%% file: main-arxiv.tex
\documentclass[10pt]{article}

\usepackage[hmargin=1in,vmargin=1.25in]{geometry}

\newcommand{\highlightrevision}[2]{#2}
\newenvironment{newpartinrevision}[1][]{}

\def\arxivonly#1{#1}
\let\oldauthor=\author
\let\oldmaketitle=\maketitle
\def\myauthor{}
\makeatletter
\def\appendtoauthors#1{\protected@edef\myauthor{\myauthor #1}}
\renewcommand\author[2][]{\appendtoauthors{\text{#2} \\}}
\def\ead#1{} 
\def\corref#1{}
\newcommand\cortext[2][]{}
\newcommand\address[2][]{\appendtoauthors{\textit{\small #2} \\[5mm]}}
\def\journal#1{\oldauthor{\myauthor}}
\def\sep{, }
\newenvironment{frontmatter}{}{\maketitle\par\tableofcontents}
\newenvironment{keyword}{ \newsavebox\savedkeywords\lrbox\savedkeywords \vbox\bgroup \vspace*{0.5cm} \noindent  \textbf{Keywords}: \noindent}{\egroup\endlrbox}
\renewenvironment{abstract}{\newsavebox\savedabstract\lrbox\savedabstract \vbox\bgroup \noindent \textbf{Abstract:}\par \noindent }{\egroup\endlrbox}
\def\maketitle{\oldmaketitle\par\noindent\usebox{\savedabstract}\par\noindent\usebox{\savedkeywords}} 
\makeatother
\date{}
\usepackage[numbers]{natbib}
\usepackage{lmodern}

\usepackage[english]{babel}

\usepackage{amsmath}
\usepackage{amsthm}
\usepackage{amssymb}
\usepackage{textcomp}
 
\usepackage{paralist}
\usepackage{colortbl}

\newtheorem{theorem}{Theorem}
\newtheorem{proposition}{Proposition}
\newtheorem{lemma}{Lemma}
\newtheorem{corollary}{Corollary}
\newtheorem{definition}{Definition}
 
\newtheorem{claim}{Claim}

\usepackage[noend]{algpseudocode}
\usepackage{tikz}

\usepackage{rotating} 

\usepackage{pxfonts}
\usepackage{times}
\usepackage[T1]{fontenc}

\usepackage{ucs}
\usepackage[utf8x]{inputenc}
\usepackage{tikz}
\usepackage{listings}

\usepackage{nameref}
\usepackage{url}

\newcommand{\cC}{{\cal C}}
\newcommand{\cD}{{\cal D}}

\newcommand{\cG}{{\cal G}}

\newcommand{\cM}{{\cal M}}

\newcommand{\cO}{{\cal O}}

\newcommand{\bA}{{\bf A}}
\newcommand{\bB}{{\bf B}}
\newcommand{\bC}{{\bf C}}

\newcommand{\bE}{{\bf E}}

\newcommand{\bI}{{\bf I}}

\newcommand{\bL}{{\bf L}}
\newcommand{\bM}{{\bf M}}

\newcommand{\bP}{{\bf P}}
\newcommand{\bQ}{{\bf Q}}
\newcommand{\bR}{{\bf R}}
\newcommand{\bS}{{\bf S}}

\newcommand{\bU}{{\bf U}}
\newcommand{\bV}{{\bf V}}
\newcommand{\bW}{{\bf W}}
\newcommand{\bY}{{\bf Y}}
\newcommand{\bX}{{\bf X}}
\newcommand{\bZ}{{\bf Z}}

\newcommand{\bs}{{\bf s}}

\newcommand{\bv}{{\bf v}}

\newcommand{\by}{{\bf y}}
\newcommand{\bx}{{\bf x}}
\newcommand{\bz}{{\bf z}}

\DeclareMathOperator{\EX}{\mathbb{E}}

\usepackage{tikz}
\usetikzlibrary{automata}
\usetikzlibrary{graphs}
\usetikzlibrary{shapes}
\usetikzlibrary{decorations.markings}

\usepackage{siunitx}
 
\usepackage{array,multirow,graphicx}

\newcommand{\gbarx}{\ensuremath{\cG_{\underline{\bX}}}}
\newcommand{\cbdg}{\cG^{pbd}_{\bX\bY}}
\newcommand{\pbdg}{\cG^{pbd}_{\bX\bY}}
\newcommand{\pbdm}{\cM^{pbd}_{\bX\bY}}
\def\pbdgmod{\cG^{pbd,\bC}_{\bX\bY}}
\newcommand{\causpaths}{\textit{Dpcp}(\bX,\bY)}
\newcommand{\pcauspaths}{\textit{PCP}(\bX,\bY)}
\newcommand{\adjustment}{\textit{Adjustment}}

\newcommand\independent{\protect\mathpalette{\protect\independenT}{\perp}}
\def\independenT#1#2{\mathrel{\rlap{$#1#2$}\mkern2mu{#1#2}}}

\newcommand{\EAninMoralGraph}{m_{a}}

\renewcommand{\frac}[2]{#1/#2}

\makeatletter
\newenvironment{algo}[4]
{\begin{center} 
\def\@currentlabelname{#1}#3

\begin{minipage}[t]{#4}
\begin{algorithmic}\Function{#1}{#2}}
{\EndFunction 
\end{algorithmic}
\end{minipage}

\end{center}}
\makeatother

\newenvironment{analal}{\begin{proof}[Analysis of the Algorithm]}{\end{proof}}

\newcommand{\call}[1]{\textsc{\nameref{#1}}}

\long\def\onlyinblackandwhite#1{}
\long\def\onlyincolor#1{#1}

\onlyinblackandwhite{\definecolor{red}{RGB}{0,0,0}\definecolor{green}{RGB}{0,0,0}\definecolor{blue}{RGB}{0,0,0}}

\journal{
Artificial Intelligence
}

\let\citet=\cite

\tikzset{adjusted/.style={fill=black!30}}
\tikzset{inner sep=2pt,minimum size=15pt}

\begin{document}

\input{body.tex}

\end{document}

%% file: body.tex
\begin{frontmatter}

\title{Separators and Adjustment Sets in Causal Graphs: \\
Complete Criteria and an Algorithmic Framework
\footnote{This is a revised and extended version of 
preliminary work presented at the 27th  \citep{TextorLiskiewicz2011} 
and 30th \citep{zander2014constructing} conferences on 
Uncertainty in Artificial Intelligence (UAI). Declarations of interest: none
}\arxivonly{\footnote{Published in Artificial Intelligence 270 (2019) 1-40. \url{https://doi.org/10.1016/j.artint.2018.12.006}}}
}

\cortext[cor1]{Corresponding author.}

\author[uzl]{Benito van der Zander\corref{cor1}}
\ead{benito@tcs.uni-luebeck.de}
\arxivonly{\address[uzl]{Institute for Theoretical Computer Science, Universit\"at zu L\"ubeck, Germany}}

\author[uzl]{Maciej Li\'{s}kiewicz}
\ead{liskiewi@tcs.uni-luebeck.de}
\address[uzl]{Institute for Theoretical Computer Science, Universit\"at zu L\"ubeck, Germany}

\author[radboud]{Johannes Textor}
\ead{Johannes.Textor@radboudumc.nl}
\address[radboud]{Institute for Computing and Information Sciences, Radboud University Nijmegen, Nijmegen, The Netherlands}

\begin{abstract}
Principled reasoning about the identifiability of causal effects
from non-experimental data is an important application of graphical
	causal models. \highlightrevision{r2c3a}{This paper focuses on effects that are identifiable by
	covariate adjustment, a commonly used estimation approach.}
We present an algorithmic framework for
efficiently testing, constructing, 
and enumerating $m$-separators in ancestral graphs (AGs),
a class of graphical causal models that can represent
uncertainty about the presence of latent confounders.
Furthermore, we prove a reduction from causal effect 
identification by covariate adjustment
to $m$-separation in a subgraph for directed acyclic
graphs (DAGs) and maximal ancestral graphs (MAGs). Jointly,
these results yield constructive criteria that characterize
all adjustment sets as well as
all minimal and minimum adjustment sets for identification
of a desired causal effect with multiple 
exposures and outcomes in the presence of latent confounding.
Our results extend several existing solutions for special
cases of these problems. 
Our efficient algorithms allowed us to empirically quantify
the identifiability gap between covariate
adjustment and the do-calculus in random DAGs and MAGs, 
covering a wide range of scenarios.
Implementations of our algorithms 
are provided in the R package {\sc dagitty}.
\end{abstract}

\end{frontmatter}

\let\inputiffinal=\input

%\maketitle 

%\tableofcontents

\section{Introduction}\label{sec:intro}

Graphical causal models are popular tools for reasoning about assumed causal relationships between random variables and their implications \citep{Pearl2009,Elwert2013}. Such models represent the mechanisms that are assumed to generate an observed joint probability distribution. By analyzing the model structure, one can deduce which variables introduce potential bias (confounding) when the causal effect of exposure variables ($ \bX $) on outcome variables ($ \bY $) is to be determined.
\highlightrevision{r3c2}{
In population research within Epidemiology \citep{RothmanGL2008}, the Social Sciences \citep{Elwert2013}, or Econometrics, causal inference often needs to be based on observational rather than experimental data for practical or ethical reasons. Confounding is a major threat to causal inference in such applications.  
}
A traditional approach to address confounding is to simply adjust for (e.g., by stratification) as many covariates as possible that temporally precede the exposure variable (pre-exposure variables) \cite{Shrier2008,Rubin2008}. This practice has roots in the analysis of data from randomized trials, in which confounding can only occur by chance. For observational data, however, it can be easily shown using graphical causal models that conditioning on pre-exposure variables can induce dependencies between other pre-treatment variables that affect exposure and outcome (Figure~\ref{fig:mbias}). If those variables have not been or cannot be measured (e.g., if we could not observe variable MR in Figure~\ref{fig:mbias}), then we could increase rather then reduce bias by conditioning on a pre-exposure variable. Such examples show that it is impossible to decide which covariates one should  adjust for to address confounding without at least some knowledge of the causal relationships between the involved covariates. To our knowledge, graphical causal models are currently the only causal inference framework in which it is possible to state a set of causal assumptions and derive from those assumptions a conclusive answer to the question whether, and how, a causal effect of interest is identifiable via covariate adjustment. 

Covariate adjustment is not complete for identification; other methods like the front-door criterion \cite{Fulcher2017} or do-calculus \cite{Pearl2009} can permit identification even if covariate adjustment is impossible. \highlightrevision{r1c2}{When multiple options are available, however, adjustment may be an attractive option for effect estimation because its statistical properties are well understood, giving access to useful methodology like robust estimators and confidence intervals.}
A famous result that characterizes valid sets of variables for covariate adjustment in a given graphical model is the back-door criterion by Pearl \citet{Pearl2009}. We call such sets of variables \emph{adjustment sets} or simply \emph{adjustments}. Although the back-door criterion is \emph{sound}, meaning that sets fulfilling the criterion are indeed valid adjustments, it is not \emph{complete} -- not all valid adjustments fulfill the criterion.  
It is even possible that none of the valid adjustments in a causal
model fulfill the back-door criterion, which would lead a user to incorrectly
conclude that a different method is required. 

Shpitser and colleagues \cite{ShpitserVR2010} gave the 
first sound and complete criterion to characterize adjustments 
in DAGs.
This criterion achieves completeness by weakening the simple but overly restrictive
condition in the back-door criterion that adjustment sets must not contain 
any descendants of the exposure variable; such variables can be allowed 
provided they do not block causal paths or create new biasing paths.
However, a characterization of adjustment sets in terms of blocked paths does not 
yield a practical algorithm for adjustment set construction. 
Although enumerating all subsets of covariates
and all paths would work, such a brute-force approach
is only feasible in very small graphs. Thus, to help researchers find the most suitable strategy to identify a given causal effect, algorithmically efficient methods are needed that exhaustively characterize all available options. 

\begin{figure}
\begin{center}
\begin{tabular}{cccc}
{\bf (a)} & {\bf (b)} & {\bf (c)} & {\bf (d)}\\ 
\begin{tikzpicture} 
\node [align=left] (z1) at (0,1.5) {FI}; % {Family income\\ during childhood};
\node [align=left] (z2) at (3,1.5) {MR}; %{Mother's genetic\\ diabetes risk};
\node (w) at (1.5,.75) {MD}; % {Mother had diabetes};
\node (x) at (0,0) {LE}; % {Low education};
\node (y) at (3,0) {D}; %{Diabetes};

\draw [very thick,->] (z1) -- (w);
\draw [very thick,->] (z1) -- (x);
\draw [->] (z2) -- (w);
\draw [->] (z2) -- (y);
\draw [very thick,->] (w) -- (y);
\draw [->] (x) -- (y);
\end{tikzpicture} &
\begin{tikzpicture} 
\node [align=left] (z1) at (0,1.5) {FI}; % {Family income\\ during childhood};
\node [align=left] (z2) at (3,1.5) {MR}; %{Mother's genetic\\ diabetes risk};
\node [adjusted] (w) at (1.5,.75) {MD}; % {Mother had diabetes};
\node (x) at (0,0) {LE}; % {Low education};
\node (y) at (3,0) {D}; %{Diabetes};

\draw [very thick,->] (z1) -- (w);
\draw [very thick,->] (z1) -- (x);
\draw [very thick,->] (z2) -- (w);
\draw [very thick,->] (z2) -- (y);
\draw [->] (w) -- (y);
\draw [->] (x) -- (y);
\end{tikzpicture} &

%{\bf (c)} & {\bf (d)} \\ 
\begin{tikzpicture} 
\node [align=left] (z1) at (0,1.5) {FI}; % {Family income\\ during childhood};
\node [adjusted,align=left] (z2) at (3,1.5) {MR}; %{Mother's genetic\\ diabetes risk};
\node [adjusted] (w) at (1.5,.75) {MD}; % {Mother had diabetes};
\node (x) at (0,0) {LE}; % {Low education};
\node (y) at (3,0) {D}; %{Diabetes};

\draw [->] (z1) -- (w);
\draw [->] (z1) -- (x);
\draw [->] (z2) -- (w);
\draw [->] (z2) -- (y);
\draw [->] (w) -- (y);
\draw [->] (x) -- (y);
\end{tikzpicture} &
\begin{tikzpicture} 
\node [adjusted,align=left] (z1) at (0,1.5) {FI}; % {Family income\\ during childhood};
\node [align=left] (z2) at (3,1.5) {MR}; %{Mother's genetic\\ diabetes risk};
\node [align=left] (w) at (1.5,.75) {MD}; % {Mother had diabetes};
\node (x) at (0,0) {LE}; % {Low education};
\node (y) at (3,0) {D}; %{Diabetes};

\draw [->] (z1) -- (w);
\draw [->] (z1) -- (x);
\draw [->] (z2) -- (w);
\draw [->] (z2) -- (y);
\draw [->] (w) -- (y);
\draw [->] (x) -- (y);
\end{tikzpicture}
\end{tabular}
\end{center}

\caption{Causal diagram \cite[Chapter 12]{RothmanGL2008} 
describing the effect of low education  
(LE) on diabetes risk (D) with the covariates
family income (FI), mother's genetic 
risk to develop diabetes (MR), and mother's diabetes
(MD). The unadjusted estimate (a) 
is biased due to the common ancestor FI
-- bias ``flows'' via the biasing path LE $\gets$ FI $\to$ MD $\to$ D
(bold edges). Adjustment for MD (b) blocks this biasing path,
but opens a new one by creating an association between FI and MR. The
minimal adjustments 
$\{\text{MD},\text{MR}\}$
(c) and $\{\text{FI}\}$ (d) close all biasing paths. Note that,
if both FI and MR were unmeasured, it would be impossible to 
know whether adjustment for only MD would increase or reduce bias.
This shows that, unlike for experimental data, conditioning on pre-exposure
variables can be both beneficial and detrimental in observational data.} 
\label{fig:mbias}
\end{figure}
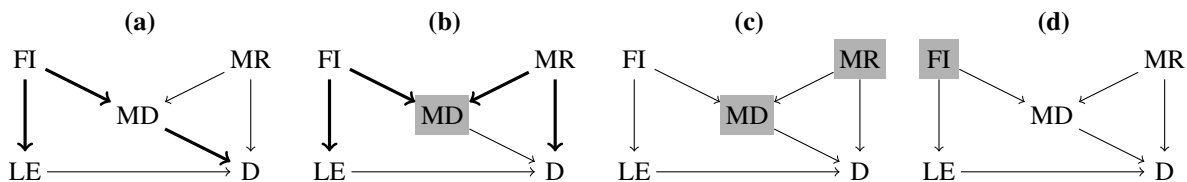

This paper provides efficient algorithms that provide exhaustive answers to the question: Given a causal graph $\cG$, which covariates
$\bZ$ can we adjust for to 
estimate the causal effect of the exposures $\bX$ on
the outcomes $\bY$? 
We address many variants of this question, such as requiring $ \bZ $ 
to be minimal or minimum as well as  imposing 
the constraint $\bI\subseteq\bZ\subseteq\bR$ for given sets $ \bI,\bR $.
All our algorithms handle sets of multiple, possibly
interrelated exposures $\bX$ and outcomes $ \bY $, 
\highlightrevision{r2c47}{
which is important in applications such as 
case-control studies that screen several putative causes
of a disease \citep{Greenland1994}}. The basis for our algorithms are theorems that
reduce adjustment to $d$-separation or $m$-separation in a subgraph, combined
with a generic, efficient and flexible algorithmic framework to find, verify, and
enumerate separating sets.
Our algorithms are guaranteed to find all valid adjustments 
for a given model with polynomial delay.
Besides adjustment, the framework is potentially useful for other
applications, as we illustrate by applying it to 
model consistency testing. %checking. 
We also perform an extensive empirical evaluation of the power of covariate
adjustment compared to other identification techniques in random causal graphs.

When there are multiple adjustments available, practical considerations can help to decide between these. We consider two different scenarios to help making such decisions. First, not all covariates might be equally  well suited for adjustment. Measuring some variables may be difficult or expensive, or could be affected by substantial measurement error.
For instance, in Figure~\ref{fig:mbias}, measuring the variable MR 
(the genetic risk inherited from the mother)
would require knowledge of the genetic variants
that lead to a higher diabetes risk and availability of the genomes
of the mothers of all study participants, but under these circumstances, it could be 
measured rather reliably.  
The family income, on the other hand, could be determined rather easily
using a questionnaire, which however could suffer from substantial 
bias (e.g., people with high incomes could be less willing to report them). In such cases adjustments in which some cost (such as measurement error) is minimized are potentially interesting. We provide algorithms to list only those sets that minimize a user-supplied cost function.

	Second, precision is an important issue when estimating effects. Even if adjustment for a certain variable is not required to minimize bias, it may still be desirable to adjust for that variable to increase precision, even if it means the resulting adjustment set is no longer minimal. In Figure~\ref{fig:mbias}, for example, variable MD (mother's diabetes) will likely be substantially correlated with variable D (child's diabetes), so its inclusion into the statistical model that estimates LE's effect on D might well increase precision even if this means that either FI or MR also need to be included. To allow such considerations, we consider the setting where we search for adjustment sets containing a pre-defined subset of variables.

A model represented as a DAG assumes causal sufficiency, i.e., that every relevant variable is included in the model. If this is not the case and there are unknown latent variables, one can use a maximal ancestral graph (MAG) model that just specifies the ancestral relations between the variables \cite{Richardson2002}. An adjustment in a MAG remains valid when the MAG is extended by any number of additional latent variables as long as the ancestral relations and $m$-separating sets are preserved. This is an important feature because the assumption that all relevant confounders have been measured and included in a model is often unrealistic. 
We show that our entire algorithmic framework for computing adjustment sets can be generalized to MAGs, allowing practitioners to substantially relax the assumption of causal sufficiency that DAG-based analyses rely on.

\newcommand{\etall}{ et al. }
\begin{table}
\centering\begin{tabular}{ccc}
&\multicolumn{2}{c}{\bf }\\
\bf Graph class & \bf First sound and complete criterion & \bf First sound and complete constructive criterion \\
DAGs
     & Shpitser\etall \cite{ShpitserVR2010} & this paper (conference version \cite{zander2014constructing}) \\
MAGs
     & this paper (conference version \cite{zander2014constructing}) & this paper (conference version \cite{zander2014constructing})\\
CPDAGs
     &Perkovi\'{c}\etall \cite{PerkovicEtAl2018} &van der Zander and Li\'{s}kiewicz \cite{vanderZander2016separators}\\
PAGs 
     &Perkovi\'{c}\etall \cite{PerkovicEtAl2018} &Perkovi\'{c}\etall \cite{PerkovicEtAl2018}\\
CGs
   &van der Zander and Li\'{s}kiewicz \cite{vanderZander2016separators}&van der Zander and Li\'{s}kiewicz \cite{vanderZander2016separators}\\
maximal PDAGs 
  & Perkovi\'{c}\etall \cite{PerkovicEtAl2017}
  & Perkovi\'{c}\etall \cite{PerkovicEtAl2017} \\
\end{tabular}
\caption{
Using covariate adjustment to estimate causal effects:
an overview of our results and related work 
for directed acyclic graphs (DAGs), 
maximal ancestral graphs (MAGs),
completed partially directed acyclic graphs (CPDAGs), 
partial ancestral graphs (PAG),  
chain graphs (CGs), and maximally oriented partially directed acyclic graphs (PDAGs).
	A criterion is sound if it is only satisfied by adjustment sets. It is complete if every adjustment set satisfies it.  It is constructive if it leads to an efficient algorithm for constructing an adjustment set.  Paper \cite{zander2014constructing} is the preliminary conference version of this work, and all subsequent results are based on the techniques developed in this paper.}\label{table:related}
\end{table}

Beyond DAGs and MAGs, our proof techniques have been successfully applied to obtain complete and constructive adjustment criteria for several other classes of graphical causal models since the publication of our preliminary version~\cite{zander2014constructing}.
Perkovi\'{c}\etall \cite{PerkovicEtAl2018} provide a generalized adjustment criterion (GAC) 
which is sound and complete for DAGs, CPDAGs, MAGs and PAGs.
CPDAGs (PAGs) are a class of causal models where a single graph represents a Markov equivalence class of DAGs (MAGs), in which all DAGs (MAGs) have the same conditional independences. % between all their random variables.  
Such models occur naturally when learning models from observed data, where in general, several possible models are consistent with the data. Finally, in \cite{vanderZander2016separators} we were able to extend our techniques to chain graphs (CGs)
and Perkovi\'{c}\etall \cite{PerkovicEtAl2017} applied these to maximally oriented partially directed acyclic graphs (PDAGs)
--  graphs providing an elegant framework for modeling and analyzing a broad range
of Markov equivalence classes. 
We summarize these results in Table~\ref{table:related}.
In this paper, we focus on complete criteria for DAGs and MAGs because this substantially simplifies the presentation of the results, and the application of these techniques to other graph classes is relatively straightforward \cite[Section~4.3]{PerkovicEtAl2018}.

This paper is structured as follows. %
Section~\ref{sec:preliminaries} introduces notation and basic concepts. 

In Section~\ref{sec:algo}, we present algorithms for verifying, constructing, and listing $m$-separating sets $ \bZ $ in ancestral graphs. All algorithms handle unconstrained separators, minimal and minimum separators, and the constraint $ \bI \subseteq \bZ \subseteq \bR $ for given sets $ \bI,\bR $. They subsume a number of earlier solutions for special cases of these problems, e.g., the Bayes-Ball algorithm for verification of $d$-separating sets \citep{Shachter1998} or the use of network flow calculations to find minimal $d$-separating sets in DAGs~\citep{TianPP1998,AcidC03}. 
Section \ref{sec:basissets} presents DAG consistency testing  as a possible application of these algorithms beyond adjustment sets.
In Section~\ref{sec:dagadjust}, we give a constructive back-door criterion (CBC) that
reduces the construction of adjustment sets to finding separating sets in DAGs,
explore variations of the CBC and compare it to Pearl's back-door criterion \cite{Pearl2009}. 
Section~\ref{sec:algorithms:adjustment} explains how to apply the algorithms of Section~\ref{sec:algo} to the graphs resulting from the criterion of Section~\ref{sec:dagadjust}, which yields polynomial time algorithms for verifying, constructing, and listing adjustment sets in DAGs. This combination leads to the first efficient implementation of the sound and complete adjustment criterion by Shpitser and colleagues \citet{ShpitserVR2010}.
Section~\ref{sec:cbcext} extends our criterion by addressing some cases where, even though covariate adjustment is not possible, it is still not necessary to invoke the do-calculus since  the causal effect can be identified in another simple manner. 
In Section~\ref{sec:experiments} we compare the results and performance of our adjustment criterion (with and without the extensions derived in Section~\ref{sec:cbcext}) with Pearl's back-door criterion as well as the IDC algorithm, which finds all identifiable causal effects, on randomly generated DAGs. 
Finally, in Section~\ref{sec:magadjust}, we extend our constructive back-door criterion of Section~\ref{sec:dagadjust} to MAGs, which generalizes the sound but incomplete adjustment criterion for MAGs without constraints by Maathuis and Colombo \cite{Maathuis2013}.

\section{Preliminaries}\label{sec:preliminaries}

We denote sets by bold upper case letters ($\bS$), and sometimes
abbreviate singleton sets $\{S\}$ as $S$. Graphs are written calligraphically $(\cG)$, and variables in upper-case $(X)$.

\subsection{General background}
\paragraph{Mixed graphs, walks and paths}
We consider mixed graphs $\cG=(\bV,\bE)$ with
nodes (vertices, variables)
$\bV$ and directed
($A \to B$), undirected ($A-B$), and bidirected
($A \leftrightarrow B$) edges $\bE$. Nodes linked
by an edge are \emph{adjacent} and \emph{neighbors} of each other.
A  \emph{walk} of length $n$ is a node
sequence $V_1,\ldots,V_{n+1}$  such that
there exists an edge sequence
$E_1,E_2,\ldots,E_{n}$ for which every edge $E_i$ 
connects $V_{i},V_{i+1}$. Then $V_1$ is called 
the \emph{start node} and $V_{n+1}$ the \emph{end node}
of the walk.
A \emph{path} is a walk in which no node occurs
more than once. Given a node set $\bX$ and a node set
$\bY$, a walk from $X \in \bX$ to $Y\in\bY$  is called
\emph{$ \bX $-proper} if only its start node is in $\bX$. 
If $ \bX $ is clear from the context, it is omitted and we just say the path is \emph{proper}.
Given a graph $\cG=(\bV,\bE)$ and a node set $\bV'$,
the \emph{induced subgraph} $\cG_{\bV'}=(\bV',\bE')$ contains
the edges $\bE' = \bE\cap (\bV'\times\bV')$ from $\cG$ that
are adjacent only to nodes in $\bV'$. The \emph{skeleton} of $\cG$ is a graph with the same nodes in which every edge is replaced by an undirected edge.
A \emph{connected component} is a subgraph in which every pair of nodes is connected by a path. A subgraph containing only a single node is also a connected component.
A connected component is a \emph{bidirectionally connected component} if for every pair of nodes there exists a connecting path that contains only bidirected edges.  
Throughout, $n$ stands for the number of nodes and $m$ for the
number of edges of a graph. 

\paragraph{Ancestry} A walk of the form $V_1 \to \ldots \to V_n$ 
is \emph{directed}, or \emph{causal}. A non-causal walk is \emph{biasing}.
A graph is \emph{acyclic} if no 
directed walk from a node to itself 
is longer than $0$.
All directed walks in an acyclic
graph are paths. 
If there is a directed path from $U$ to $V$, 
then $U$ is called an
\emph{ancestor} of $V$ and $V$ a \emph{descendant} of $U$.
A walk is \emph{anterior} if it were
directed after replacing all edges 
$U - V$ by $U \to V$.
If there is an anterior path from $U$ to $V$, then
$U$ is called an \emph{anterior} of $V$. 
All ancestors of $V$ are anteriors of $V$.
Every node is its own ancestor,
descendant, and anterior.
For a node set $\bX$, the set of all of its ancestors 
is written as $\textit{An}(\bX)$. The descendant and
anterior sets $\textit{De}(\bX)$, $\textit{Ant}(\bX)$ are
analogously defined.
Also, we denote by $\textit{Pa}(\bX)$, 
($\textit{Ch}(\bX)$, $\textit{Ne}(\bX)$), the set of parents (children, neighbors)
of $\bX$.

\paragraph{$m$-Separation} This concept is 
a generalization of the well-known $d$-separation 
criterion in DAGs  \citep{Pearl2009} for ancestral graphs.
A node $V$ on a walk $w$
is called a \emph{collider} if two arrowheads of $w$
meet at $V$, e.g., one possible collider is
$U \leftrightarrow V \gets Q$. There can be no collider
if $w$ is shorter than 2. Two nodes $U,V$ are called
\emph{collider connected} if there is a path between 
them on which all nodes except $U$ and $V$ are 
colliders. Adjacent vertices are collider connected.
Two nodes $U,V$ are 
called \emph{$m$-connected} by a set $\bZ$ if there is a path $\pi$
between them on which every node that is a collider 
is in $\textit{An}(\bZ)$ and every node that is not a 
collider is not in $\bZ$. Then $\pi$
is called an $m$-connecting path. The same definition
can be stated simpler using walks:
$U,V$ are called $m$-connected by $\bZ$ if 
there is a walk $w$ between them on which all
colliders and only colliders are in $\bZ$.
Similarly, such $w$ is called an $m$-connecting walk.
If $U,V$ are
$m$-connected by the empty set, we simply say they
are $m$-connected.
If $U,V$ are 
not $m$-connected by $\bZ$, we say that $\bZ$ \emph{$m$-separates}
them or \emph{blocks} all paths between them. Two node sets
$\bX,\bY$ are $m$-separated by $\bZ$ if all
their nodes are pairwise $m$-separated by $\bZ$.

\paragraph{DAGs and ancestral graphs} A DAG is an acyclic graph with only directed  edges.
A mixed graph
$\cG=(\bV,\bE)$ is called an \emph{ancestral graph} (AG)  \citep{Richardson2002} if the following two
conditions hold: (1) For each edge $A\gets B$ 
or $A \leftrightarrow B$, $A$ is not an anterior of $B$.
(2) For each edge $A-B$, there are no edges
$A \gets C$, $A \leftrightarrow C$, $B \gets C$ 
or $B \leftrightarrow C$. From these conditions it follows that there can be at most one edge between two nodes in an AG.
\highlightrevision{r2c12}{
If one restricts AGs to  
graphs consisting of only directed and bidirected edges, as we do in Section~\ref{sec:magadjust}, then a simple and intuitive equivalent definition characterizes an AG as a graph without directed and almost-directed cycles, i.e.,  without subgraphs of the form
$V_1 \to \ldots \to V_n \to V_1$ or 
$V_1 \leftrightarrow V_2 \to \ldots \to V_n  \to V_1$.}
Syntactically, all DAGs are AGs and all
AGs containing only directed edges are DAGs.
An AG $\cG=(\bV,\bE)$ is a \emph{maximal ancestral graph}
(MAG) if every non-adjacent pair of nodes $U,V$ 
can be $m$-separated by some $\bZ \subseteq \bV \setminus\{U,V\}$.
It is worth noting that syntactically a DAG is also a MAG \cite{Zhang2008}. 
Moreover every AG $\cG$ can be turned into a MAG $\cM$  by adding
bidirected edges between node pairs that cannot
be $m$-separated, which preserves all $m$-separation
relationships in the graph~\citep[Theorem 5.1]{Richardson2002}.

\paragraph{Graph transformations} 
A DAG $\cG = (\bV,\bE)$ is represented by a MAG $\cM = \cG[^\bS_\bL$  with nodes $ \bV \setminus (\bS \cup \bL) $ for sets $\bS,\bL\subseteq\bV$, whereby $ \cM $ has an edge between a pair of nodes $U,V$  if $U,V$ cannot be $d$-separated in $\cG$ by any $\bZ$ with $\bS \subseteq \bZ \subseteq \bV\setminus\bL$. That edge has an arrowhead at node $ V $, if $V \notin \textit{An}(U \cup \bS)$.
Notice that for empty $\bS$ and $\bL$ we have $\cG[_{\emptyset}^{\emptyset}=\cG$.
The \emph{canonical DAG} $\cC(\cM)$ of a MAG $\cM$ is the DAG obtained from $\cM$ by replacing every $\leftrightarrow$ edge with $\gets L \to$ and every $-$ edge with $\to S \gets$ with new nodes $L$ or $ S $ which form sets $\bL$, $\bS$. Clearly $\cC(\cM)[^\bS_\bL = \cM$ (for more details see \citep{Richardson2002} or Section~\ref{sec:magadjust} in our paper).

The \emph{augmented graph} $(\cG)^a$ of a certain AG $\cG$ is an undirected graph with the same nodes as $\cG$ whose edges are all pairs of nodes that are collider connected in $\cG$. Two node sets $\bX$ and $\bY$ are $m$-separated by a node set $\bZ$ in $\cG$  if and only if $\bZ$ is an $\bX$-$\bY$ node cut in $(\cG_{\textit{Ant}(\bX\cup\bY\cup\bZ)})^a$  \citep{Richardson2002}. For DAGs the augmented graph is also called the moralized graph, so the construction of the augmented graph is termed \emph{moralization}. 

For any subset $\bA$ and $\bB$ of $\bV$, by $\bA \to \bB$ we denote the set of 
all edges $A \to B$ in $\bE$, such that $A \in \bA$ and $B\in \bB$; the sets $\bA \gets  \bB$, 
$\bA \leftrightarrow  \bB$, and $\bA -  \bB$ are defined analogously.
\highlightrevision{r2c13}{Pearl's do-calculus \cite{Pearl2009} defines the following two transformations.} 
First, the graph obtained from a graph $ \cG = (\bV,\bE) $ by removing all edges entering $\bX$ is written as $ \cG_{\overline{\bX}} = (\bV, \bE \setminus (( \bV \to \bX ) \cup ( \bV \leftrightarrow \bX ) )) $. Second, the removal of all edges leaving $ \bX $ is written as 
$ \cG_{\underline{\bX}} = (\bV, \bE \setminus (( \bX \to \bV ) \cup( \bX - \bV ) )) $. The application of both these operations $ (\cG_{\overline{\bX}})_{\underline{\bX'}} $  is abbreviated as $ \cG_{\overline{\bX}\underline{\bX'}} $ matching the notation of do-calculus \cite{Pearl2009}. Descendants and ancestors in these graphs are written as set subscript at the  corresponding function without specifying the graph, e.g., $ \textit{De}_{\overline{\bX}} $ or $ \textit{An}_{\underline{\bX}} $. This notation is used for DAGs and MAGs.
Note that for DAGs in each case only the first kind of edges needs to be removed as there are no $\leftrightarrow$ or $-$ edges in DAGs.

\subsection{Identification via covariate adjustment} 

\paragraph{Do-operator and identifiability}%
A DAG $\cG$ encodes the factorization of a
joint distribution $P$ for the set of variables $\bV=\{X_1,\ldots,X_n\}$
as $P(\bv)=\prod_{j=1}^n P(x_j \mid \textit{pa}_j)$, 
where $\textit{pa}_j$ denotes a particular realization of 
the parent variables %
of $X_j$ in $\cG$.
When interpreted causally, an edge $X_i \to X_j$ is 
taken to represent a direct causal effect of $X_i$ on
$X_j$. Formally, this can be defined in terms of the do-operator \cite{Pearl2009}:
Given DAG $\cG$ and a joint probability density $P$ for $\bV$ 
the post-intervention distribution can be expressed
in a truncated factorization formula
\begin{equation}\label{do:truncated:formula}
  P(\bv \mid  \textit{do}(\bx)) = 
  \left\{
    \begin{array}{ll}
    \displaystyle
         \prod_{ X_j \in\bV \setminus \bX} P(x_j \mid  \textit{pa}_j)&
	    \text{for $\bv$ consistent with $\bx$,} \\[1mm]
      0& \text{otherwise,}
    \end{array}
  \right.
\end{equation}
	where ``$\bv$ consistent with $\bx$'' means that $\bv$ and $\bx$ assign the same values to the variables in $\bX \cap \bV$.
For disjoint 
$\bX,\bY\subseteq \bV$, 
the \emph{(total) causal effect} of $\bX$ on $\bY$ is
$P(\by \mid  \textit{do}(\bx))$ where $\textit{do}(\bx)$ represents
an intervention that sets $\bX=\bx$. 
In a DAG, this intervention corresponds to removing all edges
into $\bX$, disconnecting $\bX$ from its parents, i.e., constructing the graph $\cG_{\overline{\bX}}$.

When all variables in $\bV$ are observed (we model this by setting $\bR=\bV$), 
the causal effect $P(\by \mid  \textit{do}(\bx))$ 
of $\bX$ on $\bY$  in a given graph $\cG$ can be determined uniquely from 
the pre-intervention distribution $P$ using the factorization above.   
However, when some variables are unobserved, the question whether $P(\by \mid  \textit{do}(\bx))$  
is  \emph{identifiable}, i.e., if it can be computed from the pre-intervention distribution independent 
of the unknown quantities for the unobserved variables $\bV\setminus \bR$, 
becomes much more complex. 
For a formal definition of identifiability see~\cite[Section 3.2]{Pearl2009}.
Figure~\ref{fig:notident:ident:viaadj} shows an example DAG $\cG_1$ for which 
the causal effect of $X$ on $Y$ is not identifiable, and two DAGs $\cG_2$ and $\cG_3$
for which the effect is identifiable.

\paragraph{Adjustment sets}
Identification via covariate adjustment, studied in this paper, is defined for DAGs below
and can be extended to MAGs in the usual way \citep{Maathuis2013}
(we will give a formal definition Section~\ref{sec:magadjust}).

\begin{definition}[Covariate adjustment \cite{Pearl2009}]
\label{def:adjustment:set}
Given a DAG $\cG=(\bV,\bE)$ and 
pairwise disjoint
$\bX,\bY,\bZ \subseteq \bV$, 
$\bZ$ is called \emph{adjustment set for estimating the causal effect of $\bX$ on $\bY$}, 
or simply \emph{adjustment} (set), if for every distribution %
$P$ consistent with $\cG$ we have 
\begin{equation}\label{eq:adjustment:def}
P(\by\mid \textit{do}(\bx)) = \left\{
  \begin{array}{ll}
   P(\by\mid \bx)                       & \text{if $\bZ = \emptyset$,}\\[2mm]
   \sum_{\bz} P(\by\mid \bx,\bz) P(\bz) & \text{otherwise.}
  \end{array}
\right.
\end{equation}
\end{definition}

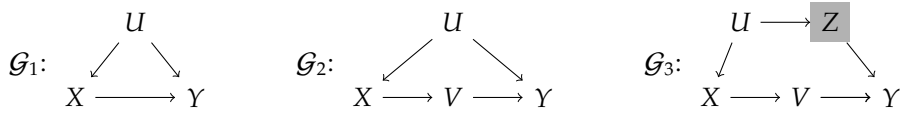
\begin{figure}
\centering
$\cG_1$:
\begin{tikzpicture}[xscale=0.8,baseline=1em]
\node (X) at (0,0) {$X$};
\node (U) at (1,1) {$U$};
\node (Y) at (2,0) {$Y$};
\graph { (X) -> (Y) ; (X) <- (U) -> (Y) ; };
\end{tikzpicture}
\hspace*{10mm}$\cG_2$:
\begin{tikzpicture}[xscale=0.8,baseline=1em]
\node (X) at (-0.5,0) {$X$};
\node (U) at (1,1) {$U$};
\node (Z) at (1,0) {$V$};
\node (Y) at (2.5,0) {$Y$};
\graph { (X) -> (Z) -> (Y) ; (X) <- (U) -> (Y) ; };
\end{tikzpicture}
\hspace*{10mm}$\cG_3$:
\begin{tikzpicture}[xscale=0.8,baseline=1em]
\node (X) at (-0.5,0) {$X$};
\node (U) at (0.0,1) {$U$};
\node[adjusted] (Z) at (1.5,1) {$Z$};
\node (V) at (1,0) {$V$};
\node (Y) at (2.5,0) {$Y$};
\graph { (X) -> (V) -> (Y) ; (X) <- (U) -> (Z) -> (Y) ; (V) -> (Y)};
\end{tikzpicture}\caption{Identification of $P(y \mid  \textit{do}(x))$ in graphs with the unobserved nodes
$\bV\setminus\bR=\{U\}$: In $\cG_1$ the effect is not identifiable. In $\cG_2$, it is identifiable 
(using the front-door method \cite{Pearl2009}) but 
for this DAG no adjustment set $\bZ$ exists, with $\bZ\subseteq\bV\setminus\{U\}$. 
In $\cG_3$ the effect can be identified via adjustment $\bZ=\{Z\}$.
}
\label{fig:notident:ident:viaadj}
\end{figure}

Identification via covariate adjustment is not complete in the sense
that there exist graphs for which $P(\by\mid \textit{do}(\bx))$ is identifiable 
but for which no adjustment set $\bZ$, with $\bZ \subseteq \bR$, exists.
For an example of such a DAG see Figure~\ref{fig:notident:ident:viaadj}.
In \cite{pearl1995causal} (see also \cite[Section~3.4]{Pearl2009}) Pearl introduced 
the \emph{do-calculus} of intervention that was proven to be complete for
identifying causal effects \cite{shpitser2006identification,huang2006pearl}.
Based on the rules of the do-calculus, the IDC algorithm proposed by 
Shpitser and Pearl \cite{ShpitserIDCAlgorithm} computes a formula for 
$P(\by \mid  \textit{do}(\bx))$ if (and only if) the effect is identifiable.
\highlightrevision{r1c1}{
Researchers might thus use IDC to decide if an effect of interest is identifiable 
at all, and then use one of the algorithms provided in this paper to verify if 
identification is also possible via adjustment, which has some statistically 
favorable properties.} However, for complex graphs, it may not be practical to 
run the IDC algorithm as we show later, whereas adjustment sets can still be found
quickly if they exist. 

%However, IDC 
%A drawback of the algorithm is that it can generate very complex formulas,
%even when a causal effect is also identifiable via covariate adjustment
%(Expression~\eqref{eq:adjustment:def}) or in another simple manner.
%Hence in the cases where adjustment is possible, but IDC returns a more complex formula, it cannot be used to determine if adjustment is possible.
%Also the runtime of the IDC algorithm is higher compared to our algorithms
%to determine adjustment sets, as we will show later.

\paragraph{Minimality}
In this paper we present methods for computing an adjustment set 
that satisfies the condition of Definition~\ref{def:adjustment:set} for a given instance. An important feature 
of our approach is that it allows also to find adjustment sets satisfying some additional desirable 
constraints, e.g., minimality. Below we define this notion formally both for adjustment
and separation sets.

For pairwise disjoint $\bX,\bY,\bZ \subseteq \bV$,
and  a subset $\bM$ of $\bV$,
an $m$-separator (resp. adjustment set) $\bZ$ relative to $(\bX,\bY)$ 
is \textit{$\bM$-minimal}, if $\bM\subseteq \bZ$ and no set $\bZ' \subsetneq \bZ$ 
with $\bM \subseteq \bZ'$ is an $ m $-separator (adjustment set) relative to $(\bX,\bY)$.
An $ m $-separator (adjustment) $\bZ$ is \textit{$ \bM $-minimum} 
according to a cost function $w: \bV\to \mathbb{R}^+$ if $\bM\subseteq \bZ$ 
and no $ m $-separator (adjustment) $\bZ'$ with $\bM \subseteq \bZ'$ 
and $ \sum_{Z\in\bZ'} w(Z) < \sum_{Z\in\bZ} w(Z) $ exists. 
In particular, $\emptyset$-minimality ($\emptyset$-minimum)
coincides with the standard notion of minimality (minimum), which we will also 
call \emph{strong-minimality} (\emph{strong-minimum}). A more detailed explanation of these concepts along with some examples is provided in Section~\ref{sec:sepalgo:dense}.

While $\bM$-minimality is defined with respect to any set $\bM \subseteq \bV$, in this 
paper we will only consider the two cases $\bM=\emptyset$ and $\bM=\bI$, i.e., we
consider $m$-separators and adjustments
$\bZ$ that are supersets of the constraining set $\bI$.
For a given constraining set $\bI$ we will abbreviate the 
$\bI$-minimal/minimum $\bZ\supseteq \bI$ as \emph{weakly-minimal/minimum} or just 
as \emph{minimal/minimum}.
Note a subtle, but important difference between
weak and strong minimality.
For example, the existence of a weakly-minimal $m$-separator 
does not necessarily imply that a strongly-minimal separator exists.
E.g., in the DAG $X \to I \gets V \to Y$, set $\bZ=\{I,V\}$ is 
an $I$-minimal $d$-separator relative to $(X,Y)$ but in the graph
there exists no strongly-minimal $d$-separator  $\bZ'$,
with $I \subseteq \bZ'$. On the other hand, it is easy to see that 
every strongly-minimal $m$-separator $\bZ$, with $\bI\subseteq \bZ$,  
is also an $\bI$-minimal one and the same holds for the minimum sets.

\section{Algorithms for $m$-separation in ancestral graphs}

\label{sec:algo}

In this section, we compile an algorithmic framework for solving a host of 
problems related to verification, construction, and enumeration of 
$m$-separating sets in an ancestral graph $\cG = (\bV,\bE)$ with 
$n$ nodes in $\bV$ and $m$ edges in $\bE$. 
The problems are defined in 
Table~\ref{fig:problems}, which also shows the asymptotic runtimes of our algorithms. %
The  goal is to test or to output %
either arbitrary $m$-separating sets $\bZ$ without further constraints 
or $m$-separating sets that have a minimal size, so that no node 
can be removed from the set without turning it into a non-separating set.  
A further constraint is that $\bZ$ should be bounded by arbitrary given sets $\bI \subseteq \bR$ as $\bI \subseteq \bZ \subseteq \bR$. The variables in $\bI$ will always be included in the $ m $-separating set, even if the set remains a separator without these variables. The set $\bV\setminus\bR$ corresponds to a set of \emph{unobserved} variables, which are known to exist, but have not been (or cannot be) measured. This constraint can also be used to model correlated errors between variables, i.e., rather than connecting such variables with a bidirected edge like in a semi-Markovian model, the variables are connected to another variable not in $\bR$.  Both constraints together are also an important technical tool for the design of efficient enumeration algorithms and the adjustment algorithms in the later sections.

The rest of this section describes our algorithms solving these problems, which are mostly 
 generalizations of existing algorithms from
\citep{Acid1996,Shachter1998,TianPP1998,TextorLiskiewicz2011}. For each problem, we present the algorithm as a function of the same name as the problem, so that the association between problem and algorithm is easy to follow and the algorithms can be built upon each other.

In contrast to the preliminary conference version of this work \cite{zander2014constructing}, we state the algorithms separately for sparse and dense graphs. We introduce the notion $\bM$-minimality to analyze the minimal sets found by the algorithms and show that they find weakly-minimal separators, since finding strong-minimal separators is intractable. % as well as that they find weakly minimal separators, because strongly minimal separators. % introducing and analyzing the new notion of strong minimality. 

\begin{table*}
\centering
\begin{tabular}{lllll}
&&&Runtime&Reference \\
\multicolumn{3}{l}{\textbf{Verification:} \text{For given $\bX, \bY,\bZ$ and constraint $\bI$
 decide if $\ldots$ }}
&&\\
\hspace*{2mm} 
   & {\sc TestSep} & $\bZ$ $m$-separates $\bX,\bY$ & $\cO(n+m)$ 
      &Proposition~\ref{prop:TestSep}\\
   & {\sc TestMinSep} & $\bZ\supseteq \bI$ $m$-separates $\bX,\bY$ and $\bZ$ is $\ldots$ &&\\
   & & \hspace*{3mm} $ \bI $-minimal & $\cO(n^2)$ 
      &Proposition~\ref{prop:TestMinSep}\\
   & &  \hspace*{3mm} strongly-minimal 
      & $\cO(n^2)$ 
      &Proposition~\ref{prop:TestMinSep} \\[2mm]
\multicolumn{3}{l}{\textbf{Construction:} \text{For given $\bX, \bY$ and constraints $\bI, \bR$, output an $\ldots$}} & & \\
\hspace*{2mm} 
  & {\sc FindSep} &    $m$-separator $\bZ$ with $\bI\subseteq\bZ\subseteq\bR$  & $\cO(n+m)$ 
     &Proposition~\ref{prop:FindSep}\\
  & {\sc FindMinSep} &   $m$-separator $\bZ$  with $\bI\subseteq\bZ\subseteq\bR$ which is $\ldots$ &&\\
     &  & \hspace*{3mm}  $ \bI $-minimal & $\cO(n^2)$ 
     &Proposition~\ref{prop:FindMinSep}\\
   &  & \hspace*{3mm}   strongly-minimal 
     & NP-hard
     &Proposition~\ref{prop:np-complete}\\
  & {\sc FindMinCostSep} & $m$-separator $\bZ$  with $\bI\subseteq\bZ\subseteq\bR$ which is $\ldots$ &&\\
  &  & \hspace*{3mm}    $ \bI $-minimum & $\cO(n^3)$ 
     &Proposition~\ref{prop:FindMinCostSepI}\\
  &  & \hspace*{3mm} strongly-minimum & $\cO(n^3)$ 
     &Proposition~\ref{prop:FindMinCostSep:proper}\\[2mm]
\multicolumn{3}{l}{\textbf{Enumeration:} \text{For given $\bX, \bY, \bI, \bR$ enumerate all $\ldots$ } }&Delay&\\
 & {\sc ListSep} 
     &$m$-separators $\bZ$  with $\bI\subseteq \bZ \subseteq  \bR$   &$\cO(n(n+m))$ 
     &Proposition~\ref{prop:ListSep}\\
 & {\sc ListMinSep} & $ \bI $-minimal $m$-separators $\bZ$  with $\bI \subseteq \bZ \subseteq  \bR$        
     &$\cO(n^3)$ 
     &Proposition~\ref{prop:ListMinSep}\\
\end{tabular}
\caption{Definitions of algorithmic tasks related to $m$-separation
in an ancestral graph $\cG$ of $n$ nodes and $m$ edges and the 
time complexities of algorithms given in this section that solve the associated problems. Throughout, $\bX,\bY,\bR$ are pairwise disjoint node sets, the set $\bZ$ is disjoint with the non-empty sets $\bX,\bY$, and each of the sets $\bI,\bR,\bZ$ can be empty. 
A minimum $m$-separator minimizes the sum $\sum_{Z\in\bZ} w(Z)$ for a cost function $w$ respecting the given constraints, %
i.e., $w(V) = \infty$ for $V \notin \bR$.
The construction
algorithms output $\bot$ if no set fulfilling the listed condition exists.
Delay complexity for {\sc ListSep} and {\sc ListMinSep} refers to the time needed per solution when there can be exponentially many solutions
(see~\cite{Takata2010}).}
\label{fig:problems}
\end{table*}

\subsection{Linear-time algorithms for testing and finding $m$-separators}

The problems {\sc TestSep} and {\sc FindSep} can be solved immediately in the ancestral graph.
{\sc TestSep} just requires us to verify the $m$-separation definition for given sets $\bX, \bY, \bZ$. In DAGs $d$-separation is usually tested with the Bayes-Ball algorithm \citep{Shachter1998}, which can be visualized as sending balls through the graph along the edges, tracking which nodes have been visited from either their parents or children, until all reachable nodes have been visited. In other words, Bayes-Ball is basically a breadth-first-search with some rules that determine if an edge pair lets the ball pass, bounces it back on the same edge or blocks it completely. These rules can be adapted to $ m $-separation as shown in 
 Figure~\ref{fig:bayes:ball}, which leads to the following algorithm for testing if 
 a given $\bZ$ $m$-separates $\bX$ and $\bY$ in an AG $\cG$:

\def\bayesball#1#2#3#4#5#6#7#8{\begin{scope}[xshift=#1]
\node[] (C) at (0,1) {T}; 
\node[#2] (M) {$M$}; 
\draw[#3] (C) -- (M);
\node[] (O) at (-0.9,-1) {$O$} ; \draw[#4,->] (M) -- (O);
\node[] (I) at (-0.3,-1) {$I$} ; \draw[#5,<-] (M) -- (I);
\node[] (B) at (0.3,-1) {$B$}  ; \draw[#6,<->] (M) -- (B);
\node[] (U) at (0.9,-1) {$U$}  ; \draw[#7] (M) -- (U);
\node[text width=2em,anchor=north,draw] at (-0.9,1.2) {#8};
\end{scope}}
\begin{figure}
\centering
\begin{tikzpicture}[b/.style={color=red,postaction={decorate,decoration={markings, mark=at position 0.5 with {\draw[-] (0,-3pt) -- (0,3pt);}}}},Z/.style={adjusted}]
\node[] at (-2,1.5) {$M\notin \bZ $:};
\bayesball{0cm}{}{->}{}{b}{b}{}{$\to \to$\\$\to -$}
\bayesball{3cm}{}{<-}{}{}{}{}{$\gets \to$\\$\gets\gets$\\$\gets\leftrightarrow$\\$\gets-$}
\bayesball{6cm}{}{<->}{}{b}{b}{}{$\leftrightarrow \to$\\$\leftrightarrow -$}
\bayesball{9cm}{}{  }{}{}{}{}{$- \to$\\$-\gets$\\$-\leftrightarrow$\\$-\ -$}

\begin{scope}[yshift=-3cm]
\node[] at (-2,1.5) {$M\in \bZ $:};
\bayesball{0cm}{Z}{->}{b}{}{}{b}{$\to\gets$\\$\to\leftrightarrow$}
\bayesball{3cm}{Z}{<-}{b}{b}{b}{b}{}
\bayesball{6cm}{Z}{<->}{b}{}{}{b}{$\leftrightarrow\gets$\\$\leftrightarrow\leftrightarrow$}
\bayesball{9cm}{Z}{  }{b}{b}{b}{b}{}
\end{scope}
\end{tikzpicture}
\caption{Expanded rules for Bayes-Ball in AGs, listing (in boxes) 
all combinations of edge pairs through which the ball is allowed to pass. 
The Bayes ball starts at the top node $ T $ and passes through the middle 
node $ M $ to one of the bottom nodes $\{O,I,B,U\}$. Forbidden passes are \onlyincolor{marked in red and }crossed out.
Here, by a pair of edges we mean an edge between $T$ (Top node) and $M$ (Middle node) 
and $M\to O$ (Out-node), resp. $M\gets I$ (In-node), $M\leftrightarrow  B$ (Bidirected edge), and  $M-U$ (Undirected edge).
The figure above shows all possible types of edges between $T$ and $M$.
We consider two cases: $M\not\in \bZ$ and $M\in \bZ$ (gray). 
The leaving edge can correspond to the entering edge, i.e., $T$
can belong to $\{O,I,B,U\}$, in which case the ball might return to $T$, which is called a bouncing ball in the literature.}\label{fig:bayes:ball}
\end{figure}
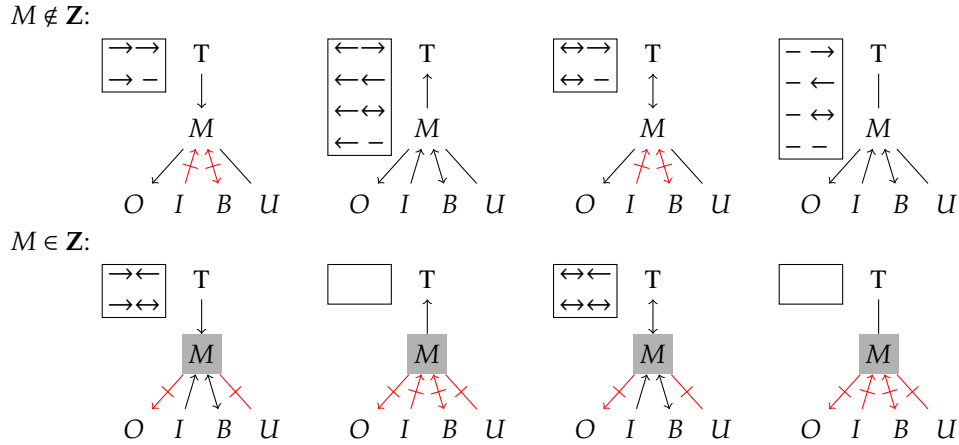

\begin{algo}{TestSep}{$\cG, \bX, \bY, \bZ$}{\label{algo:issep}}{15cm}
\State{$\bP \gets \{ (\gets,X) \mid  \text{$X \in \bX$ }\}$  }\Comment{list of pairs (type of edge, node) whose visit is pending}
\State{$\bQ \gets \bP$  }\Comment{all pending and ever visited nodes}
\While{ $ \bP $ not empty }
\State{Let $(e,T)$ be a (type of edge, node) pair of $\bP$}
\State{Remove $(e,T)$ from $\bP$}
\For{all neighbors $N$ of $T$} 
\State{Let $T$ and $N$ be connected by edge $f$}
\If{$e, T, f$ is an $m$-connecting path segment %
{\bf and } $(f,N)\notin \bQ$}
\State{Add $(f,N)$ to $\bP$ and $ \bQ $}
\EndIf
\EndFor
\If{$M \in \bY$}{ {\bf return} false}
\EndIf
\EndWhile
\State{\Return true}
\end{algo}
\begin{analal}
From the rules in Figure~\ref{fig:bayes:ball} it is obvious that the following statement about algorithm \call{algo:issep} holds:
The ball only passes through the walk segment of two consecutive edges 
if the segment is an $m$-connected walk. The correctness follows from the fact that 
the algorithm searches over all walks starting in $\bX$.
The runtime is linear as it is a breadth-first-search and each node is visited at most four times. The rules for entering a node through edges $\to$ and $\leftrightarrow$ as well as through edges $\gets$ and $-$ are the same, so an implementation could merge these cases in $ \bQ $, so each node is visited at most twice. %
\end{analal}
\begin{proposition}\label{prop:TestSep}
Using the algorithm above, the task {\sc TestSep} can be solved in time $\cO(n+m)$. 
\end{proposition}

The next problem {\sc FindSep} asks for a set $\bZ$ $m$-separating given $\bX$ and $\bY$ with $\bI \subseteq \bZ\subseteq\bR$. A canonical solution for this problem is given in the following lemma, \highlightrevision{r2c26}{which is an extended version of a lemma by Spirtes, Glymour and Scheines \cite[p. 136]{Spirtes2000} allowing the restriction $\bI$.}

\begin{lemma}\label{lemma:auxiliary:m:sep}
Let $\bX, \bY, \bI, \bR$  be sets of nodes with 
$\bI \subseteq \bR$, $\bR \cap (\bX \cup \bY) = \emptyset$. 
If there exists an $m$-separator $\bZ_0$ relative to $(\bX,\bY)$ with 
$\bI \subseteq \bZ_0 \subseteq \bR$, then $\bZ = \textit{Ant}(\bX \cup \bY \cup \bI) \cap \bR  $ 
is an $m$-separator.
\end{lemma}
\begin{proof}
Let us consider a proper walk 
$w = X , V_1, \ldots, V_{n} , Y$ with 
$X \in \bX$ and $Y \in \bY$. If $w$ 
does not contain a collider, all nodes $V_i$ 
are in $\textit{Ant}(\bX \cup \bY)$ and the walk is blocked
by $\bZ$, unless $\{V_1,\ldots, V_n\} \cap \bR = \emptyset$
in which case the walk is not  blocked by $\bZ_0$ either.
If the walk contains colliders $\bC$, it is
blocked, unless $\bC \subseteq \bZ \subseteq \bR$.
Then all nodes $V_i$ are in $\textit{Ant}(\bX \cup \bY \cup \bI)$ 
and the walk is blocked, unless $\{V_1,\ldots, V_n\}\cap \bR = \bC $.
Since $\bC \subseteq \bZ$ is a set of anteriors,
there exists a shortest (possible of length zero) %
path $\pi_j = V_j \to \ldots \to W_j$ for 
each $V_j \in \bC$ with $W_j \in \bX \cup \bY \cup \bI$ 
(it cannot contain an undirected  edge, since there is an arrow
pointing to $V_j$). Let $\pi_j' = V_j \to \ldots \to W'_j$
be the shortest subpath of $\pi_j$ that is not blocked
by $\bZ_0$. Let $w'$ be  the walk $w$ after replacing
each $V_j$ by the walk 
$V_j \to \ldots \to W'_j \gets \ldots \gets V_j$. 
If any of the $W_j$ is in $\bX \cup \bY$ we truncate 
the walk, such that we get the shortest walk between 
nodes of $\bX$ and $\bY$. Since $\pi'_j$ is not blocked, 
$w'$ contains no colliders except $w'_j$ and all
other nodes of $w'$ are not in $\bR$, $w'$ is not blocked
and $\bZ_0$ is not a separator.   
\end{proof}

This yields the following algorithm to find an $m$-separator relative to $\bX,\bY$:

\begin{algo}{FindSep}{$\cG, \bX, \bY, \bI, \bR$}{\label{algo:findsep}}{8cm}
\State{$\bR' \gets \bR \setminus (\bX\cup\bY)$}
\State{$\bZ\gets \textit{Ant}(\bX \cup \bY \cup \bI) \cap \bR'$}
\If{\call{algo:issep}$(\cG, \bX, \bY, \bZ)$}
{\Return{$\bZ$}}
\Else
{ \Return{$\bot$}}
\EndIf
\end{algo}
\begin{analal}
The algorithm finds an $ m $-separator due to Lemma~\ref{lemma:auxiliary:m:sep} and runs in linear time, since the set $\textit{Ant}(\bX \cup \bY \cup \bI) \cap \bR$  can be calculated in linear time
and the algorithm {\sc TestSep} runs in linear time as well.
\end{analal} 
\begin{proposition}%
\label{prop:FindSep}
Using the algorithm above, the task {\sc FindSep} can be solved in 
time $\cO(n+m)$. 
\end{proposition}

Set
$ \bR $ is required to be disjoint with $\bX$ and $\bY$, because $ m $-separation of a set $ \bZ $ relative to $ \bX, \bY $ is not defined when $\bZ$ contains a node of $ \bX$ or $ \bY $, so $\bZ$ is always a subset of $ \bR\subseteq \bV\setminus (\bX \cup \bY)$. However, our algorithms still remove $\bX\cup\bY$ from $\bR$, as it might prevent bugs in practical implementations and it is reasonable to find a separator that does not contain the variables to be separated.

\subsection{Polynomial time algorithms for minimal and minimum $m$-separators } % in Dense Graphs}
\label{sec:sepalgo:dense}

In this section we present algorithms for testing minimal $m$-separators 
for both variants of the minimality. We give also algorithms for constructing 
weakly- and strongly-minimum $m$-separators $\bZ$ satisfying constraints
$\bI \subseteq \bZ \subseteq \bR$ and for finding weakly-minimal separators.  
\highlightrevision{r2c20}{
To illustrate the differences between these different versions of minimality, we again consider the possible adjustment sets in Figure~\ref{fig:mbias}. The set $\{\text{FI}\}$ is a strongly-minimal adjustment set. When we use the simple cost function $c(v)=1$ that assigns to each separator a cost that is equal to the number of variables it contains, it is also a strongly-minimum adjustment set. Now, suppose we set $\bI=\{\text{MD}\}$ because $\{\text{MD}\}$ is expected to explain a lot of the variance in the outcome D. If we adjust for MD, then we also need to adjust for at least one other variable, so all adjustment sets $\bZ$ respecting $\bZ \supseteq \bI$ must have a cardinality of at least 2. Therefore, because the strongly-minimum adjustment set $\{\text{FI}\}$ has a cardinality of 1, it is clear that no strongly minimum adjustment set satisfying $\bZ \supseteq \bI$ exists. The set $\bZ=\{\text{MD},\text{MR}\}$ is a minimal adjustment set that satisfies $\bZ \supseteq \bI$, so it is both strongly minimal and $\bI$-minimal; it is also $\bI$-minimum. By contrast, the set $\bZ=\{\text{MD},\text{FI}\}$ is also $\bI$-minimum and $\bI$-minimal, but it is not strongly minimal.
}

\highlightrevision{r2c27}{
	The only one of these four variations that appears to be turned into a hard problem by the restriction $\bZ\supseteq \bI$ is computing strongly-minimal separators \textendash\ we discuss its complexity in detail in the next section.} This is a very surprising result since finding objects of minimum costs or sizes is typically, harder than constructing minimal objects. Our algorithms for the other three cases are easily implementable and have runtimes $\cO(n^2)$, resp. $\cO(n^3)$. These time complexities are reasonable particularly in case of dense ancestral graphs, i.e., graphs with a large number of edges $n \ll m \le n^2$ and $ \cO(n+m) = \cO(n^2) $. In Section~\ref{sec:separators:sparse:AGs} we propose algorithms that are faster than those presented below when the instance graphs are sparse.

The algorithms for minimal $m$-separators consist of two phases. First we convert the ancestral graph to an augmented graph (analogous to moralization for DAGs \cite{Spirtes2000}; see Preliminaries). Then, we find a minimal separator as a vertex cut in the augmented graph. This approach generalizes the $d$-separation algorithms of \cite{TianPP1998} and \cite{TextorLiskiewicz2011}, particularly the handling of the undirected graph after the moralization step is the same as in \cite{TianPP1998}. \highlightrevision{r2c29}{It is important to avoid iterating over all pairs of nodes and searching for collider connected paths between them in the construction of the augmented graph, as this would lead to a suboptimal algorithm.} Instead, the following algorithm achieves an asymptotically optimal (linear time in output size) runtime for AGs:

\begin{lemma}[Efficient AG moralization]
Given an AG $\cG$, the augmented graph $(\cG)^a$ can be computed in
time $\cO(n^2)$. 
\end{lemma}
\begin{proof}
The algorithm proceeds in four steps.
\begin{enumerate}
\item  Start by setting $(\cG)^a$ to the skeleton of $\cG$. 
\item  Partition $\cG$ in all its maximal bidirectionally connected components. 
\item  For each pair $U,V$ of nodes from 
the same component, add the edge $U-V$ to $(\cG)^a$ if it did not exist already.
\item  For each component, identify all its parents %
and link them all by undirected edges
in $(\cG)^a$. 
\end{enumerate}
Now two nodes are adjacent in $(\cG)^a$ if and only if they are 
collider connected in $\cG$. All four steps can be performed in time $\cO(n^2)$.
\end{proof}

\highlightrevision{r2c32}{
	Lemma~\ref{lemma:auxiliary:m:sep} gave us a closed form solution to find \emph{one} $m$-separator. By extending this result only slightly, we obtain a constraint on \emph{all} minimal $m$-separators:}

\begin{corollary}[Ancestry of $ \bM $-minimal separators]
Given an AG $\cG$ and three sets $\bX,\bY,\bM$, every $\bM$-minimal $m$-separator $\bZ$   %
is a subset of $\textit{Ant}(\bX\cup\bY\cup\bM)$. 
\label{lemma:moralmingeneralm}
\end{corollary}
\begin{proof}
Assume there is an $ \bM $-minimal separator $\bZ$ with 
$ \bZ \not\subseteq \textit{Ant}(\bX\cup\bY\cup\bM)$.
Setting $\bI= \bM$ and $\bR=\bZ$, Lemma~\ref{lemma:auxiliary:m:sep} shows that $\bZ' = \textit{Ant}(\bX\cup\bY\cup\bM) \cap \bZ$ is an $m$-separator with 
$\bM \subseteq \bZ'$. But $\bZ' \subseteq \textit{Ant}(\bX\cup\bY\cup\bM)$ and 
$\bZ' \subseteq \bZ$, so $\bZ \not= \bZ'$ and $\bZ$ is not an $ \bM $-minimal separator.
\end{proof}

\begin{corollary}[Ancestry of minimal separators]
Given an AG $\cG$ and three sets $\bX,\bY,\bI$, every 
$ \bI $-minimal  or $ \emptyset$-minimal   
$m$-separator %
is a subset of $\textit{Ant}(\bX\cup\bY\cup\bI)$. 
\label{lemma:moralmin}
\end{corollary}
\begin{proof}
This follows from Corollary~\ref{lemma:moralmingeneralm} with  $\bM= \bI$ or $\bM= \emptyset$.
In both cases we have  $\textit{Ant}(\bX\cup\bY\cup\bM)\subseteq\textit{Ant}(\bX\cup\bY\cup\bI)$. 
\end{proof}

Corollary~\ref{lemma:moralmin} applies to minimum separators as well because
every minimum separator must be minimal.
For all strongly-minimal (i.e., recall,  $ \emptyset $-minimal) 
or $ \bI $-minimal $ m $-separator $\bZ$ satisfying $\bI \subseteq \bZ$ the corollary implies $\textit{Ant}(\bX\cup\bY\cup\bZ) = \textit{Ant}(\bX\cup\bY\cup\bI)$ and thus $(\cG_{\textit{Ant}(\bX\cup\bY\cup\bZ)})^a = (\cG_{\textit{Ant}(\bX\cup\bY\cup\bI)})^a$, so the augmented graph $(\cG_{\textit{Ant}(\bX\cup\bY\cup\bZ)})^a$ can be constructed without knowing the actual $\bZ$. We will use $\EAninMoralGraph$ to denote the number of edges in $(\cG_{\textit{Ant}(\bX\cup\bY\cup\bI)})^a$. Since nodes in $\bI$ as well as nodes outside  $\bR$ are known to either be in $\bZ$ or to be excluded from $\bZ$, they do not need to be considered by the algorithms and can be removed from the augmented graph as shown in Figure~\ref{fig:removeIR}, similarly to the approach in \cite{Acid1996}.%

\def\dicefive#1{ 
  \node[] (M) at (0,0.5) {$V$};   
  \node[#1,color=red, cross out] at (M) {};
  \node[] (A) at (-1,-1) {$N$};   
  \node[] (B) at (1,-1) {$M$};   
  \node[] (C) at (-1,1) {$O$};   
  \node[] (D) at (1,1) {$P$};   
}
\begin{figure}\centering
\begin{tikzpicture}[]
\begin{scope}
\node at (-1,1.75) {$V$ and its neighbors   $\textit{Ne}(V)$ in $\cG^a$: };
\dicefive{}; 
\draw (M) -- (A);
\draw (M) -- (B);
\draw (M) -- (C);
\draw (M) -- (D);
\end{scope}
\begin{scope}[xshift=5cm]
\node at (-1,1.75) {\bf Case: $V \in \bI$  };
\dicefive{draw}; 
\end{scope}
\begin{scope}[xshift=10cm]
\node at (-1,1.75) {\bf Case: $V \notin \bR$  };
\dicefive{draw};
\draw (A) -- (B);\draw (A) -- (C);\draw (A) -- (D);
\draw (B) -- (C);\draw (B) -- (D);
\draw (C) -- (D);
\end{scope}
\end{tikzpicture}
\caption{This figure explains the removal of nodes in $\bI$ and outside of $\bR$ from the augmented graph $(\cG_{\textit{Ant}(\bX\cup\bY\cup\bI)})^a$. Shown is an exemplary node $V$ with all its neighbors in the  augmented graph. In the case $V \in \bI$ the node $V$ blocks all paths through $V$, so the second graph obtained by removing  $V$ has no remaining edges.
In the case $V \notin \bR$ no path is blocked by $V$, so after removing the node all its neighbors need to be linked to preserve the connectivities as shown in the third graph.
The time  needed to insert the new edges is $\cO(|\bV\setminus\bR| \cdot |\textit{Ne}(\bV\setminus\bR)|^2) = \cO(n^3)$, so this removal of nodes outside $\bR$ is only used for algorithm {\sc ListMinSep} in Section \ref{sec:algo:enum} and the other algorithms handle such nodes with different approaches.
}\label{fig:removeIR}
\end{figure}

{\sc TestMinSep} and {\sc FindMinSep}  can now be solved by a modified breadth-first-search in the graph $(\cG_{\textit{Ant}(\bX\cup\bY\cup\bM)})^a$. We start with providing the function {\sc TestMinSep}
which tests if $\bZ\subseteq \bR$ is an $\bM$-minimal $m$-separator relative to $(\bX,\bY)$ in an AG $\cG$:

\begin{algo}{TestMinSep}{$\cG, \bX, \bY, \bZ, \bM, \bR$}{\label{algo:isminimalsepmoral}}{14cm}

\If{$\bZ \setminus \textit{Ant}(\bX \cup \bY  \cup \bM) \neq \emptyset $ {\bf or }$\bZ \not\subseteq \bR $} \Return{false}\EndIf
\If{\textbf{not} \call{algo:issep}$(\cG, \bX, \bY, \bZ )$} 	{\Return{false}}\EndIf
\State{$\cG'^a \gets (\cG_{\textit{Ant}(\bX \cup \bY \cup \bM)})^a$}
\State{Remove from $\cG'^a$ all nodes of $\bM$.}
\State{$R_x \gets \{Z\in \bZ\mid \exists \text{ path from } X \text{ to } Z \text{ in } \cG'^a \text{ not intersecting } \bZ \setminus \{Z\} \}$}
\If{$ \bZ \neq R_x$} \Return{false}\EndIf
\State{$R_y \gets \{Z\in \bZ \mid \exists \text{ path from } Y \text{ to } Z \text{ in } \cG'^a \text{ not intersecting } \bZ \setminus \{Z\} \}$}
\If{$ \bZ \neq R_y$} \Return{false}\EndIf
\State{\Return{true}}
\end{algo}
\begin{analal}\call{algo:isminimalsepmoral}, 
runs in  $\cO(\EAninMoralGraph)$ 
because $R_x$ and $R_y$ can be computed 
with an ordinary search that 
aborts when a node in $\bZ$ is reached.
If each node in $\bZ$ is reachable from both $\bX$ and $\bY$, the set is $ \bM$-minimal as no node can be removed without opening a connecting path as shown in the example of Figure~\ref{fig:minseptest}.

By setting the parameter $ \bM = \bI$, the algorithm tests $\bI$-minimality of $\bZ$. 
By setting $ \bM = \emptyset $, the algorithm tests  strong-minimality.
\end{analal}
\begin{proposition}\label{prop:TestMinSep}
Using the algorithm above, the task 
{\sc TestMinSep}, both for testing $\bI$-minimality  and strong-minimality,  can be solved in time $\cO(\EAninMoralGraph) = \cO(n^2)$.
\end{proposition}

\begin{figure}
\centering
$\cG$:
\begin{tikzpicture}[xscale=0.8,baseline=1em]
\node (X) at (-2,0) {$X$};
\node[adjusted] (Z1) at (-1,1) {$Z_1$};
\node[adjusted] (Z2) at (0,1) {$Z_2$};
\node (Y) at (1,0) {$Y$};
\node (V1) at (-3,1) {$V_1$};
\node (V2) at (-0.5,-0.5) {$V_2$};
\graph { (V1) -> (X) <- (Z1) <- (Z2) -> (Y) -> (V2) <- (X); };
\end{tikzpicture}
$\cG_{\textit{Ant}(\bX \cup \bY \cup \bM)}$:
\begin{tikzpicture}[xscale=0.8,baseline=1em]
\node (X) at (-2,0) {$X$};
\node[adjusted] (Z1) at (-1,1) {$Z_1$};
\node[adjusted] (Z2) at (0,1) {$Z_2$};
\node (Y) at (1,0) {$Y$};
\node (V1) at (-3,1) {$V_1$};
\node (V2) at (-0.5,-0.5) {};
\graph { (V1) -> (X) <- (Z1) <- (Z2) -> (Y) ; };
\end{tikzpicture}
$\cG^a_{\textit{Ant}(\bX \cup \bY \cup \bM)}$:
\begin{tikzpicture}[xscale=0.8,baseline=1em]	
\node (X) at (-2,0) {$X$};
\node[adjusted] (Z1) at (-1,1) {$Z_1$};
\node[adjusted] (Z2) at (0,1) {$Z_2$};
\node (Y) at (1,0) {$Y$};
\node (V1) at (-3,1) {$V_1$};
\node (V2) at (-0.5,-0.5) {}; %
\graph { (Z1) -- (V1) -- (X) -- (Z1) -- (Z2) -- (Y) ; };
\end{tikzpicture}
\caption{The transformation of a graph $\cG$ to $\cG_{\textit{Ant}(\bX \cup \bY \cup \bM)}$ to $(\cG_{\textit{Ant}(\bX \cup \bY \cup \bM)})^a$ with $\bM = \emptyset$. The $m$-separating set $\{Z_1,Z_2\}$ is not minimal as no node is reachable from both $\bX$ and $\bY$, but each node alone $\{Z_1\}$ or $\{Z_2\}$ is a minimal $m$-separator.}\label{fig:minseptest}
\end{figure}

The difference between $ \bI $-minimal and strongly-minimal separation sets $\bZ\supseteq \bI$ 
is illustrated in Figure~\ref{fig:minseptest}. 
There are exactly two strongly-minimal separating sets: $ \{Z_1\} $ and $ \{Z_2\} $. 
No other  $m$-separator will be strongly-minimal regardless of which 
nodes are added to a constraining set $ \bI $. 
Therefore, for $ \bI = \emptyset $, both $m$-separators satisfy the constraint, 
for $ \bI= \{Z_1\} $ or $ \bI=\{Z_2\} $, only one of them does and, 
for a $ \bI = \{Z_1, Z_2\}$ or $ V_1 \in \bI $, no strongly-minimal 
$m$-separator satisfies it. 
The constraint $ \bI $ just chooses some $m$-separators of 
a fixed set of strongly-minimal $m$-separators.

On the other hand, when computing an $ \bI $-minimal $m$-separator, we treat the nodes of $ \bI $ as fixed and search for a minimal $m$-separator among all supersets of $ \bI $. In Figure~\ref{fig:minseptest}, if $ \bI$ is either $ \{Z_1\}$, $\{Z_2\} $ or $\{Z_1,Z_2\} $, then $ \bI $ itself is an $ \bI $-minimal $m$-separator and no other $ \bI $-minimal $m$-separator exists. If $ \bI = \{V_1\} $, then $ \{Z_1, V_1\} $ and $ \{Z_2, V_1\} $ are $ \bI $-minimal. This is an easier and in some sense more natural concept, since it can be modeled by removing the nodes of $ \bI  $ from the graph and searching a minimal $m$-separator for the remaining nodes.

From a covariate adjustment perspective, the definition of $ \bM $-minimality  
is most meaningful in the case $\bM=\emptyset$ or $\bM=\bI$.
However, our algorithms technically also allow $\bM$ to lie ``between'' $\emptyset$ and $\bI$
or even partly outside $\bM$, even though this is less relevant for our application.
For example if $ \emptyset \not= \bM \subset \bI $, the nodes of $ \bM $ can be ignored for the minimality, 
while the nodes of $ \bI\setminus\bM $ must be unremovable like all nodes in 
the case of $ \emptyset $-minimality. In an example in Figure~\ref{fig:minseptest},
for $ \bM=\{V_1\}, \bI= \{Z_1, V_1\} $ would only accept $ \{Z_1, V_1\} $ as $m$-separator. $ \bM=\{Z_1\}, \bI= \{Z_1, Z_2\} $ would not allow any.  
Every $m$-separator $ \bZ$ %\subseteq \bR $ 
is $\bZ$-minimal. %, and would be $ \bR $-minimal and $ \bV $-minimal if . %

Next we give an algorithm to \emph{find} an $ \bI $-minimal $ m $-separator:

\begin{algo}{FindMinSep}{$\cG, \bX, \bY, \bI, \bR$}{\label{algo:findminimalsepmoral}}{14cm}
\State{$\cG' \gets \cG_{\textit{Ant}(\bX \cup \bY \cup \bI)}$}
\State{$\cG'^a \gets (\cG_{\textit{Ant}(\bX \cup \bY \cup \bI)})^a$}
\State{Remove from $\cG'^a$ all nodes of $\bI$.}

\State{$\bZ' \gets \bR \cap \textit{Ant}(\bX \cup \bY ) \setminus (\bX\cup\bY)$}
\State{$\bZ'' \gets \{ Z \in \bZ' \mid \exists \text{ a path from } \bX \text{ to } Z \text{ in }  \cG'^a \text{ not intersecting } \bZ' \setminus \{Z\}  \}$  }
\State{$\bZ \gets \{ Z \in \bZ'' \mid \exists \text{ a path from } \bY \text{ to } Z \text{ in }  \cG'^a \text{ not intersecting } \bZ'' \setminus \{Z\}  \}$  }
\If{\textbf{not} \call{algo:issep}$(\cG', \bX, \bY, \bZ)$}
{\Return{$\bot$}}
\EndIf
\State{\Return{$\bZ\cup \bI$}}
\end{algo}

\begin{analal}

Algorithm  \call{algo:findminimalsepmoral} begins 
with the separating set $\bR \cap \textit{Ant}(\bX \cup \bY)\setminus (\bX\cup\bY)$ 
and finds a subset satisfying the conditions tested 
by algorithm \call{algo:isminimalsepmoral}. As \call{algo:isminimalsepmoral} it can be implemented in $\cO(\EAninMoralGraph)$ using a breadth-first-search starting from $\bX$ ($\bY$) that aborts when a node in $\bZ'$ ($\bZ''$) is reached.
\end{analal}

Algorithm {\sc FindMinSep} finds an $ \bI $-minimal $ m $-separator.
\highlightrevision{r3c4}{
Note that setting the argument  $ \bI$ of the algorithm to $\emptyset $ does not lead to a strongly-minimal separator $\bZ$ that satisfies the constraint $ \bI \subseteq \bZ \subseteq \bR $ for a given non-empty set $\bI$.} 
%, because then the returned set $ \bZ $ may no longer satisfy the constraint $ \bI \subseteq \bZ \subseteq \bR $.

\begin{proposition}
\label{prop:FindMinSep}
The algorithm above finds an $\bI$-minimal $m$-separator $\bZ$, with  $\bI\subseteq\bZ\subseteq\bR$, 
in time $\cO(\EAninMoralGraph) = \cO(n^2)$.
\end{proposition}

In the problem {\sc FindMinCostSep}, each node $V$ is associated with a cost $w(V)$ given by a cost function $w: \bV \to \mathbb{R}^+$ and the task is to find a set $\bZ$ $m$-separating $\bX$ and $\bY$ which minimizes the total cost $\sum_{Z\in\bZ} w(Z)$ under the constraint $\bI \subseteq \bZ\subseteq \bR$. 
In order to find an $m$-separator of minimum size, we can use a function $ w(V) = 1\ \forall V \in \bR$ that assigns unit cost to each node. Alternatively we might want to find an $m$-separator  that minimizes the cost of  measuring the variables in the separator or that minimizes the number of combinations that the values of these variables can take. 
When each node $V$ corresponds to a random variable that can take $k_V$ different values, there are $ \prod_{V\in\bV} k_V $ combinations, which can be minimized by a logarithmic cost function $ w(V) = \log k_V\ \forall V \in \bR$. %

We again construct the augmented graph and can afterwards solve the problem with any weighted min-cut algorithm.

\begin{algo}{FindMinCostSep}{$\cG, \bX, \bY, \bI, \bR, w$}{\label{algo:findminimumsep}}{16cm} 
\State{$\cG' \gets \cG_{\textit{Ant}(\bX \cup \bY\cup\bI)}$}
\State{$\cG'^a \gets (\cG_{\textit{Ant}(\bX \cup \bY\cup\bI)})^a$}
\State{Add a node $X^m$ connected to all nodes in $\bX$, and a node $Y^m$ connected to all nodes in~$\bY$.}
\State{Assign infinite cost to all nodes in $\bX \cup \bY \cup (\bV \setminus \bR)$  and  cost  $w(Z)$ to every other node $Z$.}
\State{Remove all nodes of $\bI$ from $\cG'^a$.}
\State{\Return {a  minimum vertex cut $\bZ$ separating $ X^m $ and $ Y^m $ in the undirected graph.}}
\end{algo}
\begin{analal}
The correctness %
follows from the fact that a minimum set is a minimal set and the minimum cut found in the ancestor moralized graph is therefore the minimum $ m $-separating set. The minimum cut can be found using a maximum flow algorithm in $ \cO(n^3) $ due to the well-known min-cut-max-flow theorem \cite[Chapter~6]{Even1979graphalgorithm}.
\end{analal}
\begin{proposition}
\label{prop:FindMinCostSepI}
The algorithm above solves in time $\cO(n^3)$ the task {\sc FindMinCostSep} in case of $\bI$-minimality.
\end{proposition}

The runtime primarily depends on the used min-cut/max-flow algorithm. Using a state-of-the-art max-flow algorithm recently presented by Orlin improves the runtime to $ \cO(n \EAninMoralGraph) $ \cite{maxflowOrlin2013}, although this is not necessarily an improvement in our setting, because the augmented graph can have $ m_a = \cO(n^2) $ edges and then $ \cO(n \EAninMoralGraph)$ and $\cO(n^3) $ are the same. Faster max-flow algorithms are also known for specific graph classes or specific cost functions. In the special case of a unit cost function $w(V) = 1$ on undirected graphs an $ \cO(\EAninMoralGraph \sqrt{n}) $ algorithm is known \cite{Even1979graphalgorithm}, which is not directly applicable, since algorithm {\sc FindMinCostSep} changes the nodes $ \bX \cup \bY \cup (\bV \setminus \bR) $ to have infinite costs.  However, we can apply the max-flow algorithm to a new graph containing only nodes in $ \bR' = (\bR\setminus(\bX \cup \bY)) \cap \textit{Ant}(\bX\cup\bY\cup\bI)$ by removing the nodes of $ \bV\setminus\bR $ iteratively in $\cO((n - |\bR|) n^2)$ as shown in Figure~\ref{fig:removeIR} or by creating a new graph only containing those nodes in $ \cO(|\bR'| \EAninMoralGraph) $ as described in \cite{TianPP1998}, resulting in a total runtime of $ \cO(\min((n - |\bR|) n^2, |\bR'|\EAninMoralGraph) + \EAninMoralGraph \sqrt{|\bR'|}) = \cO(|\bR'|\EAninMoralGraph) $ for a unit cost function.

The set $ \bZ $ returned by algorithm  {\sc FindMinCostSep} is also strongly-minimum, 
unless no strongly-minimum set $\bZ$ exists under the given constraints $ \bI\subseteq \bZ\subseteq \bR $. 
To see this assume $ \bZ $ is not strongly-minimum and there exists a 
strongly-minimum set $ \bZ' $ satisfying the constraint. Then $ \sum_{Z\in\bZ'} w(Z) < \sum_{Z\in\bZ} w(Z) $. But $ \bZ' $ satisfies $ \bI \subseteq \bZ' $, so  $ \bZ $ was not $ \bI $-minimum, a contradiction.

All strongly-minimum sets for a graph have the same minimum sum $ \sum_{Z\in\bZ} w(Z) $, 
regardless of the constraint $ \bI $, 
so we can test if the set returned by {\sc FindMinCostSep} is 
strongly-minimum by finding one strongly-minimum set $ \bZ' $ and 
testing if $ \sum_{Z\in\bZ} w(Z) = \sum_{Z\in\bZ'} w(Z) $.  
Such a strongly-minimum set $ \bZ' $ can be obtained by 
calling {\sc FindMinCostSep} again with parameter $ \bI = \emptyset $. 
Although $ \bZ' $ might not fulfill the constraint $ \bI \subseteq\bZ' $, it has the same required sum.
Thus, we get the following:

\begin{proposition}\label{prop:FindMinCostSep:proper}
Finding an $m$-separator $\bZ$, with $\bI\subseteq\bZ\subseteq\bR$, which is strongly-minimum
can be done in time $\cO(n^3)$.
\end{proposition}

The same arguments as above show that $ \bZ $ is $ \bM $-minimum for any $ \bM \subset \bI $, 
if an $ \bM $-minimum $ m $-separator satisfying the constraints exist.

\subsection{The hardness of strong-minimality}

For most problems it is easier to find a minimal solution than 
a minimum one, but for $m$-separators the opposite is true. 
If a strongly-minimum $m$-separator exists, it is also strongly-minimal. 
However there is a gap, where no strongly-minimum $m$-separator exists,
and the $\bI$-minimum or $ \bI $-minimal $m$-separators are not strongly-minimal.

In this section we show that it is hard to find a strongly-minimal $ m $-separator even for singleton $X, Y$ in DAGs, in which  $ m $-separation and $ d $-separation are equivalent. 
\highlightrevision{r2c34}{
Due to the equivalence between $d$-separation and vertex separators in the moralized graph, 
%Because our proof of this result will use a moralized graph
this implies that it is also NP-hard to find strongly-minimal vertex separators in undirected graphs. Together with the characterizations of adjustment sets in the coming sections, it will follow that it is also NP-hard to find strongly-minimal adjustments in DAGs or ancestral graphs.
}

\begin{proposition}\label{prop:np-complete}
It is an NP-complete problem to decide if in a given DAG there exists 
a strongly-minimal $d$-separating set 
$\bZ$ containing  $\bI$.
\end{proposition}

\begin{proof}
The problem  is in NP, since given $ \bZ $ verifying $ \bI\subseteq \bZ $ is trivial and the 
strong-minimality can be efficiently tested by algorithm \textsc{TestMinSep} with  parameter $ \bI= \emptyset $.

\begin{figure}	

\centering\begin{tikzpicture}

\node (X) at (-3,0) {$X$};

\node (Y1) at (0,-1) {$Y$} ;
\node[adjusted] (I1) at (0,0) {$I_1$} ;
\node (V1l) at (-1,1) {$V_{1,l}$} ;
\node (V1r) at (1,1) {$V_{1,r}$} ;
\node (V1) at (-1,0) {$V_1$};
\node (V1neg) at (1,0) {$\overline{V}_1$};
\draw[->] (I1) -- (V1l); \draw[->] (I1) -- (V1r); 
\draw[->] (V1l) -- (V1);  \draw[->] (V1r) -- (V1neg); 
\draw[->] (I1) -- (Y1); \draw[->] (V1) -- (Y1); \draw[->] (V1neg) -- (Y1); 

\begin{scope}[xshift=3cm]
\node (Y2) at (0,-1) {$Y$} ;
\node[adjusted] (I2) at (0,0) {$I_2$} ;
\node (V2l) at (-1,1) {$V_{2,l}$} ;
\node (V2r) at (1,1) {$V_{2,r}$} ;
\node (V2) at (-1,0) {$V_2$};
\node (V2neg) at (1,0) {$\overline{V}_2$};
\draw[->] (I2) -- (V2l); \draw[->] (I2) -- (V2r); 
\draw[->] (V2l) -- (V2);  \draw[->] (V2r) -- (V2neg); 
\draw[->] (I2) -- (Y2); \draw[->] (V2) -- (Y2); \draw[->] (V2neg) -- (Y2); 
\end{scope}

\begin{scope}[xshift=6cm]
\node (Y3) at (0,-1) {$Y$} ;
\node[adjusted] (I3) at (0,0) {$I_3$} ;
\node (V3l) at (-1,1) {$V_{3,l}$} ;
\node (V3r) at (1,1) {$V_{3,r}$} ;
\node (V3) at (-1,0) {$V_3$};
\node (V3neg) at (1,0) {$\overline{V}_3$};
\draw[->] (I3) -- (V3l); \draw[->] (I3) -- (V3r); 
\draw[->] (V3l) -- (V3);  \draw[->] (V3r) -- (V3neg); 
\draw[->] (I3) -- (Y3); \draw[->] (V3) -- (Y3); \draw (V3neg) edge [->] (Y3); 
\end{scope}

\node at (8cm, 0.5) {$\ldots$};

\node[adjusted]  (IC1) at (1.5,-2.25) {$ I'_1 $};

\node[adjusted]  (IC2) at (4.5,-2.25) {$ I'_2 $};

\draw[dashed,->] (V1neg) to (IC1)  (V2) edge[->] (IC1) (V3) edge[->] (IC1);
\draw[dashed,->] (V1neg) to (IC2)  (V2neg) edge[->] (IC2) (V3neg) edge[->,bend left=20] (IC2);

\node at (6cm, -2.5) {$\ldots$};

\draw (IC1) edge [->,bend left=15] (X);
\draw (IC2) edge [->,bend left=35] (X);

\draw (V1l.west) edge [->,bend right=20] (X);
\draw (V1r.north west) edge [->,bend right=30] (X);

\draw (V2l.north west) edge [->,bend right=35] (X);
\draw (V2r.north west) edge [->,bend right=40] (X);

\draw (V3l.north west) edge [->,bend right=45] (X);
\draw (V3r.north west) edge [->,bend right=50] (X);

\end{tikzpicture}

\caption{The graph used in the proof of Proposition~\ref{prop:np-complete}, which represents the first three variables $V_1, V_2, V_3$, and two clauses $(\overline{V}_1 \vee V_2 \vee V_3)$ and $(\overline{V}_1 \vee \overline{V}_2 \vee \overline{V}_3)$.  All shown $Y$ nodes can be considered to be a single node $Y$, but we display them separately to reduce the number of overlapping edges in the figure. %
}\label{fig:minimality:proof:graph}
\end{figure}
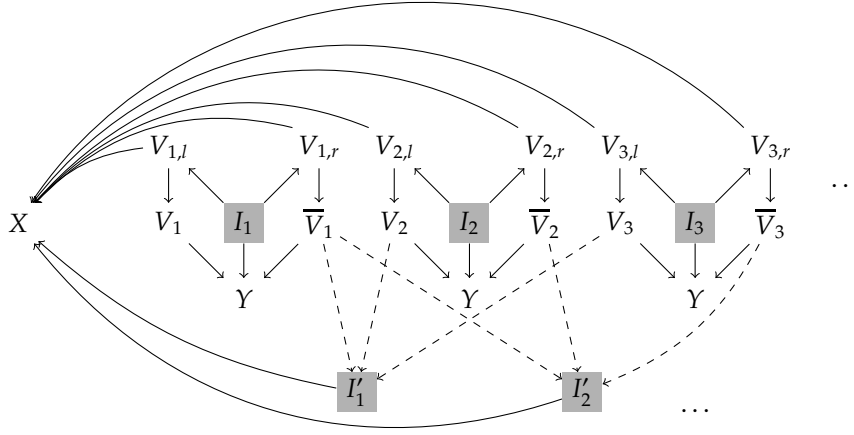

To show the NP-hardness we take an instance of 3-SAT, a canonical 
NP-complete problem \cite{michael1979computersNP}, and construct a DAG $ \cG  $ 
and a set $ \bI $, such that  a 
strongly-minimal $d$-separating set containing $\bI$ exists in $ \cG $, if and only 
if the 3-SAT formula is satisfiable. %

The 3-SAT instance consists of $k$ variables $ V_1, \ldots, V_k $, and $ \ell $ clauses $C_i = ( W_{i,1} \vee W_{i,2} \vee W_{i,3} ) $ of literals $ W_{i,j} \in \{V_1, \ldots, V_k, \overline{V}_1, \ldots, \overline{V}_k \} $. It is NP-hard to decide if there exists a Boolean assignment $ \phi: \{V_1,\ldots V_k\} \to \{\text{true},\text{false}\} $ that satisfies the formula $ C_1 \wedge \ldots \wedge C_\ell $.

\highlightrevision{r2c35}{
The basic idea of the construction is to ensure that any $m$-separator must contain some nodes to block a path between $X$ and $Y$, while no $m$-separator can contain all those nodes, because all nodes together would block all paths from $X$ to nodes in $\bI$ making the separator not $\bI$-minimal. 
The choice of a node will correspond to choosing an assignment to a variable of the satisfiable problem.
}

Let $ \cG = (\bV,\bE)  $ be defined as:

\newcommand{\forallvariables}{\mid i \in \{1,\ldots,k \}}
\newcommand{\forallclauses}{\mid i \in \{1,\ldots,\ell \}}

\[\begin{array}{rccl}
\bV &=&& \{ X, Y\} \cup \{ V_{i,l}, V_{i,r}, V_i, \overline{V}_i, I_i \forallvariables  \} \cup \{I'_i \forallclauses \}\\

\bE &=&& \{ X \gets V_{i,l} \to V_i \to Y   \forallvariables \} \\
   &&  \cup& \{ X \gets V_{i,r} \to \overline{V}_i \to Y \forallvariables \} \\
   &&  \cup& \{ I_i \to V_{i,l},\ I_i \to V_{i,r},\ I_i \to Y \forallvariables \}\\
   &&  \cup& \{ I'_i \to X \forallclauses \}\\
   &&  \cup& \{ W_{i,j} \to I'_i \forallclauses, W_{i,j} \in C_i \}\\

\bI &=&& \{I_i \forallvariables \} \cup \{I'_i \forallclauses\}
\end{array}
\]

The resulting graph is shown in Figure~\ref{fig:minimality:proof:graph}. We identify the literal $V_i$ ($ \overline{V}_i $) in the formula with the node $ V_i $ ($ \overline{V}_i $) in the graph. $V_{i,l}$ ($V_{i,r}$) is a left (right) node, $,l$ in the index should not be confused with the number of clauses $\ell$.

“$\Leftarrow$”: Let $ \phi $ be a Boolean assignment that satisfies the formula. Let 
\[
\bZ = \bI \cup \{ V_i, V_{i,r} \forallvariables,\ \phi(V_i) = \text{false} \}\cup \{ \overline{V}_i, V_{i,l} \forallvariables,\ \phi(V_i) = \text{true} \}. 
\]

To show that $ \bZ  $ $ d$-separates $ X $ and $ Y $, we begin a breadth first search at $ X $ and enumerate all reachable nodes. Immediately reachable are nodes $ I'_i $, but they are in $\bI\subseteq \bZ $, so the path stops there. The parents $ V_{i,l} $ are reachable, but they are either in $ \bZ $ or only $ V_i, I_i \in \bZ $ are reachable and the path stops there. Likewise no path through $ V_{i,r} $ can reach $ Y $. Hence $ \bZ $ is a $ d $-separator.

$ \bZ $ is also strongly-minimal: If $ V_i \in \bZ$, $ V_{i,l} \notin \bZ  $ and $ \bZ\setminus\{V_i\} $ would not block the path  $X \gets V_{i,l} \to V_i \to Y$. A similar path would be opened by removing $ \overline{V}_i,\ V_{i,l}$ or $ V_{i,r} $. $ \bZ \setminus \{ I_i\} $ is not a $ d $-separator as it would not block either the path $ X \gets V_{i,l} \to I_i \to Y $ or $ X \gets V_{i,r} \to I_i \to Y $. If a clause $ C_i $ is satisfied by a literal $ V_j$, $ \bZ \setminus \{ I'_i\} $ is not a $ d $-separator, as it would not block the path $ X \gets I'_i \gets V_j \to Y $ . Likewise $ X \gets I'_i \gets \overline{V}_j \to Y $ would be opened by removing $ I'_i $ if the clause $ C_i $ is satisfied by a literal $ \overline{V}_j $.

Therefore $ \bZ $ is a strongly-minimal $ d $-separator.

“$\Rightarrow$”: Now we show that a strongly-minimal $ d $-separator $ \bZ $ yields a satisfying assignment $ \phi $. For every $i$ the two paths $X \gets V_{i,l} \to V_i \to Y $ and $X \gets V_{i,r} \to \overline{V}_i \to Y $ need to be blocked by a node of $ \{V_{i,l}, V_i\} $ and a node of $ \{V_{i,r}, \overline{V}_i\} $. If neither $V_i$ nor $ \overline{V}_i $ are in $ \bZ $, 
both $V_{i,l}$ and $V_{i,r}$ must be in $\bZ$, so $I_i$ is not reachable from $X$, $\bZ \setminus \{I_i\}$ is a $d$-separator and $\bZ$ is not strongly-minimal. Therefore $V_i$ or $\overline{V}_i$ is in $\bZ$ and the following Boolean assignment $\phi$ to the variables is well-defined:

\[\phi(V_i) = \begin{cases}
\text{true}&  V_i \notin \bZ, \\
\text{false}& \overline{V}_i \notin \bZ,\\
\text{false}& otherwise.
\end{cases}\]

Since $I'_i \in \bI$ for all $ i $, $I'_i$ has to be reachable from $Y$, so there is an open path $I'_i \gets V_j \to Y$ (or $I'_i \gets \overline{V}_j \to Y$) and $ V_j $ (or $ \overline{V}_j $) is not in $ \bZ $ for some $ j $. This $ V_j $ (or $ \overline{V}_j $) satisfies clause $ C_i $ according to the definition of $ \phi $. Hence every clause and the formula is satisfiable.

\end{proof}

\subsection{Algorithms for minimal $m$-separators in sparse graphs}
\label{sec:separators:sparse:AGs}

Subsection~\ref{sec:sepalgo:dense} gave $\cO(n^2)$ algorithms for the problems {\sc TestMinSep} 
and {\sc FindMinSep} in case of $\bI$-minimality, which is optimal for dense graphs. 
The runtime of these algorithms is limited by the time required %for construction 
to construct the augmented graph, which might contain nodes and edges that have no influence on the minimal $ m $-separator and are thus included unnecessarily. This leads to the question, can we test or find the minimal $ m $-separator without constructing the augmented graph? The obvious way is to test the definition of minimality directly by removing each of the nodes of $ \bZ $ and testing if the smaller set is still an $ m $-separator, which yields $\cO(|\bZ|(n+m)) = \cO(| \textit{Ant}(\bX \cup \bY \cup \bI ) | (n + m)) = \cO(n(n+m))$ algorithms. This is generally slower than the previous runtimes, however in sparse graphs that have a low number of edges $ m \approxeq n $ and a low number of ancestors $ | \textit{Ant}(\bX \cup \bY \cup \bI ) | \ll n $, it might be faster.  
This alternative algorithm for {\sc TestMinSep} in sparse graphs, in case of 
$\bI$-minimality, is as follows

\begin{algo}{TestMinSepSparse}{$\cG, \bX, \bY, \bZ, \bI, \bR$}{\label{algo:isminimalsepnaive}}{10cm}

\If{$\bZ \setminus \textit{Ant}(\bX \cup \bY\cup\bI) \neq \emptyset $ {\bf or } $ \bZ \not\subseteq \bR $} \Return{false}\EndIf
\If{\textbf{not} \call{algo:issep}$(\cG, \bX, \bY, \bZ )$} \Return{false}\EndIf
\State{$\cG' \gets \cG_{\textit{Ant}(\bX \cup \bY\cup\bI)}$}
\For {all $U \gets \bZ \setminus \bI$}
\If{\call{algo:issep}$(\cG', \bX, \bY, \bZ \setminus \{U\})$} \Return{false}\EndIf
\EndFor
\Return{true}
\end{algo}
An $\bI$-minimal $ m $-separator relative to $ (\bX, \bY)  $ can be computed 
using the algorithm 

\begin{algo}{FindMinSepSparse}{$\cG, \bX, \bY, \bI, \bR$}{\label{algo:findminimalsepnaive}}{8cm}
\State{$\cG' \gets \cG_{\textit{Ant}(\bX \cup \bY \cup \bI)}$}
\State{$\bZ \gets \bR \cap \textit{Ant}(\bX \cup \bY \cup \bI) \setminus (\bX \cup \bY)$}
\If{\textbf{not} \call{algo:issep}$(\cG', \bX, \bY, \bZ)$}
{\Return{$\bot$}}
\EndIf
\For {all $U$ in $\bZ\setminus \bI$}
\If{\call{algo:issep}$(\cG', \bX, \bY, \bZ \setminus \{U\})$} \State{$\bZ \gets \bZ \setminus \{U\}$}\EndIf
\EndFor
\State{\Return{$\bZ$}}
\end{algo}

\begin{proposition}
The tasks {\sc TestMinSep} and  {\sc FindMinSep}
for $\bI$-minimal separators can be solved in time $\cO( (n+m) \cdot | \textit{Ant}(\bX \cup \bY \cup \bI) |  )$.
\end{proposition}

The correctness of these algorithms depends on the non-obvious fact that $m$-separators are monotone:  if $\bZ \setminus V$ is not an $m$-separator, no $\bZ \setminus \bV$ with $V\in \bV$ is one either. This monotonicity was proven for $d$-separation by \cite{TianPP1998} and we will state their results in terms of  $m$-separation: %

\begin{lemma}\label{lem:monotone:m:separators}
Let $\bX$, $\bY$ be two sets and let $\bZ$ be an $m$-separator relative to $(\bX,\bY)$. 
If the set $\bZ \cup \{Z_1, \ldots, Z_n\}$ is also an $ m $-separator, where $Z_1, \ldots, Z_n$ are single nodes which are not in $\bZ$ then either $\bZ \cup \{Z_1\}$, or $\bZ \cup \{Z_2\}$, $\ldots$ or $\bZ\cup \{Z_n\}$ must be another $ m $-separator between $ \bX $ and $ \bY $.
\end{lemma}

\begin{corollary}
If $\bZ$ and $\bZ \cup \bZ_n$ are m-separators, where $\bZ_n = \{Z_1, \ldots, Z_n\}$, then there exist a series of $n-1$ m-separators: $\bZ \cup \bZ_i$, $i=1,\ldots,n-1$, with
$\bZ_1 \subset \ldots \subset \bZ_{n-1} \subset \bZ_n $
such that each $\bZ_i$ contains exactly $i$ nodes.
\end{corollary}

\begin{corollary}\label{cor:monotonicity:sep}
If no single node can be removed from a m-separator $\bZ$ without destroying m-separability, then $\bZ$ is minimal.
\end{corollary}
 
To generalize the proofs of \cite{TianPP1998} to $m$-separation, 
it is sufficient to replace $d$-separation with $m$-separation, DAGs with AGs and ancestors with anteriors throughout their arguments. Therefore, we do not repeat the proofs here. %
 
In principle one needs to take care of the order in which the nodes are enumerated when searching a minimal subset of a $m$-separator by removing nodes one-by-one like we do in algorithm \call{algo:findminimalsepnaive}. For example to obtain a minimal subset of $\{Z_1, Z_2\}$ in a DAG $X \to Z_2 \gets Z_1 \to Y$ the node $Z_2$ must be removed first, or $Z_1$ could not be removed. However, this is not an issue when all nodes of $\bZ$ are anteriors in $\textit{Ant}(\bX\cup\bY\cup\bI)$ due to the correspondence of $m$-separators to separators in the augmented graph, since removing a node cannot block a path in the augmented graph. From this correspondence one can also conclude Lemma~\ref{lem:monotone:m:separators} restricted to anteriors.

\subsection{Algorithms for enumerating all $m$-separators}\label{sec:algo:enum}

Lastly, we consider the problem of listing 
\emph{all} $m$-separators and \emph{all}  minimal $m$-separators between $\bX$ and $\bY$ in $\cG$. 
Since there might be an exponential number of (minimal) $m$-separators, it is not possible to list them all in polynomial time. For example in a path $X \gets V \gets V' \gets Y$ either $V$ or $V'$ must be in a minimal $m$-separator between $X$ and $Y$, so a graph containing $k$ such paths will have at least $2^k$ different $m$-separators. Therefore we will describe algorithms running with  $\cO(n^3)$ delay, which means that at most $\cO(n^3)$ time will pass between the start or the output of an $m$-separator and the output of the next $m$-separator or the end of the algorithm. They are based on the enumeration algorithm for minimal vertex separators of \citep{Takata2010}. 

\begin{algo}{ListSep}{$\cG, \bX, \bY, \bI, \bR$}{\label{algo:enumsep}}{8cm}
    \If{ {\sc FindSep}$(\cG, \bX, \bY, \bI, \bR) \neq \bot$}
    \If{$\bI = \bR$} 
    {Output $\bI$}
    \Else
      \State  {$V \gets $ an arbitrary node of $\bR \setminus \bI$}
      \State{{\sc  ListSep}($\cG, \bX, \bY, \bI \cup \{V\}, \bR$)}
      \State{{\sc  ListSep}($\cG, \bX, \bY, \bI, \bR\setminus\{V\}$)}
    \EndIf
    \EndIf
\end{algo}

\begin{analal}
Algorithm {\sc ListSep} performs backtracking to enumerate all $\bZ$, with $\bI \subseteq \bZ \subseteq \bR$, aborting branches that will not find a valid separator. Since every leaf will output a separator, the tree height is at most $n$ and the existence check needs $\cO(n+m)$, the delay time is $\cO(n(n+m))$.  
The algorithm generates every separator exactly once: if initially $\bI \subsetneq \bR$, with $V\in \bR\setminus \bI$, then the first recursive call returns all separators $\bZ$ with $V\in \bZ$ and the second call returns all $\bZ'$ with $V\not\in \bZ'$. Thus the generated separators are pairwise disjoint.  
\end{analal}

\begin{proposition}\label{prop:ListSep}
Using the algorithm above, the task 
{\sc ListSep} can be solved with polynomial delay $\cO(n(n+m))$.
\end{proposition}

To enumerate all minimal $ m $-separators we can directly apply Takata's enumeration algorithm \cite{Takata2010} after transforming the ancestral graph to its augmented graph:

\begin{algo}{ListMinSep}{$\cG, \bX, \bY, \bI, \bR$}{\label{algo:enumminimalsep}}{12cm} 
\State{$\cG' \gets \cG_{\textit{Ant}(\bX \cup \bY\cup\bI)}$}
\State{$\cG'^a \gets (\cG_{\textit{Ant}(\bX \cup \bY\cup\bI)})^a$}
\State{Add a node $X^m$ connected to all $\bX$ nodes.}
\State{Add a node $Y^m$ connected to all $\bY$ nodes.}
\State{Remove nodes of $\bI$.}
\State{Remove nodes of $\bV \setminus \bR$ connecting the neighbors of each removed node.}
\State{Use the algorithm in~\citet{Takata2010} to list all sets separating $X^m$ and $Y^m$.}
\end{algo}

\begin{analal}
The correctness is shown by \citet{TextorLiskiewicz2011} for
adjustment sets and generalizes directly to $m$-separators,
because after moralization, both problems are equivalent to
enumerating vertex cuts of an undirected graph. The handling
of $\bI$ is shown by \citet{Acid1996}.
\end{analal}

\begin{proposition}\label{prop:ListMinSep}
The task {\sc ListMinSep} can be solved with polynomial delay $\cO(n^3)$.
\end{proposition}

\section{Empirical analysis of DAG consistency testing}

\label{sec:basissets}

The concept of $m$-separation in ancestral graphs, or $d$-separation in DAGs, is of central importance in the field of Bayesian networks and graphical causal models because it couples the graphical model structure to conditional independences in the corresponding probability distribution. Perhaps the most direct application of graphical separation is therefore \emph{consistency testing}, where we check whether a graphical causal model is in fact consistent with the dataset it is intended to represent. If a graphical causal model fails to pass this test, then any downstream analyses, such as implied adjustment sets, are potentially nullified. 
In this section, we illustrate how the algorithmic framework that we developed in Section~\ref{sec:algo} can be harnessed for graphical causal model checking by applying it to the problem of deriving \emph{basis sets} for testing the consistency of a DAG to a given probability distribution. 

A DAG $\cG$ is called \emph{consistent} with a probability distribution $P$ if for all pairwise disjoint subsets $\bX,\bY,\bZ \subseteq \bV$, where only $\bZ$ may be empty, $\bX$ and $\bY$ are conditionally independent given $\bZ$ in $P$ whenever $\bZ$ $d$-separates $\bX$ and $\bY$ in $\cG$. Therefore, to test consistency, we could in principle enumerate all $d$-separation relations in $\cG$ and then perform the corresponding conditional independence tests. However, that set of relations can be very large -- there is already an exponential number of subsets $\bX$ and $\bY$ to consider, each of which could be separated by an exponential number of sets $\bZ$. A \emph{basis set} of $d$-separation statements entails all other statements implied by $\cG$ when combining them using the axioms of conditional independence \cite{Dawid1979}. The \emph{canonical} basis set has the form
$$
\{ \forall X \in \bV: X \independent \bV \setminus (\textit{Pa}(X) \cup \textit{De}(X)) \mid \textit{Pa}(X) \},
$$
or in words, every variable must be jointly independent of its non-descendants conditioned on its parents. This basis set only contains $n$ elements but the independences to be tested can be of very high order. 

In practice, conditional independence is often tested using methods that do not distinguish between joint and pairwise independence, such as partial correlation or linear regression. In such cases, even simpler basis sets can be derived. For instance, Pearl and Meshkat \cite{PearlMeshkat1999} discuss basis sets for hierarchical linear regression models (also known as path models or structural equation models \cite{Thoemmes2017}). We here consider the vertices $\bV=\{X_1,\ldots,X_n\}$ to be indexed topologically, such that for $j > i$, we have that $X_j \notin \textit{An}(X_i)$. Then we consider basis sets having the form 
$$
\{ \forall j > i, X_i \to X_j \notin \bE  : X_i \independent X_j \mid \bZ_{ij} \},
$$
or in words, we include one separation statement for every nonadjacent pair of variables. These are pairwise rather than joint independences, which is equivalent when we measure dependence only by partial correlation \cite{PearlMeshkat1999}. An obvious choice is $\bZ_{ij}=\textit{Pa}(X_j)$, in which case we simply obtain a pairwise version of the canonical basis set above. We call that basis set the \emph{parental basis set}. However, Pearl and Meshkat show that another valid choice is any separating set $\bZ_{ij}$ that contains only nodes whose distance to $X_j$ is smaller than the distance between $X_i$ and $X_j$ (where the distance is defined length of a shortest path between two nodes). We call any such set a \emph{sparse basis set}. Note that every parental basis set is also sparse, but the separators in sparse basis sets can be substantially smaller than in parental basis sets.

\subsection{Simulation setup}

We evaluate the difference between parental and sparse basis sets on random DAGs. These are generated with 
various numbers of variables $ n \in \{10, 25, 50, 100\}$.
The edges are chosen independently with individual probabilities $P(\textit{edge})=\frac{l}{(n-1)}$, where $l \in 2,5,10,20$. For small $n$ the probabilities are capped at 1. 
For example, for $l= n = 10$, the parameter $P(\textit{edge}) = \min\{\frac{10}{(10-1)}, 1\} =1$ will only generate complete graphs.
Parameter $l$ describes the expected number of neighbors of a node in a generated DAG.
This  leads to an expected number of edges in generated graphs $\EX[m] \approxeq l n / 2$, 
 as there are $ \frac{n (n-1)}{2} $ 
possible edges in a DAG of $n$ nodes, each existing with probability $P(\textit{edge})$. Our algorithmic framework makes it straightforward to compute sparse basis sets with minimal separating sets by simply setting $\bR = \{ X_k \in \bV \mid d(X_k,X_j) < d(X_i, X_j)\}$, where we use $d(X_a, X_b)$ to denote the distance between two nodes. 

\subsection{Empirical results}

We empirically evaluate the sizes of parental and sparse basis sets in these random DAGs (Figure~\ref{fig:modelchecking}). 
The results show that the benefit of using sparse rather than parental basis sets does not depend much on the size of the DAG but rather on the amount of edges. For instance, when nodes had on average $5$ neighbors, sparse basis sets had between 20\% and 40\% fewer conditioning variables than the canonical parental basis sets.

These results provide a first example of how our framework can be applied in the context of graphical causal modeling. In the rest of this paper, we focus on covariate adjustment as our focus application area. However, given that separators play a central role in many aspects of graphical causal models, we expect that there should be many more applications in addition to those shown in this paper.

\input{experiments-modelchecking}

\section{Adjustment in DAGs}

\label{sec:dagadjust}

We now leverage the algorithmic
framework of Section~\ref{sec:algo} together with a 
constructive, sound and complete criterion 
for covariate adjustment in DAGs to solve all problems
listed in Table~\ref{fig:problems} for adjustment %
sets rather than $m$-separators in the same
asymptotic time.  %

First, note that the formal definition of 
adjustments (Definition~\ref{def:adjustment:set}) 
cannot be used to actually find an adjustment set as there 
are infinitely many probability distributions that are consistent to a certain graph. 
Fortunately, it is possible to characterize adjustment sets in graphical terms.

\begin{definition}[Adjustment criterion (AC) \citep{ShpitserVR2010}]\label{def:ac:general}
Let $\cG = (\bV,\bE)$ be a DAG, and $\bX,\bY,\bZ \subseteq \bV$ 
be pairwise disjoint subsets of variables. 
The set $\bZ$ satisfies the adjustment criterion relative to $(\bX, \bY)$ in $\cG$ if 
\begin{enumerate}
   \item[$(a)$] 
    no element in $\bZ$ is a descendant in $\cG_{\overline{\bX}}$ %$\gbarx$ 
    of any $W \in \bV \setminus \bX$ 
   which lies on a proper causal path from $\bX$ to $\bY$
   and 
   \item[$(b)$]
   all proper non-causal paths in $\cG$ from $\bX$ to $\bY$ are blocked by $\bZ$.
\end{enumerate}
\end{definition}
 
However, even this criterion does not lead to efficient algorithms to find adjustment sets, because there can be exponentially many paths and blocking one non-causal path might open other paths. We address this problem below by presenting a constructive adjustment criterion that reduces non-causal paths to ordinary $ m $-separation, which can be reduced further to a reachability problem in undirected graphs.

\subsection{Constructive back-door criterion}

\begin{figure}
\begin{center}
\begin{tabular}{ccc}
\begin{tikzpicture}[yscale=0.9,xscale=0.9]
\node (g) at (-.7,1.25) {$\cG$:};
\node (x1) at (0,0) {$X_1$};
\node (x2) at (1,0) {$X_2$};
\node (v1) at (2,0) {$D_1$};
\node (y1) at (3,0) {$Y$};
\node[adjusted] (z1) at (2,1) {$Z$};
\node (v2) at (1,-1) {$V$};
\node (v3) at (2,-1) {$D_2$};
\node (v4) at (3,-1) {$D_3$};

\draw [->] (x1) -- (x2);
\draw [->] (v1) -- (y1);
\draw [->] (z1) -- (y1);
\draw [->] (z1) -- (x2);

\draw [->] (x2) -- (v2);
\draw [->] (v1) -- (v3);
\draw [->] (y1) -- (v4);

\draw [->] (x2) -- (v1);
\draw [->] (x2) -- (v3);
\end{tikzpicture}\hspace*{10mm}
&
\begin{tikzpicture}[yscale=0.9,xscale=0.9]
\node (g) at (-0.7,1.25) {$\cbdg$:};
\node (x1) at (0,0) {$X_1$};
\node (x2) at (1,0) {$X_2$};
\node (v1) at (2,0) {$D_1$};
\node (y1) at (3,0) {$Y$};
\node[adjusted] (z1) at (2,1) {$Z$};
\node (v2) at (1,-1) {$V$};
\node (v3) at (2,-1) {$D_2$};
\node (v4) at (3,-1) {$D_3$};

\draw [->] (x1) -- (x2);
\draw [->] (v1) -- (y1);
\draw [->] (z1) -- (y1);
\draw [->] (z1) -- (x2);

\draw [->] (x2) -- (v2);
\draw [->] (v1) -- (v3);
\draw [->] (y1) -- (v4);
\draw [->] (x2) -- (v3);
\end{tikzpicture}\hspace*{10mm}
&
\begin{tikzpicture}[yscale=0.9,xscale=0.9]
\node (g) at (-0.7,1.25) {$\pbdgmod$:};
\node (x1) at (0,0) {$X_1$};
\node (x2) at (1,0) {$X_2$};
\node (v1) at (2,0) {$D_1$};
\node (y1) at (3,0) {$Y$};
\node[adjusted] (z1) at (2,1) {$Z$};
\node (v2) at (1,-1) {$V$};
\node (v3) at (2,-1) {$D_2$};
\node (v4) at (3,-1) {$D_3$};

\draw [->] (x1) -- (x2);
\draw [->] (v1) -- (y1);
\draw [->] (z1) -- (y1);
\draw [->] (z1) -- (x2);

\draw [->] (x2) -- (v2);
\draw [->] (v1) -- (v3);
\draw [->] (y1) -- (v4);

\draw[dotted] (v1.north west) rectangle (v4.south east);
\node at (3.5,-1) {$\bC$};

\end{tikzpicture}
\end{tabular}
\end{center}
\caption{A DAG that permits exactly two adjustment sets for estimating the causal effect of  
$\bX=\{X_1,X_2\}$ on $\bY=\{Y\}$:  $\bZ=\{Z\}$ and $\bZ'=\{Z,V\}$. Because $\bV \setminus \causpaths= \{X_1,X_2,Z,V\}$, every adjustment is a subset of $\{Z,V\}$. 
The nodes $D_1, D_2, D_3$ are not allowed in any adjustment as they are 
not in $ \{Z,V\}$  \textendash\ the set of descendants of a non-$\bX$ node on the (only) 
proper causal path $X_2 \to D_1 \to Y$. 
Moreover, every adjustment must contain the variable $Z$ to block the path between $X_2$ and $Y$ in $\pbdg$.
Graph  $\pbdgmod$  illustrates  a parametrized proper back-door graph (Definition~\ref{def:ac:general:bdc:prime})  
with parameter $\bC=\causpaths=\{D_1,D_2,D_3,Y\}$ surrounded by a dotted rectangle. Removing edge $X_2 \to D_2$ from $\pbdg$ simplifies the DAG
while preserving the reduction of finding adjustment sets to $d$-separation.
} 
\label{fig:CBCexample}
\end{figure}

The graph of this reduction will be called the proper back-door graph:

\begin{definition}[Proper back-door graph] \label{def:adj:graph:}
Let $\cG = (\bV,\bE)$ be a DAG, and $\bX,\bY\subseteq \bV$ 
be pairwise disjoint subsets of variables. 
The \emph{proper back-door graph}, denoted as  $\cbdg$,
is obtained from $\cG$ by removing 
the first edge of every proper causal path from $\bX$ to $\bY$. 
\end{definition}

For an example proper back-door graph  $\pbdg$, see Fig.~\ref{fig:CBCexample}.
Note the difference between the 
proper back-door graph $\cbdg$ and the famous back-door 
graph $\cG_{\underline{\bX}}$ of \cite{Pearl2009}: in $\cG_{\underline{\bX}}$ \emph{all} 
edges leaving $\bX$ are removed while in $\cbdg$ only those 
that lie on a proper causal path (see Figure~\ref{fig:ShpitserVsPearl} 
for an example illustrating the difference).
However, to construct $\cbdg$ still only elementary 
operations are sufficient. 
\highlightrevision{r2c43}{
Indeed, we remove all edges  $X \to D$ in $\bE$ such that 
$X\in \bX$ and $D$ is in the subset of nodes on proper causal paths, which we 
abbreviate as $\pcauspaths$, and which is defined as follows:}
\begin{equation}\label{eq:pcp:contsr:def}
  \pcauspaths =   
        (\textit{De}_{\overline{\bX}}(\bX) \setminus \bX) \cap \textit{An}_{\underline{\bX }}(\bY).
\end{equation}
Hence, the proper back-door graph  can be constructed from $\cG$ in linear time
$\cO(m+n)$. Note that using the notation $\textit{PCP}$ the proper back-door graph can be specified as 
$$\cbdg = (\bV,\bE \setminus (\bX \to \pcauspaths )).$$

Now we propose the following adjustment criterion. 
For short, we will denote the set $\textit{De}(\pcauspaths)$ as $\causpaths$.

\begin{definition}[Constructive back-door criterion (CBC)]\label{def:ac:general:bdc}
Let $\cG = (\bV,\bE)$ be a DAG, and 
let $\bX,\bY,\bZ \subseteq \bV$ be pairwise disjoint subsets of variables. 
The set $\bZ$ satisfies the 
\emph{constructive back-door criterion}
relative to $(\bX, \bY)$ in $\cG$  if
\begin{enumerate}
  \item[$(a)$] $\bZ \subseteq \bV \setminus \causpaths$ and 
  \item[$(b)$] 
    $\bZ$ $d$-separates $\bX$ and $\bY$ in the proper back-door graph 
      $\cbdg$.
\end{enumerate}
\end{definition}
Figure~\ref{fig:CBCexample} shows how the constructive back-door criterion  can be applied 
to find an adjustment set in an example DAG.

The CBC is surprisingly robust to modifications of its definition which 
allows a flexible handling of the forbidden nodes for adjustments and 
edges one can remove from the DAG to build a proper back-door graph.
An example of a modified proper back-door graph is given in Figure~\ref{fig:CBCexample} (rightmost graph):
we can remove some further edges from the proper back-door graph simplifying its structure 
while preserving the property of interest, i.e., finding adjustment sets via $d$-separation. 
The following extended definition incorporates various possible variants determined by specific parameters, which will turn out to be useful when proving the correctness of the CBC and other related criteria.

\begin{definition}[Parametrization of the constructive back-door criterion (CBC($\bA,\bB,\bC $))]\label{def:ac:general:bdc:prime}
Let $\cG = (\bV,\bE)$ be a DAG, and 
let $\bX,\bY,\bZ \subseteq \bV$ be pairwise disjoint subsets of variables. 
Let $\bA \subseteq \bX \cup \bY $, $\bB \subseteq \bX $, $\bC \subseteq \causpaths$.
The set $\bZ$ satisfies CBC($\bA,\bB,\bC $) %
relative to $(\bX, \bY)$ in $\cG$  if 
\begin{enumerate} 
  \item[$(a')$] $\bZ \subseteq \bV \setminus De_{\overline{\bA}\underline{\bB}}(\pcauspaths)$, and
\item[$(b')$] 
    $\bZ$ $d$-separates $\bX$ and $\bY$ in the graph $\pbdgmod := (\bV,\bE \setminus (\bX \to (\pcauspaths \cup \bC) ))$.
\end{enumerate}
\end{definition}

Clearly CBC($ \emptyset $, $ \emptyset $, $\emptyset$) = CBC.
The variants CBC($ \bX $, $ \emptyset $, $\emptyset$) and 
CBC($ \emptyset $, $ \emptyset $, $\causpaths$) might be the most interesting.
Condition $ (a') $ of CBC($ \bX $, $ \emptyset $, $\emptyset$) forbids less nodes than 
$(a)$ of the CBC, namely it excludes the descendants of $\pcauspaths$ in $ \cG_{\overline{\bX}} $. 
So, this condition $ (a') $ is identical to condition $ (a) $ of the AC (Definition~\ref{def:ac:general})
which helps to prove the correctness of the CBC.
Condition $ (a') $ of CBC($ \emptyset $, $ \emptyset $, $\causpaths$), see Figure~\ref{fig:CBCexample} for an example, 
forbids exactly the same nodes as $(a)$ of the CBC but according to the condition~$ (b') $ more edges 
can be removed than due to $(b)$ of the CBC, 
which might improve the performance in an implementation. % as there is no need to calculate separate sets for condition $ (a') $ and $ (b') $.

Note that the definition excludes CBC($ \emptyset $, $ \bY $, $\emptyset $), which could be considered as modifying condition $(a)$ to forbid the descendants of $ \pcauspaths $ in the graph $ \cG_{\underline{\bY}}$. This would not lead to a valid criterion as it would allow an adjustment set $ \{Z\} $ in the graph $ X \to Y \to Z $, where $ \{Z\} $ is not an adjustment. However, removing edges into $ \bY $ as in the graph $ \cG_{\overline{\bY}} $ of CBC($ \bY $, $ \emptyset $, $ \emptyset $) does not change the descendants at all, since the relevant $ Y $ are in $ \pcauspaths $ themselves. We can show that none of these modifications change the criterion:

\begin{lemma}\label{lem:cbc:variants}
Let $\cG = (\bV,\bE)$ be a DAG, and 
let $\bX,\bY\subseteq \bV$ be pairwise disjoint subsets of variables. 
Let $\bA \subseteq \bX \cup \bY $, $\bB \subseteq \bX $, $\bC \subseteq \textit{De}(\pcauspaths)$. Then CBC($\bA$, $\bB$, $\bC $) is equivalent to the CBC.
\end{lemma}
\begin{proof}
Let $ \bZ $ be a set that satisfies the CBC. 
Since $ \bZ \subseteq \bV \setminus \causpaths = \bV \setminus \textit{De}(\pcauspaths) $ and 
$ De_{\overline{\bA}\underline{\bB}}(\pcauspaths) \subseteq \textit{De}(\pcauspaths)$, 
the condition $(a')$   $ \bZ \subseteq \bV \setminus De_{\overline{\bA}\underline{\bB}}(\pcauspaths) $ 
is satisfied for CBC($\bA$, $\bB$, $\bC $). 
$ \bZ $ $ d $-separates $ \bX $ and $ \bY $ in $ \pbdg $, and thus also in $ \pbdgmod $, 
because every edge or path of $ \pbdgmod $ also exists in $ \pbdg $. Thus $(b')$ is true as well. 
Hence, $ \bZ $ satisfies CBC($\bA$, $\bB$, $\bC $). 

To see the other direction, let $ \bZ $ be a set that satisfies CBC($\bA$, $\bB$, $\bC $), but not CBC. 
If $ \bZ $ does not satisfies CBC~$(a)$, there exists a node $ Z\in\bV \setminus De_{\overline{\bA}\underline{\bB}}(\pcauspaths)$ that is not in $ \bV \setminus \causpaths = \bV \setminus \textit{De}(\pcauspaths)  $.
Then there must exist a proper causal path from $\bX$ to $\bY$ on which a node $W \in \bV \setminus \bX$ is an ancestor of $ Z $ in $ \cG $, but not in $ \cG_{\overline{\bA}\underline{\bB}} $, i.e., there is a causal path from a node of $ \bX  $ over $W$ to $Z$ which intersects $\bA \cup \bB	$. We can assume the nodes were chosen such that the length of the subpath between $ W $ and $ Z $ is minimal.
 
Let $W \to \ldots \to Y$ denote the suffix of the path from $\bX$ to $\bY$ 
starting in $W$. Note that this path might consist of only the vertex $W$.
Additionally, for the causal path from $W$ to $Z$, let $W \to \ldots \to A$ 
be its shortest prefix with $ A \neq W $ which ends in $\bA\cup\bB\cup\bX \subseteq \bX \cup \bY$. 
Notice that $W$ itself cannot be in $ \bB$ and, if it is in $\bA$, it does not change the paths.
Then, from the condition $(a')$, we know that no vertex of $W \to \ldots \to A$ belongs to $\bZ$. 
If $A \in \bX$, this leads to  a contradiction with the condition $(b')$ since 
$A \gets \ldots \gets W \to \ldots \to Y$ is a path $ \pbdgmod $ from $\bX$ to $\bY$ that is not blocked by $\bZ$. %
Otherwise we have $A \in \bY$, so $A \in \pcauspaths$ and the path from $A$ to $ Z $ is shorter than the path from $ W $ to $ Z $, which contradicts the choice of $ W $.  

If $ \bZ $ does not satisfies CBC~$(b)$, but satisfies CBC~$(a)$, there exists a path $ \pi $ between $\bX $ and $ \bY $ not blocked in $ \pbdg $ by $ \bZ $ that is blocked in $ \pbdgmod $ due to a removed edge $ X \to C $ with $ X \in \bX, C\in \bC $. 
If $ X \to C $ is on $ \pi $, we can assume it is the last such edge on $ \pi $. If the subpath from $ C $ to $ Y $ is causal, this edge is also removed in $ \pbdg $, a contradiction. So this subpath becomes non-causal at a  collider $ \to C' \gets $ unblocked in $ \pbdg $, which has a descendant in $\bZ$ that is also a descendant of $ C $ contradicting CBC~$(a)$. 
If the removal of the edge  $X \to C $ prevents the opening of a collider, $ C $ is also the ancestor of a node in $ \bZ $, which contradicts CBC~$(a)$ either.
\end{proof}

We will see the usefulness of the parametrization of the constructive back-door criterion
in the proof of the main result of this subsection:

\begin{theorem}\label{th:ac:equivalence}
The constructive back-door criterion (CBC) is equivalent to the adjustment criterion (AC).
\end{theorem}

\begin{proof}
First observe that the condition $(a)$ of the adjustment criterion AC is identical to condition $(a')$ 
of the constructive back-door criterion CBC($ \bX $, $ \emptyset $, $\emptyset$).
Assume conditions $(a)$ and $(b)$ of the adjustment criterion AC hold. 
Due to Lemma~\ref{lem:cbc:variants}, it is sufficient to 
show that condition $(b)$ of the constructive back-door criterion is satisfied.
Let $\pi$ be any proper path from $\bX$ to $\bY$ in $\cbdg$. 
Because $\cbdg$ does not contain causal 
paths from $\bX$ to $\bY$, $\pi$ is not causal
and has to be blocked by $\bZ$
in $\cG$ by the assumption. Since removing edges cannot
open paths, $\pi$ is blocked by $\bZ$ in $\cbdg$ 
as well.

Now we show that $(a)$ and $(b)$ of the
constructive back-door criterion CBC together 
imply $(b)$ of the adjustment criterion AC.
If that were not the case, then there could exist  a proper 
non-causal path $\pi$ from $\bX$ to $\bY$ that is blocked in $\cbdg$ but open
in $\cG$. There can be two reasons why $\pi$ is blocked 
in $\cbdg$: (1) The path starts with an edge $X \to D$ that does not exist
in $\cbdg$. Then we have $D \in \pcauspaths$. For $\pi$
to be non-causal, it would have to contain a collider 
$C \in \textit{An}(\bZ) \cap \textit{De}(D) \subseteq
\textit{An}(\bZ) \cap \causpaths$. But because of 
CBC~$(a)$, $\textit{An}(\bZ) \cap \causpaths$ is empty.
(2) A collider $C$ on $\pi$ is an ancestor of $\bZ$ 
in $\cG$, but not in $\cbdg$. Then there must be a directed path
from $C$ to $\bZ$ via an edge $X \to D$ with 
$D \in \textit{An}(\bZ) \cap \pcauspaths $, contradicting CBC~$(a)$.
\end{proof}

\subsection{CBC vs Pearl's back-door criterion}\label{sec:singletons}

\begin{figure}
\begin{center}
\begin{tabular}{ccc}
\begin{tikzpicture}[yscale=0.85,xscale=0.7]
\node (g) at (-1.5,-1.5) {$\cG$:};
\node (x1) at (0,0) {$X_1$};
\node[adjusted] (z1) at (1.2,-1) {$Z_1$};
\node[adjusted] (z2) at (.8,-2) {$Z_2$}; 
\node (x2) at (0,-3) {$X_2$}; 
\node (y1) at (2,0) {$Y_1$};
\node (y2) at (2,-3) {$Y_2$};

\draw [->] (x1) -- (z1);
\draw [->] (x1) -- (y1);
\draw [->] (z1) -- (z2);
\draw [->] (z2) -- (x2);
\draw [->] (y2) -- (z2);
\end{tikzpicture}\hspace*{15mm}
&
\begin{tikzpicture}[yscale=0.85,xscale=0.7]
\node (g) at (-1.5,-1.5) {$\gbarx$:};
\node (x1) at (0,0) {$X_1$};
\node[adjusted] (z1) at (1.2,-1) {$Z_1$};
\node[adjusted] (z2) at (.8,-2) {$Z_2$}; 
\node (x2) at (0,-3) {$X_2$}; 
\node (y1) at (2,0) {$Y_1$};
\node (y2) at (2,-3) {$Y_2$};

\draw [->] (z1) -- (z2);
\draw [->] (z2) -- (x2);
\draw [->] (y2) -- (z2);
\end{tikzpicture}\hspace*{15mm}
&
\begin{tikzpicture}[yscale=0.85,xscale=0.7]
\node (g) at (-1.5,-1.5) {$\cbdg$:};
\node (x1) at (0,0) {$X_1$};
\node[adjusted] (z1) at (1.2,-1) {$Z_1$};
\node[adjusted] (z2) at (.8,-2) {$Z_2$}; 
\node (x2) at (0,-3) {$X_2$}; 
\node (y1) at (2,0) {$Y_1$};
\node (y2) at (2,-3) {$Y_2$};

\draw [->] (x1) -- (z1);
\draw [->] (z1) -- (z2);
\draw [->] (z2) -- (x2);
\draw [->] (y2) -- (z2);
\end{tikzpicture}
\end{tabular}
\end{center}
\caption{A DAG where for $\bX=\{X_1,X_2\}$ and $\bY=\{Y_1,Y_2\}$,
 $\bZ=\{Z_1,Z_2\}$ is a valid and minimum adjustment,
but no set fulfills the back-door criterion \citep{Pearl2009} (Definition~\ref{def:back-door:pearl}),
and the parents of 
$\bX$ are not a valid adjustment set either.
} 
\label{fig:ShpitserVsPearl}
\end{figure}

In this section we relate our constructive back-door criterion to the well-known 
back-door criterion by Pearl \cite{Pearl2009}:
\begin{definition}[Pearl's back-door criterion (BC) \cite{Pearl2009}]\label{def:back-door:pearl}
A set of variables $\bZ$ satisfies the \emph{back-door criterion} relative to an ordered pair of variables $(X, Y)$ in a DAG $\cG$  if:
\begin{enumerate}
  \item[$(a)$] $\bZ \subseteq \bV \setminus \textit{De}(X)$ and 
  \item[$(b)$] $\bZ$ blocks every path between $X$ and $Y$ that contains an arrow into $X$.
\end{enumerate}
Similarly, if $\bX$ and $\bY$ are two disjoint subsets of nodes in $\cG$, then $\bZ$ is said to satisfy the back-door criterion relative to $(\bX,\bY)$ if it satisfies the back-door criterion relative to any pair $(X,Y)$ such that $X\in\bX$ and $Y\in\bY$.
\end{definition}

In Definition~\ref{def:back-door:pearl} condition $(b)$ is often replaced by the equivalent condition that $\bZ$ $d$-separates $\bX$ and $\bY$ in the back-door graph $\cG_{\underline{X}}$.

In \cite{TextorLiskiewicz2011} it was shown that for minimal adjustment sets in $\bX$-loop-free DAGs the adjustment criterion and the back-door criterion of Pearl are equivalent. A DAG is $\bX$-loop-free for an exposure set $\bX$, if no directed path between two different nodes of $\bX$ contains a node not in $\bX$. If $X$ is a singleton, there are no two different nodes of $\bX$ and every DAG is $\bX$-loop-free, so the criteria are always equivalent for minimal adjustments. In this case it is still possible that an adjustment set $\bZ$ satisfies the CBC and not the back-door criterion, but there will always be a minimal subset $\bZ' \subseteq \bZ$ that satisfies the back-door criterion. Since an adjustment set satisfying the back-door criterion also satisfies the generalized back-door criterion of \cite{Maathuis2013}, and all sets of the  generalized back-door criterion satisfy our CBC, all three criteria are equivalent to test the \emph{existence} of an (minimal) adjustment set for a singleton $X$ in DAGs.

The situation changes if the effect of multiple exposures  
is estimated. Theorem~3.2.5 in \citet{Pearl2009} claims that 
the expression for $P(\by \mid \textit{do}(\bx))$  is obtained 
by adjusting for $\textit{Pa}(\bX)$ if 
$\bY$ is disjoint from $ %
\textit{Pa}(\bX)$ in graphs without latent nodes, but,
as the DAG in Figure~\ref{fig:ShpitserVsPearl} shows, this is not true: 
the set $\bZ=\textit{Pa}(X_1,X_2) = \{Z_2\}$ is not an adjustment set 
according to $\{X_1,X_2\}$ and $\{Y_1,Y_2\}$. In this case one can identify
the causal effect
by adjusting for 
$\bZ=\{Z_1,Z_2\}$ only. Indeed, for more than
one exposure, no adjustment set may exist at all even
without latent covariates and even though
$\bY \cap (\bX\cup \textit{Pa}(\bX)) = \emptyset$, 
e.g., in the DAG
\begin{tikzpicture}[text height=.4em,xscale=1,,baseline=-.18em]
\node (x1) at (.2,0) {$X_1$};
\node (x2) at (1,0) {$X_2$};
\node (z) at (2,0) {$Z$};
\node (y) at (3,0) {$Y$};
\draw [->] (z) -- (x2);
\draw [->] (y) -- (z);
\draw [->] (x1) edge [bend left=25] (z);
\end{tikzpicture}
and for $\bX=\{X_1,X_2\}$ and $\bY=\{Y\}$.

In the case of multiple exposures $\bX$ it is also harder to use the back-door criterion to actually find an adjustment set. 
Although the back-door criterion reduces adjustment to $d$-separation in the back-door graph $\cG_{\underline{X}}$, 
this is not the graph  $\cG_{\underline{\bX}}$, so for each exposure $X \in \bX$ the criterion would find a separate 
adjustment set, which do not lead directly to a combined adjustment set for all exposures.
For an example see Figure~\ref{fig:ShpitserVsPearl}. 

Table~\ref{tab:bdc:vs:cbc} summarizes the relationships between CBC and the Pearl's back-door criterion.

\begin{table}
\begin{center}
\begin{tabular}{lcll}
\multicolumn{3}{l}{Statement for arbitrary DAGs and all sets $\bZ$:}& proof  \\[1mm]  %and trees skeletons:
%	                                      \multicolumn{3}{c}{statement $(\forall  \bZ)$}
	$\bZ$ satisfies CBC                              & $\not\Rightarrow$ & $\bZ$ satisfies back-door          & $ Z \gets X \to Y $                  \\
	$\bZ$ satisfies CBC                              & $\not\Rightarrow$ & $\exists\bZ'$ satisfying back-door & see Figure~\ref{fig:ShpitserVsPearl} \\
	$\bZ$ satisfies CBC   and $ \bZ $ is minimal     & $\not\Rightarrow$ & $\bZ$ satisfies back-door          & see Figure~\ref{fig:ShpitserVsPearl} \\
	$\bZ$ satisfies CBC   and $ \bZ $ is minimal     & $\not\Rightarrow$ & $\exists\bZ'$ satisfying back-door & see Figure~\ref{fig:ShpitserVsPearl} \\
	\\[-3mm]
\multicolumn{4}{l}{Statement for all $\bX$-loop-free DAGs (e.g., for singleton $X$)  and all sets $\bZ$:}\\[1mm]
	%\multicolumn{3}{c}{statement}              &       proof       &  \\
	$\bZ$ satisfies CBC                              &$\not\Rightarrow$ & $\bZ$ satisfies back-door          & $ Z \gets X \to Y $              \\
	$\bZ$ satisfies CBC                              &  $\Rightarrow$   & $\exists\bZ'$ satisfies back-door  & via minimal $ \bZ'\subseteq\bZ $ \\
	$\bZ$ satisfies CBC   and $ \bZ $ is minimal     &  $\Rightarrow$   & $\bZ$ satisfies back-door          & see \cite{TextorLiskiewicz2011}  \\
	$\bZ$ satisfies CBC   and $ \bZ $ is minimal     &  $\Rightarrow$   & $\exists\bZ'$ satisfying back-door & $\bZ'=\bZ$     
\end{tabular}
\end{center}\vspace*{-3mm}
\caption{A summary of the relationship between the existence of a Pearl back-door-adjustment set
and the existence of an CBC-adjustment set  in unconstrained DAGs and $\bX$-loop-free DAGs. 
Symbol $\not\Rightarrow$ means that the implication does not hold, in general.
On the other hand, due to the completeness property of the CBC, we have that 
if one replaces in the left hand sides  "CBC" by "back-door" and in the right
hand sides  "back-door" by "CBC", then  the corresponding implications are always true.}
\label{tab:bdc:vs:cbc}
\end{table}

\section{Algorithms for testing and computing adjustment sets in DAGs}
\label{sec:algorithms:adjustment}

Having proved the constructive back-door criterion, we are now licensed to apply our separation algorithms from Section~\ref{sec:algo} to solve adjustment set problems. This works because the adjustment relative to $ \bX $ and $ \bY $ in $ \cG $ corresponds to an $ m $-separator between $ \bX $ and $ \bY $ in $ \pbdg $ subject to the constraint given by CBC~$(a)$. Table~\ref{fig:problemsadj} gives an overview of the relevant tasks.

A small difference exists between testing and constructing adjustment sets when handling CBC~$(a)$: Testing requires us to check if the given set $ \bZ $ contains nodes of $ \causpaths $, whereas constructing requires that the returned set $\bZ$ must not contain any of the nodes in $\causpaths$. The latter requirement can be straightforwardly implemented by imposing the constraint $ \bZ \subseteq \bR' = \bR\setminus\causpaths $, which can be given as parameter to our separation algorithms.

\begin{table*}
\centering
\begin{tabular}{llll}
&&&Runtime \\
\multicolumn{3}{l}{\textbf{Verification:} \text{For given $\bX, \bY,\bZ$ and constraint $\bI$
 decide if $\ldots$ }}
&\\
\hspace*{2mm} 
   & {\sc TestAdj} & $\bZ$ is an adjustment for $(\bX,\bY)$ & $\cO(n+m)$ 
     \\
   & {\sc TestMinAdj} & $\bZ\supseteq \bI$ is an adjustment for $(\bX,\bY)$ and $\bZ$ is $\ldots$ &\\
   & & \hspace*{3mm} $ \bI $-minimal & $\cO(n^2)$ 
      \\
   & &  \hspace*{3mm} strongly-minimal 
      & $\cO(n^2)$ 
      \\[2mm]
\multicolumn{3}{l}{\textbf{Construction:} \text{For given $\bX, \bY$ and constraints $\bI, \bR$, output an $\ldots$}} &  \\
\hspace*{2mm} 
  & {\sc FindAdj} &   adjustment $\bZ$  for $(\bX,\bY)$ with $\bI\subseteq\bZ\subseteq\bR$  & $\cO(n+m)$ 
     \\
  & {\sc FindMinAdj} &  adjustment $\bZ$ for $(\bX,\bY)$ with $\bI\subseteq\bZ\subseteq\bR$ which is $\ldots$ &\\
     &  & \hspace*{3mm}  $ \bI $-minimal & $\cO(n^2)$ 
     \\
   &  & \hspace*{3mm}   strongly-minimal 
     & NP-hard 
     \\
  & {\sc FindMinCostAdj} & adjustment $\bZ$ for $(\bX,\bY)$ with $\bI\subseteq\bZ\subseteq\bR$ which is $\ldots$ &\\
  &  & \hspace*{3mm}    $ \bI $-minimum & $\cO(n^3)$ 
     \\
  &  & \hspace*{3mm} strongly-minimum & $\cO(n^3)$ 
    \\[2mm]
\multicolumn{3}{l}{\textbf{Enumeration:} \text{For given $\bX, \bY, \bI, \bR$ enumerate all $\ldots$ } }&Delay\\
 & {\sc ListAdj} 
     & adjustments $\bZ$   for $(\bX,\bY)$ with $\bI\subseteq \bZ \subseteq  \bR$   &$\cO(n(n+m))$ 
     \\
 & {\sc ListMinAdj} & $ \bI $-minimal adjustments $\bZ$  with $\bI \subseteq \bZ \subseteq  \bR$        
     &$\cO(n^3)$ 
    \\
\end{tabular}

\caption{Definitions of algorithmic tasks related to adjustment in DAGs. 
The meaning of parameters $\bX, \bY, \bZ, \bI,$ and $\bR$ is the same as in 
the definitions of tasks related to $m$-separation in  Table~\ref{fig:problems}.
The right column shows the associated time complexities given in this section.
Due to our linear-time reduction from causal effect identification by adjustment to $m$-separation 
in a subgraph of an input DAG, the time complexities to solve the problems above are the
same as in Table~\ref{fig:problems}.
% Throughout, $\bX,\bY,\bR$ are pairwise disjoint node sets,
% the set $\bZ$ is disjoint with the non-empty sets $\bX,\bY$, 
% and each of the sets $\bI,\bR,\bZ$ can be empty. 
% A minimum adjustment minimizes the sum $\sum_{Z\in\bZ} w(Z)$ for a cost function $w$ respecting the given constraints, %
% i.e., $w(V) = \infty$ for $V \notin \bR$.
% The construction
% algorithms will output $\bot$ if no set fulfilling the listed condition exists.
% Delay complexity  for {\sc ListAdj} or {\sc ListMinAdj} refers to the time needed to 
% output one solution when there can be exponentially many solutions
% (see~\cite{Takata2010}).
}
\label{fig:problemsadj}
 
\end{table*}

\newcommand{\forbiddenNodes}{De_{\overline{\bX}}(\{W \in \bV \mid  \exists \text{proper path } \bX \stackrel{+}{\to} W \stackrel{*}{\to} \bY \})}
\newcommand{\forbiddenEdges}{\{X \to D \mid  X \in \bX, D \in \cD \}}

Hence {\sc TestAdj} can be solved by testing if $\bZ\cap \causpaths = \emptyset$
and if $\bZ$ is a $d$-separator in the proper back-door graph $\pbdg$ using algorithm {\sc TestSep}.
Since $\pbdg$ can be constructed from $\cG$ in linear time, the total 
time complexity of this algorithm is $\cO(n+m)$.

{\sc TestMinAdj} can be solved by testing again if $\bZ\cap \causpaths = \emptyset$ and calling {\sc TestMinSep} to verify that $ \bZ $ is minimal within the back-door graph $ \pbdg $. This leads to a runtime of $ \cO(n^2) $ which is optimal for dense graphs. Alternatively {\sc TestMinSepSparse} with its runtime of $\cO(|\textit{Ant}(\bX \cup \bY)| \cdot  |\text{Edges of } (\pbdg)^a | ) = \cO(n(n+m)) $ can be used in sparse graphs. It is worth noting that since the back-door graph is formed by removing edges, it is even sparser than the input graph. 
This approach only works because the minimal adjustment corresponds to a \emph{minimal} separator in the proper back-door graph: every subset of an adjustment must still satisfy condition CBC~$(a)$. It also implies the following corollary which generalizes the result of \citet{TianPP1998} from $d$-separators to adjustment sets:

\begin{corollary}
\highlightrevision{r2c49}{
An adjustment set $\bZ$ is minimal if and only if no single node $Z$ can be removed from $\bZ$ such that the resulting set $\bZ'=\bZ\setminus Z$ is no longer 
	an adjustment set.}
\end{corollary}

The problem {\sc FindAdj} can be solved by a closed form solution. 
For a DAG $\cG=(\bV,\bE)$ and constraints $ \bI,\bR $ we define the set
\begin{newpartinrevision}[r2c50]
$$
   \adjustment(\bX,\bY) 
= \textit{An}(\bX \cup \bY \cup \bI) \cap \bR \setminus (\bX \cup \bY \cup \causpaths). 
$$
\end{newpartinrevision}
\begin{theorem}\label{th:adj:set:cosctr}
Let $\cG=(\bV,\bE)$ be a DAG, let $\bX,\bY \subseteq V$ be disjoint node sets and $ \bI, \bR $ constraining node sets with $\bI \subseteq \bR \setminus (\bX \cup \bY \cup \causpaths)$. 
Then the following statements are equivalent:
\begin{enumerate} 
\item There exists an adjustment $ \bZ $ in $\cG$ w.r.t. $\bX$ and $\bY$ with $ \bI\subseteq\bZ\subseteq\bR $.
\item $\adjustment(\bX,\bY)$ is an adjustment w.r.t. $\bX$ and $\bY$.
\item $\adjustment(\bX,\bY)$ $d$-separates $\bX$ and $\bY$ in the proper back-door graph $\cbdg$.
\end{enumerate}
\end{theorem}

\begin{proof}
The implication $\it (3)\Rightarrow(2)$ follows directly from 
the criterion Def.~\ref{def:ac:general:bdc} and the definition
of $\adjustment(\bX,\bY)$. Since the implication $\it (2)\Rightarrow(1)$
is obvious, it remains to prove  $\it (1)\Rightarrow(3)$. 
 
Assume there exists an adjustment set $\bZ_0$ w.r.t. $\bX$ and $\bY$.
From Theorem~\ref{th:ac:equivalence} we know that  
$\bZ_0 \cap \causpaths= \emptyset$ and  that 
$\bZ_0$ $d$-separates $\bX$ and $\bY$ in $\cbdg$.
Our task is to show that $\adjustment(\bX,\bY)$ $d$-separates 
$\bX$ and  $\bY$ in $\cbdg$. This follows from  Lemma~\ref{lemma:auxiliary:m:sep}
used for the proper back-door graph $\pbdg$
if we take  $\bI'=\bI$, $\bR'=\bR \setminus (\bX \cup \bY \cup \causpaths)$.
\end{proof}
From Equation~\eqref{eq:pcp:contsr:def} and the definition $\causpaths  = \textit{De}(\pcauspaths)$  
we then obtain immediately:
\begin{corollary}\label{enum:fastas}
Given two disjoint sets $\bX, \bY \subseteq \bV$,  $\adjustment(\bX,\bY)$ 
can be found in $\cO(n+m)$ time.
\end{corollary}

The remaining problems, {\sc FindMinAdj}, {\sc FindMinCostAdj}, {\sc ListAdj} and {\sc ListMinAdj} 
can be solved using the corresponding algorithms
for finding, resp. listing  $m$-separations applied to the proper 
back-door graph.  Since the proper back-door graph can be
constructed in linear time the time complexities to solve the 
problems above are the same in Table~\ref{fig:problems} and Table~\ref{fig:problemsadj}. The NP-hardness of finding 
strongly-minimal adjustment sets follows from Proposition~\ref{prop:np-complete} and the fact that the graph 
constructed in the proof of the proposition contains no causal paths between $ X $ and $ Y $, 
so there are no forbidden nodes and that graph is the same as its back-door graph.

\section{Extending the CBC}

\label{sec:cbcext}

While our complete adjustment criterion is guaranteed to find all instances in which a causal effect can be identified via covariate adjustment, it is well known that not all identifiable effects are also identifiable via adjustment. The do-calculus \cite{Pearl2009} is a complete method that characterizes all identifiable causal effects, but which comes at a price of substantially increased formula and runtime complexity. In this section, we however show that many cases in which covariate adjustment is not applicable do not require the power of the do-calculus either. 
 
Specifically, we provide three lemmas that permit identification of total causal effects in the following three cases (which are not mutually exclusive) as shown in Figure~\ref{fig:cases:beyond:cbc}: (1) $\bX$ does not have a causal effect on $\bY$; (2) \highlightrevision{r3c8}{$\bX=X$ is singleton and all its parents are observed;} (3) $\bX$ and $\bY$ partition $\bV$. While in each case the adjustment criterion may or may not be applicable, our lemmas show that identifiability is always guaranteed, and the total effect can be computed by reasonably simple formulas. Moreover, each lemma provides an easy algorithm for testing whether the corresponding case applies.

\tikzset{exp/.style={color=green!50!black}}
\tikzset{res/.style={color=blue}}
\tikzset{lat/.style={color=gray}}
\begin{figure}\arxivonly{\small}\begin{center}
	(1) \begin{tikzpicture}[baseline]
\graph  { 
 v1[res,as=$Y_1$] -!- v2[as=$V_2$],
 v3[lat,as=$V_1$] -!- v4[exp,as=$X_1$],
 v1 -> v4, v1 -> v2 -> v4, v3 -> v1, v3 -> v4
 };
\end{tikzpicture} \phantom{long cat}
% (2) \begin{tikzpicture}[baseline]
% \graph  { 
%  v1[as=$V_1$] -!- v2[exp,as=$X_1$],
%  v3[lat,as=$V_2$] -!- v4[res,as=$Y_1$],
%  v1 -> v4, 
%  v1 -> v2 -> v4, v3 -> v1, v3 -> v4
%  };
(2) \begin{tikzpicture}[baseline]
\graph  { 
 v1[res,as=$Y_1$] -!- v2[exp,as=$X_1$],
 v3[lat, as=$V_1$] -!- v4[res,as=$Y_2$],
 v1 -> v4, 
 v1 -> v2 -> v4, v3 -> v1, v3 -> v4
 };
\end{tikzpicture} \phantom{long cat}
(3) \begin{tikzpicture}[baseline]
\graph  { 
 v1[res,as=$Y_1$] -!- v2[exp,as=$X_1$],
 v3[exp,as=$X_2$] -!- v4[res,as=$Y_2$],
 v1 -> v4, 
 v1 -> v2 -> v4, v3 -> v1, v3 -> v4
 };
\end{tikzpicture} 
\end{center}
\caption{The three cases 
%Three graphs that illustrate the cases 
analyzed in this section. Exposure and outcome nodes are marked as $X$ and $Y$; latent nodes are shown in gray. In case (1) the causal effect is given by $P(y_1 \mid \textit{do}(x_1)) = P(y_1)$, in case 
%(2) by $P(y_1 \mid \textit{do}(x_1))  = \sum_{v_1} P(v_1) P(y_{1} \mid x_1, v_1)$, 
(2) by $P(y_1, y_2 \mid \textit{do}(x_1))  = P(y_1) P(y_{2} \mid x_1, y_1)$, 
and in case (3) by
$P(y_1, y_2 \mid \textit{do}(x_1, x_2)) 
%= P(y_1 \mid \textit{Pa}(Y_1) ) P(y_2 \mid \textit{Pa}(Y_2) ) 
= P(y_1 \mid x_2 ) P(y_2 \mid y_1, x_1, x_2 ) 
$.
}\label{fig:cases:beyond:cbc}
\end{figure}

\subsection{Identification by plain formulas}\label{sec:observation}

One case in which identification is trivial is if there is no causal effect of $X$ on $Y$ at all.
%When  there is a common ancestor of $X$ and $Y$, then this case is covered by the CBC; however, 
%if there is a reverse causal path from $Y$ to $X$, then the (void) effect is not obtainable 
When all non-causal paths between $X$ and $Y$ can be blocked, then this case is covered by the CBC; however, 
if there is a non-causal path consisting solely of unobserved nodes, then the (void) effect is not obtainable 
through the adjustment formula. 
%In both of these cases, 
In such cases,
however, we simply have 
$P(\by\mid \textit{do}(\bx)) = P(\by)$, which we will call the \emph{plain formula}. 
The following proposition provides a characterization of all cases in which this plain formula
works in terms of $d$-separation.

\begin{proposition}\label{prep:observation:is:causal}
Let $\cG=(\bV,\bE)$ be a DAG and let $\bX,\bY\subseteq \bV$  
be disjoint subsets of variables and 
let $\bR \subseteq \bV$ be an arbitrary set of observed variables, 
with $ \bX\cup \bY \subseteq\bR$.
Then %
$\bX$ and $\bY$ are $d$-separated in 
$\cG_{\overline{\bX}}$, expressed symbolically as 
\begin{equation}\label{eq:plain:condition}
  (\bY \independent \bX)_{\cG_{\overline{\bX}}}
\end{equation}
if and only if 
the effect of intervention  of $\bX$ on $\bY$ %
	is given by  the plain formula $P(\by \mid \textit{do}(\bx)) = P(\by)$, 
\highlightrevision{r2c53}{i.e., there is no causal effect from $\bX$ on $\bY$.}
Particularly, if $Y \in \textit{An}(X)$ then $(Y \independent X)_{\cG_{\overline{X}}}$
and thus $P(y \mid \textit{do}(x)) = P(y)$.
\end{proposition}
\begin{proof}
The soundness of the  statement %
follows directly by the application of rule 3 
(intervention/deletion of actions; 
for the precise definition of the do-calculus rules see Theorem~3.4.1 in \citep{Pearl2009}). 
The completeness of the statement can be shown similarly to the completeness of the adjustment criterion~\cite{ShpitserVR2010}. If $\bY$ and $\bX$ are not $ d $-separated in $\cG_{\overline{\bX}}$, 
there exists a shortest causal path $ X \to \ldots \to Y $ for $ X \in \bX,Y \in \bY $. 
In the subgraph $ \cG' = (\bV', \bE')$ 
consisting only of this path, 
the causal effect is given by an empty adjustment set 
$P(y\mid \textit{do}(x)) = P(y\mid x)$. 
If we take a model $P' $ where $P'(y \mid x) \neq P'(y) $ 
for some values $ x, y $, 
like e.g. in  a model on binary variables $ X$, $Y $  with 
\[
  P'(x) = \frac{1}{2}\quad  \text{and} \quad 
  P'(y\mid x) = \begin{cases}
  \frac{1}{3}& x = y,\\
  \frac{2}{3}& x \neq y,
  \end{cases}
\]
the causal effect is not given by $P'(y)$.
This model can be extended 
to a model $ P $ on the full graph $ \cG $ by assuming all other variables 
are independent, i.e., $ P(\bv) = (\frac{1}{2})^{|\bV \setminus \bV'|} P'(\bv') $. 
This model is consistent with $ \cG $ (though not faithful, but faithfulness is not required) 
and we have
\[
  P(\by\mid \textit{do}(\bx)) \ = \  P(\by \setminus Y) P(y\mid \textit{do}(x)) \ = \  P(\by \setminus Y) 
  P'(y\mid \textit{do}(x)) \neq P(\by \setminus Y) P'(y) \ = \ P(\by \setminus Y) P(y) = P(\by).
\]
\end{proof}

\subsection{Identification by generalized parent adjustment}

Another case that permits the identification of causal effects using 
simple formulas occurs if the exposure $ \bX = {X} $ is a singleton 
and all its parents are observed, i.e., $ \textit{Pa}(X) \subseteq \bR $.  
Then adjusting for the parents of $X$ blocks all biasing 
paths and suffices for identification, but one needs to 
be careful as there might be variables $ \bY_{pa} = \textit{Pa}(X) \cap \bY $ 
that are both parents and outcome nodes. Proposition~\ref{prop:singleton} below
shows that in this case the causal effect is given by 
$P(\by \mid \textit{do}(x)) \ =\ \sum_{\bz} P(\bz, \by_{pa}) P(\by_{np} \mid x, \bz, \by_{pa})$, 
where $ \bY_{pa} \cup \bY_{np}$ is a partition of $ \bY $ and $ \bY_{pa} \cup \bZ $ 
is a partition of $ \textit{Pa}(X) $. %

This is a slight generalization of identification via adjustment: 
we still sum over the values of variables $ \bZ $ and multiply the conditional
probability with a factor, but rather than multiplying with the probability of
the same variables $ P(\bz) $ that are used in the sum, we multiply with a factor  
$ P(\bz, \by_{pa}) $ involving additionally the variables in $ \bY_{pa} $.

The situation is even simpler when $ Y $ is also a singleton. 
Then one of the sets $ \bY_{pa}, \bY_{np} $ vanishes, 
so there are only two cases: either $Y \notin \textit{Pa}(X)$ and $\textit{Pa}(X)$ 
is an adjustment set \citep[Theorem~3.2.2]{Pearl2009}, or $Y \in \textit{Pa}(X)$ 
and no adjustment exists, but the causal effect is identified as $P(y \mid \textit{do}(x)) = P(y)$.  
One can see that in the case $Y\in \textit{An}(X)\setminus\textit{Pa}(X)$
the effect of intervention $\textit{do}(X=x)$ can be given both 
by the plain expression $P(y \mid \textit{do}(x)) = P(y)$ and 
by adjustment in parents of $X$. %

\begin{proposition}\label{prop:singleton}
Let $\cG=(\bV,\bE)$ be a DAG and let $X \in \bV$ be a node with observed parents 
$ \textit{Pa}(X) \subseteq \bR $ and $\bY\subseteq \bV \setminus X$.
Furthermore, let $\bY_{pa} = \bY \cap \textit{Pa}(X)$ and let $\bY_{np} = \bY \setminus \textit{Pa}(X)$  
be a partition of $ \bY = \bY_{pa} \cup \bY_{np} $ and let 
$ \bZ = \textit{Pa}(X) \setminus \bY_{pa} $ form with $ \bY_{pa} $ a partition of $ \textit{Pa}(X) = \bY_{pa} \cup \bZ $.
Then 
\[
  P(\by \mid \textit{do}(x)) \ =\ 
  \begin{cases}
    P(\by_{pa}) P(\by_{np} \mid x, \by_{pa}) & \text{if $\bZ=\emptyset$, i.e., if  $\textit{Pa}(X)\subseteq \bY$,}  \\[1mm]
    \sum_{\bz} P(\bz, \by_{pa}) P(\by_{np} \mid x, \bz, \by_{pa}) & \text{if $\bZ\not=\emptyset$,}  
  \end{cases}
\]
where $P(\by_{pa})$ (resp. $ P(\by_{np} \mid x, \by_{pa})$  
and $ P(\by_{np} \mid x, \bz, \by_{pa})$) should be read as $1$ if
$\bY_{pa}=\emptyset$ (resp.  $\bY_{np}=\emptyset$).
\end{proposition}
\begin{proof}
This follows from a straightforward calculation using the do-calculus:

\begin{align*}
P(\by \mid \textit{do}(x)) \ =\ & P(\by_{pa}, \by_{np} \mid \textit{do}(x))\\
=\ & \sum_{\bz} P(\bz, \by_{pa}, \by_{np} \mid \textit{do}(x))\\
=\ & \sum_{\bz} P(\bz, \by_{pa}\mid \textit{do}(x)) P(\by_{np} \mid \textit{do}(x), \bz, \by_{pa}) \\
=\ & \sum_{\bz} P(\bz, \by_{pa}) P(\by_{np} \mid \textit{do}(x), \bz, \by_{pa}) &\rlap{\hspace*{-2.5cm}\parbox{5.5cm}{do-calculus rule 3  in \cite[Subsect. 3.4.2]{Pearl2009}\\$(\bY_{pa}, \bZ \independent X )\text{ in }\cG_{\overline{X}}$}} \\  
=\ & \sum_{\bz} P(\bz, \by_{pa}) P(\by_{np} \mid x, \bz, \by_{pa}).     &\rlap{\hspace*{-2.5cm}\parbox{5.5cm}{do-calculus rule 2 in \cite[Subsect. 3.4.2]{Pearl2009}\\$(\bY_{np} \independent X \mid \bZ, \bY_{pa})\text{ in } \cG_{\underline{X}}$}} \\  
\end{align*}
\end{proof}

If some of parents $\textit{Pa}(X)$ are unobserved, the causal effect might not be identifiable at all, 
like e.g., in the DAG $\cG_1$ in Figure~\ref{fig:notident:ident:viaadj}.
To decide if the effect is identifiable in such a case, one can use the CBC
criterion which, like for $\cG_3$ in Figure~\ref{fig:notident:ident:viaadj}, 
can confirm the identifiability. However, while the CBC is complete 
to decide if the effect is expressible via covariate adjustment 
it is not complete to decide if the effect is identifiable or not.
For an example, see the DAG $\cG_2$  in Figure~\ref{fig:notident:ident:viaadj}.
To solve the identification problem in this case, when the CBC does not work, 
one has to use a complete criterion, like this based on the do-calculus.

\subsection{Identification when $\bX$ and $\bY$ partition $\bV$}\label{sec:special:case:v:partition}

Here we consider the case of DAGs in which $\bX$ and $\bY$ partition the set of variables $\bV$, which
implies that there are no unobserved nodes.
Again, in this case the CBC may not be applicable as there may be an arrow from $\bY$ to $\bX$, but still 
the causal effect can be given by a closed-form solution as we present below.

\begin{lemma}
Let $ \cG = (\bV, \bE) $ be a DAG and $ \bX, \bY \subset \bV  $ be 
a partition of $ \bV = \bX \cup \bY $. The following statements hold

\begin{enumerate}[$(a)$] %
\item The causal effect of $ \bX  $ on $ \bY $ is given by
\[ 
  P(\by \mid \textit{do}(\bx)) = \prod\limits_{Y\in\bY} P(Y = y \mid \textit{Pa}(Y) ). 
\]
\item If no edge $ X \to Y $ with  $X \in \bX, Y \in \bY $ exists, the causal effect is also given by the plain formula
\[ 
  P(\by \mid \textit{do}(\bx)) = P(\by). 
\]
\item The causal effect can be identified by adjustment if and only if no edge $ Y \to X $ with  $X \in \bX, Y \in \bY $ exists. 
\item If identification by adjustment is possible, the adjustment set is $ \bZ = \emptyset $ and the causal effect of $ \bX $ on $ \bY $ is given by
\[ 
  P(\by \mid \textit{do}(\bx)) = P(\by \mid \bx ).
\]
\end{enumerate}
\end{lemma}

\begin{proof}

Statement (a) follows   from the definition of the causal effect:

\begin{align*}
P(\by \mid \textit{do}(\bx)) \ =\ & \sum_{\bx'} P(\bX= \bx', \bY= \by \mid \textit{do}(\bX= \bx))\\
=\ & P(\bx, \by \mid \textit{do}(\bx)) &\rlap{\hspace*{-2.5cm}\parbox{6cm}{$\bx \neq \bx'$ makes the causal effect inconsistent}} \\  
=\ & P(\bv \mid \textit{do}(\bx)) \\
=\ & \prod_{ Y_j \in\bY } P(y_j \mid  \textit{pa}_j) &\rlap{\hspace*{-2.5cm}\parbox{4.5cm}{definition of the causal effect}} \\
=\ &  \prod_{Y\in\bY} P(Y = y \mid \textit{Pa}(Y) ) 
\end{align*}

For Statements $(b)$ and $(c)$ note that edges $ \bX \to \bX $ or $ \bY \to \bY $ do not affect $d$-connectedness between $ \bX $ and $ \bY $. 
Hence with the assumption in Statement~$(b)$, the sets $\bX$ and $\bY$ are $d$-separated in the graph $\cG_{\overline{\bX}}$, where all edges $\bY\to\bX $ and $\bX\to\bX $ are deleted. Then we know from Proposition~\ref{prep:observation:is:causal} that the causal effect is identified by a plain formula.
Since no node is outside of $ \bX \cup \bY $, the only possible adjustment set is $ \bZ =  \emptyset $, which is Statement~$(d)$.  
Finally, an adjustment set $ \bZ = \emptyset$ always satisfies the first  condition of the CBC. %
The back-door graph is formed by removing all edges $ \bX \to \bY $ as those edges form a causal path of length one. 
Thus $ \bZ $ is an adjustment set, if and only if no  edge $ \bY \to \bX $ exists, which is Statement~$(c)$. %
\end{proof}

When $ \bV = \bX\cup\bY $, the R package {\tt causaleffect} \cite{RpackageCausalEffect}, which we used in our experiments described in the next section, % that we describe later in this paper
returns the formula 
$ P(\by \mid \textit{do}(\bx)) = \prod_{Y\in\bY} P(Y \mid \textit{An}(Y) \setminus Y ) $ when configured to be fast and find \emph{any} identification formula rather than a short one. Thus it is worth to mention  that $\prod_{Y\in\bY} P(Y \mid \textit{An}(Y) \setminus Y ) = \prod_{Y\in\bY} P(Y \mid \textit{Pa}(Y) )$, because the parents of a node $Y$ block all paths from other ancestors of $Y$ to $Y$.
%It is easy to see that the above proof also holds for $ \textit{An}(Y) \setminus Y $ rather than $ \textit{Pa}(Y) $, since both sets block all paths into a node $ Y\in \bY $, so it is worth to mention  that $\prod_{Y\in\bY} P(Y \mid \textit{An}(Y) \setminus Y ) = \prod_{Y\in\bY} P(Y \mid \textit{Pa}(Y) )$.

\section{Empirical analysis of identifiability by adjustment}\label{sec:experiments}

As mentioned before, not all identifiable total effects are identifiable via covariate adjustment, but 
if covariate adjustment is possible, then it is usually preferred to other methods due to its benign statistical properties. This raises the question how often we will actually have to go beyond covariate adjustment when identifying causal effects. The completeness and algorithmic efficiency of the CBC allowed us to perform an empirical analysis of identifiability via adjustment in random graphs, including graphs of substantial size. 

The basic setup of our experiments is as follows. We (1) generate a random graph; (2) set nodes to be unobserved at random; 
(3) choose random disjoint subsets $\bX,\bY$  of pre-specified cardinalities from the observed nodes; and (4) test whether $P(\by \mid \textit{do}(\bx))$ is identifiable in the resulting graph. We test the identifiability of  $P(\by \mid \textit{do}(\bx))$ using four increasingly powerful criteria: (1) Pearl's back-door criterion  \cite{Pearl2009}; (2) the CBC; (3) an extended version of the CBC that also covers the special cases discussed in Section~\ref{sec:cbcext}; and (4) the do-calculus, which characterizes all effects that are identifiable at all. Full details are given below.

We included the classic back-door criterion (Definition~\ref{def:back-door:pearl}) in our analysis because it is still very present in the applied literature on DAGs (e.g., \cite{Elwert2013}) whereas the generalized version is still barely mentioned. It is known that the back-door criterion is not complete and can thus fail to identify an adjustment set, which raises the question how often the back-door criterion fails to find an adjustment set when our CBC criterion succeeds. In Section~\ref{sec:singletons} it was shown that this is never the case for a singleton $X$ (although the CBC may still find more adjustment sets than the BC).

Our extensions to the CBC in Section~\ref{sec:cbcext} were motivated by our observation from preliminary experiments that many cases where an effect is not identifiable by adjustment are anyway identifiable due to simple reasons like the absence of any causal path from $\bX$ to $\bY$, which can be addressed quite easily without invoking the full machinery of the do-calculus.

We now proceed to give the technical details of how we set up our empirical analysis.

\subsection{Instance generation}\label{sub:sub:sect:parameters}

We evaluate identifiability on random DAGs, which we generate as described in Section~\ref{sec:basissets}. 
The random DAGs are generated with different numbers of variables $\bV$
\[
 |\bV| = n \in \{10, 25, 50, 100, 250, 500, 1000, 2000\}.
\]

These variables are divided %
into four sets: 
ordinary observed nodes $\bR$,
unobserved nodes $\bV\setminus \bR$, 
exposure nodes $\bX\subseteq \bR$, and outcome nodes $\bY\subseteq \bR$ (with $\bX \cap \bY = \emptyset$)
depending on parameters
\[ 
   P(\textit{unobserved}) \in \{0, 0.25, 0.5, 0.75\} \quad \text{and}\quad   
  |\bX| = |\bY| = k \in  \left\{1,2,5,\lfloor \sqrt{n} \rfloor,\lfloor 0.1n \rfloor\right\}.
\]
To select those sets, we proceed as follows: Initially mark all variables in $\bV$ as observed. Next, 
for every node mark it as unobserved with  
probability $P(\textit{unobserved})$ until all nodes are considered or 
the number of nodes which remain observed reaches the threshold value $2k$.
Finally, from the observed $\bR$ pick randomly two 
disjoint subsets $\bX$ and $\bY$ of size $k$.
The expected size of $\bR$ is bounded by
$\EX[|\bR|] > n\cdot(1-P(\textit{unobserved}))$, with the difference being very small for $n \gg 2k$, but substantial for $n \gtrapprox 2k$.
For example for $n=10$ and $k = |\bX| = |\bY| = 5$, all nodes 
are in $\bR =  \bX \cup \bY = \bV$ %
regardless of the probability $P(\textit{unobserved})$ -- the case discussed 
already in Section~\ref{sec:special:case:v:partition}. %

We perform  experiments for each parametrization tuple 
\begin{equation}\label{eq:tuples}
  (n, l, k, P(\textit{unobserved})),
\end{equation}
where, recall, $l$ determines the probability $P(\textit{edge})$
as described in Section~\ref{sec:basissets}. In this section we will 
report our results in detail only for $ P(\textit{unobserved}) \in \{0,0.75\}.$
The remaining cases are shown in  the appendix.

We generated 10~000 graphs for each parameter tuple
using the function {\tt GraphGenerator.randomDAG} of our 
DAGitty library \cite{dagittyIJE} in node.js. 
Figure~\ref{fig:experiments:typical:dags} shows example instances  sampled for $ n = 10 $ 
and illustrates the four cases we are interested in.

\begin{figure}\arxivonly{\small}
	(a) \begin{tikzpicture}[baseline]
\graph  { 
 /-!-v1[exp,as=$X_1$],
 v7[exp,as=$X_2$] -!- v3[as=$V_0$] -!- /-!-v2[res,as=$Y_1$] ,
 v5[lat,as=$V_1$]  -!- v4[res,as=$Y_2$] -!-v6[lat,as=$V_2$],
 v10[exp,as=$X_3$] -!- v9[lat,as=$V_3$] -!- v8[res,as=$Y_3$],
 v1[exp]-> {v2[res],v3,v7[exp]}, v2-> {v6[lat],v8[res]}, v3-> {v4[res],v6,v8}, v4-> {v8,v9[lat]}, v5[lat]-> {v9,v10[exp]}, v6-> {}, v7-> {}, v8-> {}, v9-> {}, v10-> {}
 };
\end{tikzpicture} %adjustment
	(c) \begin{tikzpicture}[baseline]
\graph  { 
  / -!- / -!- v1[lat,as=$V_1$] -!- v2[exp,as=$X_1$],
  / -!- v3[as=$V_0$] -!-  v7[res,as=$Y_1$],
  v5[res,as=$Y_2$] -!- v4[exp,as=$X_2$] -!- v6[res,as=$Y_3$] -!- v9[lat,as=$V_2$],
  -!- v8[exp,as=$X_3$] -!- / -!- v10[lat,as=$V_3$],
  v1[lat]-> {v3,v7[res]}, v2[exp]-> {v7,v9[lat]}, v3-> {v5[res],v8[exp],v9}, v4[exp]-> {v10[lat]}, v5-> {v8}, v6[res]-> {v10}, v7-> {v9}, v8-> {}, v9-> {v10}, v10-> {}
};
\end{tikzpicture} %complex
	(p) \begin{tikzpicture}[baseline]
\graph  { 
  / -!- v2[res,as=$Y_1$] -!- /  -!- v3[lat,as=$V_1$],
  v1[lat,as=$V_2$] -!- v4[lat,as=$V_3$] -!- v7[exp,as=$X_1$] -!- v6[res,as=$Y_2$] -!- v9[exp,as=$X_2$],
  v8[res,as=$Y_3$] -!- v5[lat,as=$V_4$]  -!- v10[exp,as=$X_3$],
v1-> {v5,v8}, v2-> {v4}, v3-> {v7,v9}, v4-> {v5,v10}, v5-> {}, v6-> {v10}, v7-> {}, v8-> {}, v9-> {v10}, v10-> {}};
\end{tikzpicture} %independent
	(n) \begin{tikzpicture}[baseline]
\graph  { 
v2[exp,as=$X_1$] -!- v1[lat,as=$V_1$]  ,
v3[lat,as=$V_2$] -!- v4[lat,as=$V_3$] -!- v6[exp,as=$X_2$],
v5[lat,as=$V_4$] -!- v7[res,as=$Y_1$] -!- v8[res,as=$Y_2$],
v9[exp,as=$X_3$] -!- v10[res,as=$Y_3$],
v1[lat]-> {v3[lat],v4[lat],v6[exp]}, v2[exp]-> {v3}, v3-> {v7[res]}, v4-> {}, v5[lat]-> {v9[exp]}, v6-> {v7,v8[res]}, v7-> {v9}, v8-> {v10[res]}, v9-> {}, v10-> {}
};
\end{tikzpicture} %unidentif

\caption{Example DAGs sampled for the parameters $n = 10,  P(\textit{edge}) = 2/9,  P(\textit{unobserved}) = 0.5$, 
and $k= |\bX|=|\bY|=3$. 
Nodes are relabeled such that exposures are called 
$X_1,X_2,X_3$, outcomes are called $Y_1,Y_2,Y_3$, and all 
nodes except $V_0$ are unobserved.
Case (a)djustment is identified by using the empty set and by the formula 
$ \sum_{v_{0}}[P(y_{1}\vert x_{1})P(v_{0}\vert x_{1})P(y_{3}\vert x_{1},y_{1},v_{0},y_{2})P(y_{2}\vert x_{1},v_{0})]$ 
found by the ID-algorithm. Instance (c)omplex is identified by the complex formula 
$\sum_{v_{0}}[P(v_{0}\vert x_{1})P(y_{1}\vert x_{1},v_{0})P(y_{2}\vert v_{0})P(y_{3})]$ and 
instance  (p)lain is identified by the plain formula $ P(y_{1},y_{2},y_{3})$, although in this 
case no adjustment set exists.  The final example is (n)ot identifiable.
}\label{fig:experiments:typical:dags}
\end{figure}

\subsection{Algorithms}\label{sec:experiments:algorithms:abbreviations}

The main goal of our experiments was to examine the influence of the instance complexity,
like density of a DAG, numbers of exposures and  outcomes, and the ratio of unobserved to observed variables, on the identifiability by adjustment compared to general identifiability. Throughout, we use the following abbreviations for the algorithms we examine:
\begin{itemize}%{compactitem}
\item[CBC:] our constructive back-door criterion (Definition~\ref{def:ac:general:bdc}, Theorem~\ref{th:adj:set:cosctr}). We used our DAGitty library, specifically the function 
   {\tt GraphAnalyzer.canonicalAdjustmentSet}, which implements algorithm {\sc FindAdj}  
   based on the CBC.
 We also tested another implementation of our CBC criterion, the {\tt gac} function of the R package {\tt pcalg} \cite{pcalg}. 
\item[CBC$^+$:] combination of the CBC and plain formula (Proposition~\ref{prep:observation:is:causal}). We implement the plain formula using the DAGitty  function {\tt GraphAnalyzer.dConnected}, 
   which implements algorithm {\sc TestSep} (Proposition~\ref{prop:TestSep}).
\item[BC:] Pearl's back-door criterion (Definition~\ref{def:back-door:pearl}). 
It has been shown that if an adjustment set $\bZ$ that satisfies BC exists, it can be found 
by removing all descendants of $\bX$ from $\bZ$
\cite{PerkovicEtAl2018}. This means we can implement BC by trivial post-processing of the CBC output. 
\item[IDC:] general identifiability as determined by do-calculus (see \cite[Chapter 3.4.2]{Pearl2009}). 
% we also recall the rules in Theorem~\ref{thm:do-calculus} in the appendix). 
Specifically, we use the IDC algorithm by Shpitser and Pearl \cite{ShpitserIDCAlgorithm}, which is well known to be complete for the identification of causal effects \cite{shpitser2006identification,huang2006pearl}, meaning that the algorithm computes a formula involving only the pre-intervention distribution that expresses the causal effect if such a formula exists; otherwise it outputs that identification is impossible.
Our experiments are based on the IDC implementation 
provided by the R package {\tt causaleffect} \cite{RpackageCausalEffect}.
Due to its high time complexity, we were only able to use this algorithm for small instances.

   %In this way we count 
   %the cases that are identifiable by adjustment \emph{or} by plain  
   %formula (see Proposition~\ref{prep:observation:is:causal}).

\end{itemize}%{compactitem}

\subsection{Results}

\input{experiments-table1}

The results for all methods and parameters $n,k,l$ described above
%in subsection~\ref{sub:sub:sect:parameters},
 are shown in Table~\ref{table:global:stat:unobs0:count} 
(for the case $P(\textit{unobserved})=0$) and in Table~\ref{table:global:stat:unobs0.75:count}
($P(\textit{unobserved})=0.75$). 
We now discuss the results in more detail.

\paragraph{Identification by adjustment sets or plain formulas}
Tables~\ref{table:global:stat:unobs0:count} and  \ref{table:global:stat:unobs0.75:count}
provide counts for instances identifiable by adjustment alone
(columns CBC) or by adjustment enhanced by using the plain formula (CBC$^+$).
The number of effects only identified by the plain formula, but not by CBC,
is thus given by the difference between these columns. 

Figure~\ref{table:global:stat:grid:grey} summarizes 
the counts for CBC and CBC$^+$ reported in Table~\ref{table:global:stat:unobs0:count} 
and~\ref{table:global:stat:unobs0.75:count} for $k=1,2,5$ 
and  $n\ge 25$. We omit the instances with $n=10$, since for $k=5$ these 
cases were discussed separately in Section~\ref{sec:special:case:v:partition}.
Moreover, for  parameter values $l=10$ and $l=20$ the individual probabilities 
for edge selection, $P(\textit{edge})=\max\{l/(n-1),1\}$, 
imply that every node has $9<l$ neighbors while in our analyses we want that $l$
specifies the expected number of neighbors of a node.

Identification by plain formula and identification by adjustment are 
overlapping concepts.
Some cases can be identified using either approach,
while in other instances only one of them works.
\highlightrevision{r2c58}{
Many cases for which adjustment does not work can be solved instead by 
the plain formula, meaning that in those DAGs there is in fact no causal
effect of $\bX$ on $\bY$.}
This can be seen especially in dense graphs, e.g., DAGs
in which each node has $l=20$ neighbors on average,
and for singleton $\bX$ and $\bY$, i.e., $k=1$.
For example, for $l=20$, $k=1$, $P(\textit{unobserved}) = 0.75$, 
in DAGs with $n = 2000$  ($n=1000$, $n=500$) nodes, 
up to 65 \%  (63 \%, 61 \%) of all instances  
are identifiable by the plain formula but not by adjustment.
Furthermore, increasing $n$ from $25$ to $2000$ we observe 
that this percentage ranges between $51 \%$ and the maximum $65 \%$, 
a rather narrow range.  
The counts for CBC and CBC$^+$ are illustrated in Figure~\ref{table:global:stat:grid:grey}
as gray squares in the columns labeled as $(l,k)=(20,1)$ (case: $P(\textit{unobserved}) = 0.75$).  

The difficulty of identification by adjustment grows with increasing $k$ and $l$,
but it decreases with increasing number of nodes $n$,
both for $P(\textit{unobserved}) = 0$ and for $0.75$.
In Figure~\ref{table:global:stat:grid:grey}, 
columns are sorted increasingly by the total number of identifiable effects per column.
This shows that the most difficult case is $(l,k)=(20,5)$:
for $P(\textit{unobserved}) = 0$ the counts grows very slowly with $n$ 
reaching the maximum value of $3 \%$ of identifiable graphs for $n=2000$;
for  $P(\textit{unobserved}) = 0.75$ almost no instances are identifiable by adjustment 
(compare the upper panels in Figure~\ref{table:global:stat:grid:grey}).
However, as we can see in Table~\ref{table:global:stat:unobs0.75:count},
for $n=250$ only $2.5 \%$ of  cases are identifiable at all.
Figures~\ref{fig:global:n:curves} and \ref{fig:global:l:curves} summarize  the 
difficulty of identification stratified by $n$ (Figure~\ref{fig:global:n:curves}) 
and $l$ (Figure~\ref{fig:global:l:curves}), respectively. 

\input{experiments-grid}

%Note, however, that this identification is not unique. In some graphs the identification is  both possible by $ P(\by) $ and by an %adjustment set. 
%, or by an empty adjustment set, $ p(\by \mid \textit{do}(\bx)) = p(\by \mid \bx) $. 
%These identifications are disjoint from the identification by (Pearl) adjustment sets. There are graphs in which an empty adjustment %set exists X1 <- X2 -> Y

\paragraph{Comparison of CBC to the back-door criterion by Pearl}
%We tested for each graph in which DAGitty has found an adjustment set $ \bZ $
%, which has been shown to be an adjustment set satisfying Pearl's definition, if and only if such an adjustment set exists in the graph .
%The results show that in the easy cases, when almost every graph contains an adjustment set, and the hard cases, when no adjustment exists, Pearl's criterion and the CBC behave very similar. In the in-between case a larger percentage of graphs yield an adjustment set that is not identifiable by   Pearl's criterion.

We were also interested in how often Pearl's back-door criterion 
(BC) would fail to find an adjustment set.
Tables~\ref{table:global:stat:unobs0:count} 
and~\ref{table:global:stat:unobs0.75:count} show that the difference 
between BC and the CBC is rather small, especially for simple 
(or hard) instances where nearly every (or no) DAG has an adjustment set, 
and as expected given our results in Section~\ref{sec:singletons}, for singletons 
$\bX = \{X\}, \bY= \{Y\}$ the counts for BC and CBC are indeed equal.
However, for larger $\bX, \bY$ and parameters where only 
a few graphs have an adjustment set, the difference between BC and CBC becomes
more substantial. The 
greatest difference occurs for 
$n = 10$, $|\bX| = |\bY| = 3$, $m \approxeq n$, and  $P(\textit{unobserved}) = 0$,
where in 10\% of all cases there is an adjustment set whereas BC finds none. 
This is followed by $n = 10$, $|\bX| = |\bY| = 2$, where BC fails to find 
existing adjustment sets in 7\% to 9\% of the cases, depending on $ P(\textit{edge})$
and $P(\textit{unobserved})$.

\input{experiments-n-curves}
\input{experiments-l-curves}

\paragraph{Complete identification by do-calculus compared 
to identification by adjustment or plain formula}
As explained above, in small graphs we checked for general identifiability 
of causal effects using the IDC algorithm \cite{ShpitserIDCAlgorithm}. 
Results are shown for $P(\textit{unobserved}) = 0.75$  and $ n \leq 250 $ in 
Table~\ref{table:global:stat:unobs0.75:count}. 
Since the IDC algorithm is complete for the identification problem,
the corresponding counts also show how many instances are identifiable 
at all. It is known that in the case $ P(\textit{unobserved}) = 0$ 
the causal effect is always identifiable, so we skip the counts for IDC 
in Table~\ref{table:global:stat:unobs0:count}.

The cases with $ n=10$, $k=|\bX|=|\bY| = 5$ (Table~\ref{table:global:stat:unobs0.75:count}) 
might seem suspicious as   
the number of identifiable graphs (i.e., counts in column IDC) increases drastically compared to the 
graphs with smaller $ \bX, \bY $, while in all the other cases (see Table~\ref{table:global:stat:unobs0.75:count})
this number decreases with increasing cardinality of $ \bX, \bY $. 
However, this is explained by the cap on the number of unobserved nodes. 
When $|\bX|+|\bY|=10$ for $ n=10 $, there are no nodes outside of $ \bX\cup\bY$ 
remaining that could become unobserved regardless of $P(\textit{unobserved})$,
similarly to the cases in Table~\ref{table:global:stat:unobs0:count}, and all 
graphs must be identifiable as shown in Section~\ref{sec:special:case:v:partition}. 

Figures~\ref{fig:global:n:curves} and \ref{fig:global:l:curves} present the data for CBC$^+$
(the same data as in the lower right panel of Figure~\ref{table:global:stat:grid:grey}) in
comparison to identification by IDC. 
As we observed in Figure~\ref{table:global:stat:grid:grey}, the most difficult case for CBC$^+$
is $(l,k)=(20,5)$ and the difficulty decreases with $k$ and $l$ when $n$ is fixed
(Figure~\ref{fig:global:n:curves}). The situation is very similar for IDC.
In Figure~\ref{fig:global:l:curves}, we see that identifiability for both 
CBC$^+$ and IDC grows roughly in parallel. Similarly to the results for adjustment sets, one can see that with increasing $|\bX|, |\bY|, l$ the number of identifiable graphs decreases, when $ P(\textit{unobserved}) > 0 $. %

These experiments also provide a verification of our DAGitty implementation as every adjustment set found by the {\tt causaleffect} package has been found by DAGitty, as well as a ground truth of the unidentifiable graphs, since a causal effect not identified by the IDC algorithm cannot be identified by any method.
\highlightrevision{r2c59}{
	Moreover, as expected, for DAGs with no unobserved variables and with $|\bX| = |\bY| = 1 $, all causal effects are already identified by a plain formula or adjustment without involving the IDC algorithm (see Table~\ref{table:global:stat:unobs0:count}).}

\input{runtimes}
\input{experiments-runtimes-table}

\paragraph{Comparative runtimes of the algorithms}
Figure~\ref{fig:runtime:nodes} (black lines) shows the time needed by DAGitty for these experiments on one core of a 2.1 GHz (up 3 GHz with Turbo Core) AMD Opteron 6272 for graphs with $P(\textit{unobserved}) = 0.75$. Graphs with a lower probability $P(\textit{unobserved})$ are processed slightly faster.  For small sets $\bX$ and $\bY$ the time increases roughly linearly with the number of edges $m$. For larger sets the time also increases with the size of $\bX,\bY$, which could either mean that DAGitty does not reach the optimal asymptotic runtime of $\cO(n+m)$ due to inefficient set operations, or that the time actually only depends on $\cO(\textit{An}(\bX, \bY))$ which can be much smaller than $\cO(m)$ when the sets and degrees are small. However, for all models of a size currently used in practice, DAGitty finds the adjustment set nearly instantaneously. 

The runtimes  of the  {\tt causaleffect} package are shown \onlyincolor{as red plot }in Figure~\ref{fig:runtime:nodes}. Since the IDC algorithm is far more complex than the expression of Theorem~\ref{th:adj:set:cosctr}, it performs generally one to two orders of magnitude slower than the implementation in DAGitty, or equivalently in the same time DAGitty can process graphs that are one to two orders of magnitude  larger.  Due to this speed difference it was not possible for us to run the IDC algorithm experiments on the larger graphs.

We have also investigated a different implementation of the CBC in the R package pcalg \cite{pcalg}. The {\tt gac} function in that package implements the CBC criterion for DAGs and other graph classes. Unlike DAGitty, the pcalg package does not find an adjustment set, but only verifies whether a given set meets the criterion. Hence, after loading the graphs in R and calculating the adjacency matrices required by pcalg, we compute the canonical adjustment set $ \adjustment(\bX,\bY) $ in R as\\ %
{\tt \hspace*{1cm} Dpcp = De(G, intersect(setminus(De(GNoInX, x), x), An(GNoOutX, y)))\\
  \hspace*{1cm} z = setminus(An(G, union(x,y)), union(union(x,y),union(Dpcp, obs)))}\\
with sets {\tt x}, {\tt y}, {\tt obs}erved  nodes {\tt obs}, graphs {\tt G} $ =\cG $, {\tt GNoInX} $ = \cG_{\overline{X}} $, {\tt GNoOutX} $ = \cG_{\underline{X}} $ and helper functions {\tt An} and {\tt De} implemented using the {\tt subcomponent} function of the R package {\tt igraph}. We then compare the time required by {\tt pcalg} to test whether {\tt z} is a valid adjustment set to the time required by DAGitty to find and test a set.
The runtimes of the {\tt gac} function are plotted \onlyincolor{in purple }in Figure~\ref{fig:runtime:nodes}. 
They show that the {\tt gac} function is several orders of magnitude  slower than DAGitty. These results are expected given that the {\tt pcalg} package tests the CBC by tracing all $m$-connected paths using backtracking, an approach that suffers from exponential worst-case complexity; in fact this backtracking algorithm is even slower than the general implementation of the do-calculus in the {\tt causaleffect} package. Only the cases with small $ n $ are shown as the remaining computations did not terminate in reasonable time.

In summary, our experimental results show that many causal effects in random DAGs cannot be identified by covariate adjustment. Nevertheless, many of these cases are easily addressed by extending the CBC slightly, and then most effects become identifiable without having to resort to do-calculus, at least in the random graphs we tested. This finding is reassuring given that the implementation of our algorithmic framework in DAGitty is the only identification method of those we tested that is applicable to large graphs.

\section{Adjustment in MAGs}

\label{sec:magadjust}

Finally, in this section, we generalize our complete constructive criterion for identifying 
adjustment sets from DAGs to MAGs, making our algorithmic framework applicable to this class of graphical models as well. 
Two examples may illustrate why this generalization is not
trivial. First, take $\cG=X \to Y$. If $\cG$ is interpreted as a DAG,
then the empty set is valid for adjustment. If $\cG$ is however
taken as a MAG, then there exists no adjustment set (for a formal definition of 
an adjustment set in a MAG see below) as 
$\cG$ represents among others the 
DAG \begin{tikzpicture}[xscale=0.7,baseline=-2.3ex]
\node [anchor=north] (l) at (0,0) {$U$};
\node [anchor=north] (x) at (1,0) {$X$};
\node [anchor=north] (y) at (2,0) {$Y$};
\draw [->] (x) -- (y);
\draw [->] (l) -- (x);
\draw [->] (l) edge [bend right=20] (y);
\end{tikzpicture} where $U$ is an unobserved confounder.
Second, take $\cG=A \to X \to Y$. In that case, the empty
set is an adjustment set regardless of whether $\cG$
is interpreted as a DAG or a MAG. The reasons 
will become clear as we move on. First, let us recall
the semantics of a MAG. The following definition
can easily be given for AGs in general, but 
we do not need this generality for our purpose.
\begin{definition}[DAG representation by MAGs \citep{Richardson2002}] 
Let $\cG=(\bV,\bE)$ be a DAG, 
and let $\bS,\bL \subseteq \bV$. The MAG $\cM=\cG[^\bS_\bL$ is a graph
with nodes $\bV\setminus(\bS\cup\bL)$ and edges 
defined as follows. (1) Two nodes $U$ and $V$ are adjacent in 
$\cG[^\bS_\bL$ if they cannot be $m$-separated by 
any $\bZ$ with $\bS \subseteq \bZ
\subseteq \bV\setminus\bL$ in $\cG$. (2) 
The edge between $U$ and $V$ is
\begin{description}
\item $U-V$ if $U \in \textit{An}(\bS\cup V)$ and $V \in \textit{An}(\bS \cup U)$;
 \item $U\to V$ if $U \in \textit{An}(\bS\cup V)$ and $V \notin \textit{An}(\bS \cup U)$;
  \item $U\leftrightarrow V$ if $U \notin \textit{An}(\bS\cup V)$ and $V \notin \textit{An}(\bS \cup U)$.
\end{description}
We call $\bL$ \emph{latent} variables
and $\bS$ \emph{selection} variables. We say
there is \emph{selection bias} if $\bS \neq \emptyset$.
\label{def:mags}
\end{definition}
Hence, every MAG represents an infinite set of underlying
DAGs that all share the same ancestral relationships.

\begin{lemma}[Preservation of separating sets \citep{Richardson2002}]
Set $\bZ$ $m$-separates $\bX,\bY$ in $\cG[^\bS_\bL$ if and only
if $\bZ\cup\bS$ $m$-separates $\bX,\bY$ in $\cG$.
\end{lemma}

Selection bias
(i.e., $\bS \neq \emptyset$) substantially
complicates adjustment, and in fact nonparametric causal inference
in general \citep{Zhang2008}\footnote{
A counterexample is the graph $A \gets X \to Y$,
where we can safely assume that $A$ is the ancestor 
of a selection variable. A sufficient and 
necessary condition to recover a distribution $P(\by \mid \bx)$ from a distribution $P(\by \mid \bx,\bs)$ under selection bias is 
 $\bY \independent \bS \mid \bX$ \citep{Barenboim2014}, which is so restrictive
that most statisticians would probably not even speak of
``selection bias'' anymore in such a case.
}. 
Due to these limitations, 
we restrict ourselves to the case $\bS=\emptyset$ 
in the rest of this section.
Note however that recovery from selection bias is 
sometimes possible with additional population
data, and graphical conditions exist
to identify such cases \citep{Barenboim2014}.

We now extend the concept of adjustment to MAGs
in the usual way \citep{Maathuis2013}.

\begin{definition}[Adjustment in MAGs]
\label{def:adjustment:set:in:MAAGs}
Given a MAG $\cM=(\bV,\bE)$ and two variable sets $\bX,\bY \subseteq \bV$,
$\bZ \subseteq \bV$ is an adjustment set for $(\bX,\bY)$ in $\cM$ 
if for all DAGs $\cG=(\bV',\bE')$ for which $\cG[^\emptyset_\bL\ =\cM$ with $\bL = \bV'\setminus \bV$
the set $\bZ$ is an adjustment set for $(\bX,\bY)$ in $\cG$.
\end{definition}

This definition is equivalent to requiring that $P(\by \mid \textit{do}(\bx))$ is equal to $ \sum_{\bz} P(\by \mid \bx, \bz) P(\bz)$ for every probability distribution $P(\bv')$ consistent with a DAG $\cG=(\bV',\bE')$ for which $\cG[^\emptyset_\bL\ =\cM$ with $\bL = \bV'\setminus \bV$. If one was to extend the definition to include selection bias $ \bS $, one would need to give a requirement that holds for all DAGs $\cG=(\bV',\bE')$ with $\cG[^\bS_\bL\ =\cM$ and $\bL \cup \bS = \bV'\setminus \bV$. Thereby one can define $ P(\by \mid \textit{do}(\bx)) $ as $\sum_{\bz} P(\by \mid \bx, \bz, \bs ) P(\bz \mid \bs)$ , $  \sum_{\bz} P(\by \mid \bx, \bz, \bs ) P(\bz) $ or $ \sum_{\bs} \sum_{\bz} P(\by \mid \bx, \bz, \bs ) P(\bz, \bs)$. The last definition is equivalent to $ \bZ \cup \bS $ being an adjustment set in all these DAGs, but existing literature has used the second case\citep{Barenboim2014}. However, the first case captures the spirit of selection bias the most, since in the presence of selection bias the probability distribution is only known given some selected bias $ \bs $.  %

%There is no reason to further generalize this definition to ancestral graphs rather than MAGs as $ \cG[_\bL^\bS $ is always a %MAG and every ancestral graph can be converted to a MAG with the same conditional independences \cite{Richardson2002}.
Notice that, due to the definition of adjustment in MAGs (Def.~\ref{def:adjustment:set:in:MAAGs}), in our considerations
we can restrict MAGs to mixed graphs consisting of only directed and bidirected edges.

\subsection{Adjustment amenability}

In this section we first identify a class of MAGs in which
adjustment is impossible because of causal ambiguities -- e.g.,
the simple MAG $X \to Y$ falls into this class, but the
larger MAG $A \to X \to Y$ does not. 

\begin{definition}[Visible edge \citep{Zhang2008}]
Given a MAG $\cM=(\bV,\bE)$, an edge 
$X \to D$ in $\bE$ is called \emph{visible} if in all DAGs 
$\cG=(\bV',\bE')$ with
$\cG[^\emptyset_\bL=\cM$ for some $\bL \subseteq \bV'$,
all $d$-connected
walks between $X$ and $D$ in $\cG$ that contain only
nodes of $\bL\cup X\cup D$
are directed paths.
 Otherwise $X \to D$ is said to be \emph{invisible}.
\end{definition}

Intuitively, an invisible directed edge $X \to D$ 
means that there may exist hidden confounding factors
between $X$ and $D$, which is guaranteed not to be the case if the edge is visible.

\begin{lemma}[Graphical conditions for edge visibility  \citep{Zhang2008}]
In a MAG  $\cM=(\bV,\bE)$, an edge $X \to D$ in $\bE$ is visible if and only if
there is a node $A$ not adjacent to $D$ where 
(1) $A \to X \in \bE$ or
$A \leftrightarrow X \in \bE$, or (2) there is a collider path
$A \leftrightarrow V_1 \leftrightarrow \ldots \leftrightarrow V_n 
\leftrightarrow X$ 
or $A \to V_1 \leftrightarrow \ldots \leftrightarrow V_n 
\leftrightarrow X$ where all $V_i$ are
parents of $D$. 
\label{lemma:pureedge}
\end{lemma}

\begin{definition}
We call a MAG $\cM=(\bV,\bE)$ \emph{adjustment amenable} w.r.t.
$\bX,\bY \subseteq \bV$ if all
proper causal paths from $\bX$ to $\bY$ start with a 
visible directed edge.
\end{definition}

\begin{lemma}
If a MAG $\cM=(\bV,\bE)$ is not adjustment amenable w.r.t.
$\bX,\bY \subseteq \bV$ then there exists no valid adjustment set 
for $(\bX,\bY)$ in $\cM$.
\label{lemma:amenable}
\end{lemma}
\begin{proof}
If the first edge $X \to D$ on some causal path to $\bY$ 
in $\cM$ is 
not visible, then there exists a consistent DAG $\cG$ where there is
a non-causal path between $X$ and $\bY$ via $D$ that could
only be blocked in $\cM$ by conditioning on $D$ or some of its
descendants. But such conditioning would violate the adjustment
criterion in $\cG$.
\end{proof}

\highlightrevision{r2c65}{Note that adjustment amenability does not yet guarantee the existence of an adjustment set; the smallest example is the MAG $X \gets Y$, which is amenable but admits no adjustment set.}

Let $N(V)$ denote all nodes adjacent to $V$, 
and $\textit{Sp}(V)$ denote all spouses of $V$, i.e.,
nodes $W$ such that $W \leftrightarrow V \in \bE$.
The adjustment amenability of a graph $\cM$ w.r.t 
sets $\bX, \bY$ can be tested with the following algorithm:

\newcommand{\setvisited}{\bC}
\newcommand{\setcolliders}{\bA}
\begin{algo}{TestAdjustmentAmenability}{$\cM, \bX, \bY$}{\label{algo:isadjustmentamenable}}{10cm}
\For {all $D$ in $\textit{Ch}(\bX) \cap \pcauspaths$}

\State{$\setvisited \gets \emptyset$; $\setcolliders \gets \emptyset$}

\Function{check}{V}
\If{$\setvisited[V]$}  {\Return $\setcolliders[V]$}\EndIf
\State{$\setvisited[V] \gets$ true}
\State{$\setcolliders[V] \gets ((\textit{Pa}(V) \cup \textit{Sp}(V)) \setminus N(D) \neq \emptyset)$}
\For {all $W \in \textit{Sp}(V) \cap \textit{Pa}(D)$ }
\If{\textsc{check}$(W)$} $\setcolliders[V] \gets $ true \EndIf
\EndFor
\Return $\setcolliders[V]$
\EndFunction

\For {all $X$ in $\bX \cap \textit{Pa}(D)$ }
\If{$\neg\textsc{check}(X)$} {\Return false}\EndIf
 
\EndFor
\EndFor
\end{algo}

\begin{analal}
\highlightrevision{r2c66}{
The algorithm checks for visibility of every edge $X\to D$ by trying to find 
a node $Z$ not connected to $D$ but connected to $X$ via a collider path through
the parents of $D$, according to the conditions of Lemma~\ref{lemma:pureedge};
note that condition (1) of Lemma~\ref{lemma:pureedge} is identical 
	to condition (2) with an empty collider path.}
Since $\textsc{check}$ performs a depth-first-search
by checking every node only once and then continuing to 
its neighbors, each iteration of the outer for-loop
in the algorithm runs in linear 
time $\cO(n+m)$. Therefore, the entire algorithm runs in $\cO(k(n+m))$
where $k \leq |Ch(\bX)|$.
\end{analal}

\subsection{Adjustment criterion for MAGs}

We now show that the adjustment criterion for DAGs 
generalizes to adjustment amenable MAGs. 
The adjustment criterion and the constructive
back-door criterion are defined 
like their DAG counterparts (Definitions~\ref{def:ac:general} 
and~\ref{def:adj:graph:}), replacing “DAG” with “MAG” and $d$- 
with $m$-separation for the latter.

\begin{definition}[Adjustment criterion]\label{def:ac:general:mag}
Let $\cM = (\bV,\bE)$ be a MAG, and $\bX,\bY,\bZ \subseteq \bV$ 
be pairwise disjoint subsets of variables. 
The set $\bZ$ satisfies the adjustment criterion relative to $(\bX, \bY)$ in $\cM$ if 
\begin{enumerate}
   \item[$(a)$] 
    no element in $\bZ$ is a descendant in $\cM$ %
    of any $W \in \bV \setminus \bX$ 
   which lies on a proper causal path from $\bX$ to $\bY$
   and 
   \item[$(b)$]
   all proper non-causal paths in $\cM$ from $\bX$ to $\bY$ are blocked by $\bZ$.
\end{enumerate}
\end{definition}

Note that the above definition uses “descendants in $ \cM $” instead “descendants in $ \cM_{\overline{\bX}} $” as Definition \ref{def:ac:general}. 
However, Lemma \ref{lem:cbc:variants} implies that the conditions are equivalent.

\begin{definition}[Proper back-door graph] \label{def:adj:graph:mag}
Let $\cM = (\bV,\bE)$ be a MAG, and $\bX,\bY\subseteq \bV$ 
be pairwise disjoint subsets of variables. 
The \emph{proper back-door graph}, denoted as  $\cbdg$,
is obtained from $\cM$ by removing 
the first edge of every proper causal path from $\bX$ to $\bY$.
\end{definition}

\begin{figure}
\begin{center}
\begin{tabular}{ccccc}
\begin{tikzpicture}[yscale=0.9,xscale=0.9]
\node (g) at (-0.35,1.25) {$\cM_1$:};
\node (x1) at (0,0) {$X$};
\node (v1) at (1,0) {$V$};
\node (y1) at (2,0) {$Y$};
\node (z1) at (1,1) {$Z$};

\draw [red,thick,->] (x1) -- (v1);
\draw [->] (v1) -- (y1);
\draw [->] (z1) -- (y1);
\end{tikzpicture}
&
\begin{tikzpicture}[yscale=0.9,xscale=0.9]
\node (g) at (-0.35,1.25) {$\cM_2$:};
\node (x1) at (0,0) {$X$};
\node (v1) at (1,0) {$V$};
\node (y1) at (2,0) {$Y$};
\node[adjusted] (z1) at (1,1) {$Z$};

\draw [->] (x1) -- (v1);
\draw [->] (v1) -- (y1);
\draw [->] (z1) -- (y1);
\draw [->] (z1) -- (x1);
\end{tikzpicture}
&
\begin{tikzpicture}[yscale=0.9,xscale=0.9]
\node (g) at (-0.35,1.25) {$\cM_3$:};
\node (x1) at (0,0) {$X$};
\node (v1) at (1,0) {$V$};
\node (y1) at (2,0) {$Y$};
\node (z1) at (1,1) {$Z$};

\draw [red,thick,->] (x1) -- (v1);
\draw [->] (v1) -- (y1);
\draw [->] (z1) -- (y1);
\draw [->] (z1) -- (x1);
\draw [->] (z1) -- (v1);
\end{tikzpicture}
&
\begin{tikzpicture}[yscale=0.9,xscale=0.9]
\node (g) at (-0.35,1.25) {$\cM_4$:};
\node (x1) at (0,0) {$X$};
\node (v1) at (1,0) {$V$};
\node (y1) at (2,0) {$Y$};
\node (z1) at (1,1) {$W$};

\draw [->] (x1) -- (v1);
\draw [->] (v1) -- (y1);
\draw [<->] (z1) -- (y1);
\draw [<->] (z1) -- (x1);
 
\end{tikzpicture}
&
\begin{tikzpicture}[yscale=0.9,xscale=0.9]
\node (g) at (-0.35,1.25) {$\cM_5$:};
\node (x1) at (0,0) {$X_2$};
\node (x2) at (1,0) {$X_1$};
\node (v1) at (2,0) {$V$};
\node (y1) at (3,0) {$Y$};
\node[adjusted] (z1) at (2,1) {$Z$};

\draw [->] (x1) -- (x2);
\draw [->] (x2) -- (v1);
\draw [->] (v1) -- (y1);
\draw [->] (z1) -- (y1);
\draw [->] (z1) -- (x1);
\draw [->] (z1) -- (x2);
\draw [->] (z1) -- (v1);
 
\end{tikzpicture} 
\end{tabular}
\end{center}
\caption{Five MAGs in which we search for an adjustment relative to  $(X,Y)$ or $(\{X_1,X_2\},Y)$. $\cM_1$ and $\cM_3$ are not adjustment amenable, since the edge $ X \to V$ is not visible, so no adjustment exists. In the other three MAGs the edge is visible, due to the node $Z$ in $\cM_2$, the node $W$ in $\cM_4$ and the node $X_2$ in $\cM_5$. The only valid adjustment in $\cM_2$ and $\cM_5$ is $\{Z\}$, and in $\cM_4$ only the empty set is a valid adjustment. If $\cM_1$ and $\cM_3$ were DAGs, the set $\{Z\}$ would be an adjustment in each.} 
\label{fig:CBCMAGexample}
\end{figure}

\begin{definition}[Constructive back-door criterion (CBC)]\label{def:ac:general:bdc:mag}
Let $\cM = (\bV,\bE)$ be a MAG, and 
let $\bX,\bY,\bZ \subseteq \bV$ be pairwise disjoint subsets of variables. 
The set $\bZ$ satisfies the 
%adjustment separating criterion 
\emph{constructive back-door criterion}
relative to $(\bX, \bY)$ in $\cM$  if
\begin{enumerate}
  \item[$(a)$] $\bZ \subseteq \bV \setminus \causpaths$ and 
  \item[$(b)$] 
    $\bZ$ $m$-separates $\bX$ and $\bY$ in the proper back-door graph 
      $\cbdg$.
\end{enumerate}
\end{definition}

The main result of this section (Theorem~\ref{thm:agadjust}) shows that 
for any adjustment amenable MAG and node sets $\bX$ and $\bY$, a set $\bZ$ is an adjustment relative to 
$(\bX,\bY)$ if and only if the CBC is satisfied.
Figure~\ref{fig:CBCMAGexample} shows some examples. 
Similarly to DAGs, we provide  a generalization of the CBC for MAGs allowing parametrization of the criterion:

\begin{definition}[Parametrization of the Constructive back-door criterion (CBC($\bA,\bB,\bC $))]\label{def:ac:general:bdc:mag:prime}
Let $\cM = (\bV,\bE)$ be a MAG, and 
let $\bX,\bY,\bZ \subseteq \bV$ be pairwise disjoint subsets of variables. 
Let $\bA \subseteq \bX \cup \bY $, $\bB \subseteq \bX $, $\bC \subseteq \textit{De}(\pcauspaths)$.
The set $\bZ$ satisfies the CBC($\bA,\bB,\bC $) %of the constructive back-door criterion
relative to $(\bX, \bY)$ in $\cM$  if 
\begin{enumerate} 
  \item[$(a)$] $\bZ \subseteq \bV \setminus De_{\overline{\bA}\underline{\bB}}(\pcauspaths)$, and
\item[$(b)$] 
    $\bZ$ $d$-separates $\bX$ and $\bY$ in the graph $\pbdgmod := (\bV,\bE \setminus (\bX \to (\pcauspaths \cup \bC) ))$.
\end{enumerate}
\end{definition}

With these definitions we are ready to give:

\begin{theorem}
Given an adjustment amenable
MAG $\cM=(\bV,\bE)$ and three disjoint 
node sets $\bX,\bY,\bZ \subseteq \bV$,
the following statements are equivalent:
\begin{enumerate}
\item[$(i)$] $\bZ$ is an adjustment relative to $(\bX,\bY)$ in $\cM$.
 \item[$(ii)$] $\bZ$ fulfills the adjustment criterion (AC) w.r.t. $(\bX,\bY)$ in $\cM$.
\item[$(iii)$] $\bZ$ fulfills the constructive back-door criterion (CBC) w.r.t. $(\bX,\bY)$ in $\cM$.
\item[$(iv)$] $\bZ$ fulfills a variant of constructive back-door criterion (CBC(\bA,\bB,\bC)) w.r.t. $(\bX,\bY)$ in $\cM$
for  $\bA \subseteq \bX \cup \bY $, $\bB \subseteq \bX $, $\bC \subseteq \textit{De}(\pcauspaths)$. 
\end{enumerate}
\label{thm:agadjust}
\end{theorem}

Before proving the theorem, let us recall the concept of an  \emph{inducing path},
which we will use in our proof and 
the further analysis.  

\begin{definition}[Inducing path \citep{Richardson2002}]
Let $\cG=(\bV,\bE)$ be  a DAG and 
$\bZ,\bL \subseteq \bV$ be disjoint.
%and let $\cM=\cG[^\emptyset_\bL$.
A path $\pi=V_1,\ldots,V_{n+1}$ in $\cG$ is called
\emph{inducing} with respect to $\bZ$ and $\bL$ 
if all non-colliders on $\pi$ 
except $V_1$ and $V_{n+1}$ are in $\bL$ 
and all colliders on $\pi$ are in 
$\textit{An}(\{V_1,V_{n+1}\}\cup\bZ)$.
\label{def:inducing:path}
\end{definition}
% % Every inducing path w.r.t.
% % $\bZ, \bL$ is $m$-connected by 
% % $\bZ$.

We will use also the following notion: 
\begin{definition}[Inducing $\bZ$-trail]
Let $\cG=(\bV,\bE)$ be  a DAG and $\bZ,\bL \subseteq \bV$ be disjoint.
%and let $\cM=\cG[^\emptyset_\bL$.
\highlightrevision{r2c77}{
Let $\pi=V_1,\ldots,V_{n+1}$ be a path in $\cG[^\emptyset_\bL$ 
such that $V_2,\ldots,V_n \in \bZ$, $V_1,V_{n+1} \notin \bZ$, 
for each $i \in \{1,\ldots,n\}$, 
there is an inducing path w.r.t. $\emptyset,\bL$ 
linking $V_i, V_{i+1}$, and for each $i \in \{2,\ldots,n\}$, 
these inducing paths have arrowheads at $V_i$.}
Then $\pi$ is called an inducing $\bZ$-trail.
\label{def:inducing:trail}
\end{definition}

\begin{proof}[Proof of Theorem~\ref{thm:agadjust}]
The equivalence of $(ii)$, $(iii)$ and $(iv)$ is established by
observing that the proofs of Theorem~\ref{th:ac:equivalence} and Lemma~\ref{lem:cbc:variants}
generalize to $m$-separation. Below we establish
equivalence of $(i)$ and $(ii)$.

$\neg (ii) \Rightarrow \neg (i)$: If
$\bZ$ violates the adjustment criterion in $\cM$, it
does so in the canonical DAG $\cC(\cM)$, and thus 
 is not an adjustment in $\cM$.

$\neg (i) \Rightarrow \neg (ii)$: In the proof we rely on properties from
Lemmas~\ref{lemma:auxiliary:walkpathconversion}, \ref{lemma:auxiliary:truncatemaxsub}, and \ref{lemma:inducing:adjacent:when:visible}
presented in Subsection \ref{sec:aux:lemmas}.

Let $\cG$ be a DAG, with 
$\cG[^\emptyset_\bL=\cM$, in which $\bZ$ violates the AC. We show
that $(a)$ if $\bZ \cap \causpaths\neq \emptyset$ in $\cG$
then $\bZ \cap \causpaths \neq \emptyset$ in $\cM$ 
as well, or there exists a proper non-causal 
path in $\cM$ that cannot be $m$-separated; and 
$(b)$ if $\bZ \cap \causpaths = \emptyset$ in $\cG$ 
and $\bZ$ $d$-connects a proper
non-causal path in $\cG$, then it
$m$-connects a proper non-causal path in $\cM$.

$(a)$ 
Suppose that in $\cG$, $\bZ$ contains a node $Z$ in $\causpaths$,
and let $\bW = \textit{PCP}(\bX, \bY)\cap\textit{An}(Z)$.
If $\cM$ still contains at least 
one node $W_1  \in  \bW$,
then $W_1$ lies on a proper causal path in $\cM$ and 
$Z$ is a descendant of $W_1$ in 
$\cM$.
Otherwise, $\cM$ must contain a node $W_2 \in 
\textit{PCP}_\cG(\bX, \bY) \setminus 
\textit{An}(Z)$ (possibly 
$W_2 \in \bY$) such that $W_2 \leftrightarrow A$,
$X \to W_2$, and $X \to A$ are edges in $\cM$, where
$A \in \textit{An}(Z)$
(possibly $A=Z$; see Figure~\ref{fig:agadjust}). Then $\cM$ contains
an $m$-connected proper non-causal
path $X \to A \leftrightarrow W_2 \to \ldots \to Y$.

\begin{figure}
\begin{center}
\begin{tabular}{cc}
\begin{tikzpicture}[xscale=1.2]
\node (g) at (-1.5,-0.5) {DAG $\cG$:};
\node (x) at (0,0) {$X$};
\node (m2) at (1,-1) {$W_1$};
\node (m3) at (2,-1) {$W_2$};
\node (y) at (2,0) {$Y$};
\node [adjusted] (w) at (0,-1) {$Z$};
\draw [->] (x) -- (m2);
\draw [->] (m2) -- (m3);
\draw [->] (m3) -- (y);
\draw [->] (m2) -- (w);
\end{tikzpicture}\hspace*{15mm}
& 
\begin{tikzpicture}[xscale=1.2]
\node (g) at (-1.5,-0.5) {MAG $\cM=\cG[^\emptyset_{W_1}$:};
\node (x) at (0,0) {$X$};
\node (m3) at (2,-1) {$W_2$};
\node (y) at (2,0) {$Y$};
\node [adjusted] (w) at (0,-1) {$Z$};
\draw [->] (x) -- (m3);
\draw [->] (x) -- (w);
\draw [->] (m3) -- (y);
\draw [<->] (m3) -- (w);
\end{tikzpicture}
\end{tabular}
\end{center}
\caption{
Illustration of the case in the proof
of Theorem \ref{thm:agadjust} where $Z$ descends
from $W_1$ which in a DAG $\cG$ is on a proper causal path from $X$ to $Y$,
but is not a descendant of a node on a proper 
causal path from $X$ to $Y$ in the MAG $\cM$
after marginalizing $W_1$.
In such cases, conditioning on $Z$ will $m$-connect
$X$ and $Y$ in $\cM$ via a proper non-causal path.}
\label{fig:agadjust}
\end{figure}

$(b)$ Suppose that in $\cG$, $\bZ\cap\causpaths=\emptyset$,
and there exists an open proper non-causal path from $\bX$
to $\bY$. Then there must also
be a proper non-causal \emph{walk} $w_\cG$ from some $X \in \bX$ 
to some $Y \in \bY$ (Lemma~\ref{lemma:auxiliary:walkpathconversion}),
which is $d$-connected by $\bZ$ in $\cG$.  
Let $w_\cM$ denote the subsequence of $w_\cG$ formed 
by nodes in $\cM$. It includes all colliders on $w_\cG$, because $\bZ \cap \bL=\emptyset$. 
%{\bf \color{red} TODO: nicht klar. $A\to B \to C \to Z \gets C \gets B \gets D$. Ist $Z$ der einzige colleider? oder auch  $B$? Um $w_\cM$ zu bilden: nicht fergessen End-nodes!}.
The sequence $w_\cM$ is a walk in $\cM$, but is not
necessarily $m$-connected by $\bZ$;
all colliders on $w_\cM$ are in $\bZ$ because every non-$\bZ$ 
must be a parent of at least one of its neighbors, 
but there might be subsequences $U,Z_1,\ldots,Z_k,V$ on $w_\cM$ 
where all $Z_i \in \bZ$ but some of the $Z_i$ are not colliders
on $w_\cM$. However, then we can form from $w_\cM$ an $m$-connected
walk by bypassing some sequences of $\bZ$-nodes 
(Lemma~\ref{lemma:auxiliary:truncatemaxsub}). 
Let $w_\cM'$ be the resulting walk.

If $w_\cM'$ is a proper non-causal walk, then 
there must also exist a proper non-causal
path in $\cM$ (Lemma~\ref{lemma:auxiliary:walkpathconversion}),
violating the AC.
It therefore remains to show that 
$w_\cM'$ is not a proper causal path. This must be
the case if $w_\cG$ does not contain colliders, because then
the first edge of $w_\cM=w_\cM'$ cannot be a visible directed edge
out of $X$. Otherwise, the only way for $w_\cM'$ to be proper
causal is if all $\bZ$-nodes in $w_\cM$ have been bypassed
in $w_\cM'$ by edges pointing away from $\bX$.
In that case, one can show by several case distinctions 
that the first edge $X \to D$ of $w_\cM'$, 
where $D \notin \bZ$, cannot be visible
(see Figure~\ref{fig:agadjust:b} for an example of such a case).

\begin{figure}[h]
 
\begin{center}
\begin{tabular}{cc}
\begin{tikzpicture}[xscale=1.2]
\node (g) at (-1.5,-0.8) {DAG $\cG$:};
\node (l1) at (0.2,0) {$L_1$};
\node [adjusted] (z) at (1,0) {$Z$};
\node (y) at (2,-1) {$Y$};
\node (l2) at (1.8,0) {$L_2$};
\node (x) at (0,-1) {$X$};
\node (a) at (0.6,-1.6) {$A$};
\draw [->] (l1) -- (z);
\draw [->] (z) -- (y);
\draw [->] (l2) -- (y);
\draw [->] (z) -- (x);
\draw [->] (x) -- (y);
\draw [->] (l2) -- (z);
\draw [->] (l1) -- (x);
\draw [->] (a) -- (x);
\end{tikzpicture}\hspace*{15mm}
& 
\begin{tikzpicture}[xscale=1.2]
\node (g) at (-1.5,-0.8) {MAG $\cM=\cG[^\emptyset_{\{L_1,L_2\}}$:};
\node [adjusted] (z) at (1,0) {$Z$};
\node (y) at (2,-1) {$Y$};
\node (x) at (0,-1) {$X$};
\node (a) at (0.6,-1.6) {$A$};
\draw [->] (z) -- (y);
\draw [->] (z) -- (x);
\draw [->] (x) -- (y);
\draw [->] (a) -- (x);
\draw [->] (a) -- (y);
\end{tikzpicture}
\end{tabular}
\end{center}
\caption{Case (b) in the proof of Theorem~\ref{thm:agadjust}:
A proper non-causal path $w_\cG=X \gets L_1 \to Z \gets L_2 \to Y$ 
in a DAG 
is $d$-connected 
by $\bZ$, but the corresponding proper non-causal path 
$w_\cM=X \gets Z \to Y$ is not $m$-connected in the
MAG, and its $m$-connected subpath $w_\cM'=X \to Y$ is 
proper causal. However, this also renders the edge $X \to Y$
invisible, because otherwise $A$ could be $m$-separated from 
$Y$ by $\bU=\{X,Z\}$ in $\cM$ but not in $\cG$.
}
\label{fig:agadjust:b} 
\end{figure}

For simplicity, 
assume that $\cM$ contains a subpath $A \to X \to D$
where $A$ is not adjacent to $D$; the other cases of
edge visibility like $A \leftrightarrow X \to D$
(Lemma~\ref{lemma:pureedge})
are treated analogously. In $\cG$, there are inducing
paths (possibly several  %, see Definition~\ref{def:inducing:path}) 
$\pi_{AX}$ from $A$ to $X$ and $\pi_{XD}$ from $X$ to $D$
w.r.t $\emptyset,\bL$; $\pi_{AX}$ must have an arrowhead at $X$.
We distinguish several cases on the shape of $\pi_{XD}$.
(1) A path $\pi_{XD}$ has an arrowhead at $X$ as well.
Then $A,D$ are adjacent (Lemma~\ref{lemma:inducing:adjacent:when:visible}),
a contradiction.
(2) No inducing path $\pi_{XD}$ has an arrowhead at $X$. 
Then $w_\cG$ must start with an arrow out of $X$,
and must contain a collider $Z \in \textit{De}(X)$ because
$w_\cG$ is not causal. 
(a) $Z \in \textit{De}(D)$. This contradicts
$\bZ\cap\causpaths=\emptyset$. So 
(b) $Z \notin \textit{De}(D)$. Then by construction of $w_\cM'$
(Lemma~\ref{lemma:auxiliary:truncatemaxsub}), 
$w_\cM$ must start with an inducing $\bZ$-trail 
$X \to Z,Z_1,\ldots,Z_n,D$, which is also an inducing path
from $X$ to $D$ in $\cG$ w.r.t. $\emptyset,\bL$.
Then $Z,Z_1,\ldots,Z_n,D$ must also be an inducing path
in $\cG$ w.r.t. $\emptyset,\bL$ because 
$\textit{An}(X)\subseteq\textit{An}(Z)$. Hence $Z$ and $D$
are adjacent. We distinguish cases on the path 
$X \to Z,D$ in $\cM$. Now we can conclude:  If $X\to Z \to D$,
then $Z$ lies on a proper causal path, contradicting
$\bZ\cap\causpaths=\emptyset$; If $X\to Z \leftrightarrow D$,
or $X\to Z \gets D$,
then we get an $m$-connected proper 
non-causal walk along $Z$ and $D$.
\end{proof}

\subsection{Adjustment set construction}

In the previous section, we have already shown 
that the CBC is equivalent to the AC
for MAGs as well; hence, adjustment sets for a 
given MAG $\cM$ can be found by forming the proper
back-door graph $\pbdm$ and then applying the algorithms
from the previous section. In principle, care must
be taken when removing edges from MAGs as the result
might not be a MAG; however, this is not the case
when removing only directed edges.

\begin{lemma}[Closure of maximality under removal
of directed edges]
Given a MAG $\cM$, every graph $\cM'$ formed by removing
only directed edges from $\cM$ is also a MAG.
\end{lemma}
\begin{proof}
Suppose the converse, i.e., $\cM$ is no longer a 
MAG after removal of some edge $X \to D$. Then
$X$ and $D$ cannot be $m$-separated even 
after the edge is removed because $X$ and $D$
are collider connected via a path 
whose nodes are all ancestors 
of $X$ or $D$ \citep{Richardson2002}.
The last edge on this path must be 
 $C\leftrightarrow D$ or $C \gets D$, hence
 $C \notin \textit{An}(D)$, and thus 
we must have $C \in \textit{An}(X)$.
 But then we get $C \in \textit{An}(D)$ in 
 $\cM$ via the edge $X \to D$, a contradiction.
\end{proof}
\begin{corollary}
For every MAG $\cM$, the proper back-door graph
$\pbdm$ is also a MAG.
\end{corollary}

For MAGs that are not adjustment amenable, the CBC
might falsely indicate that an adjustment
set exists even though that set may not be valid for some
represented graph. Fortunately, adjustment amenability 
is easily tested using the graphical criteria
of Lemma~\ref{lemma:pureedge}.
For each child $D$ of $\bX$ in
$\pcauspaths$, we can test the visibility of all edges $\bX \to D$
simultaneously using depth first search. This means
that we can check all potentially problematic edges in 
time $\cO(n+m)$. If all tests pass, we are licensed
to apply the CBC, as shown above.
Hence, we can solve all algorithmic tasks in Table~\ref{fig:problemsadj} 
for MAGs in the same way as for DAGs after an $\cO(k(n+m))$
check of adjustment amenability, where $k\leq|\textit{Ch}(\bX)|$.

Hence our algorithms can construct an adjustment set for a given MAG $ \cM $ and variables $ \bX,\bY $ in $ O((k+1)(n+m)) $ time. If an additional set $ \bZ $ is given, it can be verified that $ \bZ $ is an adjustment set in the same  time.

Minimal adjustments sets can be constructed and verified in $ O(k(n+m) + n^2) $ or $ O(n(n+m)) $ time using our algorithms {\sc FindMinAdj}, {\sc TestMinAdj} for dense or sparse graphs. The algorithms for the remaining problems of finding a minimum cost adjustment set {\sc FindMinCostAdj} and enumerating adjustment sets {\sc ListAdj} or {\sc ListMinAdj} in MAGs have the same runtime as the corresponding algorithms in DAGs, since their time surpasses the time required for the adjustment amenability test.

\begin{newpartinrevision}[r2c7]

\subsection{Empirical analysis} % of identifiability by adjustment}

\begin{figure}\arxivonly{\small}
%10, 1, 1; 1; 1,  [v10], [v5], [v3,v4,v6,v9], {v1: [v4,v5], v2: [v3,v7,v8], v3: [v8,v9], v4: [v6,v10], v5: [v9], v6: [v7,v9], v7: [v10], v8: [v9], v9: [], v10: []}
\centering DAG $\cG$ and MAG $\cG[^\emptyset_\emptyset$: \begin{tikzpicture}
\graph  { 
 /-!-v3[lat,as=$V_3$],
 v2[as=$V_2$] -> v8[as=$V_8$] -!- /-!-v9[lat,as=$V_9$] ,
 v7[as=$V_7$]  <- v6[lat,as=$V_6$] -!- v4[lat,as=$V_4$] <- v1[as=$V_1$] -> v5[res,as=$Y_1$],
 /-!-v10[exp,as=$X_1$] -!- v9[lat,as=$V_9$] -!- v8[res,as=$Y_1$],
 {v1-> {v4,v5}, v2-> {v3,v7,v8}, v3-> {v8,v9}, v4-> {v6,v10}, v5-> {v9}, v6-> {v7,v9}, v7-> {v10}, v8-> {v9}, v9-> {}, v10-> {}}
 %v1[exp]-> {v2[res],v3,v7[exp]}, v2-> {v6[lat],v8[res]}, v3-> {v4[res],v6,v8}, v4-> {v8,v9[lat]}, v5[lat]-> {v9,v10[exp]}, v6-> {}, v7-> {}, v8-> {}, v9-> {}, v10-> {}
 };
\end{tikzpicture} \hspace*{10mm}
MAG $\cG[_\bL^\emptyset$:
\begin{tikzpicture}
\graph  { 
 v2[as=$V_2$] -> v8[as=$V_8$],
 /-!-            v7[as=$V_7$]  <- v1[as=$V_1$] -> v5[res,as=$Y_1$],
 /-!-  v10[exp,as=$X_1$],
 v1 -> v10,
 v1 -> v5,
 v1 -> v7,
 v2 -> v10,
 v2 -> v7,
 v2 -> v8,
 v7 -> v10
 };
\end{tikzpicture}

\caption{Example of a DAG and resulting MAG sampled for the parameter tuple $n = 10,  P(\textit{edge}) = 2/9,  P(\textit{unobserved}) = 0.75$, 
and $k= |\bX|=|\bY|=1$. 
Nodes are relabeled such that exposures are called $X_1$ and  outcomes are called $Y_1$. The nodes $\bL = \{V_3, V_4, V_6, V_9\}$ are unobserved. 
%, and all 
%nodes except $V_0$ are unobserved.
$\{V_1\}$ is an adjustment set in $\cG, \cG[^\emptyset_\emptyset$  and $ \cG[_\bL^\emptyset$.
}\label{fig:experiments:mags:example}
\end{figure}

We test for a DAG $\cG=(\bV,\bE)$, generated as described  in Section~\ref{sec:experiments}, whether the causal effect in the MAGs  $\cG[^\emptyset_\emptyset$ and $\cG[_\bL^\emptyset$ can be identified via  adjustment. Hereby $\bL = \bV\setminus\bR$ is the set of unobserved nodes, which is used to determine the MAG $\cG[_\bL^\emptyset$. Thus, all nodes  of $\cG[_\bL^\emptyset$ are considered as observed and all of them are allowed to occur in adjustment sets. The MAG  $\cG[_\emptyset^\emptyset$ is syntactically equal to $\cG$ and $\bL$ specifies forbidden nodes for adjustments. % in $\cG[^\emptyset_\emptyset$.
The MAGs  $\cG[_\bL^\emptyset$ can be constructed according to Definition~\ref{def:mags}, 
however we have implemented the DAG-to-MAG conversion algorithm from \cite{Zhang2008} in DAGitty, 
which is based on Theorem 4.2 in \citep{Richardson2002} that two nodes $A, B$ are adjacent 
in $\cG[^\emptyset_\bL$ if and only if there exists an inducing path  between $A$ and $B$  
with respect to $\emptyset, \bL$ (see Definition~\ref{def:inducing:path}), because testing if an inducing path exists appears easier than testing if two nodes are $d$-separable by any set. 
%If $\bL=\emptyset$, all inducing paths have length 1, so $\cG[_\emptyset^\emptyset$ has the same nodes and edges as $\cG$.
 
Since $\cG$ and $\cG[^\emptyset_\emptyset$ are syntactically equal,
the only difference between using the CBC in a DAG $\cG$ and the CBC in 
$\cG$ interpreted as a  MAG is that the last case needs a test for adjustment amenability of the graph. 
Figure~\ref{fig:experiments:mags:example} shows an example.

%$\cG[_\bL^\emptyset$
% For this we apply to each DAG the  as implemented in DAGitty. On the resulting MAGs we apply the CBC.
%
% shows that 

The results of our experiments are shown in Table~\ref{table:mags:0} and \ref{table:mags:0.75}. 
As expected we find fewer adjustment sets in the MAGs than in the DAGs, since any adjustment 
set found in the MAG is valid for all represented DAGs.
There are also always fewer adjustment sets in $\cG[_\bL^\emptyset$ than in $\cG[^\emptyset_\emptyset$, because with fewer nodes and edges in $\cG[_\bL^\emptyset$ there are fewer visible edges and $\cG[_\bL^\emptyset$ might not be adjustment amenable, even if $\cG[^\emptyset_\emptyset$ is.
For example in the DAG $\cG=L \to X \to Y$ and the MAG $\cG[^\emptyset_\emptyset = L \to X \to Y$ 
the empty set is an adjustment set.  
However, in the MAG $\cG[_L^\emptyset = X \to Y$ there exists no valid adjustment set. %So we can always find more adjustment sets in $\cG^\emptyset_\emptyset$ than in $\cG_\bL^\emptyset$.

For $n=10$, $l=10\ (20)$ we have $P(\textit{edge})=\max\{10/9,1\}=1$, so the experiments generate only \emph{complete} DAGs, such that %$\cG$ 
all nodes are adjacent in the DAG and corresponding MAGs.
This implies that there is either an edge $X \gets Y$, in which case adjustment is impossible in both the DAG and the MAG, or there is an edge $X \to Y$, which then would need to be visible for adjustment in the MAG due to Lemma~\ref{lemma:amenable}. However, there are no visible edges in complete graphs, since a visible edge $X\to Y$ would require a node not adjacent to $Y$. Thus, no adjustment sets are found for these parameters.

The runtimes of these experiments are listed in Table~\ref{table:mags:runtimes}.
The time required to find the adjustment set in the MAG is similar to finding it in 
the original DAG and quick enough to be negligible. We did not run the 
experiments on even larger graphs since constructing a MAG from a DAG
has quadratic runtime complexity in the worst case (as a quadratic number of 
edges needs to be added), which makes constructing the MAGs too slow. 

\end{newpartinrevision}

\input{experiments-mags-plot.tex}

\input{experiments-mags-table.tex}

\input{experiments-mags-runtimes.tex}

\color{black}

\subsection{Auxiliary lemmas for proof of Theorem~\ref{thm:agadjust}} %Section \ref{sec:magadjust}}
\label{sec:aux:lemmas}

In this section, we present the auxiliary lemmas used 
in the proof of Theorem~\ref{thm:agadjust}. For Lemmas~\ref{lemma:auxiliary:truncatemaxsub} and~\ref{lemma:inducing:adjacent:when:visible}, we also give separate preparing claims and existing results before stating the lemmas themselves.

\begin{lemma}
Given a DAG $\cG$ and sets $\bX, \bY, \bZ\subseteq \bV$ satisfying
$\bZ\cap\causpaths=\emptyset$, $\bZ$ $m$-connects 
a proper non-causal path
between $\bX$ and $\bY$ if and
only if it $m$-connects a proper non-causal
walk between $\bX$ and~$\bY$.
\label{lemma:auxiliary:walkpathconversion}
\end{lemma}
\begin{proof}
$\Leftarrow$: Let $w$ be the $m$-connected proper non-causal walk. 
It can be transformed to an $m$-connected path $\pi$ by removing loops of nodes
that are visited multiple times. Since no nodes have been added,
$\pi$ remains proper, and the first edges of $\pi$ and $w$ are the same.
So if $w$ does not start with a $\to$ edge, $\pi$ is non-causal. If $w$ starts
with an edge $X \to D$, there exists a collider with a descendant in $\bZ$ 
which is in $\textit{De}(D)$. So $\pi$ has to be non-causal, or it
would contradict $\bZ\cap\causpaths=\emptyset$.

$\Rightarrow$: Let $\pi$ be an $m$-connected proper non-causal path. It can be
changed to an $m$-connected walk $w$ by inserting 
$C_i \to \ldots \to Z_i \gets \ldots \gets C_i$ for every collider $C_i$ on 
$\pi$ and a corresponding $Z_i \in \bZ$. 
Since no edges are removed from $\pi$,
$w$ is non-causal, but not necessarily proper, since the 
inserted walks might contain 
nodes of $\bX$. However, in that case, $w$ can be truncated to a proper walk 
$w'$ starting at the last node of $\bX$ on $w$. Then $w'$ is non-causal, since 
it contains the subpath $\bX \gets \ldots \gets C_i$.
\end{proof}

In all of the below, $\cG=(\bV,\bE)$ is a DAG,
$\bZ,\bL \subseteq \bV$ are disjoint, 
and $\cM=\cG[^\emptyset_\bL$.
We notice first, that every inducing path w.r.t.
$\bZ$ and $\bL$ is $m$-connected by 
$\bZ$.

% \begin{definition}[Inducing walk]
% A walk 
% $W=V_1,\ldots,V_{n+1}$ is called
% \emph{inducing walk} with respect to $\bZ, \bL$ 
% if all non-colliders on $W$
% except $V_1$ and $V_{n+1}$ are in $\bL$,
% and all colliders on $W$ are in $\bZ$.
% \end{definition}
% We can always get inducing paths from inducing walks, though the
% converse is not true.
% 
% \begin{lemma}
% \label{lemma:auxiliary:indwalkstoindpaths}
% If there is an inducing walk $W$ from
% $U \in \bV$ to $V \in \bV$ with respect
% to  $\bZ, \bL$ then there is an inducing path
% $\pi$ from $U \in \bV$ to $V \in \bV$ with respect
% to  $\bZ, \bL$.
% \end{lemma}
% 
% \begin{proof}
% Transform $W$ into a path $\pi$ by
% removing loops of nodes visited twice. Then
% $\pi$ is $m$-connected by $\bZ$. Hence
% non-colliders on $\pi$
% are not in $\bZ$, so they must be non-colliders 
% on $W$ as well and thus be in $\bL$. 
% Colliders on $\pi$ must be in $\textit{An}(\bZ)$.
% So $\pi$ is inducing w.r.t. $\bZ, \bL$.
% \end{proof}

\begin{lemma}[\citet{Richardson2002}]
\label{lemma:inducing:nosep}
If there is an inducing path
$\pi$ from $U \in \bV$ to $V \in \bV$ with respect
to  $\bZ,\bL$, then there exists no set $\bZ'$ with 
$\bZ \subseteq \bZ' \subseteq (\bV\setminus\bL)$ 
such that $\bZ'$ $d$-separates $U$ and $V$ in $\cG$ or
$m$-separates $U$ and $V$ in $\cG[^\emptyset_\bL$.
\end{lemma}

\begin{proof}
This is Theorem~4.2, cases (v) and (vi), in~\cite{Richardson2002}.
\end{proof}

\begin{claim}
\label{lemma:inducing:bypass}
Two nodes $U,V$ are 
adjacent in $\cG[^\emptyset_\bL$ if and only if $\cG$ contains 
an inducing path $\pi$ between $U$ and $V$ with respect to $\emptyset,\bL$. 
Moreover, the edge between $U,V$ in  $\cG[^\emptyset_\bL$ 
can only have an arrowhead at 
$U$ ($V$) if all such $\pi$ have
an arrowhead at $U$ ($V$) in $\cG$.
\end{claim}

\begin{proof}
The first part on adjacency is proved in~\citep{Richardson2002}. For the 
second part on arrowheads, suppose $\pi$ does not have an arrowhead at $U$, 
then $\pi$ starts with an edge $U \to D$. Hence $D \notin \textit{An}(U)$, so
$D \in \textit{An}(V)$ because $\pi$ is an inducing path
and therefore also $U \in \textit{An}(V)$. Hence, the edge between $U$ and $V$
in $\cG[^\emptyset_\bL$ must be $U \to V$. The argument for $V$ is identical.
\end{proof}

\begin{claim}
\label{lemma:inducing:concat}
Suppose $Z_0,Z_1,Z_2$ is a path in $\cG[^\emptyset_\bL$ 
on which $Z_1$ is a non-collider. Suppose an
inducing path $\pi_{01}$ from $Z_0$ to $Z_1$ w.r.t. $\emptyset,\bL$ 
in $\cG$ has an arrowhead at $Z_1$, and an
inducing path $\pi_{12}$ from $Z_1$ to $Z_2$ w.r.t. $\emptyset,\bL$ 
has an arrowhead at $Z_1$. Then the walk
$w_{012} = \pi_{01}\pi_{12}$ can be 
truncated to an inducing path from
$Z_0$ to $Z_2$ w.r.t. $\emptyset,\bL$ in $\cG$.
\end{claim}

\begin{proof}
The walk $w_{012}$ does not contain more non-colliders than
those on $\pi_{01}$ or $\pi_{12}$, so they must all be in $\bL$.
It remains to show that the colliders on $w_{012}$ are 
in $\textit{An}(Z_0\cup Z_2)$. Because $Z_1$ is not a collider
on $Z_0,Z_1,Z_2$, at least one of the edges $Z_0,Z_1$ 
and $Z_1,Z_2$ must be a directed edge pointing away from $Z_1$. Assume
without loss of generality that $Z_0 \gets Z_1$ is that
edge. Then all colliders on $\pi_{01}$ are in 
$\textit{An}(Z_0 \cup Z_1)=\textit{An}(Z_0) \subseteq \textit{An}(Z_0 \cup Z_2)$,
and all colliders on $\pi_{12}$ are in $\textit{An}(Z_1 \cup Z_2)
\subseteq \textit{An}(Z_0 \cup Z_2)$. $Z_1$ itself is 
a collider on $w_{012}$ and is also in $\textit{An}(Z_0)$.
Hence, the walk $w_{012}$ is $d$-connected,
and can be truncated to an inducing path that starts
with the first arrow of $\pi_{01}$ and ends
with the last arrow of $\pi_{12}$.
\end{proof}

\begin{claim}
\label{lemma:inducing:invariance}
Let $\pi=V_1,\ldots,V_{n+1}$ be an inducing $\bZ$-trail, 
and let $\pi'$ be a subsequence of $\pi$ formed by removing
one node $V_i$ of $\pi$ such that $V_i \in \bZ$ is
a non-collider on $\pi$. Then $\pi'$ is an inducing $\bZ$-trail.
\end{claim}

\begin{proof}
According to Claim~\ref{lemma:inducing:concat}, 
if $V_i$ is a non-collider on $\pi$, then $V_{i-1}$ and
$V_{i+1}$ are linked by an inducing path $\pi$ that contains
an arrowhead at $V_{i-1}$ ($V_{i+1}$) if $V_{i-1}\in\bZ$ ($V_{i+1}\in\bZ$).
Therefore,  $V_{i-1}$ and $V_{i+1}$ are themselves adjacent,
$\pi'$ is a path, and is a $\bZ$-trail.  
\end{proof}

\begin{corollary}\label{cor:inducing:z:trail:has:subpath}
Every inducing $\bZ$-trail $\pi=V_1,\ldots,V_{n+1}$
has a subpath $\pi'$ that is
$m$-connected by $\bZ$.
\end{corollary}

\begin{proof}
Transform $\pi$ into $\pi'$ by replacing non-collider nodes
in $\bZ$ by the direct edge linking their neighbors until
no such node exists anymore. By inductively applying 
Claim~\ref{lemma:inducing:invariance}, we see that $\pi'$
is also an inducing $\bZ$-trail, and 
every node in $\bZ$ is a collider because otherwise we
would have continued transforming.
So $\pi'$ must be $m$-connected by $\bZ$.
\end{proof}

\begin{lemma}
\label{lemma:auxiliary:truncatemaxsub}
Let $w_\cG$ be a walk from $X$ to $Y$ in $\cG$,
$X, Y \notin \bL$, 
that is $d$-connected by $\bZ$.
Let $w_\cM=V_1,\ldots,V_{n+1}$ 
be the subsequence of $w_\cG$ consisting only of the nodes in 
$\cM=\cG[^\emptyset_\bL$.
Then $\bZ$ $m$-connects $X$ and $Y$ in $\cM$  via a path along a 
subsequence $w_\cM'$ formed from $w_\cM$ by removing some nodes
in $\bZ$ (possibly $w_\cM'=w_\cM$). 
\end{lemma}

\begin{proof}
First, truncate from $w_\cM$ all
subwalks between nodes in $\bZ$ that occur more than once.
Now consider all subsequences $V_1,\ldots,V_{n+1}$, $n>1$,
of $w_\cM$ where $V_2,\ldots,V_n \in \bZ$, $V_1,V_{n+1} \notin \bZ$,
which now are all paths in $w_\cM$.
On those subsequences, every $V_i$ must be adjacent in $\cG$ to  $V_{i+1}$
via a path containing no colliders, and all non-endpoints on that path
must be in $\bL$. So there are inducing paths w.r.t. $\emptyset,\bL$ between
all $V_{i}, V_{i+1}$, which have arrowheads at 
$V_i$ ($V_{i+1}$) if $V_i\in\bZ$ ($V_{i+1}\in\bZ$). So 
$V_1,\ldots,V_{n+1}$ is an inducing $\bZ$-trail, and has a subpath
which $m$-connects $V_1$, $V_{n+1}$ given $\bZ$ due to Corollary~\ref{cor:inducing:z:trail:has:subpath}. Transform 
$w_\cM$ to $w_\cM'$ by replacing all inducing $\bZ$-trails by
their $m$-connected subpaths. According to
Claim~\ref{lemma:inducing:bypass}, non-colliders on $w_\cM$ 
cannot be colliders on $w_\cM'$, as bypassing inducing paths
can remove but not create arrowheads.
Moreover, all nodes in $\bZ$ on
$w_\cM'$ are colliders. Hence $w_\cM'$ is $m$-connected by
$\bZ$.
\end{proof}

\begin{corollary}
Each edge on $w_\cM'$ as defined above corresponds to 
an inducing path w.r.t $\emptyset,\bL$ in $\cG$ along
nodes on $w_\cG$.
\end{corollary}

\begin{claim}\label{lemma:inducing:concat:collider}
Suppose there exists an inducing path $\pi_{01}$ from $Z_0$ to $Z_1$ w.r.t. $\bS, \bL$ with an arrowhead at $Z_1$ and an inducing path from $Z_1$ to $Z_2$  w.r.t. $\bS', \bL$ with an arrowhead at $Z_1$.  Then the walk
$w_{012} = \pi_{01}\pi_{12}$ can be 
truncated to an inducing path from
$Z_0$ to $Z_2$ w.r.t. $\bS \cup \bS' \cup \{Z_1\},\bL$ in $\cG$.
\end{claim}
\begin{proof}
The walk $w_{012}$ does not contain more non-colliders than
those on $\pi_{01}$ or $\pi_{12}$, so they must all be in $\bL$.
All colliders on $\pi_{0,1}$ and $\pi_{1,2}$ as well as $Z_1$ are in $\textit{An}(Z_0, Z_{1}, Z_2,\bS,\bS')$,
and therefore also all colliders of $w_{012}$.

Hence, the walk $w_{012}$ is $d$-connected,
and can be truncated to an inducing path that starts
with the first arrow of $\pi_{01}$ and ends
with the last arrow of $\pi_{12}$.
\end{proof}
 
\begin{claim}\label{lemma:inducing:concat:colliderseq}
Suppose $Z_0,Z_1,\ldots,Z_{k+1}$ is a path in $\cG[^\emptyset_\bL$ 
with an arrowhead at $Z_{k+1}$ on which all $Z_1, \ldots, Z_k$ are colliders. 
Then there exists an inducing path from $Z_0$ to $Z_{k+1}$ w.r.t.  $\{Z_1, \ldots, Z_k\},\bL$ with an arrowhead at $Z_{k+1}$.
%Suppose an inducing path $\pi_{01}$ from $Z_0$ to $Z_1$ w.r.t. $\emptyset,\bL$ 
%in $\cG$ has an arrowhead at $Z_1$, and an
%inducing path $\pi_{12}$ from $Z_1$ to $Z_2$ w.r.t. $\emptyset,\bL$ 
%has an arrowhead at $Z_1$. Then the walk
%$w_{012} = \pi_{01}\pi_{12}$ can be 
%truncated to an inducing path from
%$Z_0$ to $Z_2$ w.r.t. $\{Z_1\},\bL$ in $\cG$.
\end{claim}

\begin{proof}
Because all $Z_i, Z_{i+1}$ are adjacent and all $Z_1, \ldots, Z_k$ are colliders 
there exist inducing paths $\pi_{i,i+1}$ w.r.t. $\emptyset, \bL$ from $Z_i$ to $Z_{i+1}$ that have   arrowheads at $Z_1, \ldots, Z_k$ 
(Claim~\ref{lemma:inducing:bypass}).
The claim follows by repeatedly applying Claim~\ref{lemma:inducing:concat:collider} to the $\pi_{i,i+1}$'s.
\end{proof}

\begin{lemma}
\label{lemma:inducing:adjacent:when:visible}
Suppose $A \to V_1 \leftrightarrow \ldots \leftrightarrow V_k \leftrightarrow X \to D$ or $A \leftrightarrow V_1 \leftrightarrow \ldots \leftrightarrow V_k \leftrightarrow X \to D$ is a path in $\cG[^\emptyset_\bL$ (possibly $k=0$), each $V_i$ is a parent of $D$ %there exists inducing paths from $V_i$ to $D$ w.r.t $\emptyset, \bL$  that only have one arrowhead at $D$ 
and there exists an inducing path $\pi_{XD}$ from $X$ to $D$ w.r.t $\emptyset, \bL$ that has arrowheads on both ends.
Then $A$ and $D$ cannot be $m$-separated in  $\cG[^\emptyset_\bL$.
\end{lemma}

\begin{proof}
Assume the path is $A \to V_1 \leftrightarrow \ldots \leftrightarrow V_k \leftrightarrow X \to D$. 
The case where the path starts with $A \leftrightarrow V_1$ can be handled identically, 
since the first arrowhead does not affect $m$-separation.

Assume $A$ and $D$ can be $m$-separated in  $\cG[^\emptyset_\bL$, and let $\bZ$ be such a separator. If $V_1$ is not in $\bZ$ then the path $A \to V_1 \to D$ is not blocked, so $V_1 \in \bZ$. Inductively it follows, if $V_i$ is not in $\bZ$, but all $\forall j<i: V_j\in\bZ$  then the path $A \to V_1 \leftrightarrow \ldots \leftrightarrow V_{i-1} \leftrightarrow V_i \to D$ is not blocked, so $V_i \in \bZ$ for all $i$.

%If $k > 0$, 
There exist an inducing path $\pi_{AX}$ from $A$ to $X$ with an arrowhead at $X$  
w.r.t. to $\{V_1,\ldots,V_k\},\bL$ (Claim~\ref{lemma:inducing:concat:colliderseq})
which can be combined with $\pi_{XD}$ to an inducing path from $A$ to $D$ w.r.t.
to $\{V_1,\ldots,V_k,X\},\bL$ (Claim~\ref{lemma:inducing:concat:collider}).

%Let $\pi_{A1}$ be the inducing path from $A$ to $V_1$ with an arrowhead at $V_1$, $\pi_{i,i+1}$ the inducing path from $V_i$ to $V_{i+1}$ with arrowheads at $V_i$ and $ V_{i+1}$, $\pi_{k,X}$ the inducing path from $V_k$ to $X$ with arrowheads at $V_k$ and $X$, $\pi_{X,D}$ the inducing path from $D$ to $X$ with arrowheads at $D$ and $X$, and $\pi_{i,D}$ the inducing path from $V_i$ to $D$ with an arrowhead at $D$. (for $k=0$ we set $\pi_{A1}$ to the inducing path from $A$ to $X$ )

%Let $W = \pi_{A,1}\pi_{1,2}\ldots\pi_{k,X}\pi_{X,D}$. % and $w_{i} = \pi_{A, 1}\pi_{1,2}\ldots\pi_{i-1,i}\pi_{i,D}$. 
%$W$ does not contain more non-colliders than the paths, so all its non-colliders are in $\bL$. %$w_i$ only contains an additional non-collider $V_i$, and all others are in $\bL$. 
%All colliders on $W$ are in $\textit{An}(\{A, V_1,\ldots, V_k,X,D\})$, so $W$ is a d-connected walk, and can be truncated to an inducing path W.r.t. $\{X,V_1,\ldots,V_k\},\bL$.  

Hence no $m$-separator of $A,D$ can contain  $\{X,V_1,\ldots,V_k\}$ 
(Lemma~\ref{lemma:inducing:nosep}). 
Then there
cannot exist an $m$-separator, because every separator must include $V_1,\ldots,V_k$ and the path $A \to V_1 \leftrightarrow V_2 \leftrightarrow \ldots \leftrightarrow V_k \leftrightarrow X \to D$ is open without $X \in \bZ$.
\end{proof}

\section{Discussion}

We provide a framework of efficient algorithms to verify, find, and enumerate $m$-separating sets in MAGs, which we then harness to solve the same problems for adjustment sets in DAGs and MAGs. In both graph classes, this provides a complete and informative answer to the question when, and how, a desired causal effect between multiple exposures and outcomes can be estimated by covariate adjustment. 

For DAGs, our results show that from a computational complexity perspective, there is no disadvantage of using our complete constructive back-door criterion (CBC) instead of Pearl's back-door criterion (BC) -- in other words, using CBC instead of BC gives us the guarantee of completeness ``for free''. Therefore, at least for implementation in software packages, we would see no reason to use BC instead of CBC. Nevertheless, our empirical evaluation also suggests that the number of cases covered by CBC but not by BC might be relatively small.

In contrast to DAGs, MAGs are not widely used to encode causal models, as their semantics are more complex and direct edges do not necessarily correspond to direct causal relationships. Still, our CBC for MAGs can also be used as a form of ``sensitivity analysis'' for researchers performing DAG-based analyses because every DAG can be converted to a MAG -- if it contains no latent variables, it can simply be read as if it were a MAG. If the MAG resulting from that conversion still admits covariate adjustment, then we have shown that the adjustment set postulated for the original DAG is in fact valid for an infinite set of DAGs, namely all those represented by the MAG. This strategy might allow researchers to partly relax the often untenable ``causal sufficiency assumption'' that all relevant variables are known and were measured.

Our results rest on two key concepts: reduction of adjustment to $m$-separation in a subgraph (the \emph{proper back-door graph}), and  \emph{adjustment amenability} for graphical models that are more causally ambiguous than DAGs. Since the publication of the preliminary version of this work \cite{TextorLiskiewicz2011,zander2014constructing}, these techniques were shown to be applicable to adjustment in four additional classes of graphical causal models: CPDAGs \cite{PerkovicEtAl2018}, PAGs \cite{PerkovicEtAl2018}, chain graphs \cite{vanderZander2016separators}, and maximal PDAGs \cite{PerkovicEtAl2017}. Likewise, it has been shown that our algorithms can be applied to extended graphical criteria that allow to deal with selection bias \cite{CorreaBarenboim2017}. As we have illustrated briefly in Section~\ref{sec:basissets}, we expect our algorithmic framework to be useful in other areas as well, due to the central role of $m$-separation in the theory of graphical models. In \cite{ZanderTL15,ZanderL16} we have demonstrated how $d$-separators can be harnessed for efficiently finding generalized instrumental variables: for this purpose we apply constrained separators but using different restrictions than the ones discussed in this paper.

Our empirical analysis in this paper shows that our algorithmic framework is efficient enough to be used in practice, even on large (MAGs) or very large (DAGs) models. The practical feasibility of our algorithms is illustrated by the fact that they underpin both the web application ``dagitty.net'' as well as the associated R package \cite{dagittyIJE}, which currently have a substantial user community. We hope that future work will expand on our initial empirical results, and there are many potential avenues to follow. For instance, we generated random DAGs, but these are likely not representative of ``typical'' causal graphs encountered in practical applications. While our understanding of the ``typical'' structure of causal graphs is currently limited, one could test the robustness of our findings to the graph structure by considering other well-known graph-generating models, such as scale-free~\cite{barabasi_emergence_1999}, small-world~\cite{watts_collective_1998}, or lattice-like~\cite{Ozik2004} networks.

A further interesting open question to be pursued in future research would be whether the approaches presented here could be generalized to accommodate confounding that arose by chance in a given sample, rather than for structural reasons \cite{Greenland2015}.

\section{Acknowledgments}
This work was supported by the Deutsche Forschungsgemeinschaft (DFG) grant LI 634/4-1 and LI 634/4-2.

We also thank Marcel Wienöbst for help in performing the experiments, particularly for implementing the parsing of the generated graph files in R.

\section{References}
\bibliographystyle{elsarticle-num}
\bibliography{main}

%\end{document}
%%% FINAL VERSION ENDS HERE
%\clearpage

%\appendix

\section{Appendix: further experimental results}

% % \subsection{Do-calculus}
% % 
% % Some of our proofs use the rules of the do-calculus \cite{Pearl2009} that exchange actions and observations or remove them altogether. We state those rules here for sake of self-containedness. 
% % 
% % \begin{theorem}[Do-Calculus \cite{Pearl2009}]\label{thm:do-calculus}
% % Given a DAG $ \cG $ and disjoint sets $ \bX,\bY,\bZ,\bW $ the following rules are valid for all probability distributions $ P $ consistent with $ \cG $
% % 
% % \textbf{Rule 1.} (Insertion/deletion of observations) 
% % 
% % \hspace{1em}$P(\by\mid \textit{do}(\bx),\bz,\bw) = P(\by\mid \textit{do}(\bx),\bw)$ if $(\bY \independent \bZ\mid \bX, \bW)$ in $ \cG_{\overline{\bX}} $
% % 
% % \textbf{Rule 2.} (Exchange of actions/observations) 
% % 
% % \hspace{1em}$P(\by\mid \textit{do}(\bx),\textit{do}(\bz), \bw) = P(\by\mid \textit{do}(\bx),\bz,\bw)$ if $ (\bY \independent \bZ\mid \bX,\bW)\text{ in }{\cG_{\overline{\bX}\underline{\bZ}}} $
% % 
% % \textbf{Rule 3.} (Insertion/deletion of actions) 
% % 
% % \hspace{1em}$P(\by\mid \textit{do}(\bx),\textit{do}(\bz),\bw) = P(\by\mid \textit{do}(\bx),\bw)$ if $(\bY\independent \bZ\mid\bX,\bW) \text{ in } {\cG_{\overline{\bX}\ \overline{\bZ(\bW)}}}$ where $ \bZ(\bW) $ is short for  $ \bZ \setminus An_{\overline{\bX}}(\bW) $.
% % 
% % \end{theorem}

%\subsection{Further experimental results}

Tables~\ref{table:global:stat:unobs0.25:count} and~\ref{table:global:stat:unobs0.5:count} 
in this section show the results of versions of the experiments presented in Section~\ref{sec:experiments} in Tables~\ref{table:global:stat:unobs0:count} and~\ref{table:global:stat:unobs0.75:count} in which the parameter controlling 
the number of unobserved variables is set to $0.25$ or $0.5$.

%Figure~\ref{fig:experiments:increasing:blue:yellow} 
%shows how the percentage of identifiable graphs increases depending on their parameters.

\input{experiments-table2}

% % \inputiffinal{experiments1}
% % 
% % \inputiffinal{experimentsplot}
% % 
% % \inputiffinal{experiments2}

%\inputiffinal{experimentsnew}

%% file: experiments-modelchecking.tex
\begin{figure}

{
\footnotesize
\gdef\modelcheckinginput{code/modelchecking/normal.tex}
\arxivonly{\gdef\modelcheckinginput{normal.tex}}
\input{\modelcheckinginput}

}

\caption{Model checking with parental or sparse basis sets for random DAGs. Here, $n$ denotes the number of nodes in the DAG, and $l$ denotes the expected number of neighbours of each node. The left panel shows the total size of the conditioning sets as a function of $n$, showing the expected quadratic increase. The right panel emphasizes that the benefit of using a sparse basis set is greatest if the graph is also sparse, in which case the total number of variables that need to be conditioned on can be reduced by up to 90\%.}
\label{fig:modelchecking}
\end{figure}

%% file: normal.tex
% Created by tikzDevice version 0.10.1 on 2018-01-26 10:48:39
% !TEX encoding = UTF-8 Unicode
\begin{tikzpicture}[x=1pt,y=1pt]
\definecolor{fillColor}{RGB}{255,255,255}
\path[use as bounding box,fill=fillColor,fill opacity=0.00] (0,0) rectangle (231.26,144.54);
\begin{scope}
\path[clip] ( 38.40, 38.40) rectangle (229.34,134.94);
\definecolor{drawColor}{RGB}{0,0,0}

\path[draw=drawColor,line width= 0.4pt,line join=round,line cap=round] ( 50.27, 42.15) -- ( 99.61, 43.95);

\path[draw=drawColor,line width= 0.4pt,line join=round,line cap=round] (109.20, 44.44) -- (217.48, 51.53);

\path[draw=drawColor,line width= 0.4pt,line join=round,line cap=round] ( 45.47, 41.98) circle (  1.80);

\path[draw=drawColor,line width= 0.4pt,line join=round,line cap=round] (104.41, 44.12) circle (  1.80);

\path[draw=drawColor,line width= 0.4pt,line join=round,line cap=round] (222.27, 51.84) circle (  1.80);
\end{scope}
\begin{scope}
\path[clip] (  0.00,  0.00) rectangle (231.26,144.54);
\definecolor{drawColor}{RGB}{0,0,0}

\path[draw=drawColor,line width= 0.4pt,line join=round,line cap=round] ( 80.83, 38.40) -- (222.27, 38.40);

\path[draw=drawColor,line width= 0.4pt,line join=round,line cap=round] ( 80.83, 38.40) -- ( 80.83, 33.60);

\path[draw=drawColor,line width= 0.4pt,line join=round,line cap=round] (127.98, 38.40) -- (127.98, 33.60);

\path[draw=drawColor,line width= 0.4pt,line join=round,line cap=round] (175.13, 38.40) -- (175.13, 33.60);

\path[draw=drawColor,line width= 0.4pt,line join=round,line cap=round] (222.27, 38.40) -- (222.27, 33.60);

\node[text=drawColor,anchor=base,inner sep=0pt, outer sep=0pt, scale=  1.00] at ( 80.83, 21.12) {40};

\node[text=drawColor,anchor=base,inner sep=0pt, outer sep=0pt, scale=  1.00] at (127.98, 21.12) {60};

\node[text=drawColor,anchor=base,inner sep=0pt, outer sep=0pt, scale=  1.00] at (175.13, 21.12) {80};

\node[text=drawColor,anchor=base,inner sep=0pt, outer sep=0pt, scale=  1.00] at (222.27, 21.12) {100};

\path[draw=drawColor,line width= 0.4pt,line join=round,line cap=round] ( 38.40, 41.38) -- ( 38.40,126.55);

\path[draw=drawColor,line width= 0.4pt,line join=round,line cap=round] ( 38.40, 41.38) -- ( 33.60, 41.38);

\path[draw=drawColor,line width= 0.4pt,line join=round,line cap=round] ( 38.40, 58.41) -- ( 33.60, 58.41);

\path[draw=drawColor,line width= 0.4pt,line join=round,line cap=round] ( 38.40, 75.44) -- ( 33.60, 75.44);

\path[draw=drawColor,line width= 0.4pt,line join=round,line cap=round] ( 38.40, 92.48) -- ( 33.60, 92.48);

\path[draw=drawColor,line width= 0.4pt,line join=round,line cap=round] ( 38.40,109.51) -- ( 33.60,109.51);

\path[draw=drawColor,line width= 0.4pt,line join=round,line cap=round] ( 38.40,126.55) -- ( 33.60,126.55);

\node[text=drawColor,rotate= 90.00,anchor=base,inner sep=0pt, outer sep=0pt, scale=  1.00] at ( 26.88, 41.38) {0};

\node[text=drawColor,rotate= 90.00,anchor=base,inner sep=0pt, outer sep=0pt, scale=  1.00] at ( 26.88, 75.44) {20000};

\node[text=drawColor,rotate= 90.00,anchor=base,inner sep=0pt, outer sep=0pt, scale=  1.00] at ( 26.88,109.51) {40000};

\path[draw=drawColor,line width= 0.4pt,line join=round,line cap=round] ( 38.40,134.94) --
	( 38.40, 38.40) --
	(229.34, 38.40);
\end{scope}
\begin{scope}
\path[clip] (  0.00,  0.00) rectangle (231.26,144.54);
\definecolor{drawColor}{RGB}{0,0,0}

\node[text=drawColor,anchor=base,inner sep=0pt, outer sep=0pt, scale=  1.00] at (133.87,  1.92) {$n$};

\node[text=drawColor,rotate= 90.00,anchor=base,inner sep=0pt, outer sep=0pt, scale=  1.00] at (  7.68, 86.67) {sum of conditioning set sizes};
\end{scope}
\begin{scope}
\path[clip] ( 38.40, 38.40) rectangle (229.34,134.94);
\onlyincolor{\definecolor{drawColor}{RGB}{255,0,0}}
\onlyinblackandwhite{\definecolor{drawColor}{RGB}{0,0,0}}

\path[draw=drawColor,line width= 0.4pt,dash pattern=on 4pt off 4pt ,line join=round,line cap=round] ( 50.25, 43.19) -- ( 99.62, 47.37);

\path[draw=drawColor,line width= 0.4pt,dash pattern=on 4pt off 4pt ,line join=round,line cap=round] (109.14, 48.59) -- (217.54, 67.27);

\path[draw=drawColor,line width= 0.4pt,line join=round,line cap=round] ( 45.47, 42.78) circle (  1.80);

\path[draw=drawColor,line width= 0.4pt,line join=round,line cap=round] (104.41, 47.77) circle (  1.80);

\path[draw=drawColor,line width= 0.4pt,line join=round,line cap=round] (222.27, 68.09) circle (  1.80);
\onlyincolor{\definecolor{drawColor}{RGB}{0,205,0}}

\path[draw=drawColor,line width= 0.4pt,dash pattern=on 1pt off 3pt ,line join=round,line cap=round] ( 50.22, 44.00) -- ( 99.66, 51.40);

\path[draw=drawColor,line width= 0.4pt,dash pattern=on 1pt off 3pt ,line join=round,line cap=round] (108.97, 53.60) -- (217.71, 89.16);

\path[draw=drawColor,line width= 0.4pt,line join=round,line cap=round] ( 45.47, 43.29) circle (  1.80);

\path[draw=drawColor,line width= 0.4pt,line join=round,line cap=round] (104.41, 52.11) circle (  1.80);

\path[draw=drawColor,line width= 0.4pt,line join=round,line cap=round] (222.27, 90.65) circle (  1.80);
\onlyincolor{\definecolor{drawColor}{RGB}{0,0,255}}

\path[draw=drawColor,line width= 0.4pt,dash pattern=on 1pt off 3pt on 4pt off 3pt ,line join=round,line cap=round] ( 50.12, 43.61) -- ( 99.76, 56.34);

\path[draw=drawColor,line width= 0.4pt,dash pattern=on 1pt off 3pt on 4pt off 3pt ,line join=round,line cap=round] (108.47, 60.08) -- (218.20,128.82);

\path[draw=drawColor,line width= 0.4pt,line join=round,line cap=round] ( 45.47, 42.42) circle (  1.80);

\path[draw=drawColor,line width= 0.4pt,line join=round,line cap=round] (104.41, 57.54) circle (  1.80);

\path[draw=drawColor,line width= 0.4pt,line join=round,line cap=round] (222.27,131.36) circle (  1.80);
\definecolor{drawColor}{RGB}{0,0,0}

\path[draw=drawColor,line width= 0.4pt,line join=round,line cap=round] ( 50.27, 41.51) -- ( 99.61, 41.79);

\path[draw=drawColor,line width= 0.4pt,line join=round,line cap=round] (109.21, 41.83) -- (217.47, 42.08);
\definecolor{fillColor}{RGB}{0,0,0}

\path[draw=drawColor,line width= 0.4pt,line join=round,line cap=round,fill=fillColor] ( 45.47, 41.48) circle (  1.80);

\path[draw=drawColor,line width= 0.4pt,line join=round,line cap=round,fill=fillColor] (104.41, 41.82) circle (  1.80);

\path[draw=drawColor,line width= 0.4pt,line join=round,line cap=round,fill=fillColor] (222.27, 42.09) circle (  1.80);
\onlyincolor{\definecolor{drawColor}{RGB}{255,0,0}}

\path[draw=drawColor,line width= 0.4pt,dash pattern=on 4pt off 4pt ,line join=round,line cap=round] ( 50.26, 42.76) -- ( 99.61, 45.64);

\path[draw=drawColor,line width= 0.4pt,dash pattern=on 4pt off 4pt ,line join=round,line cap=round] (109.18, 46.39) -- (217.50, 57.25);
\onlyincolor{\definecolor{fillColor}{RGB}{255,0,0}}

\path[draw=drawColor,line width= 0.4pt,line join=round,line cap=round,fill=fillColor] ( 45.47, 42.48) circle (  1.80);

\path[draw=drawColor,line width= 0.4pt,line join=round,line cap=round,fill=fillColor] (104.41, 45.92) circle (  1.80);

\path[draw=drawColor,line width= 0.4pt,line join=round,line cap=round,fill=fillColor] (222.27, 57.73) circle (  1.80);
\onlyincolor{\definecolor{drawColor}{RGB}{0,205,0}}

\path[draw=drawColor,line width= 0.4pt,dash pattern=on 1pt off 3pt ,line join=round,line cap=round] ( 50.22, 43.93) -- ( 99.65, 50.90);

\path[draw=drawColor,line width= 0.4pt,dash pattern=on 1pt off 3pt ,line join=round,line cap=round] (109.00, 52.96) -- (217.68, 85.88);
\onlyincolor{\definecolor{fillColor}{RGB}{0,205,0}}

\path[draw=drawColor,line width= 0.4pt,line join=round,line cap=round,fill=fillColor] ( 45.47, 43.25) circle (  1.80);

\path[draw=drawColor,line width= 0.4pt,line join=round,line cap=round,fill=fillColor] (104.41, 51.57) circle (  1.80);

\path[draw=drawColor,line width= 0.4pt,line join=round,line cap=round,fill=fillColor] (222.27, 87.27) circle (  1.80);
\onlyincolor{\definecolor{drawColor}{RGB}{0,0,255}}

\path[draw=drawColor,line width= 0.4pt,dash pattern=on 1pt off 3pt on 4pt off 3pt ,line join=round,line cap=round] ( 50.12, 43.60) -- ( 99.75, 56.27);

\path[draw=drawColor,line width= 0.4pt,dash pattern=on 1pt off 3pt on 4pt off 3pt ,line join=round,line cap=round] (108.48, 60.00) -- (218.20,128.29);
\onlyincolor{\definecolor{fillColor}{RGB}{0,0,255}}

\path[draw=drawColor,line width= 0.4pt,line join=round,line cap=round,fill=fillColor] ( 45.47, 42.41) circle (  1.80);

\path[draw=drawColor,line width= 0.4pt,line join=round,line cap=round,fill=fillColor] (104.41, 57.46) circle (  1.80);

\path[draw=drawColor,line width= 0.4pt,line join=round,line cap=round,fill=fillColor] (222.27,130.83) circle (  1.80);
\definecolor{drawColor}{RGB}{0,0,0}

\path[draw=drawColor,line width= 0.4pt,line join=round,line cap=round] ( 40.56,125.34) -- ( 54.96,125.34);
\onlyincolor{\definecolor{drawColor}{RGB}{255,0,0}}

\path[draw=drawColor,line width= 0.4pt,dash pattern=on 4pt off 4pt ,line join=round,line cap=round] ( 40.56,115.74) -- ( 54.96,115.74);
\onlyincolor{\definecolor{drawColor}{RGB}{0,205,0}}

\path[draw=drawColor,line width= 0.4pt,dash pattern=on 1pt off 3pt ,line join=round,line cap=round] ( 40.56,106.14) -- ( 54.96,106.14);
\onlyincolor{\definecolor{drawColor}{RGB}{0,0,255}}

\path[draw=drawColor,line width= 0.4pt,dash pattern=on 1pt off 3pt on 4pt off 3pt ,line join=round,line cap=round] ( 40.56, 96.54) -- ( 54.96, 96.54);
\definecolor{drawColor}{RGB}{0,0,0}

\path[draw=drawColor,line width= 0.4pt,line join=round,line cap=round] ( 47.76, 86.94) circle (  1.80);
\definecolor{fillColor}{RGB}{0,0,0}

\path[draw=drawColor,line width= 0.4pt,line join=round,line cap=round,fill=fillColor] ( 47.76, 77.34) circle (  1.80);

\node[text=drawColor,anchor=base west,inner sep=0pt, outer sep=0pt, scale=  1.00] at ( 62.16,121.90) {$l$=2};

\node[text=drawColor,anchor=base west,inner sep=0pt, outer sep=0pt, scale=  1.00] at ( 62.16,112.30) {$l$=5};

\node[text=drawColor,anchor=base west,inner sep=0pt, outer sep=0pt, scale=  1.00] at ( 62.16,102.70) {$l$=10};

\node[text=drawColor,anchor=base west,inner sep=0pt, outer sep=0pt, scale=  1.00] at ( 62.16, 93.10) {$l$=20};

\node[text=drawColor,anchor=base west,inner sep=0pt, outer sep=0pt, scale=  1.00] at ( 62.16, 83.50) {parental basis};

\node[text=drawColor,anchor=base west,inner sep=0pt, outer sep=0pt, scale=  1.00] at ( 62.16, 73.90) {sparse basis};
\end{scope}
\end{tikzpicture}
\begin{tikzpicture}[x=1pt,y=1pt]
\definecolor{fillColor}{RGB}{255,255,255}
\path[use as bounding box,fill=fillColor,fill opacity=0.00] (0,0) rectangle (231.26,144.54);
\begin{scope}
\path[clip] ( 38.40, 38.40) rectangle (229.34,134.94);
\definecolor{drawColor}{RGB}{0,0,0}

\path[draw=drawColor,line width= 0.4pt,line join=round,line cap=round] (220.00,111.93) -- (195.07, 65.45);

\path[draw=drawColor,line width= 0.4pt,line join=round,line cap=round] (188.28, 59.61) -- (148.22, 45.39);

\path[draw=drawColor,line width= 0.4pt,line join=round,line cap=round] (138.89, 43.72) -- ( 50.27, 42.55);

\path[draw=drawColor,line width= 0.4pt,line join=round,line cap=round] (222.27,116.16) circle (  1.80);

\path[draw=drawColor,line width= 0.4pt,line join=round,line cap=round] (192.81, 61.22) circle (  1.80);

\path[draw=drawColor,line width= 0.4pt,line join=round,line cap=round] (143.69, 43.78) circle (  1.80);

\path[draw=drawColor,line width= 0.4pt,line join=round,line cap=round] ( 45.47, 42.49) circle (  1.80);
\end{scope}
\begin{scope}
\path[clip] (  0.00,  0.00) rectangle (231.26,144.54);
\definecolor{drawColor}{RGB}{0,0,0}

\path[draw=drawColor,line width= 0.4pt,line join=round,line cap=round] (192.81, 38.40) -- ( 45.47, 38.40);

\path[draw=drawColor,line width= 0.4pt,line join=round,line cap=round] (192.81, 38.40) -- (192.81, 33.60);

\path[draw=drawColor,line width= 0.4pt,line join=round,line cap=round] (143.69, 38.40) -- (143.69, 33.60);

\path[draw=drawColor,line width= 0.4pt,line join=round,line cap=round] ( 94.58, 38.40) -- ( 94.58, 33.60);

\path[draw=drawColor,line width= 0.4pt,line join=round,line cap=round] ( 45.47, 38.40) -- ( 45.47, 33.60);

\node[text=drawColor,anchor=base,inner sep=0pt, outer sep=0pt, scale=  1.00] at ( 45.47, 21.12) {20};

\node[text=drawColor,anchor=base,inner sep=0pt, outer sep=0pt, scale=  1.00] at ( 94.58, 21.12) {15};

\node[text=drawColor,anchor=base,inner sep=0pt, outer sep=0pt, scale=  1.00] at (143.69, 21.12) {10};

\node[text=drawColor,anchor=base,inner sep=0pt, outer sep=0pt, scale=  1.00] at (192.81, 21.12) {5};

\path[draw=drawColor,line width= 0.4pt,line join=round,line cap=round] ( 38.40, 41.98) -- ( 38.40,131.36);

\path[draw=drawColor,line width= 0.4pt,line join=round,line cap=round] ( 38.40, 41.98) -- ( 33.60, 41.98);

\path[draw=drawColor,line width= 0.4pt,line join=round,line cap=round] ( 38.40, 59.85) -- ( 33.60, 59.85);

\path[draw=drawColor,line width= 0.4pt,line join=round,line cap=round] ( 38.40, 77.73) -- ( 33.60, 77.73);

\path[draw=drawColor,line width= 0.4pt,line join=round,line cap=round] ( 38.40, 95.61) -- ( 33.60, 95.61);

\path[draw=drawColor,line width= 0.4pt,line join=round,line cap=round] ( 38.40,113.49) -- ( 33.60,113.49);

\path[draw=drawColor,line width= 0.4pt,line join=round,line cap=round] ( 38.40,131.36) -- ( 33.60,131.36);

\node[text=drawColor,rotate= 90.00,anchor=base,inner sep=0pt, outer sep=0pt, scale=  1.00] at ( 26.88, 41.98) {0};

\node[text=drawColor,rotate= 90.00,anchor=base,inner sep=0pt, outer sep=0pt, scale=  1.00] at ( 26.88, 59.85) {20};

\node[text=drawColor,rotate= 90.00,anchor=base,inner sep=0pt, outer sep=0pt, scale=  1.00] at ( 26.88, 95.61) {60};

\node[text=drawColor,rotate= 90.00,anchor=base,inner sep=0pt, outer sep=0pt, scale=  1.00] at ( 26.88,131.36) {100};

\path[draw=drawColor,line width= 0.4pt,line join=round,line cap=round] ( 38.40,134.94) --
	( 38.40, 38.40) --
	(229.34, 38.40);
\end{scope}
\begin{scope}
\path[clip] (  0.00,  0.00) rectangle (231.26,144.54);
\definecolor{drawColor}{RGB}{0,0,0}

\node[text=drawColor,anchor=base,inner sep=0pt, outer sep=0pt, scale=  1.00] at (133.87,  1.92) {$l$};

\node[text=drawColor,rotate= 90.00,anchor=base,inner sep=0pt, outer sep=0pt, scale=  1.00] at (  7.68, 86.67) {\% reduction in conditoning set size};
\end{scope}
\begin{scope}
\path[clip] ( 38.40, 38.40) rectangle (229.34,134.94);
\onlyincolor{\definecolor{drawColor}{RGB}{255,0,0}}
\onlyinblackandwhite{\definecolor{drawColor}{RGB}{0,0,0}}

\path[draw=drawColor,line width= 0.4pt,dash pattern=on 4pt off 4pt ,line join=round,line cap=round] (219.81,113.27) -- (195.27, 72.24);

\path[draw=drawColor,line width= 0.4pt,dash pattern=on 4pt off 4pt ,line join=round,line cap=round] (188.41, 66.18) -- (148.09, 48.41);

\path[draw=drawColor,line width= 0.4pt,dash pattern=on 4pt off 4pt ,line join=round,line cap=round] (138.90, 46.28) -- ( 50.27, 42.59);

\path[draw=drawColor,line width= 0.4pt,line join=round,line cap=round] (222.27,117.39) circle (  1.80);

\path[draw=drawColor,line width= 0.4pt,line join=round,line cap=round] (192.81, 68.12) circle (  1.80);

\path[draw=drawColor,line width= 0.4pt,line join=round,line cap=round] (143.69, 46.48) circle (  1.80);

\path[draw=drawColor,line width= 0.4pt,line join=round,line cap=round] ( 45.47, 42.39) circle (  1.80);
\onlyincolor{\definecolor{drawColor}{RGB}{0,205,0}}

\path[draw=drawColor,line width= 0.4pt,dash pattern=on 1pt off 3pt ,line join=round,line cap=round] (219.78,121.17) -- (195.29, 80.75);

\path[draw=drawColor,line width= 0.4pt,dash pattern=on 1pt off 3pt ,line join=round,line cap=round] (188.65, 74.23) -- (147.84, 50.52);

\path[draw=drawColor,line width= 0.4pt,dash pattern=on 1pt off 3pt ,line join=round,line cap=round] (138.90, 47.84) -- ( 50.26, 42.78);

\path[draw=drawColor,line width= 0.4pt,line join=round,line cap=round] (222.27,125.27) circle (  1.80);

\path[draw=drawColor,line width= 0.4pt,line join=round,line cap=round] (192.81, 76.64) circle (  1.80);

\path[draw=drawColor,line width= 0.4pt,line join=round,line cap=round] (143.69, 48.11) circle (  1.80);

\path[draw=drawColor,line width= 0.4pt,line join=round,line cap=round] ( 45.47, 42.51) circle (  1.80);
\definecolor{drawColor}{RGB}{0,0,0}

\path[draw=drawColor,line width= 0.4pt,line join=round,line cap=round] ( 40.56,125.34) -- ( 54.96,125.34);
\onlyincolor{\definecolor{drawColor}{RGB}{255,0,0}}

\path[draw=drawColor,line width= 0.4pt,dash pattern=on 4pt off 4pt ,line join=round,line cap=round] ( 40.56,115.74) -- ( 54.96,115.74);
\onlyincolor{\definecolor{drawColor}{RGB}{0,205,0}}

\path[draw=drawColor,line width= 0.4pt,dash pattern=on 1pt off 3pt ,line join=round,line cap=round] ( 40.56,106.14) -- ( 54.96,106.14);
\definecolor{drawColor}{RGB}{0,0,0}

\node[text=drawColor,anchor=base west,inner sep=0pt, outer sep=0pt, scale=  1.00] at ( 62.16,121.90) {$n$=25};

\node[text=drawColor,anchor=base west,inner sep=0pt, outer sep=0pt, scale=  1.00] at ( 62.16,112.30) {$n$=50};

\node[text=drawColor,anchor=base west,inner sep=0pt, outer sep=0pt, scale=  1.00] at ( 62.16,102.70) {$n$=100};
\end{scope}
\end{tikzpicture}

%% file: experiments-table1.tex
{

\setlength\tabcolsep{5pt}
\begin{table}
  \begin{center}
  \scriptsize
  \begin{tabular}{|rr | rrr| rrr| rrr| rrr|}
  \hline
   &&\multicolumn{3}{c|}{$l=2$} &\multicolumn{3}{c|}{$l=5$} &\multicolumn{3}{c|}{$l=10$} &\multicolumn{3}{c|}{$l=20$}\\
   \bfseries $n$ & \bfseries $k$
& \bfseries BC & \bfseries CBC & \bfseries CBC$^+$
& \bfseries BC & \bfseries CBC & \bfseries CBC$^+$
& \bfseries BC & \bfseries CBC & \bfseries CBC$^+$
& \bfseries BC & \bfseries CBC & \bfseries CBC$^+$
\\\hline
10&1&
8893&8893&10000&7205&7205&10000&5034&5034&10000&4934&4934&10000\\
10&2&
\cellcolor[gray]{0.85}5543&\cellcolor[gray]{0.85}6061&8618&\cellcolor[gray]{0.85}1033&\cellcolor[gray]{0.85}1980&4322&\cellcolor[gray]{0.85}0&\cellcolor[gray]{0.85}660&2197&\cellcolor[gray]{0.85}0&\cellcolor[gray]{0.85}686&2417\\
10&3&
\cellcolor[gray]{0.85}2359&\cellcolor[gray]{0.85}3395&5817&\cellcolor[gray]{0.85}57&\cellcolor[gray]{0.85}548&1425&0&168&663&0&174&689\\
10&5&
\cellcolor[gray]{0.85}200&\cellcolor[gray]{0.85}886&1712&0&108&216&0&36&71&0&31&65\\
\hline
25&1&
9573&9573&10000&8936&8936&10000&7905&7905&10000&5843&5843&10000\\
25&2&
8117&8247&9651&\cellcolor[gray]{0.85}4243&\cellcolor[gray]{0.85}4735&7003&\cellcolor[gray]{0.85}1033&\cellcolor[gray]{0.85}1553&3674&70&401&2118\\
25&3&
\cellcolor[gray]{0.85}6013&\cellcolor[gray]{0.85}6424&8520&\cellcolor[gray]{0.85}1203&\cellcolor[gray]{0.85}1852&3524&46&212&1046&0&34&587\\
25&5&
\cellcolor[gray]{0.85}2298&\cellcolor[gray]{0.85}3021&5055&39&243&646&0&11&93&0&1&53\\
\hline
50&1&
9832&9832&10000&9476&9476&10000&8997&8997&10000&7832&7832&10000\\
50&2&
9095&9128&9882&6688&6938&8388&2657&3049&4927&527&866&2729\\
50&5&
\cellcolor[gray]{0.85}5104&\cellcolor[gray]{0.85}5535&7489&462&799&1613&3&16&198&0&2&58\\
50&7&
\cellcolor[gray]{0.85}2473&\cellcolor[gray]{0.85}3120&4799&27&119&302&0&1&25&0&0&6\\
\hline
100&1&
9907&9907&10000&9783&9783&10000&9494&9494&10000&8966&8966&10000\\
100&2&
9585&9593&9971&8262&8353&9165&4600&4834&6162&1507&1762&3492\\
100&5&
7425&7591&9090&\cellcolor[gray]{0.85}1932&\cellcolor[gray]{0.85}2336&3441&43&102&393&1&4&92\\
100&10&
\cellcolor[gray]{0.85}2499&\cellcolor[gray]{0.85}3040&4479&15&48&137&0&0&2&0&0&0\\
\hline
250&1&
9947&9947&10000&9894&9894&10000&9774&9774&10000&9621&9621&10000\\
250&2&
9835&9840&9991&9284&9327&9696&6569&6689&7358&3151&3285&4502\\
250&5&
8956&8994&9807&5051&5325&6261&469&544&994&7&17&164\\
250&15&
\cellcolor[gray]{0.85}3102&\cellcolor[gray]{0.85}3537&4864&18&32&58&0&0&1&0&0&0\\
250&25&
319&536&731&0&0&0&0&0&0&0&0&0\\
\hline
500&1&
9977&9977&10000&9946&9946&10000&9883&9883&10000&9799&9799&10000\\
500&2&
9923&9923&9996&9674&9684&9872&7704&7750&8116&4184&4266&4988\\
500&5&
9477&9490&9948&7249&7368&8069&1170&1265&1686&43&48&216\\
500&22&
3012&3288&4413&3&14&17&0&0&0&0&0&0\\
500&50&
10&29&31&0&0&0&0&0&0&0&0&0\\
\hline
1000&1&
9990&9990&10000&9973&9973&10000&9942&9942&10000&9885&9885&10000\\
1000&2&
9965&9966&10000&9844&9845&9952&8416&8434&8640&5130&5173&5577\\
1000&5&
9734&9736&9986&8679&8726&9173&2310&2396&2686&136&149&319\\
1000&32&
2923&3191&4163&2&6&7&0&0&0&0&0&0\\
1000&100&
0&0&0&0&0&0&0&0&0&0&0&0\\
\hline
2000&1&
9999&9999&10000&9988&9988&10000&9972&9972&10000&9938&9938&10000\\
2000&2&
9973&9973&10000&9940&9940&9981&9023&9029&9119&5928&5954&6156\\
2000&5&
9880&9880&9996&9450&9471&9713&3608&3648&3869&287&300&469\\
2000&45&
3000&3210&4122&4&8&8&0&0&2&0&0&0\\
2000&200&
0&0&0&0&0&0&0&0&0&0&0&0\\
\hline\end{tabular}
  \end{center}\vspace*{-3mm}
  \caption{Numbers of instances for $P(\textit{unobserved}) = 0$, i.e., all variables are observed, that are identifiable by use of BC, CBC, CBC$^+$ (as defined in Section~\ref{sec:experiments:algorithms:abbreviations}). 
      Gray cells highlight where the CBC was able to identify at least 400 more graphs than the BC.
      Since all variables are observed, all instances are identifiable, thus IDC is not used in this Table.
}\label{table:global:stat:unobs0:count}
  \end{table}
\begin{table}
  \begin{center}
  \scriptsize
  \begin{tabular}{|rr | rrrr| rrrr| rrrr| rrrr|}
  \hline
   &&\multicolumn{4}{c|}{$l=2$} &\multicolumn{4}{c|}{$l=5$} &\multicolumn{4}{c|}{$l=10$} &\multicolumn{4}{c|}{$l=20$}\\
   \bfseries $n$ & \bfseries $k$
& \bfseries BC & \bfseries CBC & \bfseries CBC$^+$& \bfseries IDC
& \bfseries BC & \bfseries CBC & \bfseries CBC$^+$& \bfseries IDC
& \bfseries BC & \bfseries CBC & \bfseries CBC$^+$& \bfseries IDC
& \bfseries BC & \bfseries CBC & \bfseries CBC$^+$& \bfseries IDC
\\\hline
10&1&
6333&6333&9604&9609&1935&1935&7475&7476&978&978&5944&5944&936&936&5877&5877\\
10&2&
2008&2339&6889&8740&   103&228&2854&4137&     0&113&1721&2260&      0&114&1752&2294\\
10&3&
610&980&4193&8056&0&21&1061&1995&0&9&512&763&0&10&547&789\\
10&5&
\cellcolor[gray]{0.85}185&\cellcolor[gray]{0.85}859&1756&10000&0&98&190&10000&0&43&76&10000&0&26&75&10000\\
\hline
25&1&
8414&8414&9923&9930&3647&3647&8727&8742&1340&1340&6884&6888&557&557&5696&5696\\
25&2&
5164&5331&8939&9731&   601&728&4630&6417&  77&130&2501&3469&    4&41&1847&2299\\
25&3&
2350&2632&6958&9270&73&144&2141&4327&2&17&872&1518&0&6&554&780\\
25&5&
277&449&3008&8157&0&1&456&1462&0&0&114&251&0&0&49&71\\
\hline
50&1&
9082&9082&9975&9979&4651&4651&9237&9253&1699&1699&7547&7555&697&697&6031&6032\\
50&2&
6985&7059&9599&9908&      1098&1189&6078&7686&     133&160&3353&4270&      23&40&2061&2543\\
50&5&
1440&1663&5452&9394&      5&16&868&3030&        0&1&178&482&         0&0&73&125\\
50&7&
254&388&2596&8648&      0&0&186&1226&        0&0&19&80&       0&0&3&10\\
\hline
100&1&
9527&9527&9992&9993&5585&5585&9602&9618&1985&1985&7980&7991&744&744&6414&6416\\
100&2&
8295&8316&9884&9980&1846&1886&7303&8618&195&217&3799&4989&49&56&2413&2940\\
100&5&
3391&3562&7636&9804&20&30&1800&4690&0&0&331&802&0&0&84&159\\
100&10&
252&375&2364&9263&0&0&74&956&0&0&3&15&0&0&0&0\\
\hline
250&1&
9791&9791&10000&10000&6832&6832&9814&9827&2493&2493&8564&8579&846&846&6793&6795\\
250&2&
9205&9209&9974&9990&3099&3138&8509&9360&277&286&4914&6073&46&50&2881&3439\\
250&5&
6110&6182&9269&9962&106&123&3344&6764&1&1&599&1323&0&0&136&248\\
250&15&
232&306&2281&9697&       0&0&16&703&        0&0&0&4&          0&0&0&0\\
250&25&
3&4&221&9008&0&0&0&28&0&0&0&0&0&0&0&0\\
\hline
500&1&
9882&9882&9999&-&7646&7646&9919&-&2935&2935&8885&-&946&946&7184&-\\
500&2&
9596&9598&9993&-&4267&4280&9117&-&401&406&5722&-&33&34&3166&-\\
500&5&
7774&7801&9754&-&273&285&4973&-&1&1&990&-&0&0&226&-\\
500&22&
150&184&1757&-&0&0&3&-&0&0&0&-&0&0&2&-\\
500&50&
0&0&4&-&0&0&0&-&0&0&0&-&0&0&0&-\\
\hline
1000&1&
9936&9936&10000&-&8394&8394&9970&-&3181&3181&9137&-&1061&1061&7422&-\\
1000&2&
9789&9790&9999&-&5498&5507&9568&-&525&526&6361&-&51&51&3546&-\\
1000&5&
8797&8803&9947&-&666&676&6482&-&2&2&1471&-&0&0&245&-\\
1000&32&
94&107&1511&-&0&0&1&-&0&0&0&-&0&0&0&-\\
1000&100&
0&0&0&-&0&0&0&-&0&0&0&-&0&0&0&-\\
\hline
2000&1&
9975&9975&10000&-&8914&8914&9988&-&3685&3685&9361&-&1099&1099&7613&-\\
2000&2&
9879&9879&9999&-&6774&6777&9791&-&714&714&7048&-&57&57&3858&-\\
2000&5&
9383&9384&9980&-&1519&1535&7906&-&0&0&2159&-&0&0&342&-\\
2000&45&
81&90&1399&-&0&0&0&-&0&0&0&-&0&0&0&-\\
2000&200&
0&0&0&-&0&0&0&-&0&0&0&-&0&0&0&-\\
\hline\end{tabular}
  \end{center}\vspace*{-3mm}
  \caption{Numbers of instances for $P(\textit{unobserved}) = 0.75$ that are identifiable 
    by use of BC, CBC, CBC$^+$ and IDC (as defined in Section~\ref{sec:experiments:algorithms:abbreviations}).
    %, and 
    %by the complete identification do-calculus based algorithm (IDC). 
    Gray cells highlight where the CBC was able to identify at least 400 more graphs than the BC.
    Due to its high time complexity, we were unable to run the IDC algorithm on instances labelled
	with ``-''.}\label{table:global:stat:unobs0.75:count}
  \end{table}

}

%% file: experiments-grid.tex
\begin{figure}\centering\tiny
\begin{center}
\begin{tabular}{ccc}
\begin{tikzpicture}[minimum size=0]
% % \node[fill=white!89!black,thick,minimum width=0.4cm,minimum height=0.4cm] at (4.6, 0.2) {};
% % \node[fill=white!61!black,thick,minimum width=0.4cm,minimum height=0.4cm] at (3.8, 0.2) {};
% % \node[fill=white!9!black,thick,minimum width=0.4cm,minimum height=0.4cm] at (2.2, 0.2) {};
% % \node[fill=white!72!black,thick,minimum width=0.4cm,minimum height=0.4cm] at (4.2, 0.2) {};
% % \node[fill=white!20!black,thick,minimum width=0.4cm,minimum height=0.4cm] at (2.6, 0.2) {};
% % \node[fill=white!1!black,thick,minimum width=0.4cm,minimum height=0.4cm] at (1.4, 0.2) {};
% % \node[fill=white!50!black,thick,minimum width=0.4cm,minimum height=0.4cm] at (3.4, 0.2) {};
% % \node[fill=white!7!black,thick,minimum width=0.4cm,minimum height=0.4cm] at (1.8, 0.2) {};
% % \node[fill=white!0!black,thick,minimum width=0.4cm,minimum height=0.4cm] at (0.6, 0.2) {};
% % \node[fill=white!49!black,thick,minimum width=0.4cm,minimum height=0.4cm] at (3, 0.2) {};
% % \node[fill=white!7!black,thick,minimum width=0.4cm,minimum height=0.4cm] at (1, 0.2) {};
% % \node[fill=white!0!black,thick,minimum width=0.4cm,minimum height=0.4cm] at (0.2, 0.2) {};
\node[fill=white!79!black,thick,minimum width=0.4cm,minimum height=0.4cm] at (3.4, 0.6) {};
\node[fill=white!16!black,thick,minimum width=0.4cm,minimum height=0.4cm] at (1.8, 0.6) {};
\node[fill=white!0!black,thick,minimum width=0.4cm,minimum height=0.4cm] at (0.6, 0.6) {};
\node[fill=white!96!black,thick,minimum width=0.4cm,minimum height=0.4cm] at (4.6, 0.6) {};
\node[fill=white!82!black,thick,minimum width=0.4cm,minimum height=0.4cm] at (3.8, 0.6) {};
\node[fill=white!30!black,thick,minimum width=0.4cm,minimum height=0.4cm] at (2.2, 0.6) {};
\node[fill=white!58!black,thick,minimum width=0.4cm,minimum height=0.4cm] at (3, 0.6) {};
\node[fill=white!4!black,thick,minimum width=0.4cm,minimum height=0.4cm] at (1, 0.6) {};
\node[fill=white!0!black,thick,minimum width=0.4cm,minimum height=0.4cm] at (0.2, 0.6) {};
\node[fill=white!89!black,thick,minimum width=0.4cm,minimum height=0.4cm] at (4.2, 0.6) {};
\node[fill=white!47!black,thick,minimum width=0.4cm,minimum height=0.4cm] at (2.6, 0.6) {};
\node[fill=white!2!black,thick,minimum width=0.4cm,minimum height=0.4cm] at (1.4, 0.6) {};
\node[fill=white!98!black,thick,minimum width=0.4cm,minimum height=0.4cm] at (4.6, 1) {};
\node[fill=white!91!black,thick,minimum width=0.4cm,minimum height=0.4cm] at (3.8, 1) {};
\node[fill=white!55!black,thick,minimum width=0.4cm,minimum height=0.4cm] at (2.2, 1) {};
\node[fill=white!95!black,thick,minimum width=0.4cm,minimum height=0.4cm] at (4.2, 1) {};
\node[fill=white!69!black,thick,minimum width=0.4cm,minimum height=0.4cm] at (2.6, 1) {};
\node[fill=white!8!black,thick,minimum width=0.4cm,minimum height=0.4cm] at (1.4, 1) {};
\node[fill=white!90!black,thick,minimum width=0.4cm,minimum height=0.4cm] at (3.4, 1) {};
\node[fill=white!30!black,thick,minimum width=0.4cm,minimum height=0.4cm] at (1.8, 1) {};
\node[fill=white!0!black,thick,minimum width=0.4cm,minimum height=0.4cm] at (0.6, 1) {};
\node[fill=white!78!black,thick,minimum width=0.4cm,minimum height=0.4cm] at (3, 1) {};
\node[fill=white!9!black,thick,minimum width=0.4cm,minimum height=0.4cm] at (1, 1) {};
\node[fill=white!0!black,thick,minimum width=0.4cm,minimum height=0.4cm] at (0.2, 1) {};
\node[fill=white!95!black,thick,minimum width=0.4cm,minimum height=0.4cm] at (3.4, 1.4) {};
\node[fill=white!48!black,thick,minimum width=0.4cm,minimum height=0.4cm] at (1.8, 1.4) {};
\node[fill=white!1!black,thick,minimum width=0.4cm,minimum height=0.4cm] at (0.6, 1.4) {};
\node[fill=white!90!black,thick,minimum width=0.4cm,minimum height=0.4cm] at (3, 1.4) {};
\node[fill=white!18!black,thick,minimum width=0.4cm,minimum height=0.4cm] at (1, 1.4) {};
\node[fill=white!0!black,thick,minimum width=0.4cm,minimum height=0.4cm] at (0.2, 1.4) {};
\node[fill=white!99!black,thick,minimum width=0.4cm,minimum height=0.4cm] at (4.6, 1.4) {};
\node[fill=white!96!black,thick,minimum width=0.4cm,minimum height=0.4cm] at (3.8, 1.4) {};
\node[fill=white!76!black,thick,minimum width=0.4cm,minimum height=0.4cm] at (2.2, 1.4) {};
\node[fill=white!98!black,thick,minimum width=0.4cm,minimum height=0.4cm] at (4.2, 1.4) {};
\node[fill=white!84!black,thick,minimum width=0.4cm,minimum height=0.4cm] at (2.6, 1.4) {};
\node[fill=white!23!black,thick,minimum width=0.4cm,minimum height=0.4cm] at (1.4, 1.4) {};
\node[fill=white!99!black,thick,minimum width=0.4cm,minimum height=0.4cm] at (4.6, 1.8) {};
\node[fill=white!98!black,thick,minimum width=0.4cm,minimum height=0.4cm] at (3.8, 1.8) {};
\node[fill=white!90!black,thick,minimum width=0.4cm,minimum height=0.4cm] at (2.2, 1.8) {};
\node[fill=white!96!black,thick,minimum width=0.4cm,minimum height=0.4cm] at (3, 1.8) {};
\node[fill=white!33!black,thick,minimum width=0.4cm,minimum height=0.4cm] at (1, 1.8) {};
\node[fill=white!0!black,thick,minimum width=0.4cm,minimum height=0.4cm] at (0.2, 1.8) {};
\node[fill=white!99!black,thick,minimum width=0.4cm,minimum height=0.4cm] at (4.2, 1.8) {};
\node[fill=white!93!black,thick,minimum width=0.4cm,minimum height=0.4cm] at (2.6, 1.8) {};
\node[fill=white!53!black,thick,minimum width=0.4cm,minimum height=0.4cm] at (1.4, 1.8) {};
\node[fill=white!98!black,thick,minimum width=0.4cm,minimum height=0.4cm] at (3.4, 1.8) {};
\node[fill=white!67!black,thick,minimum width=0.4cm,minimum height=0.4cm] at (1.8, 1.8) {};
\node[fill=white!5!black,thick,minimum width=0.4cm,minimum height=0.4cm] at (0.6, 1.8) {};
\node[fill=white!99!black,thick,minimum width=0.4cm,minimum height=0.4cm] at (4.2, 2.2) {};
\node[fill=white!97!black,thick,minimum width=0.4cm,minimum height=0.4cm] at (2.6, 2.2) {};
\node[fill=white!74!black,thick,minimum width=0.4cm,minimum height=0.4cm] at (1.4, 2.2) {};
\node[fill=white!99!black,thick,minimum width=0.4cm,minimum height=0.4cm] at (3.4, 2.2) {};
\node[fill=white!78!black,thick,minimum width=0.4cm,minimum height=0.4cm] at (1.8, 2.2) {};
\node[fill=white!13!black,thick,minimum width=0.4cm,minimum height=0.4cm] at (0.6, 2.2) {};
\node[fill=white!98!black,thick,minimum width=0.4cm,minimum height=0.4cm] at (3, 2.2) {};
\node[fill=white!100!black,thick,minimum width=0.4cm,minimum height=0.4cm] at (4.6, 2.2) {};
\node[fill=white!43!black,thick,minimum width=0.4cm,minimum height=0.4cm] at (1, 2.2) {};
\node[fill=white!99!black,thick,minimum width=0.4cm,minimum height=0.4cm] at (3.8, 2.2) {};
\node[fill=white!0!black,thick,minimum width=0.4cm,minimum height=0.4cm] at (0.2, 2.2) {};
\node[fill=white!95!black,thick,minimum width=0.4cm,minimum height=0.4cm] at (2.2, 2.2) {};
\node[fill=white!99!black,thick,minimum width=0.4cm,minimum height=0.4cm] at (3.4, 2.6) {};
\node[fill=white!84!black,thick,minimum width=0.4cm,minimum height=0.4cm] at (1.8, 2.6) {};
\node[fill=white!24!black,thick,minimum width=0.4cm,minimum height=0.4cm] at (0.6, 2.6) {};
\node[fill=white!99!black,thick,minimum width=0.4cm,minimum height=0.4cm] at (3, 2.6) {};
\node[fill=white!100!black,thick,minimum width=0.4cm,minimum height=0.4cm] at (4.6, 2.6) {};
\node[fill=white!52!black,thick,minimum width=0.4cm,minimum height=0.4cm] at (1, 2.6) {};
\node[fill=white!100!black,thick,minimum width=0.4cm,minimum height=0.4cm] at (3.8, 2.6) {};
\node[fill=white!1!black,thick,minimum width=0.4cm,minimum height=0.4cm] at (0.2, 2.6) {};
\node[fill=white!97!black,thick,minimum width=0.4cm,minimum height=0.4cm] at (2.2, 2.6) {};
\node[fill=white!100!black,thick,minimum width=0.4cm,minimum height=0.4cm] at (4.2, 2.6) {};
\node[fill=white!98!black,thick,minimum width=0.4cm,minimum height=0.4cm] at (2.6, 2.6) {};
\node[fill=white!87!black,thick,minimum width=0.4cm,minimum height=0.4cm] at (1.4, 2.6) {};
\node[fill=white!99!black,thick,minimum width=0.4cm,minimum height=0.4cm] at (3, 3) {};
\node[fill=white!100!black,thick,minimum width=0.4cm,minimum height=0.4cm] at (4.6, 3) {};
\node[fill=white!60!black,thick,minimum width=0.4cm,minimum height=0.4cm] at (1, 3) {};
\node[fill=white!100!black,thick,minimum width=0.4cm,minimum height=0.4cm] at (3.8, 3) {};
\node[fill=white!3!black,thick,minimum width=0.4cm,minimum height=0.4cm] at (0.2, 3) {};
\node[fill=white!99!black,thick,minimum width=0.4cm,minimum height=0.4cm] at (2.2, 3) {};
\node[fill=white!100!black,thick,minimum width=0.4cm,minimum height=0.4cm] at (4.2, 3) {};
\node[fill=white!99!black,thick,minimum width=0.4cm,minimum height=0.4cm] at (2.6, 3) {};
\node[fill=white!95!black,thick,minimum width=0.4cm,minimum height=0.4cm] at (1.4, 3) {};
\node[fill=white!100!black,thick,minimum width=0.4cm,minimum height=0.4cm] at (3.4, 3) {};
\node[fill=white!90!black,thick,minimum width=0.4cm,minimum height=0.4cm] at (1.8, 3) {};
\node[fill=white!36!black,thick,minimum width=0.4cm,minimum height=0.4cm] at (0.6, 3) {};
\node[] at (2.4, 3.95){\normalsize $P(\textit{unobserved}) = 0$};
%\node[] at (2.4, -0.7) {\footnotesize instance complexity};
\draw[] (0, 0.4) -- (4.8, 0.4);
%\draw (-0.2, 0.37) -- (-0.2, 0.43);
\node[rotate=90,text width=0.9cm] at (-0.2, -0.1) {\makebox[4mm]{\hfill$l$}, $k$};
\draw (0.2, 0.37) -- (0.2, 0.43);
\node[rotate=90,text width=0.9cm] at (0.2, -0.1) {\makebox[4mm]{\hfill$20$}, $5$};
\draw (0.6, 0.37) -- (0.6, 0.43);
\node[rotate=90,text width=0.9cm] at (0.6, -0.1) {\makebox[4mm]{\hfill$10$}, $5$};
\draw (1, 0.37) -- (1, 0.43);
\node[rotate=90,text width=0.9cm] at (1, -0.1) {\makebox[4mm]{\hfill$20$}, $2$};
\draw (1.4, 0.37) -- (1.4, 0.43);
\node[rotate=90,text width=0.9cm] at (1.4, -0.1) {\makebox[4mm]{\hfill$5$}, $5$};
\draw (1.8, 0.37) -- (1.8, 0.43);
\node[rotate=90,text width=0.9cm] at (1.8, -0.1) {\makebox[4mm]{\hfill$10$}, $2$};
\draw (2.2, 0.37) -- (2.2, 0.43);
\node[rotate=90,text width=0.9cm] at (2.2, -0.1) {\makebox[4mm]{\hfill$2$}, $5$};
\draw (2.6, 0.37) -- (2.6, 0.43);
\node[rotate=90,text width=0.9cm] at (2.6, -0.1) {\makebox[4mm]{\hfill$5$}, $2$};
\draw (3, 0.37) -- (3, 0.43);
\node[rotate=90,text width=0.9cm] at (3, -0.1) {\makebox[4mm]{\hfill$20$}, $1$};
\draw (3.4, 0.37) -- (3.4, 0.43);
\node[rotate=90,text width=0.9cm] at (3.4, -0.1) {\makebox[4mm]{\hfill$10$}, $1$};
\draw (3.8, 0.37) -- (3.8, 0.43);
\node[rotate=90,text width=0.9cm] at (3.8, -0.1) {\makebox[4mm]{\hfill$2$}, $2$};
\draw (4.2, 0.37) -- (4.2, 0.43);
\node[rotate=90,text width=0.9cm] at (4.2, -0.1) {\makebox[4mm]{\hfill$5$}, $1$};
\draw (4.6, 0.37) -- (4.6, 0.43);
\node[rotate=90,text width=0.9cm] at (4.6, -0.1) {\makebox[4mm]{\hfill$2$}, $1$};
\node[rotate=90] at (-2.5, 1.6) {\normalsize CBC};
\node[rotate=90] at (-1.5, 1.6) {\footnotesize number of nodes};
\draw[] (0, 0.4) -- (0, 3.2);
% % \draw (0.37, 0.2) -- (0.03, 0.2);
% % \node[] at (-0.75, 0.2) {10};
\draw (-0.03, 0.6) -- (0.03, 0.6);
\node[] at (-0.75, 0.6) {25};
\draw (-0.03, 1) -- (0.03, 1);
\node[] at (-0.75, 1) {50};
\draw (-0.03, 1.4) -- (0.03, 1.4);
\node[] at (-0.75, 1.4) {100};
\draw (-0.03, 1.8) -- (0.03, 1.8);
\node[] at (-0.75, 1.8) {250};
\draw (-0.03, 2.2) -- (0.03, 2.2);
\node[] at (-0.75, 2.2) {500};
\draw (-0.03, 2.6) -- (0.03, 2.6);
\node[] at (-0.75, 2.6) {1000};
\draw (-0.03, 3) -- (0.03, 3);
\node[] at (-0.75, 3) {2000};
%\node[rotate=90] at (6.3, 1.6) {number of nodes};
\draw[] (4.8, 0.4) -- (4.8, 3.2);
% % \draw (4.77, 0.2) -- (4.83, 0.2);
% % \node[] at (5.55, 0.2) {10};
%\draw (4.77, 0.6) -- (4.83, 0.6);
%\node[] at (5.55, 0.6) {25};
%\draw (4.77, 1) -- (4.83, 1);
%\node[] at (5.55, 1) {50};
%\draw (4.77, 1.4) -- (4.83, 1.4);
%\node[] at (5.55, 1.4) {100};
%\draw (4.77, 1.8) -- (4.83, 1.8);
%\node[] at (5.55, 1.8) {250};
%\draw (4.77, 2.2) -- (4.83, 2.2);
%\node[] at (5.55, 2.2) {500};
%\draw (4.77, 2.6) -- (4.83, 2.6);
%\node[] at (5.55, 2.6) {1000};
%\draw (4.77, 3) -- (4.83, 3);
%\node[] at (5.55, 3) {2000};
\end{tikzpicture}
&%top right:
\begin{tikzpicture}[minimum size=0]
% % \node[fill=white!63!black,thick,minimum width=0.4cm,minimum height=0.4cm] at (4.6, 0.2) {};
% % \node[fill=white!25!black,thick,minimum width=0.4cm,minimum height=0.4cm] at (4.2, 0.2) {};
% % \node[fill=white!9!black,thick,minimum width=0.4cm,minimum height=0.4cm] at (3.4, 0.2) {};
% % \node[fill=white!19!black,thick,minimum width=0.4cm,minimum height=0.4cm] at (3.8, 0.2) {};
% % \node[fill=white!2!black,thick,minimum width=0.4cm,minimum height=0.4cm] at (3, 0.2) {};
% % \node[fill=white!1!black,thick,minimum width=0.4cm,minimum height=0.4cm] at (1.8, 0.2) {};
% % \node[fill=white!10!black,thick,minimum width=0.4cm,minimum height=0.4cm] at (2.6, 0.2) {};
% % \node[fill=white!1!black,thick,minimum width=0.4cm,minimum height=0.4cm] at (1.4, 0.2) {};
% % \node[fill=white!0!black,thick,minimum width=0.4cm,minimum height=0.4cm] at (0.6, 0.2) {};
% % \node[fill=white!9!black,thick,minimum width=0.4cm,minimum height=0.4cm] at (2.2, 0.2) {};
% % \node[fill=white!1!black,thick,minimum width=0.4cm,minimum height=0.4cm] at (1, 0.2) {};
% % \node[fill=white!0!black,thick,minimum width=0.4cm,minimum height=0.4cm] at (0.2, 0.2) {};
\node[fill=white!13!black,thick,minimum width=0.4cm,minimum height=0.4cm] at (2.6, 0.6) {};
\node[fill=white!0!black,thick,minimum width=0.4cm,minimum height=0.4cm] at (1.4, 0.6) {};  %n=25,l= ,k= 
\node[fill=white!0!black,thick,minimum width=0.4cm,minimum height=0.4cm] at (0.6, 0.6) {};
\node[fill=white!84!black,thick,minimum width=0.4cm,minimum height=0.4cm] at (4.6, 0.6) {};
\node[fill=white!51!black,thick,minimum width=0.4cm,minimum height=0.4cm] at (4.2, 0.6) {}; %n=25,l=2,k=2
\node[fill=white!4!black,thick,minimum width=0.4cm,minimum height=0.4cm] at (3.4, 0.6) {};
\node[fill=white!6!black,thick,minimum width=0.4cm,minimum height=0.4cm] at (2.2, 0.6) {};
\node[fill=white!0!black,thick,minimum width=0.4cm,minimum height=0.4cm] at (1, 0.6) {};
\node[fill=white!0!black,thick,minimum width=0.4cm,minimum height=0.4cm] at (0.2, 0.6) {};
\node[fill=white!36!black,thick,minimum width=0.4cm,minimum height=0.4cm] at (3.8, 0.6) {};
\node[fill=white!1!black,thick,minimum width=0.4cm,minimum height=0.4cm] at (3, 0.6) {};     %n=25,l=5,k=2
\node[fill=white!0!black,thick,minimum width=0.4cm,minimum height=0.4cm] at (1.8, 0.6) {};
\node[fill=white!91!black,thick,minimum width=0.4cm,minimum height=0.4cm] at (4.6, 1) {}; %n=50,l=2,k=1
\node[fill=white!71!black,thick,minimum width=0.4cm,minimum height=0.4cm] at (4.2, 1) {}; %n=50,l=2,k=2
\node[fill=white!17!black,thick,minimum width=0.4cm,minimum height=0.4cm] at (3.4, 1) {}; %n=50,l=2,k=5
\node[fill=white!47!black,thick,minimum width=0.4cm,minimum height=0.4cm] at (3.8, 1) {}; %n=50,l=5,k=1
\node[fill=white!12!black,thick,minimum width=0.4cm,minimum height=0.4cm] at (3, 1) {};    %n=50,l=5,k=2
\node[fill=white!0!black,thick,minimum width=0.4cm,minimum height=0.4cm] at (1.8, 1) {};  %n=50,l=5,k=5
\node[fill=white!17!black,thick,minimum width=0.4cm,minimum height=0.4cm] at (2.6, 1) {}; %n=50,l=10,k=1
\node[fill=white!0!black,thick,minimum width=0.4cm,minimum height=0.4cm] at (1.4, 1) {};  %n=50,l=10,k=2
\node[fill=white!0!black,thick,minimum width=0.4cm,minimum height=0.4cm] at (0.6, 1) {};  %n=50,l=10,k=5
\node[fill=white!7!black,thick,minimum width=0.4cm,minimum height=0.4cm] at (2.2, 1) {};  %n=50,l=20,k=1
\node[fill=white!0!black,thick,minimum width=0.4cm,minimum height=0.4cm] at (1, 1) {};    %n=50,l=20,k=2
\node[fill=white!0!black,thick,minimum width=0.4cm,minimum height=0.4cm] at (0.2, 1) {};  %n=50,l=20,k=5
\node[fill=white!20!black,thick,minimum width=0.4cm,minimum height=0.4cm] at (2.6, 1.4) {};
\node[fill=white!2!black,thick,minimum width=0.4cm,minimum height=0.4cm] at (1.4, 1.4) {};
\node[fill=white!0!black,thick,minimum width=0.4cm,minimum height=0.4cm] at (0.6, 1.4) {};
\node[fill=white!7!black,thick,minimum width=0.4cm,minimum height=0.4cm] at (2.2, 1.4) {};
\node[fill=white!1!black,thick,minimum width=0.4cm,minimum height=0.4cm] at (1, 1.4) {};
\node[fill=white!0!black,thick,minimum width=0.4cm,minimum height=0.4cm] at (0.2, 1.4) {};
\node[fill=white!95!black,thick,minimum width=0.4cm,minimum height=0.4cm] at (4.6, 1.4) {};
\node[fill=white!83!black,thick,minimum width=0.4cm,minimum height=0.4cm] at (4.2, 1.4) {};
\node[fill=white!36!black,thick,minimum width=0.4cm,minimum height=0.4cm] at (3.4, 1.4) {};
\node[fill=white!56!black,thick,minimum width=0.4cm,minimum height=0.4cm] at (3.8, 1.4) {};
\node[fill=white!19!black,thick,minimum width=0.4cm,minimum height=0.4cm] at (3, 1.4) {};
\node[fill=white!0!black,thick,minimum width=0.4cm,minimum height=0.4cm] at (1.8, 1.4) {};
\node[fill=white!98!black,thick,minimum width=0.4cm,minimum height=0.4cm] at (4.6, 1.8) {};
\node[fill=white!92!black,thick,minimum width=0.4cm,minimum height=0.4cm] at (4.2, 1.8) {};
\node[fill=white!62!black,thick,minimum width=0.4cm,minimum height=0.4cm] at (3.4, 1.8) {};
\node[fill=white!8!black,thick,minimum width=0.4cm,minimum height=0.4cm] at (2.2, 1.8) {};
\node[fill=white!1!black,thick,minimum width=0.4cm,minimum height=0.4cm] at (1, 1.8) {};
\node[fill=white!0!black,thick,minimum width=0.4cm,minimum height=0.4cm] at (0.2, 1.8) {};
\node[fill=white!68!black,thick,minimum width=0.4cm,minimum height=0.4cm] at (3.8, 1.8) {};
\node[fill=white!31!black,thick,minimum width=0.4cm,minimum height=0.4cm] at (3, 1.8) {};
\node[fill=white!1!black,thick,minimum width=0.4cm,minimum height=0.4cm] at (1.8, 1.8) {};
\node[fill=white!25!black,thick,minimum width=0.4cm,minimum height=0.4cm] at (2.6, 1.8) {};
\node[fill=white!3!black,thick,minimum width=0.4cm,minimum height=0.4cm] at (1.4, 1.8) {};
\node[fill=white!0!black,thick,minimum width=0.4cm,minimum height=0.4cm] at (0.6, 1.8) {};
\node[fill=white!76!black,thick,minimum width=0.4cm,minimum height=0.4cm] at (3.8, 2.2) {};
\node[fill=white!43!black,thick,minimum width=0.4cm,minimum height=0.4cm] at (3, 2.2) {};
\node[fill=white!3!black,thick,minimum width=0.4cm,minimum height=0.4cm] at (1.8, 2.2) {};
\node[fill=white!29!black,thick,minimum width=0.4cm,minimum height=0.4cm] at (2.6, 2.2) {};
\node[fill=white!4!black,thick,minimum width=0.4cm,minimum height=0.4cm] at (1.4, 2.2) {};
\node[fill=white!0!black,thick,minimum width=0.4cm,minimum height=0.4cm] at (0.6, 2.2) {};
\node[fill=white!9!black,thick,minimum width=0.4cm,minimum height=0.4cm] at (2.2, 2.2) {};
\node[fill=white!99!black,thick,minimum width=0.4cm,minimum height=0.4cm] at (4.6, 2.2) {};
\node[fill=white!0!black,thick,minimum width=0.4cm,minimum height=0.4cm] at (1, 2.2) {};
\node[fill=white!96!black,thick,minimum width=0.4cm,minimum height=0.4cm] at (4.2, 2.2) {};
\node[fill=white!0!black,thick,minimum width=0.4cm,minimum height=0.4cm] at (0.2, 2.2) {};
\node[fill=white!78!black,thick,minimum width=0.4cm,minimum height=0.4cm] at (3.4, 2.2) {};
\node[fill=white!32!black,thick,minimum width=0.4cm,minimum height=0.4cm] at (2.6, 2.6) {};
\node[fill=white!5!black,thick,minimum width=0.4cm,minimum height=0.4cm] at (1.4, 2.6) {};
\node[fill=white!0!black,thick,minimum width=0.4cm,minimum height=0.4cm] at (0.6, 2.6) {};
\node[fill=white!11!black,thick,minimum width=0.4cm,minimum height=0.4cm] at (2.2, 2.6) {};
\node[fill=white!99!black,thick,minimum width=0.4cm,minimum height=0.4cm] at (4.6, 2.6) {};
\node[fill=white!1!black,thick,minimum width=0.4cm,minimum height=0.4cm] at (1, 2.6) {};
\node[fill=white!98!black,thick,minimum width=0.4cm,minimum height=0.4cm] at (4.2, 2.6) {};
\node[fill=white!0!black,thick,minimum width=0.4cm,minimum height=0.4cm] at (0.2, 2.6) {};
\node[fill=white!88!black,thick,minimum width=0.4cm,minimum height=0.4cm] at (3.4, 2.6) {};
\node[fill=white!84!black,thick,minimum width=0.4cm,minimum height=0.4cm] at (3.8, 2.6) {};
\node[fill=white!55!black,thick,minimum width=0.4cm,minimum height=0.4cm] at (3, 2.6) {};
\node[fill=white!7!black,thick,minimum width=0.4cm,minimum height=0.4cm] at (1.8, 2.6) {};
\node[fill=white!11!black,thick,minimum width=0.4cm,minimum height=0.4cm] at (2.2, 3) {};
\node[fill=white!100!black,thick,minimum width=0.4cm,minimum height=0.4cm] at (4.6, 3) {};
\node[fill=white!1!black,thick,minimum width=0.4cm,minimum height=0.4cm] at (1, 3) {};
\node[fill=white!99!black,thick,minimum width=0.4cm,minimum height=0.4cm] at (4.2, 3) {};
\node[fill=white!0!black,thick,minimum width=0.4cm,minimum height=0.4cm] at (0.2, 3) {};
\node[fill=white!94!black,thick,minimum width=0.4cm,minimum height=0.4cm] at (3.4, 3) {};
\node[fill=white!89!black,thick,minimum width=0.4cm,minimum height=0.4cm] at (3.8, 3) {};
\node[fill=white!68!black,thick,minimum width=0.4cm,minimum height=0.4cm] at (3, 3) {};
\node[fill=white!15!black,thick,minimum width=0.4cm,minimum height=0.4cm] at (1.8, 3) {};
\node[fill=white!37!black,thick,minimum width=0.4cm,minimum height=0.4cm] at (2.6, 3) {};
\node[fill=white!7!black,thick,minimum width=0.4cm,minimum height=0.4cm] at (1.4, 3) {};
\node[fill=white!0!black,thick,minimum width=0.4cm,minimum height=0.4cm] at (0.6, 3) {};
\node[] at (2.4, 3.95){\normalsize $P(\textit{unobserved}) = 0.75$};
%\node[] at (2.4, -0.7) {\footnotesize instance complexity};
\draw[] (0, 0.4) -- (4.8, 0.4);
%\draw (-0.2, 0.37) -- (-0.2, 0.43);
\node[rotate=90,text width=0.9cm] at (-0.2, -0.1) {\makebox[4mm]{\hfill$l$}, $k$};
\draw (0.2, 0.37) -- (0.2, 0.43);
\node[rotate=90,text width=0.9cm] at (0.2, -0.1) {\makebox[4mm]{\hfill$20$}, $5$};
\draw (0.6, 0.37) -- (0.6, 0.43);
\node[rotate=90,text width=0.9cm] at (0.6, -0.1) {\makebox[4mm]{\hfill$10$}, $5$};
\draw (1, 0.37) -- (1, 0.43);
\node[rotate=90,text width=0.9cm] at (1, -0.1) {\makebox[4mm]{\hfill$20$}, $2$};
\draw (1.4, 0.37) -- (1.4, 0.43);
\node[rotate=90,text width=0.9cm] at (1.4, -0.1) {\makebox[4mm]{\hfill$10$}, $2$};
\draw (1.8, 0.37) -- (1.8, 0.43);
\node[rotate=90,text width=0.9cm] at (1.8, -0.1) {\makebox[4mm]{\hfill$5$}, $5$};
\draw (2.2, 0.37) -- (2.2, 0.43);
\node[rotate=90,text width=0.9cm] at (2.2, -0.1) {\makebox[4mm]{\hfill$20$}, $1$};
\draw (2.6, 0.37) -- (2.6, 0.43);
\node[rotate=90,text width=0.9cm] at (2.6, -0.1) {\makebox[4mm]{\hfill$10$}, $1$};
\draw (3, 0.37) -- (3, 0.43);
\node[rotate=90,text width=0.9cm] at (3, -0.1) {\makebox[4mm]{\hfill$5$}, $2$};
\draw (3.4, 0.37) -- (3.4, 0.43);
\node[rotate=90,text width=0.9cm] at (3.4, -0.1) {\makebox[4mm]{\hfill$2$}, $5$};
\draw (3.8, 0.37) -- (3.8, 0.43);
\node[rotate=90,text width=0.9cm] at (3.8, -0.1) {\makebox[4mm]{\hfill$5$}, $1$};
\draw (4.2, 0.37) -- (4.2, 0.43);
\node[rotate=90,text width=0.9cm] at (4.2, -0.1) {\makebox[4mm]{\hfill$2$}, $2$};
\draw (4.6, 0.37) -- (4.6, 0.43);
\node[rotate=90,text width=0.9cm] at (4.6, -0.1) {\makebox[4mm]{\hfill$2$}, $1$};
%\node[rotate=90] at (-1.5, 1.6) {\footnotesize number of nodes};
\draw[] (0, 0.4) -- (0, 3.2);
% % \draw (-0.03, 0.2) -- (0.03, 0.2);
% % \node[] at (-0.75, 0.2) {10};
%\draw (-0.03, 0.6) -- (0.03, 0.6);
%\node[] at (-0.75, 0.6) {25};
%\draw (-0.03, 1) -- (0.03, 1);
%\node[] at (-0.75, 1) {50};
%\draw (-0.03, 1.4) -- (0.03, 1.4);
%\node[] at (-0.75, 1.4) {100};
%\draw (-0.03, 1.8) -- (0.03, 1.8);
%\node[] at (-0.75, 1.8) {250};
%\draw (-0.03, 2.2) -- (0.03, 2.2);
%\node[] at (-0.75, 2.2) {500};
%\draw (-0.03, 2.6) -- (0.03, 2.6);
%\node[] at (-0.75, 2.6) {1000};
%\draw (-0.03, 3) -- (0.03, 3);
%\node[] at (-0.75, 3) {2000};
%\node[rotate=90] at (6.3, 1.6) {number of nodes};
\draw[] (4.8, 0.4) -- (4.8, 3.2);
% % \draw (4.77, 0.2) -- (4.83, 0.2);
% % \node[] at (5.55, 0.2) {10};
%\draw (4.77, 0.6) -- (4.83, 0.6);
%\node[] at (5.55, 0.6) {25};
%\draw (4.77, 1) -- (4.83, 1);
%\node[] at (5.55, 1) {50};
%\draw (4.77, 1.4) -- (4.83, 1.4);
%\node[] at (5.55, 1.4) {100};
%\draw (4.77, 1.8) -- (4.83, 1.8);
%\node[] at (5.55, 1.8) {250};
%\draw (4.77, 2.2) -- (4.83, 2.2);
%\node[] at (5.55, 2.2) {500};
%\draw (4.77, 2.6) -- (4.83, 2.6);
%\node[] at (5.55, 2.6) {1000};
%\draw (4.77, 3) -- (4.83, 3);
%\node[] at (5.55, 3) {2000};
\end{tikzpicture}&
\multirow{2}{2cm}[3cm]{
\begin{tikzpicture}[yscale=0.003cm]
%\node at (0,105) {\normalsize Legend:};
%\foreach \x in {0,5,10,15,20,25,30,35,40,45,50,55,60,65,70,75,80,85,90,95,100} {
%  \node[fill=white!\x!black,thick,minimum width=0.7cm,minimum height=0.4cm] at (0,\x) {\x \%};
%}
\node at (0.5,107) {\normalsize Legend:};
\foreach \x in {50,55,60,65,70,75,80,85,90,95,100} {
  \node[fill=white!\x!black,thick,minimum width=0.7cm,minimum height=0.4cm] at (0,\x) {\x \%};
}
\begin{scope}[yshift=50cm]
\foreach \x in {0,5,10,15,20,25,30,35,40,45,50} {
  \node[fill=white!\x!black,thick,minimum width=0.7cm,minimum height=0.4cm] at (1,\x) {\color{white} \x \%};
}
\end{scope}
\end{tikzpicture}
}
%%
%}
\\[3mm]
%%
%%left bottom:
\begin{tikzpicture}[minimum size=0]
% % \node[fill=white!100!black,thick,minimum width=0.4cm,minimum height=0.4cm] at (3.4, 0.2) {};
% % \node[fill=white!86!black,thick,minimum width=0.4cm,minimum height=0.4cm] at (3, 0.2) {};
% % \node[fill=white!17!black,thick,minimum width=0.4cm,minimum height=0.4cm] at (2.2, 0.2) {};
% % \node[fill=white!100!black,thick,minimum width=0.4cm,minimum height=0.4cm] at (3.8, 0.2) {};
% % \node[fill=white!43!black,thick,minimum width=0.4cm,minimum height=0.4cm] at (2.6, 0.2) {};
% % \node[fill=white!2!black,thick,minimum width=0.4cm,minimum height=0.4cm] at (1.4, 0.2) {};
% % \node[fill=white!100!black,thick,minimum width=0.4cm,minimum height=0.4cm] at (4.2, 0.2) {};
% % \node[fill=white!22!black,thick,minimum width=0.4cm,minimum height=0.4cm] at (1.8, 0.2) {};
% % \node[fill=white!1!black,thick,minimum width=0.4cm,minimum height=0.4cm] at (0.6, 0.2) {};
% % \node[fill=white!100!black,thick,minimum width=0.4cm,minimum height=0.4cm] at (4.6, 0.2) {};
% % \node[fill=white!24!black,thick,minimum width=0.4cm,minimum height=0.4cm] at (1, 0.2) {};
% % \node[fill=white!1!black,thick,minimum width=0.4cm,minimum height=0.4cm] at (0.2, 0.2) {};
\node[fill=white!100!black,thick,minimum width=0.4cm,minimum height=0.4cm] at (4.2, 0.6) {};
\node[fill=white!37!black,thick,minimum width=0.4cm,minimum height=0.4cm] at (1.8, 0.6) {};
\node[fill=white!1!black,thick,minimum width=0.4cm,minimum height=0.4cm] at (0.6, 0.6) {};
\node[fill=white!100!black,thick,minimum width=0.4cm,minimum height=0.4cm] at (3.4, 0.6) {};
\node[fill=white!97!black,thick,minimum width=0.4cm,minimum height=0.4cm] at (3, 0.6) {};
\node[fill=white!51!black,thick,minimum width=0.4cm,minimum height=0.4cm] at (2.2, 0.6) {};
\node[fill=white!100!black,thick,minimum width=0.4cm,minimum height=0.4cm] at (4.6, 0.6) {};
\node[fill=white!21!black,thick,minimum width=0.4cm,minimum height=0.4cm] at (1, 0.6) {};
\node[fill=white!1!black,thick,minimum width=0.4cm,minimum height=0.4cm] at (0.2, 0.6) {};
\node[fill=white!100!black,thick,minimum width=0.4cm,minimum height=0.4cm] at (3.8, 0.6) {};
\node[fill=white!70!black,thick,minimum width=0.4cm,minimum height=0.4cm] at (2.6, 0.6) {};
\node[fill=white!6!black,thick,minimum width=0.4cm,minimum height=0.4cm] at (1.4, 0.6) {};
\node[fill=white!100!black,thick,minimum width=0.4cm,minimum height=0.4cm] at (3.4, 1) {};
\node[fill=white!99!black,thick,minimum width=0.4cm,minimum height=0.4cm] at (3, 1) {};
\node[fill=white!75!black,thick,minimum width=0.4cm,minimum height=0.4cm] at (2.2, 1) {};
\node[fill=white!100!black,thick,minimum width=0.4cm,minimum height=0.4cm] at (3.8, 1) {};
\node[fill=white!84!black,thick,minimum width=0.4cm,minimum height=0.4cm] at (2.6, 1) {};
\node[fill=white!16!black,thick,minimum width=0.4cm,minimum height=0.4cm] at (1.4, 1) {};
\node[fill=white!100!black,thick,minimum width=0.4cm,minimum height=0.4cm] at (4.2, 1) {};
\node[fill=white!49!black,thick,minimum width=0.4cm,minimum height=0.4cm] at (1.8, 1) {};
\node[fill=white!2!black,thick,minimum width=0.4cm,minimum height=0.4cm] at (0.6, 1) {};
\node[fill=white!100!black,thick,minimum width=0.4cm,minimum height=0.4cm] at (4.6, 1) {};
\node[fill=white!27!black,thick,minimum width=0.4cm,minimum height=0.4cm] at (1, 1) {};
\node[fill=white!1!black,thick,minimum width=0.4cm,minimum height=0.4cm] at (0.2, 1) {};
\node[fill=white!100!black,thick,minimum width=0.4cm,minimum height=0.4cm] at (4.2, 1.4) {};
\node[fill=white!62!black,thick,minimum width=0.4cm,minimum height=0.4cm] at (1.8, 1.4) {};
\node[fill=white!4!black,thick,minimum width=0.4cm,minimum height=0.4cm] at (0.6, 1.4) {};
\node[fill=white!100!black,thick,minimum width=0.4cm,minimum height=0.4cm] at (4.6, 1.4) {};
\node[fill=white!35!black,thick,minimum width=0.4cm,minimum height=0.4cm] at (1, 1.4) {};
\node[fill=white!1!black,thick,minimum width=0.4cm,minimum height=0.4cm] at (0.2, 1.4) {};
\node[fill=white!100!black,thick,minimum width=0.4cm,minimum height=0.4cm] at (3.4, 1.4) {};
\node[fill=white!100!black,thick,minimum width=0.4cm,minimum height=0.4cm] at (3, 1.4) {};
\node[fill=white!91!black,thick,minimum width=0.4cm,minimum height=0.4cm] at (2.2, 1.4) {};
\node[fill=white!100!black,thick,minimum width=0.4cm,minimum height=0.4cm] at (3.8, 1.4) {};
\node[fill=white!92!black,thick,minimum width=0.4cm,minimum height=0.4cm] at (2.6, 1.4) {};
\node[fill=white!34!black,thick,minimum width=0.4cm,minimum height=0.4cm] at (1.4, 1.4) {};
\node[fill=white!100!black,thick,minimum width=0.4cm,minimum height=0.4cm] at (3.4, 1.8) {};
\node[fill=white!100!black,thick,minimum width=0.4cm,minimum height=0.4cm] at (3, 1.8) {};
\node[fill=white!98!black,thick,minimum width=0.4cm,minimum height=0.4cm] at (2.2, 1.8) {};
\node[fill=white!100!black,thick,minimum width=0.4cm,minimum height=0.4cm] at (4.6, 1.8) {};
\node[fill=white!45!black,thick,minimum width=0.4cm,minimum height=0.4cm] at (1, 1.8) {};
\node[fill=white!2!black,thick,minimum width=0.4cm,minimum height=0.4cm] at (0.2, 1.8) {};
\node[fill=white!100!black,thick,minimum width=0.4cm,minimum height=0.4cm] at (3.8, 1.8) {};
\node[fill=white!97!black,thick,minimum width=0.4cm,minimum height=0.4cm] at (2.6, 1.8) {};
\node[fill=white!63!black,thick,minimum width=0.4cm,minimum height=0.4cm] at (1.4, 1.8) {};
\node[fill=white!100!black,thick,minimum width=0.4cm,minimum height=0.4cm] at (4.2, 1.8) {};
\node[fill=white!74!black,thick,minimum width=0.4cm,minimum height=0.4cm] at (1.8, 1.8) {};
\node[fill=white!10!black,thick,minimum width=0.4cm,minimum height=0.4cm] at (0.6, 1.8) {};
\node[fill=white!100!black,thick,minimum width=0.4cm,minimum height=0.4cm] at (3.8, 2.2) {};
\node[fill=white!99!black,thick,minimum width=0.4cm,minimum height=0.4cm] at (2.6, 2.2) {};
\node[fill=white!81!black,thick,minimum width=0.4cm,minimum height=0.4cm] at (1.4, 2.2) {};
\node[fill=white!100!black,thick,minimum width=0.4cm,minimum height=0.4cm] at (4.2, 2.2) {};
\node[fill=white!81!black,thick,minimum width=0.4cm,minimum height=0.4cm] at (1.8, 2.2) {};
\node[fill=white!17!black,thick,minimum width=0.4cm,minimum height=0.4cm] at (0.6, 2.2) {};
\node[fill=white!100!black,thick,minimum width=0.4cm,minimum height=0.4cm] at (4.6, 2.2) {};
\node[fill=white!100!black,thick,minimum width=0.4cm,minimum height=0.4cm] at (3.4, 2.2) {};
\node[fill=white!50!black,thick,minimum width=0.4cm,minimum height=0.4cm] at (1, 2.2) {};
\node[fill=white!100!black,thick,minimum width=0.4cm,minimum height=0.4cm] at (3, 2.2) {};
\node[fill=white!2!black,thick,minimum width=0.4cm,minimum height=0.4cm] at (0.2, 2.2) {};
\node[fill=white!99!black,thick,minimum width=0.4cm,minimum height=0.4cm] at (2.2, 2.2) {};
\node[fill=white!100!black,thick,minimum width=0.4cm,minimum height=0.4cm] at (4.2, 2.6) {};
\node[fill=white!86!black,thick,minimum width=0.4cm,minimum height=0.4cm] at (1.8, 2.6) {};
\node[fill=white!27!black,thick,minimum width=0.4cm,minimum height=0.4cm] at (0.6, 2.6) {};
\node[fill=white!100!black,thick,minimum width=0.4cm,minimum height=0.4cm] at (4.6, 2.6) {};
\node[fill=white!100!black,thick,minimum width=0.4cm,minimum height=0.4cm] at (3.4, 2.6) {};
\node[fill=white!56!black,thick,minimum width=0.4cm,minimum height=0.4cm] at (1, 2.6) {};
\node[fill=white!100!black,thick,minimum width=0.4cm,minimum height=0.4cm] at (3, 2.6) {};
\node[fill=white!3!black,thick,minimum width=0.4cm,minimum height=0.4cm] at (0.2, 2.6) {};
\node[fill=white!100!black,thick,minimum width=0.4cm,minimum height=0.4cm] at (2.2, 2.6) {};
\node[fill=white!100!black,thick,minimum width=0.4cm,minimum height=0.4cm] at (3.8, 2.6) {};
\node[fill=white!100!black,thick,minimum width=0.4cm,minimum height=0.4cm] at (2.6, 2.6) {};
\node[fill=white!92!black,thick,minimum width=0.4cm,minimum height=0.4cm] at (1.4, 2.6) {};
\node[fill=white!100!black,thick,minimum width=0.4cm,minimum height=0.4cm] at (4.6, 3) {};
\node[fill=white!100!black,thick,minimum width=0.4cm,minimum height=0.4cm] at (3.4, 3) {};
\node[fill=white!62!black,thick,minimum width=0.4cm,minimum height=0.4cm] at (1, 3) {};
\node[fill=white!100!black,thick,minimum width=0.4cm,minimum height=0.4cm] at (3, 3) {};
\node[fill=white!5!black,thick,minimum width=0.4cm,minimum height=0.4cm] at (0.2, 3) {};
\node[fill=white!100!black,thick,minimum width=0.4cm,minimum height=0.4cm] at (2.2, 3) {};
\node[fill=white!100!black,thick,minimum width=0.4cm,minimum height=0.4cm] at (3.8, 3) {};
\node[fill=white!100!black,thick,minimum width=0.4cm,minimum height=0.4cm] at (2.6, 3) {};
\node[fill=white!97!black,thick,minimum width=0.4cm,minimum height=0.4cm] at (1.4, 3) {};
\node[fill=white!100!black,thick,minimum width=0.4cm,minimum height=0.4cm] at (4.2, 3) {};
\node[fill=white!91!black,thick,minimum width=0.4cm,minimum height=0.4cm] at (1.8, 3) {};
\node[fill=white!39!black,thick,minimum width=0.4cm,minimum height=0.4cm] at (0.6, 3) {};
\node[] at (2.4, -0.7) {\footnotesize instance complexity};
\draw[] (0, 0.4) -- (4.8, 0.4);
%\draw (-0.2, 0.37) -- (-0.2, 0.43);
\node[rotate=90,text width=0.9cm] at (-0.2, -0.1) {\makebox[4mm]{\hfill$l$}, $k$};
\draw (0.2, 0.37) -- (0.2, 0.43);
\node[rotate=90,text width=0.9cm] at (0.2, -0.1) {\makebox[4mm]{\hfill$20$}, $5$};
\draw (0.6, 0.37) -- (0.6, 0.43);
\node[rotate=90,text width=0.9cm] at (0.6, -0.1) {\makebox[4mm]{\hfill$10$}, $5$};
\draw (1, 0.37) -- (1, 0.43);
\node[rotate=90,text width=0.9cm] at (1, -0.1) {\makebox[4mm]{\hfill$20$}, $2$};
\draw (1.4, 0.37) -- (1.4, 0.43);
\node[rotate=90,text width=0.9cm] at (1.4, -0.1) {\makebox[4mm]{\hfill$5$}, $5$};
\draw (1.8, 0.37) -- (1.8, 0.43);
\node[rotate=90,text width=0.9cm] at (1.8, -0.1) {\makebox[4mm]{\hfill$10$}, $2$};
\draw (2.2, 0.37) -- (2.2, 0.43);
\node[rotate=90,text width=0.9cm] at (2.2, -0.1) {\makebox[4mm]{\hfill$2$}, $5$};
\draw (2.6, 0.37) -- (2.6, 0.43);
\node[rotate=90,text width=0.9cm] at (2.6, -0.1) {\makebox[4mm]{\hfill$5$}, $2$};
\draw (3, 0.37) -- (3, 0.43);
\node[rotate=90,text width=0.9cm] at (3, -0.1) {\makebox[4mm]{\hfill$2$}, $2$};
\draw (3.4, 0.37) -- (3.4, 0.43);
\node[rotate=90,text width=0.9cm] at (3.4, -0.1) {\makebox[4mm]{\hfill$2$}, $1$};
\draw (3.8, 0.37) -- (3.8, 0.43);
\node[rotate=90,text width=0.9cm] at (3.8, -0.1) {\makebox[4mm]{\hfill$5$}, $1$};
\draw (4.2, 0.37) -- (4.2, 0.43);
\node[rotate=90,text width=0.9cm] at (4.2, -0.1) {\makebox[4mm]{\hfill$10$}, $1$};
\draw (4.6, 0.37) -- (4.6, 0.43);
\node[rotate=90,text width=0.9cm] at (4.6, -0.1) {\makebox[4mm]{\hfill$20$}, $1$};
\node[rotate=90] at (-1.5, 1.6) {\footnotesize number of nodes};
\draw[] (0, 0.4) -- (0, 3.2);
% % \draw (-0.03, 0.2) -- (0.03, 0.2);
% % \node[] at (-0.75, 0.2) {10};
\draw (-0.03, 0.6) -- (0.03, 0.6);
\node[] at (-0.75, 0.6) {25};
\draw (-0.03, 1) -- (0.03, 1);
\node[] at (-0.75, 1) {50};
\draw (-0.03, 1.4) -- (0.03, 1.4);
\node[] at (-0.75, 1.4) {100};
\draw (-0.03, 1.8) -- (0.03, 1.8);
\node[] at (-0.75, 1.8) {250};
\draw (-0.03, 2.2) -- (0.03, 2.2);
\node[] at (-0.75, 2.2) {500};
\draw (-0.03, 2.6) -- (0.03, 2.6);
\node[] at (-0.75, 2.6) {1000};
\draw (-0.03, 3) -- (0.03, 3);
\node[] at (-0.75, 3) {2000};
\node[rotate=90] at (-2.5, 1.6) {\normalsize CBC$^+$};
%\node[rotate=90] at (6.3, 1.6) {number of nodes};
\draw[] (4.8, 0.4) -- (4.8, 3.2);
% % \draw (4.77, 0.2) -- (4.83, 0.2);
% % \node[] at (5.55, 0.2) {10};
%\draw (4.77, 0.6) -- (4.83, 0.6);
%\node[] at (5.55, 0.6) {25};
%\draw (4.77, 1) -- (4.83, 1);
%\node[] at (5.55, 1) {50};
%\draw (4.77, 1.4) -- (4.83, 1.4);
%\node[] at (5.55, 1.4) {100};
%\draw (4.77, 1.8) -- (4.83, 1.8);
%\node[] at (5.55, 1.8) {250};
%\draw (4.77, 2.2) -- (4.83, 2.2);
%\node[] at (5.55, 2.2) {500};
%\draw (4.77, 2.6) -- (4.83, 2.6);
%\node[] at (5.55, 2.6) {1000};
%\draw (4.77, 3) -- (4.83, 3);
%\node[] at (5.55, 3) {2000};
\end{tikzpicture}
&
\begin{tikzpicture}[minimum size=0]
% % \node[fill=white!96!black,thick,minimum width=0.4cm,minimum height=0.4cm] at (4.6, 0.2) {};
% % \node[fill=white!70!black,thick,minimum width=0.4cm,minimum height=0.4cm] at (4.2, 0.2) {};
% % \node[fill=white!18!black,thick,minimum width=0.4cm,minimum height=0.4cm] at (2.6, 0.2) {};
% % \node[fill=white!75!black,thick,minimum width=0.4cm,minimum height=0.4cm] at (3.8, 0.2) {};
% % \node[fill=white!29!black,thick,minimum width=0.4cm,minimum height=0.4cm] at (3, 0.2) {};
% % \node[fill=white!2!black,thick,minimum width=0.4cm,minimum height=0.4cm] at (1.4, 0.2) {};
% % \node[fill=white!59!black,thick,minimum width=0.4cm,minimum height=0.4cm] at (3.4, 0.2) {};
% % \node[fill=white!17!black,thick,minimum width=0.4cm,minimum height=0.4cm] at (1.8, 0.2) {};
% % \node[fill=white!1!black,thick,minimum width=0.4cm,minimum height=0.4cm] at (0.6, 0.2) {};
% % \node[fill=white!59!black,thick,minimum width=0.4cm,minimum height=0.4cm] at (2.2, 0.2) {};
% % \node[fill=white!17!black,thick,minimum width=0.4cm,minimum height=0.4cm] at (1, 0.2) {};
% % \node[fill=white!1!black,thick,minimum width=0.4cm,minimum height=0.4cm] at (0.2, 0.2) {};
\begin{scope}[yshift=0.4cm]
\node[fill=white!69!black,thick,minimum width=0.4cm,minimum height=0.4cm] at (3.4, 0.2) {};
\node[fill=white!26!black,thick,minimum width=0.4cm,minimum height=0.4cm] at (1.8, 0.2) {};
\node[fill=white!1!black,thick,minimum width=0.4cm,minimum height=0.4cm] at (0.6, 0.2) {};
\node[fill=white!99!black,thick,minimum width=0.4cm,minimum height=0.4cm] at (4.6, 0.2) {};
\node[fill=white!88!black,thick,minimum width=0.4cm,minimum height=0.4cm] at (4.2, 0.2) {};
\node[fill=white!30!black,thick,minimum width=0.4cm,minimum height=0.4cm] at (3, 0.2) {};
\node[fill=white!57!black,thick,minimum width=0.4cm,minimum height=0.4cm] at (2.2, 0.2) {};
\node[fill=white!18!black,thick,minimum width=0.4cm,minimum height=0.4cm] at (1, 0.2) {};
\node[fill=white!0!black,thick,minimum width=0.4cm,minimum height=0.4cm] at (0.2, 0.2) {};
\node[fill=white!87!black,thick,minimum width=0.4cm,minimum height=0.4cm] at (3.8, 0.2) {};
\node[fill=white!46!black,thick,minimum width=0.4cm,minimum height=0.4cm] at (2.6, 0.2) {};
\node[fill=white!5!black,thick,minimum width=0.4cm,minimum height=0.4cm] at (1.4, 0.2) {};
\node[fill=white!100!black,thick,minimum width=0.4cm,minimum height=0.4cm] at (4.6, 0.6) {}; %n=50,k=2,l=1
\node[fill=white!96!black,thick,minimum width=0.4cm,minimum height=0.4cm] at (4.2, 0.6) {};%n=50,k=2,l=2
\node[fill=white!55!black,thick,minimum width=0.4cm,minimum height=0.4cm] at (3, 0.6) {};  %n=50,k=2,l=5
\node[fill=white!92!black,thick,minimum width=0.4cm,minimum height=0.4cm] at (3.8, 0.6) {};%n=50,k=5,l=1
\node[fill=white!61!black,thick,minimum width=0.4cm,minimum height=0.4cm] at (2.6, 0.6) {};%n=50,k=5,l=2
\node[fill=white!0!black,thick,minimum width=0.4cm,minimum height=0.4cm] at (1.4, 0.6) {}; %n=50,k=5,l=5
\node[fill=white!75!black,thick,minimum width=0.4cm,minimum height=0.4cm] at (3.4, 0.6) {};%n=50,k=10,l=1
\node[fill=white!32!black,thick,minimum width=0.4cm,minimum height=0.4cm] at (1.8, 0.6) {};%n=50,k=10,l=2
\node[fill=white!0!black,thick,minimum width=0.4cm,minimum height=0.4cm] at (0.6, 0.6) {}; %n=50,k=10,l=5
\node[fill=white!60!black,thick,minimum width=0.4cm,minimum height=0.4cm] at (2.2, 0.6) {};%n=50,k=20,l=1
\node[fill=white!21!black,thick,minimum width=0.4cm,minimum height=0.4cm] at (1, 0.6) {};  %n=50,k=20,l=2
\node[fill=white!0!black,thick,minimum width=0.4cm,minimum height=0.4cm] at (0.2, 0.6) {}; %n=50,k=20,l=5
\node[fill=white!80!black,thick,minimum width=0.4cm,minimum height=0.4cm] at (3.4, 1) {};  
\node[fill=white!38!black,thick,minimum width=0.4cm,minimum height=0.4cm] at (1.8, 1) {};
\node[fill=white!3!black,thick,minimum width=0.4cm,minimum height=0.4cm] at (0.6, 1) {};
\node[fill=white!64!black,thick,minimum width=0.4cm,minimum height=0.4cm] at (2.2, 1) {};
\node[fill=white!24!black,thick,minimum width=0.4cm,minimum height=0.4cm] at (1, 1) {};
\node[fill=white!1!black,thick,minimum width=0.4cm,minimum height=0.4cm] at (0.2, 1) {};
\node[fill=white!100!black,thick,minimum width=0.4cm,minimum height=0.4cm] at (4.6, 1) {};
\node[fill=white!99!black,thick,minimum width=0.4cm,minimum height=0.4cm] at (4.2, 1) {};
\node[fill=white!76!black,thick,minimum width=0.4cm,minimum height=0.4cm] at (3, 1) {};
\node[fill=white!96!black,thick,minimum width=0.4cm,minimum height=0.4cm] at (3.8, 1) {};
\node[fill=white!73!black,thick,minimum width=0.4cm,minimum height=0.4cm] at (2.6, 1) {};
\node[fill=white!18!black,thick,minimum width=0.4cm,minimum height=0.4cm] at (1.4, 1) {};
\node[fill=white!100!black,thick,minimum width=0.4cm,minimum height=0.4cm] at (4.6, 1.4) {};
\node[fill=white!100!black,thick,minimum width=0.4cm,minimum height=0.4cm] at (4.2, 1.4) {};
\node[fill=white!93!black,thick,minimum width=0.4cm,minimum height=0.4cm] at (3, 1.4) {};
\node[fill=white!68!black,thick,minimum width=0.4cm,minimum height=0.4cm] at (2.2, 1.4) {};
\node[fill=white!29!black,thick,minimum width=0.4cm,minimum height=0.4cm] at (1, 1.4) {};
\node[fill=white!1!black,thick,minimum width=0.4cm,minimum height=0.4cm] at (0.2, 1.4) {};
\node[fill=white!98!black,thick,minimum width=0.4cm,minimum height=0.4cm] at (3.8, 1.4) {};
\node[fill=white!85!black,thick,minimum width=0.4cm,minimum height=0.4cm] at (2.6, 1.4) {};
\node[fill=white!33!black,thick,minimum width=0.4cm,minimum height=0.4cm] at (1.4, 1.4) {};
\node[fill=white!86!black,thick,minimum width=0.4cm,minimum height=0.4cm] at (3.4, 1.4) {};
\node[fill=white!49!black,thick,minimum width=0.4cm,minimum height=0.4cm] at (1.8, 1.4) {};
\node[fill=white!6!black,thick,minimum width=0.4cm,minimum height=0.4cm] at (0.6, 1.4) {};
\node[fill=white!99!black,thick,minimum width=0.4cm,minimum height=0.4cm] at (3.8, 1.8) {};
\node[fill=white!91!black,thick,minimum width=0.4cm,minimum height=0.4cm] at (2.6, 1.8) {};
\node[fill=white!50!black,thick,minimum width=0.4cm,minimum height=0.4cm] at (1.4, 1.8) {};
\node[fill=white!89!black,thick,minimum width=0.4cm,minimum height=0.4cm] at (3.4, 1.8) {};
\node[fill=white!57!black,thick,minimum width=0.4cm,minimum height=0.4cm] at (1.8, 1.8) {};
\node[fill=white!10!black,thick,minimum width=0.4cm,minimum height=0.4cm] at (0.6, 1.8) {};
\node[fill=white!72!black,thick,minimum width=0.4cm,minimum height=0.4cm] at (2.2, 1.8) {};
\node[fill=white!100!black,thick,minimum width=0.4cm,minimum height=0.4cm] at (4.6, 1.8) {};
\node[fill=white!32!black,thick,minimum width=0.4cm,minimum height=0.4cm] at (1, 1.8) {};
\node[fill=white!100!black,thick,minimum width=0.4cm,minimum height=0.4cm] at (4.2, 1.8) {};
\node[fill=white!2!black,thick,minimum width=0.4cm,minimum height=0.4cm] at (0.2, 1.8) {};
\node[fill=white!98!black,thick,minimum width=0.4cm,minimum height=0.4cm] at (3, 1.8) {};
\node[fill=white!91!black,thick,minimum width=0.4cm,minimum height=0.4cm] at (3.4, 2.2) {};
\node[fill=white!64!black,thick,minimum width=0.4cm,minimum height=0.4cm] at (1.8, 2.2) {};
\node[fill=white!15!black,thick,minimum width=0.4cm,minimum height=0.4cm] at (0.6, 2.2) {};
\node[fill=white!74!black,thick,minimum width=0.4cm,minimum height=0.4cm] at (2.2, 2.2) {};
\node[fill=white!100!black,thick,minimum width=0.4cm,minimum height=0.4cm] at (4.6, 2.2) {};
\node[fill=white!35!black,thick,minimum width=0.4cm,minimum height=0.4cm] at (1, 2.2) {};
\node[fill=white!100!black,thick,minimum width=0.4cm,minimum height=0.4cm] at (4.2, 2.2) {};
\node[fill=white!2!black,thick,minimum width=0.4cm,minimum height=0.4cm] at (0.2, 2.2) {};
\node[fill=white!99!black,thick,minimum width=0.4cm,minimum height=0.4cm] at (3, 2.2) {};
\node[fill=white!100!black,thick,minimum width=0.4cm,minimum height=0.4cm] at (3.8, 2.2) {};
\node[fill=white!96!black,thick,minimum width=0.4cm,minimum height=0.4cm] at (2.6, 2.2) {};
\node[fill=white!65!black,thick,minimum width=0.4cm,minimum height=0.4cm] at (1.4, 2.2) {};
\node[fill=white!76!black,thick,minimum width=0.4cm,minimum height=0.4cm] at (2.2, 2.6) {};
\node[fill=white!100!black,thick,minimum width=0.4cm,minimum height=0.4cm] at (4.6, 2.6) {};
\node[fill=white!39!black,thick,minimum width=0.4cm,minimum height=0.4cm] at (1, 2.6) {};
\node[fill=white!100!black,thick,minimum width=0.4cm,minimum height=0.4cm] at (4.2, 2.6) {};
\node[fill=white!3!black,thick,minimum width=0.4cm,minimum height=0.4cm] at (0.2, 2.6) {};
\node[fill=white!100!black,thick,minimum width=0.4cm,minimum height=0.4cm] at (3, 2.6) {};
\node[fill=white!100!black,thick,minimum width=0.4cm,minimum height=0.4cm] at (3.8, 2.6) {};
\node[fill=white!98!black,thick,minimum width=0.4cm,minimum height=0.4cm] at (2.6, 2.6) {};
\node[fill=white!79!black,thick,minimum width=0.4cm,minimum height=0.4cm] at (1.4, 2.6) {};
\node[fill=white!94!black,thick,minimum width=0.4cm,minimum height=0.4cm] at (3.4, 2.6) {};
\node[fill=white!70!black,thick,minimum width=0.4cm,minimum height=0.4cm] at (1.8, 2.6) {};
\node[fill=white!22!black,thick,minimum width=0.4cm,minimum height=0.4cm] at (0.6, 2.6) {};
\end{scope}
\node[] at (2.4, -0.7) {\footnotesize instance complexity};
\draw[] (0, 0.4) -- (4.8, 0.4);
%\draw (-0.2, 0.37) -- (-0.2, 0.43);
\node[rotate=90,text width=0.9cm] at (-0.2, -0.1) {\makebox[4mm]{\hfill$l$}, $k$};
\draw (0.2, 0.37) -- (0.2, 0.43);
\node[rotate=90,text width=0.9cm] at (0.2, -0.1) {\makebox[4mm]{\hfill$20$}, $5$};
\draw (0.6, 0.37) -- (0.6, 0.43);
\node[rotate=90,text width=0.9cm] at (0.6, -0.1) {\makebox[4mm]{\hfill$10$}, $5$};
\draw (1, 0.37) -- (1, 0.43);
\node[rotate=90,text width=0.9cm] at (1, -0.1) {\makebox[4mm]{\hfill$20$}, $2$};
\draw (1.4, 0.37) -- (1.4, 0.43);
\node[rotate=90,text width=0.9cm] at (1.4, -0.1) {\makebox[4mm]{\hfill$5$}, $5$};
\draw (1.8, 0.37) -- (1.8, 0.43);
\node[rotate=90,text width=0.9cm] at (1.8, -0.1) {\makebox[4mm]{\hfill$10$}, $2$};
\draw (2.2, 0.37) -- (2.2, 0.43);
\node[rotate=90,text width=0.9cm] at (2.2, -0.1) {\makebox[4mm]{\hfill$20$}, $1$};
\draw (2.6, 0.37) -- (2.6, 0.43);
\node[rotate=90,text width=0.9cm] at (2.6, -0.1) {\makebox[4mm]{\hfill$5$}, $2$};
\draw (3, 0.37) -- (3, 0.43);
\node[rotate=90,text width=0.9cm] at (3, -0.1) {\makebox[4mm]{\hfill$2$}, $5$};
\draw (3.4, 0.37) -- (3.4, 0.43);
\node[rotate=90,text width=0.9cm] at (3.4, -0.1) {\makebox[4mm]{\hfill$10$}, $1$};
\draw (3.8, 0.37) -- (3.8, 0.43);
\node[rotate=90,text width=0.9cm] at (3.8, -0.1) {\makebox[4mm]{\hfill$5$}, $1$};
\draw (4.2, 0.37) -- (4.2, 0.43);
\node[rotate=90,text width=0.9cm] at (4.2, -0.1) {\makebox[4mm]{\hfill$2$}, $2$};
\draw (4.6, 0.37) -- (4.6, 0.43);
\node[rotate=90,text width=0.9cm] at (4.6, -0.1) {\makebox[4mm]{\hfill$2$}, $1$};
%\node[rotate=90] at (-1.5, 1.6) {\footnotesize number of nodes};
\draw[] (0, 0.4) -- (0, 3.2);
% % \draw (-0.03, 0.2) -- (0.03, 0.2);
% % \node[] at (-0.75, 0.2) {10};
%\draw (-0.03, 0.6) -- (0.03, 0.6);
%\node[] at (-0.75, 0.6) {25};
%\draw (-0.03, 1) -- (0.03, 1);
%\node[] at (-0.75, 1) {50};
%\draw (-0.03, 1.4) -- (0.03, 1.4);
%\node[] at (-0.75, 1.4) {100};
%\draw (-0.03, 1.8) -- (0.03, 1.8);
%\node[] at (-0.75, 1.8) {250};
%\draw (-0.03, 2.2) -- (0.03, 2.2);
%\node[] at (-0.75, 2.2) {500};
%\draw (-0.03, 2.6) -- (0.03, 2.6);
%\node[] at (-0.75, 2.6) {1000};
%\draw (-0.03, 3) -- (0.03, 3);
%\node[] at (-0.75, 3) {2000};
%\node[rotate=90] at (6.3, 1.6) {number of nodes};
\draw[] (4.8, 0.4) -- (4.8, 3.2);
% % \draw (4.77, 0.2) -- (4.83, 0.2);
%\node[] at (5.55, 0.2) {10};
%\draw (4.77, 0.6) -- (4.83, 0.6);
%\node[] at (5.55, 0.6) {25};
%\draw (4.77, 1) -- (4.83, 1);
%\node[] at (5.55, 1) {50};
%\draw (4.77, 1.4) -- (4.83, 1.4);
%\node[] at (5.55, 1.4) {100};
%\draw (4.77, 1.8) -- (4.83, 1.8);
%\node[] at (5.55, 1.8) {250};
%\draw (4.77, 2.2) -- (4.83, 2.2);
%\node[] at (5.55, 2.2) {500};
%\draw (4.77, 2.6) -- (4.83, 2.6);
%\node[] at (5.55, 2.6) {1000};
%\draw (4.77, 3) -- (4.83, 3);
%\node[] at (5.55, 3) {2000};
\end{tikzpicture}
\end{tabular}
\end{center}\vspace*{-3mm}
\caption{Heatmaps visualizing the number of instances that are identifiable by 
use our constructive back-door criterion (CBC, top row), and by the use 
of CBC or the plain formula (CBC$^+$, bottom row). 
Black squares depict the worst case in which $0\%$ of instances are identifiable by use
 of CBC or CBC$^+$, respectively.  
White squares mean that $100\%$ of instances are identifiable.
% and the different shades of grey illustrate 
%the values in between linearly. 
The instance complexities $l,k$
(where the expected number of neighbors of a node equals $l$
and $|\bX|=|\bY|=k$) are sorted by total amount of
identifiable instances.
}\label{table:global:stat:grid:grey}
\end{figure}

%% file: experiments-n-curves.tex
\begin{figure}\centering\footnotesize
\begin{tikzpicture}[minimum size=0,xscale=0.95,yscale=0.5]
\node[] at (6, -1) {};
%\draw[] (0, 0) -- (12, 0);
%\draw (0, -0.03) -- (0, 0.03);
\node[rotate=90,text width=2.4cm] at (0.5, -.5) {\makebox[4mm]{\hfill}\phantom{,} \makebox[6mm]{\hfill$l$, $k$}};
\draw (1, -0.03) -- (1, 0.03);
\node[rotate=90,text width=2.4cm] at (1, -0.5) {\makebox[4mm]{\hfill}\phantom{,} \makebox[6mm]{\hfill$20$, $5$}};
\draw (2, -0.03) -- (2, 0.03);
\node[rotate=90,text width=2.4cm] at (2, -0.5) {\makebox[4mm]{\hfill}\phantom{,} \makebox[6mm]{\hfill$10$, $5$}};
\draw (3, -0.03) -- (3, 0.03);
\node[rotate=90,text width=2.4cm] at (3, -0.5) {\makebox[4mm]{\hfill}\phantom{,} \makebox[6mm]{\hfill$5$, $5$}};
\draw (4, -0.03) -- (4, 0.03);
\node[rotate=90,text width=2.4cm] at (4, -0.5) {\makebox[4mm]{\hfill}\phantom{,} \makebox[6mm]{\hfill$2$, $5$}};
\draw (5, -0.03) -- (5, 0.03);
\node[rotate=90,text width=2.4cm] at (5, -0.5) {\makebox[4mm]{\hfill}\phantom{,} \makebox[6mm]{\hfill$20$, $2$}};
\draw (6, -0.03) -- (6, 0.03);
\node[rotate=90,text width=2.4cm] at (6, -0.5) {\makebox[4mm]{\hfill}\phantom{,} \makebox[6mm]{\hfill$10$, $2$}};
\draw (7, -0.03) -- (7, 0.03);
\node[rotate=90,text width=2.4cm] at (7, -0.5) {\makebox[4mm]{\hfill}\phantom{,} \makebox[6mm]{\hfill$5$, $2$}};
\draw (8, -0.03) -- (8, 0.03);
\node[rotate=90,text width=2.4cm] at (8, -0.5) {\makebox[4mm]{\hfill}\phantom{,} \makebox[6mm]{\hfill$2$, $2$}};
\draw (9, -0.03) -- (9, 0.03);
\node[rotate=90,text width=2.4cm] at (9, -0.5) {\makebox[4mm]{\hfill}\phantom{,} \makebox[6mm]{\hfill$20$, $1$}};
\draw (10, -0.03) -- (10, 0.03);
\node[rotate=90,text width=2.4cm] at (10, -0.5) {\makebox[4mm]{\hfill}\phantom{,} \makebox[6mm]{\hfill$10$, $1$}};
\draw (11, -0.03) -- (11, 0.03);
\node[rotate=90,text width=2.4cm] at (11, -0.5) {\makebox[4mm]{\hfill}\phantom{,} \makebox[6mm]{\hfill$5$, $1$}};
\draw (12, -0.03) -- (12, 0.03);
\node[rotate=90,text width=2.4cm] at (12, -0.5) {\makebox[4mm]{\hfill}\phantom{,} \makebox[6mm]{\hfill$2$, $1$}};
\node[rotate=90] at (-0.25, 5) {\% of identified instances};
%\draw[] (0, 0) -- (0, 10);
%\draw (-0.03, 0) -- (0.03, 0);
\node[] at (0.5, 0) {0\%};
%\draw (-0.03, 1) -- (0.03, 1);
\node[] at (0.5, 1) {10\%};
%\draw (-0.03, 2) -- (0.03, 2);
\node[] at (0.5, 2) {20\%};
%\draw (-0.03, 3) -- (0.03, 3);
\node[] at (0.5, 3) {30\%};
%\draw (-0.03, 4) -- (0.03, 4);
\node[] at (0.5, 4) {40\%};
%\draw (-0.03, 5) -- (0.03, 5);
\node[] at (0.5, 5) {50\%};
%\draw (-0.03, 6) -- (0.03, 6);
\node[] at (0.5, 6) {60\%};
%\draw (-0.03, 7) -- (0.03, 7);
\node[] at (0.5, 7) {70\%};
%\draw (-0.03, 8) -- (0.03, 8);
\node[] at (0.5, 8) {80\%};
%\draw (-0.03, 9) -- (0.03, 9);
\node[] at (0.5, 9) {90\%};
%\draw (-0.03, 10) -- (0.03, 10);
\node[] at (0.5, 10) {100\%};
%\node[rotate=90] at (13.5, 5) {\% of identified instances};
\draw[] (12, 0) -- (12, 10);
%\draw (11.97, 0) -- (12.03, 0);
%\node[] at (12.75, 0) {0\%};
%\draw (11.97, 1) -- (12.03, 1);
%\node[] at (12.75, 1) {10\%};
%\draw (11.97, 2) -- (12.03, 2);
%\node[] at (12.75, 2) {20\%};
%\draw (11.97, 3) -- (12.03, 3);
%\node[] at (12.75, 3) {30\%};
%\draw (11.97, 4) -- (12.03, 4);
%\node[] at (12.75, 4) {40\%};
%\draw (11.97, 5) -- (12.03, 5);
%\node[] at (12.75, 5) {50\%};
%\draw (11.97, 6) -- (12.03, 6);
%\node[] at (12.75, 6) {60\%};
%\draw (11.97, 7) -- (12.03, 7);
%\node[] at (12.75, 7) {70\%};
%\draw (11.97, 8) -- (12.03, 8);
%\node[] at (12.75, 8) {80\%};
%\draw (11.97, 9) -- (12.03, 9);
%\node[] at (12.75, 9) {90\%};
%\draw (11.97, 10) -- (12.03, 10);
%\node[] at (12.75, 10) {100\%};
\draw[        ] (1, 0) -- ( 4, 0) (5, 0) -- ( 8, 0) (9, 0) -- (12, 0);
\draw[black!30] (1, 1) -- ( 4, 1) (5, 1) -- ( 8, 1) (9, 1) -- (12, 1);
\draw[black!30] (1, 2) -- ( 4, 2) (5, 2) -- ( 8, 2) (9, 2) -- (12, 2);
\draw[black!30] (1, 3) -- ( 4, 3) (5, 3) -- ( 8, 3) (9, 3) -- (12, 3);
\draw[black!30] (1, 4) -- ( 4, 4) (5, 4) -- ( 8, 4) (9, 4) -- (12, 4);
\draw[black!30] (1, 5) -- ( 4, 5) (5, 5) -- ( 8, 5) (9, 5) -- (12, 5);
\draw[black!30] (1, 6) -- ( 4, 6) (5, 6) -- ( 8, 6) (9, 6) -- (12, 6);
\draw[black!30] (1, 7) -- ( 4, 7) (5, 7) -- ( 8, 7) (9, 7) -- (12, 7);
\draw[black!30] (1, 8) -- ( 4, 8) (5, 8) -- ( 8, 8) (9, 8) -- (12, 8);
\draw[black!30] (1, 9) -- ( 4, 9) (5, 9) -- ( 8, 9) (9, 9) -- (12, 9);
\draw[        ] (1,10) -- ( 4,10) (5,10) -- ( 8,10) (9,10) -- (12,10);
\draw[] (1, 0) -- (1, 10);
\draw[] (4, 0) -- (4, 10);
\draw[] (5, 0) -- (5, 10);
\draw[] (8, 0) -- (8, 10);
\draw[] (9, 0) -- (9, 10);
\draw[black] (1, 0.34) -- (2, 2.16) -- (3, 7.91) -- (4, 9.98)  
            (5, 3.86) -- (6, 7.05) -- (7, 9.79) -- (8, 10) 
            (9, 7.61) -- (10, 9.36) -- (11, 9.99) -- (12, 10); 
\draw[white!10!black, dash pattern={on 8pt off 2pt}] 
            (1, 0.24) -- (2, 1.47) -- (3, 6.48) -- (4, 9.95)  
            (5, 3.55) -- (6, 6.36) -- (7, 9.57) -- (8, 10) 
            (9, 7.42) -- (10, 9.14) -- (11, 9.97) -- (12, 10); 
\draw[white!15!black, dash pattern={on 6.5pt off 2pt}] 
            (1, 0.23) -- (2, 0.99) -- (3, 4.97) -- (4, 9.75) 
            (5, 3.17) -- (6, 5.72) -- (7, 9.12) -- (8, 9.99) 
            (9, 7.18) -- (10, 8.88) -- (11, 9.92) -- (12, 10); 
\draw[white!20!black, dash pattern={on 5pt off 2pt}] 
            (1, 0.14) -- (2, 0.6) -- (3, 3.34) -- (4, 9.27) 
            (5, 2.88) -- (6, 4.91) -- (7, 8.51) -- (8, 9.97)
            (9, 6.79) -- (10, 8.56) -- (11, 9.81) -- (12, 10); 
\draw[white!25!black, dash pattern={on 3.5pt off 2pt}] 
            (1, 0.08) -- (2, 0.33) -- (3, 1.8) -- (4, 7.64) 
            (5, 2.41) -- (6, 3.8) -- (7, 7.3) -- (8, 9.88) 
            (9, 6.41) -- (10, 7.98) -- (11, 9.6) -- (12, 9.99); 
%\draw[white!30!red, dash pattern={on 2pt off 2pt}]           old results
%            (1, 0.05) -- (2, 0.17) -- (3, 0.89) -- (4, 5.46) 
%            (5, 2.25) -- (6, 3.3) -- (7, 5.82) -- (8, 9.52) 
%            (9, 6.03) -- (10, 7.55) -- (11, 9.24) -- (12, 9.98); 
\draw[white!30!black, dash pattern={on 2pt off 2pt}] 
            (1, 0.07) -- (2, 0.18) -- (3, 0.87) -- (4, 5.45) 
            (5, 2.06) -- (6, 3.2) -- (7, 6.08) -- (8, 9.6) 
            (9, 6.03) -- (10, 7.55) -- (11, 9.24) -- (12, 9.98); 
\draw[white!35!black, dash pattern={on 1pt off 2pt}] (1, 0.05) -- (2, 0.11) -- (3, 0.46) -- (4, 3.01) 
              (5, 1.85) -- (6, 2.5) -- (7, 4.63) -- (8, 8.94) 
              (9, 5.7) -- (10, 6.88) -- (11, 8.73) -- (12, 9.92); 
\draw[draw=black,fill=white] (8.75, 5.25) rectangle (11.5, 2.75);
\draw[black] (9, 4.95) -- (10, 4.95); 
\node[anchor=west] at (10.1, 4.95) {\tiny $n = 2000$};
\draw[white!10!black, dash pattern={on 8pt off 2pt}] (9, 4.65) -- (10, 4.65); 
\node[anchor=west] at (10.1, 4.65) {\tiny $n = 1000$};
\draw[white!15!black, dash pattern={on 6.5pt off 2pt}] (9, 4.35) -- (10, 4.35); 
\node[anchor=west] at (10.1, 4.35) {\tiny $n = 500$};
\draw[white!20!black, dash pattern={on 5pt off 2pt}] (9, 4.05) -- (10, 4.05); 
\node[anchor=west] at (10.1, 4.05) {\tiny $n = 250$};
\draw[white!25!black, dash pattern={on 3.5pt off 2pt}] (9, 3.75) -- (10, 3.75); 
\node[anchor=west] at (10.1, 3.75) {\tiny $n = 100$};
\draw[white!30!black, dash pattern={on 2pt off 2pt}] (9, 3.45) -- (10, 3.45); 
\node[anchor=west] at (10.1, 3.45) {\tiny $n = 50$};
\draw[white!35!black, dash pattern={on 1pt off 2pt}] (9, 3.15) -- (10, 3.15); 
\node[anchor=west] at (10.1, 3.15) {\tiny $n = 25$};
%(9, 2.85) -- (10, 2.85); 
%\node[draw,black!60!orange,cross out] at (9.5, 2.85) {};
%\node[anchor=west] at (10.1, 2.85) {$n = 250$};
%(9, 2.55) -- (10, 2.55); 
%\node[draw,black!40!orange,cross out] at (9.5, 2.55) {};
%\node[anchor=west] at (10.1, 2.55) {$n = 100$};
%(9, 2.25) -- (10, 2.25); 
%\node[draw,black!20!orange,cross out] at (9.5, 2.25) {};
%\node[anchor=west] at (10.1, 2.25) {$n = 50$};
%(9, 1.95) -- (10, 1.95); 
%\node[draw,orange,cross out] at (9.5, 1.95) {};
%\node[anchor=west] at (10.1, 1.95) {$n = 25$};
\end{tikzpicture}
\caption{Case $P(\textit{unobserved}=0.75)$:
Percent of identifiable graphs for fixed numbers of nodes $n\in\{25,50,100,250,500,1000,2000\}$ and
with varying  expected number of node neighbors  $l$ and cardinalities 
$|\bX|=|\bY|=k$. The horizontal axis is labeled by $(l,k)$ sorted lexicographically.
The curves show the data for CBC$^+$, i.e., for instances identifiable by adjustment 
or by plain formula. 
%; The crosses represent the data for IDC, i.e. for cases that 
%are identifiable at all. The high time complexity of the IDC algorithm precluded computations
%for graphs of sizes  $n\ge 500$.
}\label{fig:global:n:curves}
\end{figure}

%% file: experiments-l-curves.tex
\begin{figure}\centering\footnotesize
\begin{tikzpicture}[minimum size=0,xscale=0.53,yscale=0.5]
\onlyinblackandwhite{\definecolor{orange}{gray}{0.5}}
\node[] at (10.5, -1) {};
%\draw[] (0, 0) -- (21, 0);
%\draw (0, -0.03) -- (0, 0.03);
\node[rotate=90,text width=2.4cm] at (0, -0.5) {\makebox[6mm]{\hfill$n$}, $k$\phantom{,} \makebox[4mm]{\hfill}};
\draw (1, -0.03) -- (1, 0.03);
\node[rotate=90,text width=2.4cm] at (1, -0.5) {\makebox[6mm]{\hfill$25$}, $5$\phantom{,} \makebox[4mm]{\hfill}};
\draw (2, -0.03) -- (2, 0.03);
\node[rotate=90,text width=2.4cm] at (2, -0.5) {\makebox[6mm]{\hfill$50$}, $5$\phantom{,} \makebox[4mm]{\hfill}};
\draw (3, -0.03) -- (3, 0.03);
\node[rotate=90,text width=2.4cm] at (3, -0.5) {\makebox[6mm]{\hfill$100$}, $5$\phantom{,} \makebox[4mm]{\hfill}};
\draw (4, -0.03) -- (4, 0.03);
\node[rotate=90,text width=2.4cm] at (4, -0.5) {\makebox[6mm]{\hfill$250$}, $5$\phantom{,} \makebox[4mm]{\hfill}};
\draw (5, -0.03) -- (5, 0.03);
\node[rotate=90,text width=2.4cm] at (5, -0.5) {\makebox[6mm]{\hfill$500$}, $5$\phantom{,} \makebox[4mm]{\hfill}};
\draw (6, -0.03) -- (6, 0.03);
\node[rotate=90,text width=2.4cm] at (6, -0.5) {\makebox[6mm]{\hfill$1000$}, $5$\phantom{,} \makebox[4mm]{\hfill}};
\draw (7, -0.03) -- (7, 0.03);
\node[rotate=90,text width=2.4cm] at (7, -0.5) {\makebox[6mm]{\hfill$2000$}, $5$\phantom{,} \makebox[4mm]{\hfill}};
\draw (8, -0.03) -- (8, 0.03);
\node[rotate=90,text width=2.4cm] at (8, -0.5) {\makebox[6mm]{\hfill$25$}, $2$\phantom{,} \makebox[4mm]{\hfill}};
\draw (9, -0.03) -- (9, 0.03);
\node[rotate=90,text width=2.4cm] at (9, -0.5) {\makebox[6mm]{\hfill$50$}, $2$\phantom{,} \makebox[4mm]{\hfill}};
\draw (10, -0.03) -- (10, 0.03);
\node[rotate=90,text width=2.4cm] at (10, -0.5) {\makebox[6mm]{\hfill$100$}, $2$\phantom{,} \makebox[4mm]{\hfill}};
\draw (11, -0.03) -- (11, 0.03);
\node[rotate=90,text width=2.4cm] at (11, -0.5) {\makebox[6mm]{\hfill$250$}, $2$\phantom{,} \makebox[4mm]{\hfill}};
\draw (12, -0.03) -- (12, 0.03);
\node[rotate=90,text width=2.4cm] at (12, -0.5) {\makebox[6mm]{\hfill$500$}, $2$\phantom{,} \makebox[4mm]{\hfill}};
\draw (13, -0.03) -- (13, 0.03);
\node[rotate=90,text width=2.4cm] at (13, -0.5) {\makebox[6mm]{\hfill$1000$}, $2$\phantom{,} \makebox[4mm]{\hfill}};
\draw (14, -0.03) -- (14, 0.03);
\node[rotate=90,text width=2.4cm] at (14, -0.5) {\makebox[6mm]{\hfill$2000$}, $2$\phantom{,} \makebox[4mm]{\hfill}};
\draw (15, -0.03) -- (15, 0.03);
\node[rotate=90,text width=2.4cm] at (15, -0.5) {\makebox[6mm]{\hfill$25$}, $1$\phantom{,} \makebox[4mm]{\hfill}};
\draw (16, -0.03) -- (16, 0.03);
\node[rotate=90,text width=2.4cm] at (16, -0.5) {\makebox[6mm]{\hfill$50$}, $1$\phantom{,} \makebox[4mm]{\hfill}};
\draw (17, -0.03) -- (17, 0.03);
\node[rotate=90,text width=2.4cm] at (17, -0.5) {\makebox[6mm]{\hfill$100$}, $1$\phantom{,} \makebox[4mm]{\hfill}};
\draw (18, -0.03) -- (18, 0.03);
\node[rotate=90,text width=2.4cm] at (18, -0.5) {\makebox[6mm]{\hfill$250$}, $1$\phantom{,} \makebox[4mm]{\hfill}};
\draw (19, -0.03) -- (19, 0.03);
\node[rotate=90,text width=2.4cm] at (19, -0.5) {\makebox[6mm]{\hfill$500$}, $1$\phantom{,} \makebox[4mm]{\hfill}};
\draw (20, -0.03) -- (20, 0.03);
\node[rotate=90,text width=2.4cm] at (20, -0.5) {\makebox[6mm]{\hfill$1000$}, $1$\phantom{,} \makebox[4mm]{\hfill}};
\draw (21, -0.03) -- (21, 0.03);
\node[rotate=90,text width=2.4cm] at (21, -0.5) {\makebox[6mm]{\hfill$2000$}, $1$\phantom{,} \makebox[4mm]{\hfill}};
\node[rotate=90] at (-1.5, 5) {\% of identified instances};
\draw[] (1, 0) -- (1, 10);
\draw[] (7, 0) -- (7, 10);
\draw[] (8, 0) -- (8, 10);
\draw[] (14, 0) -- (14, 10);
\draw[] (15, 0) -- (15, 10);
%\draw (-0.03, 0) -- (0.03, 0);
\node[] at (0, 0) {0\%};
%\draw (-0.03, 1) -- (0.03, 1);
\node[] at (0, 1) {10\%};
%\draw (-0.03, 2) -- (0.03, 2);
\node[] at (0, 2) {20\%};
%\draw (-0.03, 3) -- (0.03, 3);
\node[] at (0, 3) {30\%};
%\draw (-0.03, 4) -- (0.03, 4);
\node[] at (0, 4) {40\%};
%\draw (-0.03, 5) -- (0.03, 5);
\node[] at (0, 5) {50\%};
%\draw (-0.03, 6) -- (0.03, 6);
\node[] at (0, 6) {60\%};
%\draw (-0.03, 7) -- (0.03, 7);
\node[] at (0, 7) {70\%};
%\draw (-0.03, 8) -- (0.03, 8);
\node[] at (0, 8) {80\%};
%\draw (-0.03, 9) -- (0.03, 9);
\node[] at (0, 9) {90\%};
%\draw (-0.03, 10) -- (0.03, 10);
\node[] at (0, 10) {100\%};
%\node[rotate=90] at (22.5, 5) {\% of identified instances};
\draw[] (21, 0) -- (21, 10);
%\draw (20.97, 0) -- (21.03, 0);
%\node[] at (21.75, 0) {0\%};
%\draw (20.97, 1) -- (21.03, 1);
%\node[] at (21.75, 1) {10\%};
%\draw (20.97, 2) -- (21.03, 2);
%\node[] at (21.75, 2) {20\%};
%\draw (20.97, 3) -- (21.03, 3);
%\node[] at (21.75, 3) {30\%};
%\draw (20.97, 4) -- (21.03, 4);
%\node[] at (21.75, 4) {40\%};
%\draw (20.97, 5) -- (21.03, 5);
%\node[] at (21.75, 5) {50\%};
%\draw (20.97, 6) -- (21.03, 6);
%\node[] at (21.75, 6) {60\%};
%\draw (20.97, 7) -- (21.03, 7);
%\node[] at (21.75, 7) {70\%};
%\draw (20.97, 8) -- (21.03, 8);
%\node[] at (21.75, 8) {80\%};
%\draw (20.97, 9) -- (21.03, 9);
%\node[] at (21.75, 9) {90\%};
%\draw (20.97, 10) -- (21.03, 10);
%\node[] at (21.75, 10) {100\%};
\draw[        ] (1, 0) -- (7, 0) (8, 0) -- (14, 0) (15, 0) -- (21, 0); 
\draw[black!30] (1, 1) -- (7, 1) (8, 1) -- (14, 1) (15, 1) -- (21, 1); 
\draw[black!30] (1, 2) -- (7, 2) (8, 2) -- (14, 2) (15, 2) -- (21, 2); 
\draw[black!30] (1, 3) -- (7, 3) (8, 3) -- (14, 3) (15, 3) -- (21, 3); 
\draw[black!30] (1, 4) -- (7, 4) (8, 4) -- (14, 4) (15, 4) -- (21, 4); 
\draw[black!30] (1, 5) -- (7, 5) (8, 5) -- (14, 5) (15, 5) -- (21, 5); 
\draw[black!30] (1, 6) -- (7, 6) (8, 6) -- (14, 6) (15, 6) -- (21, 6); 
\draw[black!30] (1, 7) -- (7, 7) (8, 7) -- (14, 7) (15, 7) -- (21, 7); 
\draw[black!30] (1, 8) -- (7, 8) (8, 8) -- (14, 8) (15, 8) -- (21, 8); 
\draw[black!30] (1, 9) -- (7, 9) (8, 9) -- (14, 9) (15, 9) -- (21, 9); 
\draw[        ] (1, 10) -- (7, 10) (8, 10) -- (14, 10) (15, 10) -- (21, 10); 
\draw[black!60!orange] %l = 2
   (1, 3.01) -- (2, 5.45) -- (3, 7.64) -- (4, 9.27) -- (5, 9.75) -- (6, 9.95) -- (7, 9.98) 
   (8, 8.94) -- (9, 9.53) -- (10, 9.88) -- (11, 9.97) -- (12, 9.99) -- (13, 10) -- (14, 10)
   (15, 9.92) -- (16, 9.98) -- (17, 9.99) -- (18, 10) -- (19, 10) -- (20, 10) -- (21, 10); 
\draw[black!40!orange] %l = 5
   (1, 0.46) -- (2, 0.87) -- (3, 1.8) -- (4, 3.34) -- (5, 4.97) -- (6, 6.48) -- (7, 7.91)
   (8, 4.63) -- (9, 6.08) -- (10, 7.3) -- (11, 8.51) -- (12, 9.12) -- (13, 9.57) -- (14, 9.79)
   (15, 8.73) -- (16, 9.24) -- (17, 9.6) -- (18, 9.81) -- (19, 9.92) -- (20, 9.97) -- (21, 9.99); 
\draw[black!20!orange]  %l=10
   (1, 0.11) -- (2, 0.18) -- (3, 0.33) -- (4, 0.6) -- (5, 0.99) -- (6, 1.47) -- (7, 2.16)
   (8, 2.5) -- (9, 3.35) -- (10, 3.8) -- (11, 4.91) -- (12, 5.72) -- (13, 6.36) -- (14, 7.05)
   (15, 6.88) -- (16, 7.55) -- (17, 7.98) -- (18, 8.56) -- (19, 8.88) -- (20, 9.14) -- (21, 9.36); 
\draw[orange] %l=20
   (1, 0.05) -- (2, 0.07) -- (3, 0.08) -- (4, 0.14) -- (5, 0.23) -- (6, 0.24) -- (7, 0.34)
   (8, 1.85) -- (9, 2.01) -- (10, 2.41) -- (11, 2.88) -- (12, 3.17) -- (13, 3.55) -- (14, 3.86)
   (15, 5.7) -- (16, 6.03) -- (17, 6.41) -- (18, 6.79) -- (19, 7.18) -- (20, 7.42) -- (21, 7.61); 
\node[draw,black!60!orange,cross out] at (1, 8.16) {};
\node[draw,black!60!orange,cross out] at (2, 9.4) {};
\node[draw,black!60!orange,cross out] at (3, 9.8) {};
\node[draw,black!60!orange,cross out] at (4, 9.96) {};
\node[draw,black!60!orange,cross out] at (8, 9.73) {};
\node[draw,black!60!orange,cross out] at (9, 9.9) {};
\node[draw,black!60!orange,cross out] at (10, 9.98) {};
\node[draw,black!60!orange,cross out] at (11, 9.99) {};
\node[draw,black!60!orange,cross out] at (15, 9.93) {};
\node[draw,black!60!orange,cross out] at (16, 9.98) {};
\node[draw,black!60!orange,cross out] at (17, 9.99) {};
\node[draw,black!60!orange,cross out] at (18, 10) {};
\node[draw,black!40!orange,cross out] at (1, 1.46) {};
\node[draw,black!40!orange,cross out] at (2, 3.03) {};
\node[draw,black!40!orange,cross out] at (3, 4.69) {};
\node[draw,black!40!orange,cross out] at (4, 6.76) {};
\node[draw,black!40!orange,cross out] at (8, 6.41) {};
\node[draw,black!40!orange,cross out] at (9, 7.69) {};
\node[draw,black!40!orange,cross out] at (10, 8.62) {};
\node[draw,black!40!orange,cross out] at (11, 9.36) {};
\node[draw,black!40!orange,cross out] at (15, 8.74) {};
\node[draw,black!40!orange,cross out] at (16, 9.25) {};
\node[draw,black!40!orange,cross out] at (17, 9.62) {};
\node[draw,black!40!orange,cross out] at (18, 9.83) {};
\node[draw,black!20!orange,cross out] at (1, 0.25) {};
\node[draw,black!20!orange,cross out] at (2, 0.48) {};
\node[draw,black!20!orange,cross out] at (3, 0.8) {};
\node[draw,black!20!orange,cross out] at (4, 1.32) {};
\node[draw,black!20!orange,cross out] at (8, 3.47) {};
\node[draw,black!20!orange,cross out] at (9, 4.27) {};
\node[draw,black!20!orange,cross out] at (10, 4.99) {};
\node[draw,black!20!orange,cross out] at (11, 6.07) {};
\node[draw,black!20!orange,cross out] at (15, 6.89) {};
\node[draw,black!20!orange,cross out] at (16, 7.56) {};
\node[draw,black!20!orange,cross out] at (17, 7.99) {};
\node[draw,black!20!orange,cross out] at (18, 8.58) {};
\node[draw,orange,cross out] at (1, 0.07) {};
\node[draw,orange,cross out] at (2, 0.12) {};
\node[draw,orange,cross out] at (3, 0.16) {};
\node[draw,orange,cross out] at (4, 0.25) {};
\node[draw,orange,cross out] at (8, 2.3) {};
\node[draw,orange,cross out] at (9, 2.54) {};
\node[draw,orange,cross out] at (10, 2.94) {};
\node[draw,orange,cross out] at (11, 3.44) {};
\node[draw,orange,cross out] at (15, 5.7) {};
\node[draw,orange,cross out] at (16, 6.03) {};
\node[draw,orange,cross out] at (17, 6.42) {};
\node[draw,orange,cross out] at (18, 6.8) {};
\draw[draw=black,fill=white] (16.75, 5.25) rectangle (19.75, 1.85);
\draw[black!60!orange] (17, 4.95) -- (18, 4.95); 
\node[anchor=west] at (18.1, 4.95) {\tiny $l = 2$};
\draw[black!40!orange] (17, 4.55) -- (18, 4.55); 
\node[anchor=west] at (18.1, 4.55) {\tiny $l = 5$};
\draw[black!20!orange] (17, 4.15) -- (18, 4.15); 
\node[anchor=west] at (18.1, 4.15) {\tiny $l = 10$};
\draw[orange] (17, 3.75) -- (18, 3.75); 
\node[anchor=west] at (18.1, 3.75) {\tiny $l = 20$};
(17, 3.75) -- (18, 3.75); 
\node[draw,black!60!orange,cross out] at (17.5, 3.35) {};
\node[anchor=west] at (18.1, 3.35) {\tiny $l = 2$};
(17, 3.45) -- (18, 3.45); 
\node[draw,black!40!orange,cross out] at (17.5, 2.95) {};
\node[anchor=west] at (18.1, 2.95) {\tiny $l = 5$};
(17, 3.15) -- (18, 3.15); 
\node[draw,black!20!orange,cross out] at (17.5, 2.55) {};
\node[anchor=west] at (18.1, 2.55) {\tiny $l = 10$};
(17, 2.85) -- (18, 2.85); 
\node[draw,orange,cross out] at (17.5, 2.15) {};
\node[anchor=west] at (18.1, 2.15) {\tiny $l = 20$};
\end{tikzpicture}
\caption{Case $P(\textit{unobserved}=0.75)$:
Percent of identifiable graphs for fixed density parameter values $l\in\{2,5,10,20\}$ and
with varying  the number of nodes $n$ and cardinalities 
$|\bX|=|\bY|=k$. The horizontal axis is labeled by $(n,k)$ sorted lexicographically.
The curves show the data for CBC$^+$, i.e. for instances identifiable by adjustment 
or by a plain formula; Crosses show data for IDC, i.e., they show how many cases
are identifiable at all. The high time complexity of the IDC algorithm precluded computations
for graphs of sizes  $n\ge 500$.
%Similarly as in Fig.~\ref{fig:global:n:curves},
%the curves show the data for CBC$^+$ and the crosses represent the data for IDC.
}\label{fig:global:l:curves}
\end{figure}

%% file: runtimes.tex
\begin{figure}
\onlyincolor{
\def\colorcausaleffecttext{red}
\definecolor{colorcausaleffect}{RGB}{255,0,0}
\def\colorpcalgtext{purple}
%\definecolor{colorcausaleffect}{RGB}{255,0,0}
\colorlet{colorpcalg}{violet}
\def\causaleffectpoints#1{}
\def\pcalgpoints#1{}
}
\onlyinblackandwhite{
\def\colorcausaleffecttext{dark gray $\times$}
\definecolor{colorcausaleffect}{gray}{0.25}
\def\colorpcalgtext{light gray $\circ$}
\definecolor{colorpcalg}{gray}{0.5}
\def\causaleffectpoints#1{
	\foreach \x/\y in {#1}
	{ \node[] at (\x, \y) {$\times$} ; }
}
\def\pcalgpoints#1{
	\foreach \x/\y in {#1}
	{ \node[] at (\x, \y) {$\circ$} ; }
}
}

\begin{minipage}{0.5\textwidth}\tiny 
\begin{tikzpicture}[minimum size=0,scale=0.6]
\node[] at (6, -1) {\scriptsize number of nodes};
\draw[->] (0, 0) -- (10.5, 0);

\draw (0.53, -0.03) -- (0.53, 0.03);
\draw[white!90!black] (0.53, 19.5) -- (0.53, 14.5);
\draw[white!90!black] (0.53, 0) -- (0.53, 14);
\node[] at (0.53, -0.5) {10};

\draw (2.13, -0.03) -- (2.13, 0.03);
\draw[white!90!black] (2.13, -0.0) -- (2.13, 14);
\draw[white!90!black] (2.13, 19.5) -- (2.13, 14.5);
\node[] at (2.13, -0.5) {25};

\draw (3.35, -0.03) -- (3.35, 0.03);
\draw[white!90!black] (3.35, -0.0) -- (3.35, 14);
\draw[white!90!black] (3.35, 19.5) -- (3.35, 14.5);
\node[] at (3.35, -0.5) {50};

\draw (4.56, -0.03) -- (4.56, 0.03);
\draw[white!90!black] (4.56, -0.0) -- (4.56, 14);
\draw[white!90!black] (4.56, 19.5) -- (4.56, 14.5);
\node[] at (4.56, -0.5) {100};

\draw (6.16, -0.03) -- (6.16, 0.03);
\draw[white!90!black] (6.16, -0.0) -- (6.16, 14);
\draw[white!90!black] (6.16, 19.5) -- (6.16, 14.5);
\node[] at (6.16, -0.5) {250};

\draw (7.38, -0.03) -- (7.38, 0.03);
\draw[white!90!black] (7.38, -0.0) -- (7.38, 14);
\draw[white!90!black] (7.38, 19.5) -- (7.38, 14.5);
\node[] at (7.38, -0.5) {500};

\draw (8.59, -0.03) -- (8.59, 0.03);
\draw[white!90!black] (8.59, -0.0) -- (8.59, 14);
\draw[white!90!black] (8.59, 19.5) -- (8.59, 14.5);
\node[] at (8.59, -0.5) {1000};

\draw (9.8, -0.03) -- (9.8, 0.03);
\draw[white!90!black] (9.8, -0.0) -- (9.8, 14);
\draw[white!90!black] (9.8, 19.5) -- (9.8, 14.5);
\node[] at (9.8, -0.5) {2000};

\node[rotate=90] at (-1.6, 7) {\scriptsize average time};
\draw[-] (0, 0) -- (0, 14);
\node[] at (0,14.25) {$\approx$};
\draw (-0.03, 0) -- (0.03, 0);
\node[] at (-0.75, 0) {0.3 ms};

\draw[white!90!black] (0, 1) -- (9.8, 1);
\draw (-0.03, 1) -- (0.03, 1);
\node[] at (-0.75, 1) {0.7 ms};

\draw[white!90!black] (0, 2) -- (9.8, 2);
\draw (-0.03, 2) -- (0.03, 2);
\node[] at (-0.75, 2) {1.6 ms};

\draw[white!90!black] (0, 3) -- (9.8, 3);
\draw (-0.03, 3) -- (0.03, 3);
\node[] at (-0.75, 3) {3.6 ms};

\draw[white!90!black] (0, 4) -- (9.8, 4);
\draw (-0.03, 4) -- (0.03, 4);
\node[] at (-0.75, 4) {8.0 ms};

\draw[white!90!black] (0, 5) -- (9.8, 5);
\draw (-0.03, 5) -- (0.03, 5);
\node[] at (-0.75, 5) {17.8 ms};

\draw[white!90!black] (0, 6) -- (9.8, 6);
\draw (-0.03, 6) -- (0.03, 6);
\node[] at (-0.75, 6) {39.6 ms};

\draw[white!90!black] (0, 7) -- (9.8, 7);
\draw (-0.03, 7) -- (0.03, 7);
\node[] at (-0.75, 7) {88.2 ms};

\draw[white!90!black] (0, 8) -- (9.8, 8);
\draw (-0.03, 8) -- (0.03, 8);
\node[] at (-0.75, 8) {196 ms};

\draw[white!90!black] (0, 9) -- (9.8, 9);
\draw (-0.03, 9) -- (0.03, 9);
\node[] at (-0.75, 9) {437 ms};

\draw[white!90!black] (0, 10) -- (9.8, 10);
\draw (-0.03, 10) -- (0.03, 10);
\node[] at (-0.75, 10) {973 ms};

\draw[white!90!black] (0, 11) -- (9.8, 11);
\draw (-0.03, 11) -- (0.03, 11);
\node[] at (-0.75, 11) {2.2 s};

\draw[white!90!black] (0, 12) -- (9.8, 12);
\draw (-0.03, 12) -- (0.03, 12);
\node[] at (-0.75, 12) {4.8 s};

\draw[white!90!black] (0, 13) -- (9.8, 13);
\draw (-0.03, 13) -- (0.03, 13);
\node[] at (-0.75, 13) {10.7 s};

\draw[white!90!black] (0, 14) -- (9.8, 14);
\draw (-0.03, 14) -- (0.03, 14);
\node[] at (-0.75, 14) {23.9 s};

%2nd y axis grid
\draw[white!90!black] (0, 14.5) -- (9.8, 14.5);
\draw[white!90!black] (0, 15.5) -- (9.8, 15.5);
\draw[white!90!black] (0, 16.5) -- (9.8, 16.5);
\draw[white!90!black] (0, 17.5) -- (9.8, 17.5);
\draw[white!90!black] (0, 18.5) -- (9.8, 18.5);
\draw[white!90!black] (0, 19.5) -- (9.8, 19.5);
%2nd y axis
\draw[->] (0, 14.5) -- (0, 19.5);
\draw (-0.03, 14.5) -- (0.03, 14.5);
\node[] at (-0.75, 14.5) {2.9 m};
\draw (-0.03, 15.5) -- (0.03, 15.5);
\node[] at (-0.75, 15.5) {6.5 m};
\draw (-0.03, 16.5) -- (0.03, 16.5);
\node[] at (-0.75, 16.5) {14.6 m};
\draw (-0.03, 17.5) -- (0.03, 17.5);
\node[] at (-0.75, 17.5) {32.4 m};
\draw (-0.03, 18.5) -- (0.03, 18.5);
\node[] at (-0.75, 18.5) {1.2 h};
\draw (-0.03, 19.5) -- (0.03, 19.5);
\node[] at (-0.75, 19.5) {2.7 h};

\draw[colorpcalg,dashed] (0.53, 4.74) -- (2.13, 7.11) -- (3.35, 9.16) -- (6.16, 13.2); 
\pcalgpoints{0.53/4.74,2.13/7.11,3.35/9.16,6.16/13.2}

%gac n=25 {
%\draw[colorpcalg,] (2.13, 18.3);                 
%\draw[colorpcalg,] (0.53, 4.7);                  
%\node[draw,colorpcalg,cross out] at (0.53, 4.7) {};
%\draw[colorpcalg,] (2.13, 7.31);                 
%\node[draw,colorpcalg,cross out] at (2.13, 7.31) {};
%\draw[colorpcalg,] (0.53, 6.53);                 
%\node[draw,colorpcalg,circle] at (0.53, 6.53) {};
%\draw[colorpcalg,] (2.13, 17.76);                
%\node[draw,colorpcalg,circle] at (2.13, 17.76) {}; 
%\draw[colorpcalg,] (3.35, 9.49);    

%\node[colorpcalg,anchor=west] at (2.13, 17.76) {\tiny$\ 3$};
%\draw[cyan] (0.53, 7.89) --  (2.13, 18.3);
\draw[colorpcalg] (0.53, 7.89) --  (1.47, 14); \draw[colorpcalg] (1.55, 14.5) -- (2.13, 18.3);
\pcalgpoints{0.53/7.89,2.13/18.3}
%}

\draw[colorcausaleffect,dashed] (0.53, 7.34) -- (2.13, 6.74) -- (3.35, 6.74) -- (4.56, 6.28) -- (6.16, 6); 
\causaleffectpoints{0.53/7.34,2.13/6.74,3.35/6.74,4.56/6.28,6.16/6}

% % \node[draw,colorcausaleffect,cross out] at (2.13, 6.69) {};
% % \node[draw,colorcausaleffect,cross out] at (3.35, 6.78) {};
% % \node[draw,colorcausaleffect,cross out] at (4.56, 6.28) {};
% % \node[draw,colorcausaleffect,cross out] at (6.16, 6) {};
\draw[dashed] (0.53, 0.49) -- (2.13, 1) -- (3.35, 1.53) -- (4.56, 2.03) -- (6.16, 2.87) -- (7.38, 3.56) -- (8.59, 4.48) -- (9.8, 5.43); 
% % \node[draw,cross out] at (0.53, 0.46) {};
% % \node[draw,cross out] at (2.13, 1.01) {};
% % \node[draw,cross out] at (3.35, 1.59) {};
% % \node[draw,cross out] at (4.56, 2.03) {};
% % \node[draw,cross out] at (6.16, 2.87) {};
% % \node[draw,cross out] at (7.38, 3.56) {};
% % \node[draw,cross out] at (8.59, 4.48) {};
% % \node[draw,cross out] at (9.8, 5.43) {};
\draw[colorcausaleffect,] (0.53, 7.73) -- (2.13, 8.02) -- (3.35, 9.45) -- (4.56, 10.79) -- (6.16, 13.2); 
\causaleffectpoints{0.53/7.73,2.13/8.02,3.35/9.45,4.56/10.79,6.16/13.2}
% % \node[draw,colorcausaleffect,circle] at (2.13, 8.11) {};
% % \node[draw,colorcausaleffect,circle] at (3.35, 9.39) {};
% % \node[draw,colorcausaleffect,circle] at (4.56, 10.79) {};
% % \node[draw,colorcausaleffect,circle] at (6.16, 13.2) {};
\draw[] (0.53, 1.5) -- (2.13, 3.36) -- (3.35, 4.16) -- (4.56, 4.88) -- (6.16, 5.9) -- (7.38, 6.69) -- (8.59, 7.52) -- (9.8, 8.42); 
% % \node[draw,circle] at (0.53, 1.49) {};
% % \node[draw,circle] at (2.13, 3.34) {};
% % \node[draw,circle] at (3.35, 4.28) {};
% % \node[draw,circle] at (4.56, 4.88) {};
% % \node[draw,circle] at (6.16, 5.9) {};
% % \node[draw,circle] at (7.38, 6.69) {};
% % \node[draw,circle] at (8.59, 7.52) {};
% % \node[draw,circle] at (9.8, 8.42) {};
\draw[] (0.5295239127395801, -0.038227069031457894) -- (0.5295239127395801, 0.8796746783413806); 
\draw[] (0.4595239127395801, -0.038227069031457894) -- (0.5995239127395802, -0.038227069031457894); 
\draw[] (0.4595239127395801, 0.8796746783413806) -- (0.5995239127395802, 0.8796746783413806); 
\draw[] (2.1330326935193513, 0.4883447024200209) -- (2.1330326935193513, 1.4244156858392103); 
\draw[] (2.0630326935193515, 0.4883447024200209) -- (2.203032693519351, 0.4883447024200209); 
\draw[] (2.0630326935193515, 1.4244156858392103) -- (2.203032693519351, 1.4244156858392103); 
\draw[] (3.3460402594992553, 1.124533404909374) -- (3.3460402594992553, 1.8798125373984065); 
\draw[] (3.2760402594992555, 1.124533404909374) -- (3.416040259499255, 1.124533404909374); 
\draw[] (3.2760402594992555, 1.8798125373984065) -- (3.416040259499255, 1.8798125373984065); 
\draw[] (4.55904782547916, 1.7640806096035786) -- (4.55904782547916, 2.292205371679999); 
\draw[] (4.48904782547916, 1.7640806096035786) -- (4.6290478254791605, 1.7640806096035786); 
\draw[] (4.48904782547916, 2.292205371679999) -- (4.6290478254791605, 2.292205371679999); 
\draw[] (6.162556606258931, 2.7147240822906573) -- (6.162556606258931, 3.0395759849045954); 
\draw[] (6.092556606258931, 2.7147240822906573) -- (6.232556606258932, 2.7147240822906573); 
\draw[] (6.092556606258931, 3.0395759849045954) -- (6.232556606258932, 3.0395759849045954); 
\draw[] (7.375564172238835, 3.4639578599077394) -- (7.375564172238835, 3.6711868436387234); 
\draw[] (7.305564172238835, 3.4639578599077394) -- (7.445564172238836, 3.4639578599077394); 
\draw[] (7.305564172238835, 3.6711868436387234) -- (7.445564172238836, 3.6711868436387234); 
\draw[] (8.58857173821874, 4.403695050435442) -- (8.58857173821874, 4.552091685896256); 
\draw[] (8.51857173821874, 4.403695050435442) -- (8.65857173821874, 4.403695050435442); 
\draw[] (8.51857173821874, 4.552091685896256) -- (8.65857173821874, 4.552091685896256); 
\draw[] (9.801579304198643, 5.40300583484659) -- (9.801579304198643, 5.470306279580765); 
\draw[] (9.731579304198643, 5.40300583484659) -- (9.871579304198644, 5.40300583484659); 
\draw[] (9.731579304198643, 5.470306279580765) -- (9.871579304198644, 5.470306279580765); 
\draw[] (0.5295239127395801, 0.8336418653917506) -- (0.5295239127395801, 2.008475881258346); 
\draw[] (0.4595239127395801, 0.8336418653917506) -- (0.5995239127395802, 0.8336418653917506); 
\draw[] (0.4595239127395801, 2.008475881258346) -- (0.5995239127395802, 2.008475881258346); 
\draw[] (2.1330326935193513, 2.7403215754480716) -- (2.1330326935193513, 3.852134166910307); 
\draw[] (2.0630326935193515, 2.7403215754480716) -- (2.203032693519351, 2.7403215754480716); 
\draw[] (2.0630326935193515, 3.852134166910307) -- (2.203032693519351, 3.852134166910307); 
\draw[] (3.3460402594992553, 3.510240881507237) -- (3.3460402594992553, 4.712805845720004); 
\draw[] (3.2760402594992555, 3.510240881507237) -- (3.416040259499255, 3.510240881507237); 
\draw[] (3.2760402594992555, 4.712805845720004) -- (3.416040259499255, 4.712805845720004); 
\draw[] (4.55904782547916, 4.227339058438194) -- (4.55904782547916, 5.437699967450618); 
\draw[] (4.48904782547916, 4.227339058438194) -- (4.6290478254791605, 4.227339058438194); 
\draw[] (4.48904782547916, 5.437699967450618) -- (4.6290478254791605, 5.437699967450618); 
\draw[] (6.162556606258931, 5.272435018182909) -- (6.162556606258931, 6.449746899164245); 
\draw[] (6.092556606258931, 5.272435018182909) -- (6.232556606258932, 5.272435018182909); 
\draw[] (6.092556606258931, 6.449746899164245) -- (6.232556606258932, 6.449746899164245); 
\draw[] (7.375564172238835, 6.075388580690969) -- (7.375564172238835, 7.230239584287947); 
\draw[] (7.305564172238835, 6.075388580690969) -- (7.445564172238836, 6.075388580690969); 
\draw[] (7.305564172238835, 7.230239584287947) -- (7.445564172238836, 7.230239584287947); 
\draw[] (8.58857173821874, 6.927132481611958) -- (8.58857173821874, 8.05533332415563); 
\draw[] (8.51857173821874, 6.927132481611958) -- (8.65857173821874, 6.927132481611958); 
\draw[] (8.51857173821874, 8.05533332415563) -- (8.65857173821874, 8.05533332415563); 
\draw[] (9.801579304198643, 7.8750921333435695) -- (9.801579304198643, 8.916736455715656); 
\draw[] (9.731579304198643, 7.8750921333435695) -- (9.871579304198644, 7.8750921333435695); 
\draw[] (9.731579304198643, 8.916736455715656) -- (9.871579304198644, 8.916736455715656); 
\draw[colorcausaleffect] (0.5295239127395801, 5.843895757667827) -- (0.5295239127395801, 8.262529891127162); 
\draw[colorcausaleffect] (0.4595239127395801, 5.843895757667827) -- (0.5995239127395802, 5.843895757667827); 
\draw[colorcausaleffect] (0.4595239127395801, 8.262529891127162) -- (0.5995239127395802, 8.262529891127162); 
\draw[colorcausaleffect] (2.1330326935193513, 5.536905308064979) -- (2.1330326935193513, 7.507684792410407); 
\draw[colorcausaleffect] (2.0630326935193515, 5.536905308064979) -- (2.203032693519351, 5.536905308064979); 
\draw[colorcausaleffect] (2.0630326935193515, 7.507684792410407) -- (2.203032693519351, 7.507684792410407); 
\draw[colorcausaleffect] (3.3460402594992553, 5.465707359393734) -- (3.3460402594992553, 7.597217191385898); 
\draw[colorcausaleffect] (3.2760402594992555, 5.465707359393734) -- (3.416040259499255, 5.465707359393734); 
\draw[colorcausaleffect] (3.2760402594992555, 7.597217191385898) -- (3.416040259499255, 7.597217191385898); 
\draw[colorcausaleffect] (4.55904782547916, 5.3990401811740245) -- (4.55904782547916, 7.0432209559597005); 
\draw[colorcausaleffect] (4.48904782547916, 5.3990401811740245) -- (4.6290478254791605, 5.3990401811740245); 
\draw[colorcausaleffect] (4.48904782547916, 7.0432209559597005) -- (4.6290478254791605, 7.0432209559597005); 
\draw[colorcausaleffect] (6.162556606258931, 5.465963553986509) -- (6.162556606258931, 6.609427947983168); 
\draw[colorcausaleffect] (6.092556606258931, 5.465963553986509) -- (6.232556606258932, 5.465963553986509); 
\draw[colorcausaleffect] (6.092556606258931, 6.609427947983168) -- (6.232556606258932, 6.609427947983168); 
\draw[colorcausaleffect] (0.5295239127395801, 6.516667564785509) -- (0.5295239127395801, 8.551365978772376); 
\draw[colorcausaleffect] (0.4595239127395801, 6.516667564785509) -- (0.5995239127395802, 6.516667564785509); 
\draw[colorcausaleffect] (0.4595239127395801, 8.551365978772376) -- (0.5995239127395802, 8.551365978772376); 
\draw[colorcausaleffect] (2.1330326935193513, 7.0202651196260835) -- (2.1330326935193513, 8.549811718292661); 
\draw[colorcausaleffect] (2.0630326935193515, 7.0202651196260835) -- (2.203032693519351, 7.0202651196260835); 
\draw[colorcausaleffect] (2.0630326935193515, 8.549811718292661) -- (2.203032693519351, 8.549811718292661); 
\draw[colorcausaleffect] (3.3460402594992553, 8.334891440177048) -- (3.3460402594992553, 9.9533838516436); 
\draw[colorcausaleffect] (3.2760402594992555, 8.334891440177048) -- (3.416040259499255, 8.334891440177048); 
\draw[colorcausaleffect] (3.2760402594992555, 9.9533838516436) -- (3.416040259499255, 9.9533838516436); 
\draw[colorcausaleffect] (4.55904782547916, 9.922154727978247) -- (4.55904782547916, 11.29600681171898); 
\draw[colorcausaleffect] (4.48904782547916, 9.922154727978247) -- (4.6290478254791605, 9.922154727978247); 
\draw[colorcausaleffect] (4.48904782547916, 11.29600681171898) -- (4.6290478254791605, 11.29600681171898); 
\draw[colorcausaleffect] (6.162556606258931, 12.067469863459765) -- (6.162556606258931, 13.82141169080745); 
\draw[colorcausaleffect] (6.092556606258931, 12.067469863459765) -- (6.232556606258932, 12.067469863459765); 
\draw[colorcausaleffect] (6.092556606258931, 13.82141169080745) -- (6.232556606258932, 13.82141169080745); 
\draw[colorpcalg] (0.5295239127395801, 4.6647728290951544) -- (0.5295239127395801, 5.0066990200907755); 
\draw[colorpcalg] (0.4595239127395801, 4.6647728290951544) -- (0.5995239127395802, 4.6647728290951544); 
\draw[colorpcalg] (0.4595239127395801, 5.0066990200907755) -- (0.5995239127395802, 5.0066990200907755); 
\draw[colorpcalg] (2.1330326935193513, 6.747872649888979) -- (2.1330326935193513, 7.3217230670859195); 
\draw[colorpcalg] (2.0630326935193515, 6.747872649888979) -- (2.203032693519351, 6.747872649888979); 
\draw[colorpcalg] (2.0630326935193515, 7.3217230670859195) -- (2.203032693519351, 7.3217230670859195); 
\draw[colorpcalg] (3.3460402594992553, 8.41144957320413) -- (3.3460402594992553, 9.553461300069118); 
\draw[colorpcalg] (3.2760402594992555, 8.41144957320413) -- (3.416040259499255, 8.41144957320413); 
\draw[colorpcalg] (3.2760402594992555, 9.553461300069118) -- (3.416040259499255, 9.553461300069118); 
\draw[colorpcalg] (4.55904782547916, 10.007265481185287) -- (4.55904782547916, 11.464517977749281); 
\draw[colorpcalg] (4.48904782547916, 10.007265481185287) -- (4.6290478254791605, 10.007265481185287); 
\draw[colorpcalg] (4.48904782547916, 11.464517977749281) -- (4.6290478254791605, 11.464517977749281); 
\draw[colorpcalg] (6.162556606258931, 12.408602810657852) -- (6.162556606258931, 13.892322850727718); 
\draw[colorpcalg] (6.092556606258931, 12.408602810657852) -- (6.232556606258932, 12.408602810657852); 
\draw[colorpcalg] (6.092556606258931, 13.892322850727718) -- (6.232556606258932, 13.892322850727718); 

%gac n=25{??
\draw[colorpcalg] (0.5295239127395801, 7.80926301821167) -- (0.5295239127395801, 8.089230109355912); 
\draw[colorpcalg] (0.4595239127395801, 7.80926301821167) -- (0.5995239127395802, 7.80926301821167); 
\draw[colorpcalg] (0.4595239127395801, 8.089230109355912) -- (0.5995239127395802, 8.089230109355912); 
\draw[colorpcalg] (2.1330326935193513, 16.82939419947397) -- (2.1330326935193513, 18.833610720149988); 
\draw[colorpcalg] (2.0630326935193515, 16.82939419947397) -- (2.203032693519351, 16.82939419947397); 
\draw[colorpcalg] (2.0630326935193515, 18.833610720149988) -- (2.203032693519351, 18.833610720149988); 

\end{tikzpicture}
\end{minipage}
\begin{minipage}{0.5\textwidth}\tiny 
\begin{tikzpicture}[minimum size=0,scale=0.6]
\node[] at (6, -1) {\scriptsize expected number of edges};
\draw[->] (0, 0) -- (12, 0);
\draw (0, -0.03) -- (0, 0.03);
\node[] at (0, -0.5) {9};

\draw[white!90!black] (1, 0) -- (1, 14);
\draw[white!90!black] (1, 19.5) -- (1, 14.5);
\draw (1, -0.03) -- (1, 0.03);
\node[] at (1, -0.5) {17};

\draw[white!90!black] (2, 19.5) -- (2, 14.5);
\draw[white!90!black] (2, 0) -- (2, 14);
\draw (2, -0.03) -- (2, 0.03);
\node[] at (2, -0.5) {36};

\draw[white!90!black] (3, 19.5) -- (3, 14.5);
\draw[white!90!black] (3, 0) -- (3, 14);
\draw (3, -0.03) -- (3, 0.03);
\node[] at (3, -0.5) {73};

\draw[white!90!black] (4, 19.5) -- (4, 14.5);
\draw[white!90!black] (4, 0) -- (4, 14);
\draw (4, -0.03) -- (4, 0.03);
\node[] at (4, -0.5) {148};

\draw[white!90!black] (5, 19.5) -- (5, 14.5);
\draw[white!90!black] (5, 0) -- (5, 14);
\draw (5, -0.03) -- (5, 0.03);
\node[] at (5, -0.5) {303};

\draw[white!90!black] (6, 19.5) -- (6, 14.5);
\draw[white!90!black] (6, 0) -- (6, 14);
\draw (6, -0.03) -- (6, 0.03);
\node[] at (6, -0.5) {619};

\draw[white!90!black] (7, 19.5) -- (7, 14.5);
\draw[white!90!black] (7, 0) -- (7, 14);
\draw (7, -0.03) -- (7, 0.03);
\node[] at (7, -0.5) {1265};

\draw[white!90!black] (8, 19.5) -- (8, 14.5);
\draw[white!90!black] (8, 0) -- (8, 14);
\draw (8, -0.03) -- (8, 0.03);
\node[] at (8, -0.5) {2584};

\draw[white!90!black] (9, 19.5) -- (9, 14.5);
\draw[white!90!black] (9, 0) -- (9, 14);
\draw (9, -0.03) -- (9, 0.03);
\node[] at (9, -0.5) {5279};

\draw[white!90!black] (10, 19.5) -- (10, 14.5);
\draw[white!90!black] (10, 0) -- (10, 14);
\draw (10, -0.03) -- (10, 0.03);
\node[] at (10, -0.5) {10783};

\draw[white!90!black] (11, 19.5) -- (11, 14.5);
\draw[white!90!black] (11, 0) -- (11, 14);
\draw (11, -0.03) -- (11, 0.03);
\node[] at (11, -0.5) {22026};

\draw (12, -0.03) -- (12, 0.03);
\node[] at (12, -0.5) {44994};

\node[rotate=90] at (-1.6, 7) {\scriptsize average time};
\draw[-] (0, 0) -- (0, 14);
\node[] at (0,14.25) {$\approx$};
\draw (-0.03, 0) -- (0.03, 0);
\node[] at (-0.75, 0) {0.3 ms};

\draw[white!90!black] (0, 1) -- (11, 1);
\draw (-0.03, 1) -- (0.03, 1);
\node[] at (-0.75, 1) {0.7 ms};

\draw[white!90!black] (0, 2) -- (11, 2);
\draw (-0.03, 2) -- (0.03, 2);
\node[] at (-0.75, 2) {1.6 ms};

\draw[white!90!black] (0, 3) -- (11, 3);
\draw (-0.03, 3) -- (0.03, 3);
\node[] at (-0.75, 3) {3.6 ms};

\draw[white!90!black] (0, 4) -- (11, 4);
\draw (-0.03, 4) -- (0.03, 4);
\node[] at (-0.75, 4) {8.0 ms};

\draw[white!90!black] (0, 5) -- (11, 5);
\draw (-0.03, 5) -- (0.03, 5);
\node[] at (-0.75, 5) {17.8 ms};

\draw[white!90!black] (0, 6) -- (11, 6);
\draw (-0.03, 6) -- (0.03, 6);
\node[] at (-0.75, 6) {39.6 ms};

\draw[white!90!black] (0, 7) -- (11, 7);
\draw (-0.03, 7) -- (0.03, 7);
\node[] at (-0.75, 7) {88.2 ms};

\draw[white!90!black] (0, 8) -- (11, 8);
\draw (-0.03, 8) -- (0.03, 8);
\node[] at (-0.75, 8) {196 ms};

\draw[white!90!black] (0, 9) -- (11, 9);
\draw (-0.03, 9) -- (0.03, 9);
\node[] at (-0.75, 9) {437 ms};

\draw[white!90!black] (0, 10) -- (11, 10);
\draw (-0.03, 10) -- (0.03, 10);
\node[] at (-0.75, 10) {973 ms};

\draw[white!90!black] (0, 11) -- (11, 11);
\draw (-0.03, 11) -- (0.03, 11);
\node[] at (-0.75, 11) {2.2 s};

\draw[white!90!black] (0, 12) -- (11, 12);
\draw (-0.03, 12) -- (0.03, 12);
\node[] at (-0.75, 12) {4.8 s};

\draw[white!90!black] (0, 13) -- (11, 13);
\draw (-0.03, 13) -- (0.03, 13);
\node[] at (-0.75, 13) {10.7 s};

\draw[white!90!black] (0, 14) -- (11, 14);
\draw (-0.03, 14) -- (0.03, 14);
\node[] at (-0.75, 14) {23.9 s};

%2nd y axis grid
\draw[white!90!black] (0, 14.5) -- (11, 14.5);
\draw[white!90!black] (0, 15.5) -- (11, 15.5);
\draw[white!90!black] (0, 16.5) -- (11, 16.5);
\draw[white!90!black] (0, 17.5) -- (11, 17.5);
\draw[white!90!black] (0, 18.5) -- (11, 18.5);
\draw[white!90!black] (0, 19.5) -- (11, 19.5);
%2nd y axis
\draw[->] (0, 14.5) -- (0, 19.5);
\draw (-0.03, 14.5) -- (0.03, 14.5);
\node[] at (-0.75, 14.5) {2.9 m};
\draw (-0.03, 15.5) -- (0.03, 15.5);
\node[] at (-0.75, 15.5) {6.5 m};
\draw (-0.03, 16.5) -- (0.03, 16.5);
\node[] at (-0.75, 16.5) {14.6 m};
\draw (-0.03, 17.5) -- (0.03, 17.5);
\node[] at (-0.75, 17.5) {32.4 m};
\draw (-0.03, 18.5) -- (0.03, 18.5);
\node[] at (-0.75, 18.5) {1.2 h};
\draw (-0.03, 19.5) -- (0.03, 19.5);
\node[] at (-0.75, 19.5) {2.7 h};
%\draw[] (0.53, 0.46) -- (2.13, 1.01) -- (3.35, 1.59) -- (4.56, 2.03) -- (6.16, 2.87) -- (7.38, 3.56) -- (8.59, 4.48) -- (9.8, 5.43); 

\draw[colorpcalg,dashed] (0.22, 4.74) -- (1.51, 7.11) -- (2.48, 9.16) -- (4.73, 13.2); 
\pcalgpoints{0.22/4.74,1.51/7.11,2.48/9.16,4.73/13.2}
%\draw[colorpcalg,] (3.45, 7.89); 
%\node[draw,colorpcalg,circle] at (3.45, 7.89) {};

%gac n=25{
%\node[draw,colorpcalg,circle] at (4.73, 17.76) {};
%\node[colorpcalg,anchor=west] at (4.73, 17.76) {\tiny$\ 3$};
\draw[colorpcalg] (3.447238260383328, 7.80926301821167) -- (3.447238260383328, 7.994897964505697); 
\draw[colorpcalg] (3.377238260383328, 7.80926301821167) -- (3.5172382603833277, 7.80926301821167); 
\draw[colorpcalg] (3.377238260383328, 7.994897964505697) -- (3.5172382603833277, 7.994897964505697); 
\draw[colorpcalg] (4.730045285007144, 16.82939419947397) -- (4.730045285007144, 18.833610720149988); 
\draw[colorpcalg] (4.660045285007143, 16.82939419947397) -- (4.800045285007144, 16.82939419947397); 
\draw[colorpcalg] (4.660045285007143, 18.833610720149988) -- (4.800045285007144, 18.833610720149988); 

%\draw[cyan](3.45, 7.89) -- (4.73, 18.3);
\draw[colorpcalg](3.45, 7.89) -- (4.21, 14); \draw[colorpcalg](4.26, 14.5) -- (4.73, 18.3);
\pcalgpoints{3.45/7.89,4.73/18.3}
%}

%MAIN LINES:
\draw[colorcausaleffect,dashed] (0.22, 7.34) -- (1.51, 6.74) -- (2.48, 6.74) -- (3.45, 6.28) -- (4.73, 6); 
\causaleffectpoints{0.22/7.34,1.51/6.74,2.48/6.74,3.45/6.28,4.73/6}
% % \node[draw,colorcausaleffect,cross out] at (1.51, 6.69) {};
% % \node[draw,colorcausaleffect,cross out] at (2.48, 6.78) {};
% % \node[draw,colorcausaleffect,cross out] at (3.45, 6.28) {};
% % \node[draw,colorcausaleffect,cross out] at (4.73, 6) {};
\draw[dashed] (0.22, 0.49) -- (1.51, 1) -- (2.48, 1.53) -- (3.45, 2.03) -- (4.73, 2.87) -- (5.7, 3.56) -- (6.67, 4.48) -- (7.64, 5.43); 
% % \node[draw,cross out] at (0.22, 0.46) {};
% % \node[draw,cross out] at (1.51, 1.01) {};
% % \node[draw,cross out] at (2.48, 1.59) {};
% % \node[draw,cross out] at (3.45, 2.03) {};
% % \node[draw,cross out] at (4.73, 2.87) {};
% % \node[draw,cross out] at (5.7, 3.56) {};
% % \node[draw,cross out] at (6.67, 4.48) {};
% % \node[draw,cross out] at (7.64, 5.43) {};
\draw[colorcausaleffect,] (3.45, 7.73) -- (4.73, 8.02) -- (5.7, 9.45) -- (6.67, 10.79) -- (7.95, 13.2); 
\causaleffectpoints{3.45/7.73,4.73/8.02,5.7/9.45,6.67/10.79,7.95/13.2}
% % \node[draw,colorcausaleffect,circle] at (4.73, 8.11) {};
% % \node[draw,colorcausaleffect,circle] at (5.7, 9.39) {};
% % \node[draw,colorcausaleffect,circle] at (6.67, 10.79) {};
% % \node[draw,colorcausaleffect,circle] at (7.95, 13.2) {};
\draw[] (3.45, 1.5) -- (4.73, 3.36) -- (5.7, 4.16) -- (6.67, 4.88) -- (7.95, 5.9) -- (8.92, 6.69) -- (9.89, 7.52) -- (10.86, 8.42); 
% % \node[draw,circle] at (3.45, 1.49) {};
% % \node[draw,circle] at (4.73, 3.34) {};
% % \node[draw,circle] at (5.7, 4.28) {};
% % \node[draw,circle] at (6.67, 4.88) {};
% % \node[draw,circle] at (7.95, 5.9) {};
% % \node[draw,circle] at (8.92, 6.69) {};
% % \node[draw,circle] at (9.89, 7.52) {};
% % \node[draw,circle] at (10.86, 8.42) {};
\draw[] (0.2236191301916639, -0.038227069031457894) -- (0.2236191301916639, 0.8796746783413806); 
\draw[] (0.1536191301916639, -0.038227069031457894) -- (0.2936191301916639, -0.038227069031457894); 
\draw[] (0.1536191301916639, 0.8796746783413806) -- (0.2936191301916639, 0.8796746783413806); 
\draw[] (1.5064261548154807, 0.4883447024200209) -- (1.5064261548154807, 1.4244156858392103); 
\draw[] (1.4364261548154806, 0.4883447024200209) -- (1.5764261548154808, 0.4883447024200209); 
\draw[] (1.4364261548154806, 1.4244156858392103) -- (1.5764261548154808, 1.4244156858392103); 
\draw[] (2.4768322075994043, 1.124533404909374) -- (2.4768322075994043, 1.8798125373984065); 
\draw[] (2.4068322075994044, 1.124533404909374) -- (2.546832207599404, 1.124533404909374); 
\draw[] (2.4068322075994044, 1.8798125373984065) -- (2.546832207599404, 1.8798125373984065); 
\draw[] (3.447238260383328, 1.7640806096035786) -- (3.447238260383328, 2.292205371679999); 
\draw[] (3.377238260383328, 1.7640806096035786) -- (3.5172382603833277, 1.7640806096035786); 
\draw[] (3.377238260383328, 2.292205371679999) -- (3.5172382603833277, 2.292205371679999); 
\draw[] (4.730045285007144, 2.7147240822906573) -- (4.730045285007144, 3.0395759849045954); 
\draw[] (4.660045285007143, 2.7147240822906573) -- (4.800045285007144, 2.7147240822906573); 
\draw[] (4.660045285007143, 3.0395759849045954) -- (4.800045285007144, 3.0395759849045954); 
\draw[] (5.700451337791067, 3.4639578599077394) -- (5.700451337791067, 3.6711868436387234); 
\draw[] (5.630451337791067, 3.4639578599077394) -- (5.770451337791068, 3.4639578599077394); 
\draw[] (5.630451337791067, 3.6711868436387234) -- (5.770451337791068, 3.6711868436387234); 
\draw[] (6.670857390574991, 4.403695050435442) -- (6.670857390574991, 4.552091685896256); 
\draw[] (6.6008573905749905, 4.403695050435442) -- (6.740857390574991, 4.403695050435442); 
\draw[] (6.6008573905749905, 4.552091685896256) -- (6.740857390574991, 4.552091685896256); 
\draw[] (7.641263443358914, 5.40300583484659) -- (7.641263443358914, 5.470306279580765); 
\draw[] (7.571263443358914, 5.40300583484659) -- (7.711263443358915, 5.40300583484659); 
\draw[] (7.571263443358914, 5.470306279580765) -- (7.711263443358915, 5.470306279580765); 
\draw[] (3.447238260383328, 0.8336418653917506) -- (3.447238260383328, 2.008475881258346); 
\draw[] (3.377238260383328, 0.8336418653917506) -- (3.5172382603833277, 0.8336418653917506); 
\draw[] (3.377238260383328, 2.008475881258346) -- (3.5172382603833277, 2.008475881258346); 
\draw[] (4.730045285007144, 2.7403215754480716) -- (4.730045285007144, 3.852134166910307); 
\draw[] (4.660045285007143, 2.7403215754480716) -- (4.800045285007144, 2.7403215754480716); 
\draw[] (4.660045285007143, 3.852134166910307) -- (4.800045285007144, 3.852134166910307); 
\draw[] (5.700451337791067, 3.510240881507237) -- (5.700451337791067, 4.712805845720004); 
\draw[] (5.630451337791067, 3.510240881507237) -- (5.770451337791068, 3.510240881507237); 
\draw[] (5.630451337791067, 4.712805845720004) -- (5.770451337791068, 4.712805845720004); 
\draw[] (6.670857390574991, 4.227339058438194) -- (6.670857390574991, 5.437699967450618); 
\draw[] (6.6008573905749905, 4.227339058438194) -- (6.740857390574991, 4.227339058438194); 
\draw[] (6.6008573905749905, 5.437699967450618) -- (6.740857390574991, 5.437699967450618); 
\draw[] (7.9536644151988085, 5.272435018182909) -- (7.9536644151988085, 6.449746899164245); 
\draw[] (7.883664415198808, 5.272435018182909) -- (8.023664415198809, 5.272435018182909); 
\draw[] (7.883664415198808, 6.449746899164245) -- (8.023664415198809, 6.449746899164245); 
\draw[] (8.924070467982732, 6.075388580690969) -- (8.924070467982732, 7.230239584287947); 
\draw[] (8.854070467982732, 6.075388580690969) -- (8.994070467982732, 6.075388580690969); 
\draw[] (8.854070467982732, 7.230239584287947) -- (8.994070467982732, 7.230239584287947); 
\draw[] (9.894476520766656, 6.927132481611958) -- (9.894476520766656, 8.05533332415563); 
\draw[] (9.824476520766655, 6.927132481611958) -- (9.964476520766656, 6.927132481611958); 
\draw[] (9.824476520766655, 8.05533332415563) -- (9.964476520766656, 8.05533332415563); 
\draw[] (10.864882573550577, 7.8750921333435695) -- (10.864882573550577, 8.916736455715656); 
\draw[] (10.794882573550577, 7.8750921333435695) -- (10.934882573550578, 7.8750921333435695); 
\draw[] (10.794882573550577, 8.916736455715656) -- (10.934882573550578, 8.916736455715656); 
\draw[colorcausaleffect] (0.2236191301916639, 5.843895757667827) -- (0.2236191301916639, 8.262529891127162); 
\draw[colorcausaleffect] (0.1536191301916639, 5.843895757667827) -- (0.2936191301916639, 5.843895757667827); 
\draw[colorcausaleffect] (0.1536191301916639, 8.262529891127162) -- (0.2936191301916639, 8.262529891127162); 
\draw[colorcausaleffect] (1.5064261548154807, 5.536905308064979) -- (1.5064261548154807, 7.507684792410407); 
\draw[colorcausaleffect] (1.4364261548154806, 5.536905308064979) -- (1.5764261548154808, 5.536905308064979); 
\draw[colorcausaleffect] (1.4364261548154806, 7.507684792410407) -- (1.5764261548154808, 7.507684792410407); 
\draw[colorcausaleffect] (2.4768322075994043, 5.465707359393734) -- (2.4768322075994043, 7.597217191385898); 
\draw[colorcausaleffect] (2.4068322075994044, 5.465707359393734) -- (2.546832207599404, 5.465707359393734); 
\draw[colorcausaleffect] (2.4068322075994044, 7.597217191385898) -- (2.546832207599404, 7.597217191385898); 
\draw[colorcausaleffect] (3.447238260383328, 5.3990401811740245) -- (3.447238260383328, 7.0432209559597005); 
\draw[colorcausaleffect] (3.377238260383328, 5.3990401811740245) -- (3.5172382603833277, 5.3990401811740245); 
\draw[colorcausaleffect] (3.377238260383328, 7.0432209559597005) -- (3.5172382603833277, 7.0432209559597005); 
\draw[colorcausaleffect] (4.730045285007144, 5.465963553986509) -- (4.730045285007144, 6.609427947983168); 
\draw[colorcausaleffect] (4.660045285007143, 5.465963553986509) -- (4.800045285007144, 5.465963553986509); 
\draw[colorcausaleffect] (4.660045285007143, 6.609427947983168) -- (4.800045285007144, 6.609427947983168); 
\draw[colorcausaleffect] (3.447238260383328, 6.516667564785509) -- (3.447238260383328, 8.551365978772376); 
\draw[colorcausaleffect] (3.377238260383328, 6.516667564785509) -- (3.5172382603833277, 6.516667564785509); 
\draw[colorcausaleffect] (3.377238260383328, 8.551365978772376) -- (3.5172382603833277, 8.551365978772376); 
\draw[colorcausaleffect] (4.730045285007144, 7.0202651196260835) -- (4.730045285007144, 8.549811718292661); 
\draw[colorcausaleffect] (4.660045285007143, 7.0202651196260835) -- (4.800045285007144, 7.0202651196260835); 
\draw[colorcausaleffect] (4.660045285007143, 8.549811718292661) -- (4.800045285007144, 8.549811718292661); 
\draw[colorcausaleffect] (5.700451337791067, 8.334891440177048) -- (5.700451337791067, 9.9533838516436); 
\draw[colorcausaleffect] (5.630451337791067, 8.334891440177048) -- (5.770451337791068, 8.334891440177048); 
\draw[colorcausaleffect] (5.630451337791067, 9.9533838516436) -- (5.770451337791068, 9.9533838516436); 
\draw[colorcausaleffect] (6.670857390574991, 9.922154727978247) -- (6.670857390574991, 11.29600681171898); 
\draw[colorcausaleffect] (6.6008573905749905, 9.922154727978247) -- (6.740857390574991, 9.922154727978247); 
\draw[colorcausaleffect] (6.6008573905749905, 11.29600681171898) -- (6.740857390574991, 11.29600681171898); 
\draw[colorcausaleffect] (7.9536644151988085, 12.067469863459765) -- (7.9536644151988085, 13.82141169080745); 
\draw[colorcausaleffect] (7.883664415198808, 12.067469863459765) -- (8.023664415198809, 12.067469863459765); 
\draw[colorcausaleffect] (7.883664415198808, 13.82141169080745) -- (8.023664415198809, 13.82141169080745); 
\draw[colorpcalg] (0.2236191301916639, 4.6647728290951544) -- (0.2236191301916639, 5.0066990200907755); 
\draw[colorpcalg] (0.1536191301916639, 4.6647728290951544) -- (0.2936191301916639, 4.6647728290951544); 
\draw[colorpcalg] (0.1536191301916639, 5.0066990200907755) -- (0.2936191301916639, 5.0066990200907755); 
\draw[colorpcalg] (1.5064261548154807, 6.747872649888979) -- (1.5064261548154807, 7.3217230670859195); 
\draw[colorpcalg] (1.4364261548154806, 6.747872649888979) -- (1.5764261548154808, 6.747872649888979); 
\draw[colorpcalg] (1.4364261548154806, 7.3217230670859195) -- (1.5764261548154808, 7.3217230670859195); 
\draw[colorpcalg] (2.4768322075994043, 8.41144957320413) -- (2.4768322075994043, 9.553461300069118); 
\draw[colorpcalg] (2.4068322075994044, 8.41144957320413) -- (2.546832207599404, 8.41144957320413); 
\draw[colorpcalg] (2.4068322075994044, 9.553461300069118) -- (2.546832207599404, 9.553461300069118); 
\draw[colorpcalg] (3.447238260383328, 10.007265481185287) -- (3.447238260383328, 11.464517977749281); 
\draw[colorpcalg] (3.377238260383328, 10.007265481185287) -- (3.5172382603833277, 10.007265481185287); 
\draw[colorpcalg] (3.377238260383328, 11.464517977749281) -- (3.5172382603833277, 11.464517977749281); 
\draw[colorpcalg] (4.730045285007144, 12.408602810657852) -- (4.730045285007144, 13.892322850727718); 
\draw[colorpcalg] (4.660045285007143, 12.408602810657852) -- (4.800045285007144, 12.408602810657852); 
\draw[colorpcalg] (4.660045285007143, 13.892322850727718) -- (4.800045285007144, 13.892322850727718); 
\draw[colorpcalg] (3.447238260383328, 7.80926301821167) -- (3.447238260383328, 8.089230109355912); 
\draw[colorpcalg] (3.377238260383328, 7.80926301821167) -- (3.5172382603833277, 7.80926301821167); 
\draw[colorpcalg] (3.377238260383328, 8.089230109355912) -- (3.5172382603833277, 8.089230109355912); 
\draw[colorpcalg] (4.730045285007144, 16.82939419947397) -- (4.730045285007144, 18.833610720149988); 
\draw[colorpcalg] (4.660045285007143, 16.82939419947397) -- (4.800045285007144, 16.82939419947397); 
\draw[colorpcalg] (4.660045285007143, 18.833610720149988) -- (4.800045285007144, 18.833610720149988); 

\end{tikzpicture}
\end{minipage}
   
   \vspace*{-1mm}
% % \centering\begin{tikzpicture}[minimum size=0]\scriptsize 
% %   \node[draw,cross out] at (0,0) {}; 
% %       \node[anchor=south west] at (0, 0) {\scriptsize$\,k$};
% %       \node at (0.5,0) {and};  
% %       \draw[dashed] (1, 0) -- (1.6, 0) node[anchor=west] {\hspace*{2mm} DAGs with $m \approxeq 1n$;};
% %   \node[draw,circle] at (0+4.5,0) {};  \node at (0.5+4.5,0) {and};  
% %       \draw (1+4.5, 0) -- (1.6+4.5, 0) node[anchor=west] {\hspace*{2mm} DAGs with $m \approxeq 10n$;};
% %   \node[anchor=west] at (0+3,-0.7) {\scriptsize black: {\tt DAGitty} package; \hspace*{0.15cm} {\color{colorcausaleffect}\colorcausaleffecttext}: {\tt causaleffect} package;};  
% % \end{tikzpicture}
% %     
\begin{center}\scriptsize
%    \begin{tikzpicture}[minimum size=0]\scriptsize 
%          \node[draw,cross out] at (0,0) {}; 
%          \node[anchor=south west] at (0, 0) {\scriptsize$\,k$};
%     \end{tikzpicture}
%     \quad DAGs of $\EX[m] = 1n$ edges and $|\bX|=|\bY|=k$, and \quad
     \begin{tikzpicture}[minimum size=0]\scriptsize 
          \draw[dashed] (1, 0) -- (1.6, 0);
          \node at (1,-0.03) {}; 
     \end{tikzpicture}
     \quad DAGs of $\EX[m] =1n$ edges and $|\bX|=|\bY|\in\{1,2,5\}$  
     \quad
%    \begin{tikzpicture}[minimum size=0]\scriptsize 
%          \node[draw,circle] at (0,0) {}; 
%          \node[anchor=south west] at (0, 0) {\scriptsize$\,k$};
%     \end{tikzpicture}
%     \quad DAGs of $\EX[m] =10n$ edges and $|\bX|=|\bY|=k$, and \quad
     \begin{tikzpicture}[minimum size=0]\scriptsize 
          \draw (1, 0) -- (1.6, 0);
          \node at (1,-0.03) {}; 
     \end{tikzpicture}
     \quad DAGs of $\EX[m] = 10n$ edges and $|\bX|=|\bY|\in\{1,2,5\}$  \\[2mm]
  black: {\tt DAGitty} package; \hspace*{0.15cm} {\color{colorcausaleffect}\colorcausaleffecttext}: {\tt causaleffect} package; \hspace*{0.15cm} {\color{colorpcalg}\colorpcalgtext}: {\tt pcalg} package;
\end{center}

\caption{Average time needed to find an adjustment set and verify it according
  to the CBC in a single graph with  $P(\textit{unobserved}) = 0.75$,
  values $n \in \{10, 25, 50, 100, 250, 500, 1000, 2000\}$, $\EX[m] \in \{1n, 10n\}$, 
  and various cardinalities $k$ of $\bX,\bY$.
  The left and the right plot show the same data, but with a different metric on the horizontal axis: 
  the number of nodes  $n$ (left) and the expected number of edges $m$ (right). 
  In black we show the statistics for {\tt DAGitty}, in \colorcausaleffecttext \textendash\ the data for the R package {\tt causaleffect} and in \colorpcalgtext \textendash\ the data for the {\tt gac} function of the R package {\tt pcalg} {\tt R}.
  %
  %For $k\in \{1,2,5\}$ and specific values $n$ and $\EX[m]$ we compute the average time over all such instances
  %and then linearly interpolated the obtained data points by a curve using black or \colorcausaleffecttext color
  %and dashed or solid shape to indicate the case, as explained above. % labeled as $(1,2,5)$.
	Error bars show the minimum and maximum time taken.
 % For the remaining parameter values $k$, single data points, labeled by $k$, are shown.
 % The data are presented using a logarithmic scale. 
  The plot shows that all small graphs ($n \leq 100$) are solved nearly instantaneously (time $\leq 100ms$) by 
  {\tt DAGitty}. Only graphs with a high number of edges and huge $\bX,\bY$ can require a few seconds. 
  Thus {\tt DAGitty} is one to two magnitudes faster than the {\tt causaleffect} 
  package or the {\tt pcalg} package.
 %   (left)  The runtime on realistic sized models with graphes $n \leq 250$ or $|\bX| \leq 50$. Some values ($n\geq 500$, $m \geq 5n$, $|\bX| > 20$) are off the scale with 12.09 seconds. All small graphs ($n \leq 100$) are cluttered in the left bottom corner which shows that they are solved nearly instantanousely.  (right) The runtime for high cardinalities of $\bX$ and $\bY$.  
    }\label{fig:runtime:nodes}\end{figure}

%% file: experiments-runtimes-table.tex
\begin{table}
  \begin{center}
  \scriptsize
  \begin{tabular}{|rr | rrr| rrr| rrr| rrr|}
  \hline
   &&\multicolumn{3}{c|}{$l=2$} &\multicolumn{3}{c|}{$l=5$} &\multicolumn{3}{c|}{$l=10$} &\multicolumn{3}{c|}{$l=20$}\\
   \bfseries $n$ & \bfseries $k$
&   \bfseries CBC& \bfseries IDC& \bfseries GAC
&   \bfseries CBC& \bfseries IDC& \bfseries GAC
&   \bfseries CBC& \bfseries IDC& \bfseries GAC
&   \bfseries CBC& \bfseries IDC& \bfseries GAC
\\\hline
10&1&
0.3 ms&35.0 ms&13.6 ms&0.5 ms&53.6 ms&49.6 ms&0.6 ms&60.4 ms&303 ms&0.6 ms&59.9 ms&169 ms\\
10&2&
0.5 ms&70.5 ms&17.9 ms&0.6 ms&100 ms&61.9 ms&0.8 ms&110 ms&173 ms&1.0 ms&110 ms&211 ms\\
10&3&
0.5 ms&124 ms&14.1 ms&0.8 ms&182 ms&32.5 ms&1.1 ms&190 ms&121 ms&1.1 ms&187 ms&60.7 ms\\
10&5&
0.7 ms&242 ms&15.3 ms&1.0 ms&296 ms&52.4 ms&1.6 ms&306 ms&170 ms&1.6 ms&305 ms&173 ms\\
\hline
25&1&
0.5 ms&27.4 ms&72.1 ms&0.9 ms&46.9 ms&1.9 s&1.6 ms&78.7 ms&18.5 s&2.9 ms&89.7 ms&1.6 h\\
25&2&
0.7 ms&54.9 ms&112 ms&1.3 ms&121 ms&680 ms&2.2 ms&193 ms&16.4 s&4.4 ms&206 ms&-\\
25&3&
0.8 ms&70.3 ms&113 ms&1.5 ms&127 ms&695 ms&2.6 ms&168 ms&15.1 s&5.7 ms&168 ms&39.9 min\\
25&5&
1.0 ms&132 ms&114 ms&1.7 ms&276 ms&522 ms&3.1 ms&309 ms&9.5 s&7.1 ms&305 ms&18.9 min\\
\hline
50&1&
0.8 ms&25.9 ms&273 ms&1.6 ms&51.7 ms&-&2.8 ms&151 ms&-&5.4 ms&257 ms&-\\
50&2&
1.1 ms&46.9 ms&557 ms&2.0 ms&157 ms&-&3.8 ms&475 ms&-&7.7 ms&691 ms&-\\
50&5&
1.5 ms&142 ms&680 ms&2.8 ms&414 ms&-&5.6 ms&796 ms&-&14.2 ms&937 ms&-\\
50&7&
1.7 ms&224 ms&744 ms&3.3 ms&649 ms&-&6.7 ms&1.1 s&-&17.9 ms&1.1 s&-\\
\hline
100&1&
1.3 ms&24.5 ms&978 ms&2.5 ms&53.1 ms&-&4.8 ms&366 ms&-&9.6 ms&914 ms&-\\
100&2&
1.6 ms&32.4 ms&1.7 s&3.0 ms&128 ms&-&6.1 ms&836 ms&-&13.6 ms&1.8 s&-\\
100&5&
2.0 ms&91.3 ms&3.1 s&4.1 ms&440 ms&-&9.2 ms&1.7 s&-&25.3 ms&2.7 s&-\\
100&10&
2.6 ms&248 ms&4.1 s&5.6 ms&843 ms&-&13.8 ms&2.4 s&-&47.5 ms&3.0 s&-\\
\hline
250&1&
2.9 ms&25.9 ms&6.7 s&5.1 ms&74.1 ms&-&10.2 ms&966 ms&-&22.2 ms&5.1 s&-\\
250&2&
3.1 ms&28.2 ms&9.2 s&6.0 ms&130 ms&-&12.9 ms&2.6 s&-&30.6 ms&12.0 s&-\\
250&5&
3.7 ms&64.6 ms&21.9 s&7.7 ms&499 ms&-&19.1 ms&6.2 s&-&56.8 ms&20.7 s&-\\
250&15&
5.3 ms&462 ms&48.5 s&13.3 ms&2.8 s&-&38.9 ms&19.3 s&-&144 ms&44.1 s&-\\
250&25&
6.4 ms&644 ms&1.1 min&17.4 ms&3.2 s&-&57.0 ms&12.9 s&-&252 ms&22.9 s&-\\
\hline
500&1&
5.2 ms&-&-&8.7 ms&-&-&18.1 ms&-&-&42.1 ms&-&-\\
500&2&
5.5 ms&-&-&10.1 ms&-&-&22.7 ms&-&-&58.1 ms&-&-\\
500&5&
6.2 ms&-&-&12.3 ms&-&-&33.8 ms&-&-&106 ms&-&-\\
500&22&
9.0 ms&-&-&24.5 ms&-&-&89.5 ms&-&-&391 ms&-&-\\
500&50&
13.4 ms&-&-&46.3 ms&-&-&192 ms&-&-&886 ms&-&-\\
\hline
1000&1&
11.1 ms&-&-&17.2 ms&-&-&35.4 ms&-&-&83.2 ms&-&-\\
1000&2&
11.8 ms&-&-&19.0 ms&-&-&42.6 ms&-&-&114 ms&-&-\\
1000&5&
12.4 ms&-&-&21.8 ms&-&-&61.8 ms&-&-&205 ms&-&-\\
1000&32&
17.1 ms&-&-&51.3 ms&-&-&225 ms&-&-&1.1 s&-&-\\
1000&100&
33.5 ms&-&-&143 ms&-&-&715 ms&-&-&3.4 s&-&-\\
\hline
2000&1&
24.6 ms&-&-&35.1 ms&-&-&70.8 ms&-&-&178 ms&-&-\\
2000&2&
25.1 ms&-&-&37.7 ms&-&-&83.6 ms&-&-&241 ms&-&-\\
2000&5&
26.0 ms&-&-&41.8 ms&-&-&120 ms&-&-&409 ms&-&-\\
2000&45&
34.5 ms&-&-&107 ms&-&-&549 ms&-&-&2.8 s&-&-\\
2000&200&
92.3 ms&-&-&505 ms&-&-&2.7 s&-&-&13.4 s&-&-\\
\hline\end{tabular}
  \end{center}\vspace*{-3mm}
  \caption{Average time to run DAGitty (CBC), the R package {causaleffect}  (IDC) or the {\tt gac} function of the R package {\tt pcalg} (GAC) on one graph with $P(\textit{unobserved}=0.75)$.  Omitted values indicate experiments we did not run or were unable to run due to runtime issues. On some parametrizations we only ran the  {\tt gac} function on the first 100 graphs rather than the full set of 10000 graphs due to time constraints.
}\label{table:runtimes}
  \end{table}

%% file: experiments-mags-plot.tex
\begin{figure}\begin{center}
\onlyinblackandwhite{\definecolor{green}{gray}{0.5}}
\begin{tikzpicture}[minimum size=0,xscale=0.53,yscale=0.5]
\node[] at (1.5, -1) {};
\draw[] (0, 0) -- (3, 0);
\draw (-1, -0.03) -- (-1, 0.03);
\node[rotate=90,text width=2.4cm] at (-1, -0.5) {\makebox[6mm]{\hfill$n$}, $k$\phantom{,} \makebox[4mm]{\hfill}};
\draw (0, -0.03) -- (0, 0.03);
\node[rotate=90,text width=2.4cm] at (0, -0.5) {\makebox[6mm]{\hfill$25$}, $5$\phantom{,} \makebox[4mm]{\hfill}};
\draw (1, -0.03) -- (1, 0.03);
\node[rotate=90,text width=2.4cm] at (1, -0.5) {\makebox[6mm]{\hfill$50$}, $5$\phantom{,} \makebox[4mm]{\hfill}};
\draw (2, -0.03) -- (2, 0.03);
\node[rotate=90,text width=2.4cm] at (2, -0.5) {\makebox[6mm]{\hfill$100$}, $5$\phantom{,} \makebox[4mm]{\hfill}};
\draw (3, -0.03) -- (3, 0.03);
\node[rotate=90,text width=2.4cm] at (3, -0.5) {\makebox[6mm]{\hfill$250$}, $5$\phantom{,} \makebox[4mm]{\hfill}};
\node[] at (5.5, -1) {};
\draw[] (4, 0) -- (7, 0);
\draw (4, -0.03) -- (4, 0.03);
\node[rotate=90,text width=2.4cm] at (4, -0.5) {\makebox[6mm]{\hfill$25$}, $2$\phantom{,} \makebox[4mm]{\hfill}};
\draw (5, -0.03) -- (5, 0.03);
\node[rotate=90,text width=2.4cm] at (5, -0.5) {\makebox[6mm]{\hfill$50$}, $2$\phantom{,} \makebox[4mm]{\hfill}};
\draw (6, -0.03) -- (6, 0.03);
\node[rotate=90,text width=2.4cm] at (6, -0.5) {\makebox[6mm]{\hfill$100$}, $2$\phantom{,} \makebox[4mm]{\hfill}};
\draw (7, -0.03) -- (7, 0.03);
\node[rotate=90,text width=2.4cm] at (7, -0.5) {\makebox[6mm]{\hfill$250$}, $2$\phantom{,} \makebox[4mm]{\hfill}};
\node[] at (9.5, -1) {};
\draw[] (8, 0) -- (11, 0);
\draw (8, -0.03) -- (8, 0.03);
\node[rotate=90,text width=2.4cm] at (8, -0.5) {\makebox[6mm]{\hfill$25$}, $1$\phantom{,} \makebox[4mm]{\hfill}};
\draw (9, -0.03) -- (9, 0.03);
\node[rotate=90,text width=2.4cm] at (9, -0.5) {\makebox[6mm]{\hfill$50$}, $1$\phantom{,} \makebox[4mm]{\hfill}};
\draw (10, -0.03) -- (10, 0.03);
\node[rotate=90,text width=2.4cm] at (10, -0.5) {\makebox[6mm]{\hfill$100$}, $1$\phantom{,} \makebox[4mm]{\hfill}};
\draw (11, -0.03) -- (11, 0.03);
\node[rotate=90,text width=2.4cm] at (11, -0.5) {\makebox[6mm]{\hfill$250$}, $1$\phantom{,} \makebox[4mm]{\hfill}};
\node[rotate=90] at (-2, 5) {\% of identified graphs};
\draw[] (0, 0) -- (0, 10);
\draw (-0.03, 0) -- (0.03, 0);
\node[] at (-0.75, 0) {0\%};
\draw (-0.03, 1) -- (0.03, 1);
\node[] at (-0.75, 1) {10\%};
\draw (-0.03, 2) -- (0.03, 2);
\node[] at (-0.75, 2) {20\%};
\draw (-0.03, 3) -- (0.03, 3);
\node[] at (-0.75, 3) {30\%};
\draw (-0.03, 4) -- (0.03, 4);
\node[] at (-0.75, 4) {40\%};
\draw (-0.03, 5) -- (0.03, 5);
\node[] at (-0.75, 5) {50\%};
\draw (-0.03, 6) -- (0.03, 6);
\node[] at (-0.75, 6) {60\%};
\draw (-0.03, 7) -- (0.03, 7);
\node[] at (-0.75, 7) {70\%};
\draw (-0.03, 8) -- (0.03, 8);
\node[] at (-0.75, 8) {80\%};
\draw (-0.03, 9) -- (0.03, 9);
\node[] at (-0.75, 9) {90\%};
\draw (-0.03, 10) -- (0.03, 10);
\node[] at (-0.75, 10) {100\%};
\draw[black!30] (0, 1) -- (3, 1); 
\draw[black!30] (4, 1) -- (7, 1); 
\draw[black!30] (8, 1) -- (11, 1); 
\draw[black!30] (0, 2) -- (3, 2); 
\draw[black!30] (4, 2) -- (7, 2); 
\draw[black!30] (8, 2) -- (11, 2); 
\draw[black!30] (0, 3) -- (3, 3); 
\draw[black!30] (4, 3) -- (7, 3); 
\draw[black!30] (8, 3) -- (11, 3); 
\draw[black!30] (0, 4) -- (3, 4); 
\draw[black!30] (4, 4) -- (7, 4); 
\draw[black!30] (8, 4) -- (11, 4); 
\draw[black!30] (0, 5) -- (3, 5); 
\draw[black!30] (4, 5) -- (7, 5); 
\draw[black!30] (8, 5) -- (11, 5); 
\draw[black!30] (0, 6) -- (3, 6); 
\draw[black!30] (4, 6) -- (7, 6); 
\draw[black!30] (8, 6) -- (11, 6); 
\draw[black!30] (0, 7) -- (3, 7); 
\draw[black!30] (4, 7) -- (7, 7); 
\draw[black!30] (8, 7) -- (11, 7); 
\draw[black!30] (0, 8) -- (3, 8); 
\draw[black!30] (4, 8) -- (7, 8); 
\draw[black!30] (8, 8) -- (11, 8); 
\draw[black!30] (0, 9) -- (3, 9); 
\draw[black!30] (4, 9) -- (7, 9); 
\draw[black!30] (8, 9) -- (11, 9); 
\draw[black] (0, 10) -- (3, 10); 
\draw[black] (4, 10) -- (7, 10); 
\draw[black] (8, 10) -- (11, 10); 
\draw[black] (0, 0) -- (0, 10); 
\draw[black] (3, 0) -- (3, 10); 
\draw[black] (4, 0) -- (4, 10); 
\draw[black] (7, 0) -- (7, 10); 
\draw[black] (8, 0) -- (8, 10); 
\draw[black] (11, 0) -- (11, 10); 
\draw[green] (0, 1.19) -- (1, 3.34) -- (2, 5.74) -- (3, 7.95); 
\draw[green] (4, 6.88) -- (5, 8.3) -- (6, 9.09) -- (7, 9.61); 
\draw[green] (8, 9.07) -- (9, 9.58) -- (10, 9.76) -- (11, 9.9); 
\draw[black!40!green] (0, 0.05) -- (1, 0.36) -- (2, 1.42) -- (3, 3.8); 
\draw[black!40!green] (4, 3.3) -- (5, 5.54) -- (6, 7.22) -- (7, 8.53); 
\draw[black!40!green] (8, 7.73) -- (9, 8.74) -- (10, 9.25) -- (11, 9.57); 
\draw[black!60!green] (0, 0) -- (1, 0) -- (2, 0.06) -- (3, 0.45); 
\draw[black!60!green] (4, 0.79) -- (5, 2.3) -- (6, 4.1) -- (7, 5.94); 
\draw[black!60!green] (8, 6.08) -- (9, 7.87) -- (10, 8.67) -- (11, 9.18); 
\draw[black!80!green] (0, 0) -- (1, 0) -- (2, 0) -- (3, 0.01); 
\draw[black!80!green] (4, 0.03) -- (5, 0.55) -- (6, 1.48) -- (7, 3.04); 
\draw[black!80!green] (8, 2.17) -- (9, 6.5) -- (10, 8.13) -- (11, 9.07); 
\draw[green,dotted] (0, 1.19) -- (1, 3.34) -- (2, 5.74) -- (3, 7.95); 
\draw[green,dotted] (4, 6.88) -- (5, 8.3) -- (6, 9.09) -- (7, 9.61); 
\draw[green,dotted] (8, 9.07) -- (9, 9.58) -- (10, 9.76) -- (11, 9.9); 
\draw[black!40!green,dotted] (0, 0.05) -- (1, 0.36) -- (2, 1.42) -- (3, 3.8); 
\draw[black!40!green,dotted] (4, 3.3) -- (5, 5.54) -- (6, 7.22) -- (7, 8.53); 
\draw[black!40!green,dotted] (8, 7.73) -- (9, 8.74) -- (10, 9.25) -- (11, 9.57); 
\draw[black!60!green,dotted] (0, 0) -- (1, 0) -- (2, 0.06) -- (3, 0.45); 
\draw[black!60!green,dotted] (4, 0.79) -- (5, 2.3) -- (6, 4.1) -- (7, 5.94); 
\draw[black!60!green,dotted] (8, 6.08) -- (9, 7.87) -- (10, 8.67) -- (11, 9.18); 
\draw[black!80!green,dotted] (0, 0) -- (1, 0) -- (2, 0) -- (3, 0.01); 
\draw[black!80!green,dotted] (4, 0.03) -- (5, 0.55) -- (6, 1.48) -- (7, 3.04); 
\draw[black!80!green,dotted] (8, 2.17) -- (9, 6.5) -- (10, 8.13) -- (11, 9.07); 
\end{tikzpicture}
\hspace{1cm}
\begin{tikzpicture}[minimum size=0,xscale=0.53,yscale=0.5]
\node[] at (1.5, -1) {};
\draw[] (0, 0) -- (3, 0);
\draw (-1, -0.03) -- (-1, 0.03);
\node[rotate=90,text width=2.4cm] at (-1, -0.5) {\makebox[6mm]{\hfill$n$}, $k$\phantom{,} \makebox[4mm]{\hfill}};
\draw (0, -0.03) -- (0, 0.03);
\node[rotate=90,text width=2.4cm] at (0, -0.5) {\makebox[6mm]{\hfill$25$}, $5$\phantom{,} \makebox[4mm]{\hfill}};
\draw (1, -0.03) -- (1, 0.03);
\node[rotate=90,text width=2.4cm] at (1, -0.5) {\makebox[6mm]{\hfill$50$}, $5$\phantom{,} \makebox[4mm]{\hfill}};
\draw (2, -0.03) -- (2, 0.03);
\node[rotate=90,text width=2.4cm] at (2, -0.5) {\makebox[6mm]{\hfill$100$}, $5$\phantom{,} \makebox[4mm]{\hfill}};
\draw (3, -0.03) -- (3, 0.03);
\node[rotate=90,text width=2.4cm] at (3, -0.5) {\makebox[6mm]{\hfill$250$}, $5$\phantom{,} \makebox[4mm]{\hfill}};
\node[] at (5.5, -1) {};
\draw[] (4, 0) -- (7, 0);
\draw (4, -0.03) -- (4, 0.03);
\node[rotate=90,text width=2.4cm] at (4, -0.5) {\makebox[6mm]{\hfill$25$}, $2$\phantom{,} \makebox[4mm]{\hfill}};
\draw (5, -0.03) -- (5, 0.03);
\node[rotate=90,text width=2.4cm] at (5, -0.5) {\makebox[6mm]{\hfill$50$}, $2$\phantom{,} \makebox[4mm]{\hfill}};
\draw (6, -0.03) -- (6, 0.03);
\node[rotate=90,text width=2.4cm] at (6, -0.5) {\makebox[6mm]{\hfill$100$}, $2$\phantom{,} \makebox[4mm]{\hfill}};
\draw (7, -0.03) -- (7, 0.03);
\node[rotate=90,text width=2.4cm] at (7, -0.5) {\makebox[6mm]{\hfill$250$}, $2$\phantom{,} \makebox[4mm]{\hfill}};
\node[] at (9.5, -1) {};
\draw[] (8, 0) -- (11, 0);
\draw (8, -0.03) -- (8, 0.03);
\node[rotate=90,text width=2.4cm] at (8, -0.5) {\makebox[6mm]{\hfill$25$}, $1$\phantom{,} \makebox[4mm]{\hfill}};
\draw (9, -0.03) -- (9, 0.03);
\node[rotate=90,text width=2.4cm] at (9, -0.5) {\makebox[6mm]{\hfill$50$}, $1$\phantom{,} \makebox[4mm]{\hfill}};
\draw (10, -0.03) -- (10, 0.03);
\node[rotate=90,text width=2.4cm] at (10, -0.5) {\makebox[6mm]{\hfill$100$}, $1$\phantom{,} \makebox[4mm]{\hfill}};
\draw (11, -0.03) -- (11, 0.03);
\node[rotate=90,text width=2.4cm] at (11, -0.5) {\makebox[6mm]{\hfill$250$}, $1$\phantom{,} \makebox[4mm]{\hfill}};
\node[rotate=90] at (-2, 5) {\% of identified graphs};
\draw[] (0, 0) -- (0, 10);
\draw (-0.03, 0) -- (0.03, 0);
\node[] at (-0.75, 0) {0\%};
\draw (-0.03, 1) -- (0.03, 1);
\node[] at (-0.75, 1) {10\%};
\draw (-0.03, 2) -- (0.03, 2);
\node[] at (-0.75, 2) {20\%};
\draw (-0.03, 3) -- (0.03, 3);
\node[] at (-0.75, 3) {30\%};
\draw (-0.03, 4) -- (0.03, 4);
\node[] at (-0.75, 4) {40\%};
\draw (-0.03, 5) -- (0.03, 5);
\node[] at (-0.75, 5) {50\%};
\draw (-0.03, 6) -- (0.03, 6);
\node[] at (-0.75, 6) {60\%};
\draw (-0.03, 7) -- (0.03, 7);
\node[] at (-0.75, 7) {70\%};
\draw (-0.03, 8) -- (0.03, 8);
\node[] at (-0.75, 8) {80\%};
\draw (-0.03, 9) -- (0.03, 9);
\node[] at (-0.75, 9) {90\%};
\draw (-0.03, 10) -- (0.03, 10);
\node[] at (-0.75, 10) {100\%};
\draw[black!30] (0, 1) -- (3, 1); 
\draw[black!30] (4, 1) -- (7, 1); 
\draw[black!30] (8, 1) -- (11, 1); 
\draw[black!30] (0, 2) -- (3, 2); 
\draw[black!30] (4, 2) -- (7, 2); 
\draw[black!30] (8, 2) -- (11, 2); 
\draw[black!30] (0, 3) -- (3, 3); 
\draw[black!30] (4, 3) -- (7, 3); 
\draw[black!30] (8, 3) -- (11, 3); 
\draw[black!30] (0, 4) -- (3, 4); 
\draw[black!30] (4, 4) -- (7, 4); 
\draw[black!30] (8, 4) -- (11, 4); 
\draw[black!30] (0, 5) -- (3, 5); 
\draw[black!30] (4, 5) -- (7, 5); 
\draw[black!30] (8, 5) -- (11, 5); 
\draw[black!30] (0, 6) -- (3, 6); 
\draw[black!30] (4, 6) -- (7, 6); 
\draw[black!30] (8, 6) -- (11, 6); 
\draw[black!30] (0, 7) -- (3, 7); 
\draw[black!30] (4, 7) -- (7, 7); 
\draw[black!30] (8, 7) -- (11, 7); 
\draw[black!30] (0, 8) -- (3, 8); 
\draw[black!30] (4, 8) -- (7, 8); 
\draw[black!30] (8, 8) -- (11, 8); 
\draw[black!30] (0, 9) -- (3, 9); 
\draw[black!30] (4, 9) -- (7, 9); 
\draw[black!30] (8, 9) -- (11, 9); 
\draw[black] (0, 10) -- (3, 10); 
\draw[black] (4, 10) -- (7, 10); 
\draw[black] (8, 10) -- (11, 10); 
\draw[black] (0, 0) -- (0, 10); 
\draw[black] (3, 0) -- (3, 10); 
\draw[black] (4, 0) -- (4, 10); 
\draw[black] (7, 0) -- (7, 10); 
\draw[black] (8, 0) -- (8, 10); 
\draw[black] (11, 0) -- (11, 10); 
\draw[green] (0, 0.12) -- (1, 0.8) -- (2, 2.5) -- (3, 5.32); 
\draw[green] (4, 4.1) -- (5, 6.18) -- (6, 7.81) -- (7, 8.98); 
\draw[green] (8, 7.84) -- (9, 8.8) -- (10, 9.38) -- (11, 9.74); 
\draw[black!40!green] (0, 0) -- (1, 0) -- (2, 0.01) -- (3, 0.06); 
\draw[black!40!green] (4, 0.23) -- (5, 0.64) -- (6, 1.4) -- (7, 2.72); 
\draw[black!40!green] (8, 2.45) -- (9, 3.79) -- (10, 4.94) -- (11, 6.46); 
\draw[black!60!green] (0, 0) -- (1, 0) -- (2, 0) -- (3, 0); 
\draw[black!60!green] (4, 0) -- (5, 0.02) -- (6, 0.06) -- (7, 0.12); 
\draw[black!60!green] (8, 0.35) -- (9, 0.8) -- (10, 1.18) -- (11, 1.8); 
\draw[black!80!green] (0, 0) -- (1, 0) -- (2, 0) -- (3, 0); 
\draw[black!80!green] (4, 0) -- (5, 0) -- (6, 0) -- (7, 0.01); 
\draw[black!80!green] (8, 0.01) -- (9, 0.09) -- (10, 0.19) -- (11, 0.35); 
\draw[green,dotted] (0, 0.15) -- (1, 0.91) -- (2, 2.62) -- (3, 5.42); 
\draw[green,dotted] (4, 4.31) -- (5, 6.32) -- (6, 7.88) -- (7, 9); 
\draw[green,dotted] (8, 7.94) -- (9, 8.86) -- (10, 9.4) -- (11, 9.74); 
\draw[black!40!green,dotted] (0, 0) -- (1, 0) -- (2, 0.01) -- (3, 0.06); 
\draw[black!40!green,dotted] (4, 0.27) -- (5, 0.69) -- (6, 1.46) -- (7, 2.76); 
\draw[black!40!green,dotted] (8, 2.64) -- (9, 3.92) -- (10, 5.07) -- (11, 6.52); 
\draw[black!60!green,dotted] (0, 0) -- (1, 0) -- (2, 0) -- (3, 0); 
\draw[black!60!green,dotted] (4, 0.01) -- (5, 0.02) -- (6, 0.08) -- (7, 0.14); 
\draw[black!60!green,dotted] (8, 0.43) -- (9, 0.92) -- (10, 1.31) -- (11, 1.94); 
\draw[black!80!green,dotted] (0, 0) -- (1, 0) -- (2, 0) -- (3, 0); 
\draw[black!80!green,dotted] (4, 0) -- (5, 0) -- (6, 0) -- (7, 0.01); 
\draw[black!80!green,dotted] (8, 0.02) -- (9, 0.11) -- (10, 0.24) -- (11, 0.42); 
\draw[draw=black,fill=white] (8.75, 5.25) rectangle (12.2, 2.55);
\draw[green] (9, 4.95) -- (10, 4.95); 
\node[anchor=west] at (10.1, 4.95) {\tiny $l = 2$};
\draw[black!40!green] (9, 4.65) -- (10, 4.65); 
\node[anchor=west] at (10.1, 4.65) {\tiny $l = 5$};
\draw[black!60!green] (9, 4.35) -- (10, 4.35); 
\node[anchor=west] at (10.1, 4.35) {\tiny $l = 10$};
\draw[black!80!green] (9, 4.05) -- (10, 4.05); 
\node[anchor=west] at (10.1, 4.05) {\tiny $l = 20$};
\draw[green,dotted] (9, 3.75) -- (10, 3.75); 
\node[anchor=west] at (10.1, 3.75) {\tiny $l = 2$};
\draw[black!40!green,dotted] (9, 3.45) -- (10, 3.45); 
\node[anchor=west] at (10.1, 3.45) {\tiny $l = 5$};
\draw[black!60!green,dotted] (9, 3.15) -- (10, 3.15); 
\node[anchor=west] at (10.1, 3.15) {\tiny $l = 10$};
\draw[black!80!green,dotted] (9, 2.85) -- (10, 2.85); 
\node[anchor=west] at (10.1, 2.85) {\tiny $l = 20$};
\end{tikzpicture}
\hspace{1cm}
\end{center}\vspace*{-3mm}
  \caption{%Case $P(\textit{unobserved}=||$p[1]||)$:
  Percent of identifiable MAGs for $P(\textit{unobserved}) = 0$ (left) and  $P(\textit{unobserved}) = 0.75$ (right).
  The horizontal axis is labeled by $(n,k)$ sorted lexicographically by % depending on 
  the number of nodes $n$ and cardinalities 
  $|\bX|=|\bY|=k$. 
  The dotted lines show the results for $\cG[^\emptyset_\emptyset$  and the solid lines for $\cG[^\emptyset_\bL$ with various density parameter values $l\in\{2,5,10,20\}$.% is shown as one line, 
  }\label{figure:mags:plot}
  \end{figure}
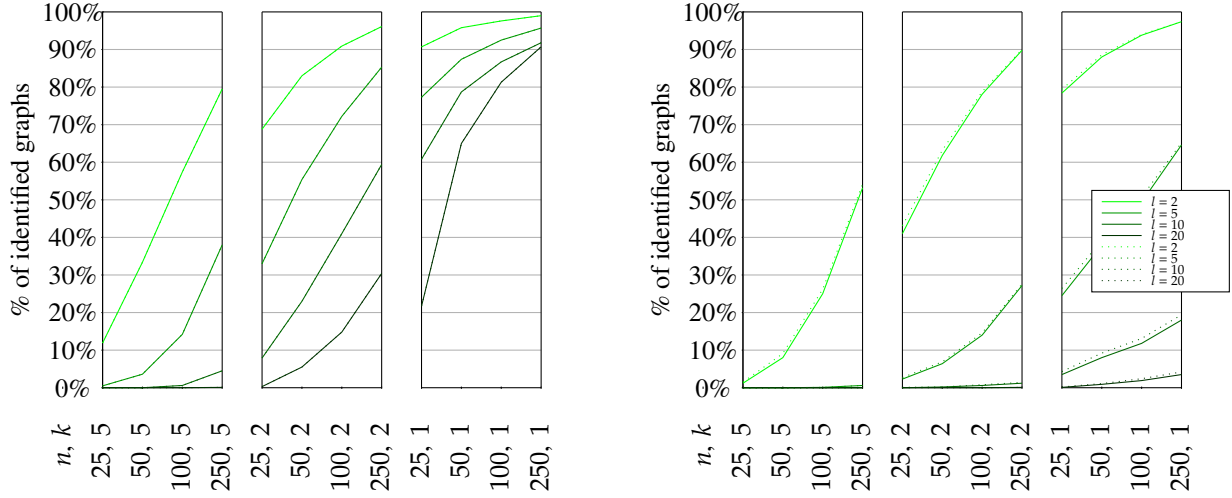

%% file: experiments-mags-table.tex
\begin{table}
  \begin{center}
  \scriptsize
  \begin{tabular}{|rr | rr| rr| rr| rr|}
  \hline
   &&\multicolumn{2}{c|}{$l=2$} &\multicolumn{2}{c|}{$l=5$} &\multicolumn{2}{c|}{$l=10$} &\multicolumn{2}{c|}{$l=20$}\\
   \bfseries $n$ & \bfseries $k$
& \bfseries CBC & \bfseries MAG$^\emptyset_\emptyset$
& \bfseries CBC & \bfseries MAG$^\emptyset_\emptyset$
& \bfseries CBC & \bfseries MAG$^\emptyset_\emptyset$
& \bfseries CBC & \bfseries MAG$^\emptyset_\emptyset$
\\\hline
10&1&
8893&7711&7205&4600&5034&0&4934&0\\
10&2&
6061&3773&1980&529&660&0&686&0\\
10&3&
3395&1243&548&32&168&0&174&0\\
10&5&
886&64&108&0&36&0&31&0\\
\hline
25&1&
9573&9066&8936&7731&7905&6080&5843&2168\\
25&2&
8247&6876&4735&3300&1553&789&401&33\\
25&3&
6424&4399&1852&913&212&61&34&0\\
25&5&
3021&1186&243&51&11&0&1&0\\
\hline
50&1&
9832&9582&9476&8742&8997&7869&7832&6501\\
50&2&
9128&8305&6938&5538&3049&2301&866&549\\
50&5&
5535&3339&799&361&16&5&2&0\\
50&7&
3120&1234&119&21&1&0&0&0\\
\hline
100&1&
9907&9756&9783&9249&9494&8674&8966&8126\\
100&2&
9593&9092&8353&7216&4834&4098&1762&1476\\
100&5&
7591&5742&2336&1416&102&59&4&1\\
100&10&
3040&1117&48&9&0&0&0&0\\
\hline
250&1&
9947&9898&9894&9573&9774&9176&9621&9069\\
250&2&
9840&9613&9327&8534&6689&5942&3285&3035\\
250&5&
8994&7954&5325&3802&544&453&17&12\\
250&15&
3537&1360&32&6&0&0&0&0\\
250&25&
536&50&0&0&0&0&0&0\\
\hline\end{tabular}
  \end{center}\vspace*{-3mm}
  \caption{Number of instances for $P(\textit{unobserved}) = 0$ that are identified using CBC in the DAG or after converting the DAG to a MAG.}\label{table:mags:0}
  \end{table}
\begin{table}
  \begin{center}
  \scriptsize
  \begin{tabular}{|rr | rrr| rrr| rrr| rrr|}
  \hline
   &&\multicolumn{3}{c|}{$l=2$} &\multicolumn{3}{c|}{$l=5$} &\multicolumn{3}{c|}{$l=10$} &\multicolumn{3}{c|}{$l=20$}\\
   \bfseries $n$ & \bfseries $k$
& \bfseries CBC & \bfseries MAG$^\emptyset_\emptyset$& \bfseries MAG$^\emptyset_\bL$ 
& \bfseries CBC & \bfseries MAG$^\emptyset_\emptyset$& \bfseries MAG$^\emptyset_\bL$ 
& \bfseries CBC & \bfseries MAG$^\emptyset_\emptyset$& \bfseries MAG$^\emptyset_\bL$ 
& \bfseries CBC & \bfseries MAG$^\emptyset_\emptyset$& \bfseries MAG$^\emptyset_\bL$ 
\\\hline
10&1&
6333&5398&5157&1935&645&543&978&0&0&936&0&0\\
10&2&
2339&1371&1204&228&8&4&113&0&0&114&0&0\\
10&3&
980&391&304&21&0&0&9&0&0&10&0&0\\
10&5&
859&56&56&98&0&0&43&0&0&26&0&0\\
\hline
25&1&
8414&7945&7838&3647&2635&2446&1340&429&346&557&18&13\\
25&2&
5331&4309&4098&728&268&234&130&6&3&41&0&0\\
25&3&
2632&1647&1491&144&16&14&17&0&0&6&0&0\\
25&5&
449&149&119&1&0&0&0&0&0&0&0&0\\
\hline
50&1&
9082&8856&8805&4651&3925&3788&1699&923&798&697&112&93\\
50&2&
7059&6320&6177&1189&689&643&160&22&16&41&1&0\\
50&5&
1663&906&805&16&0&0&1&0&0&0&0&0\\
50&7&
388&130&102&0&0&0&0&0&0&0&0&0\\
\hline
100&1&
9527&9395&9375&5585&5066&4943&1985&1311&1176&744&243&194\\
100&2&
8316&7880&7808&1886&1462&1398&217&77&64&56&3&3\\
100&5&
3562&2620&2501&30&11&10&0&0&0&0&0&0\\
100&10&
375&109&84&0&0&0&0&0&0&0&0&0\\
\hline
250&1&
9791&9743&9735&6832&6517&6462&2493&1944&1802&846&424&354\\
250&2&
9209&9005&8980&3138&2760&2715&286&136&123&50&6&6\\
250&5&
6182&5424&5321&123&63&57&1&0&0&0&0&0\\
250&15&
306&98&82&0&0&0&0&0&0&0&0&0\\
250&25&
4&1&0&0&0&0&0&0&0&0&0&0\\
\hline\end{tabular}
  \end{center}\vspace*{-3mm}
  \caption{Number of instances for $P(\textit{unobserved}) = 0.75$ that are identified using CBC in the DAG or after converting the DAG to a MAG, either a MAG$^\emptyset_\bL$ with latent nodes being removed from the graph or a MAG$^\emptyset_\emptyset$ with latent nodes left and marked as latent.}\label{table:mags:0.75}
  \end{table}

%% file: experiments-mags-runtimes.tex
\begin{table}
  \begin{center}
  \scriptsize
  \begin{tabular}{|rr | r|r|r|r | rr|rr|rr|rr| }
  \hline && \multicolumn{4}{c|}{$P(unobserved) = 0$}& \multicolumn{8}{c|}{$P(unobserved) = 0.75$}\\
  \cline{3-14}
  && \multicolumn{1}{c|}{$l = 2$}&\multicolumn{1}{c|}{$l = 5$}&\multicolumn{1}{c|}{$l = 10$}&\multicolumn{1}{c|}{$l = 20$}&\multicolumn{2}{c|}{$l = 2$}&\multicolumn{2}{c|}{$l = 5$}&\multicolumn{2}{c|}{$l = 10$}&\multicolumn{2}{c|}{$l = 20$}
   \\
   \bfseries $n$ & \bfseries $k$ &\bfseries MAG$^\emptyset_\emptyset$&\bfseries MAG$^\emptyset_\emptyset$&\bfseries MAG$^\emptyset_\emptyset$&\bfseries MAG$^\emptyset_\emptyset$&\bfseries MAG$^\emptyset_\emptyset$&\bfseries MAG$^\emptyset_\bL$&\bfseries MAG$^\emptyset_\emptyset$&\bfseries MAG$^\emptyset_\bL$&\bfseries MAG$^\emptyset_\emptyset$&\bfseries MAG$^\emptyset_\bL$&\bfseries MAG$^\emptyset_\emptyset$&\bfseries MAG$^\emptyset_\bL$
\\\hline
10&1&
0.5, 0.3&1.1, 0.5&1.6, 0.8&1.7, 0.8&0.2, 0.2&0.5, 0.4&0.2, 0.2&1.1, 0.6&0.3, 0.3&1.7, 0.9&0.3, 0.3&1.7, 0.9\\
10&2&
0.7, 0.5&1.5, 0.9&2.2, 1.8&2.4, 1.8&0.2, 0.3&0.5, 0.5&0.4, 0.5&1.1, 0.8&0.4, 0.5&1.7, 1.6&0.4, 0.5&1.7, 1.6\\
10&3&
0.8, 0.6&1.3, 1.1&2.1, 2.5&2.4, 2.7&0.5, 0.5&0.7, 0.6&0.9, 1.0&1.4, 1.0&0.9, 1.1&2.1, 2.4&1.1, 1.2&2.4, 2.5\\
10&5&
0.8, 0.9&1.6, 1.7&2.3, 4.7&2.4, 4.7&0.7, 0.9&0.7, 0.8&1.5, 1.7&1.4, 1.6&2.4, 4.9&2.4, 4.9&2.4, 4.7&2.4, 4.7\\
\hline
25&1&
3.4, 0.5&10, 1.0&22, 1.8&50, 3.9&0.8, 0.3&3.5, 0.5&2.0, 0.4&9.6, 1.1&3.4, 0.6&24, 2.0&5.2, 0.6&49, 4.1\\
25&2&
3.7, 0.7&10, 1.4&22, 2.6&49, 7.0&0.8, 0.4&3.6, 0.7&2.1, 0.6&9.8, 1.4&3.5, 0.9&24, 2.8&5.4, 1.1&50, 7.4\\
25&3&
3.8, 0.8&10, 1.7&23, 3.3&50, 10&1.0, 0.5&3.7, 0.9&2.4, 0.9&9.6, 1.7&3.9, 1.4&24, 3.5&6.0, 1.6&50, 11\\
25&5&
3.7, 1.1&10, 2.2&22, 4.7&49, 17&1.7, 0.8&3.7, 1.1&4.5, 2.1&9.6, 2.2&6.8, 3.9&24, 4.5&10, 4.9&49, 15\\
\hline
50&1&
14, 0.8&38, 1.6&92, 3.1&271, 6.3&2.9, 0.3&13, 0.9&12, 0.7&36, 1.8&23, 1.2&97, 3.3&38, 1.6&267, 6.8\\
50&2&
13, 0.9&38, 2.1&91, 4.2&265, 9.4&3.0, 0.4&14, 1.1&12, 0.9&36, 2.3&22, 1.9&93, 4.4&37, 2.9&260, 9.8\\
50&5&
14, 1.4&38, 3.2&92, 7.0&271, 19&3.1, 0.7&13, 1.5&12, 2.0&36, 3.3&23, 5.6&98, 7.5&40, 8.4&273, 20\\
50&7&
13, 1.7&35, 3.8&95, 8.6&263, 25&4.0, 0.9&13, 1.8&16, 3.5&36, 4.0&29, 11&98, 9.4&47, 16&274, 26\\
\hline
100&1&
51, 1.3&146, 2.7&361, 5.5&1129, 12&12, 0.4&50, 1.5&74, 1.4&138, 3.1&168, 3.1&380, 6.1&299, 5.1&1158, 12\\
100&2&
51, 1.5&147, 3.4&348, 7.2&1120, 17&12, 0.5&49, 1.8&74, 1.7&138, 3.6&163, 5.1&366, 7.7&292, 11&1135, 17\\
100&5&
51, 2.1&147, 5.0&348, 11&1117, 31&11, 0.8&49, 2.3&75, 3.1&138, 5.1&163, 15&367, 12&292, 36&1138, 32\\
100&10&
51, 2.9&147, 7.3&357, 18&1122, 57&12, 1.2&50, 3.1&76, 7.1&138, 7.6&171, 43&383, 20&294, 93&1147, 59\\
\hline
250&1&
314, 2.4&894, 4.5&2967, 10&12194, 23&64, 0.8&303, 3.1&663, 5.6&915, 5.6&1829, 18&3044, 12&3341, 40&12460, 25\\
250&2&
313, 2.6&889, 5.4&2941, 14&12327, 33&65, 0.9&308, 3.4&633, 6.1&907, 6.5&1770, 26&2934, 15&3315, 93&12119, 34\\
250&5&
312, 3.1&882, 7.8&2944, 22&12441, 66&65, 1.2&308, 4.0&637, 8.5&905, 8.7&1775, 75&2938, 23&3302, 366&12279, 63\\
250&15&
312, 4.9&889, 15&2932, 49&12049, 169&64, 2.1&303, 5.9&649, 25&935, 16&1835, 412&2975, 49&3357, 1860&12550, 177\\
250&25&
314, 6.8&872, 22&2938, 76&12219, 275&65, 3.1&310, 8.0&640, 54&908, 25&1783, 931&2925, 80&3397, 3645&12767, 292\\
\hline\end{tabular}
  \end{center}\vspace*{-3mm}
  \caption{Time (in milliseconds) to first construct the MAG from the DAG and then check for the existence of an adjustment set in that MAG, for $P(\textit{unobserved}) = 0$, respectively $P(\textit{unobserved}) = 0.75$.
  % Keeping the latent nodes in the MAG, reduces the construction time, but increases the time required for finding the adjustment set.
  }\label{table:mags:runtimes}
  \end{table}

%% file: experiments-table2.tex
\begin{table}
  \begin{center}
  \scriptsize
  \begin{tabular}{|rr | rrr| rrr| rrr| rrr|}
  \hline
   &&\multicolumn{3}{c|}{$l=2$} &\multicolumn{3}{c|}{$l=5$} &\multicolumn{3}{c|}{$l=10$} &\multicolumn{3}{c|}{$l=20$}\\
   \bfseries $n$ & \bfseries $k$
& \bfseries BC & \bfseries CBC & \bfseries CBC$^+$
& \bfseries BC & \bfseries CBC & \bfseries CBC$^+$
& \bfseries BC & \bfseries CBC & \bfseries CBC$^+$
& \bfseries BC & \bfseries CBC & \bfseries CBC$^+$
\\\hline
10&1&
8235&8235&9901&4772&4772&8938&2452&2452&7443&2373&2373&7469\\
10&2&
\cellcolor[gray]{0.85}4307&\cellcolor[gray]{0.85}4840&8035&      \cellcolor[gray]{0.85}528&\cellcolor[gray]{0.85}1120&3545&    0& 349&1977&      0&379&2004\\
10&3&
\cellcolor[gray]{0.85}1603&\cellcolor[gray]{0.85}2353&5173&36&313&1168&0&96&591&0&107&594\\
10&5&
\cellcolor[gray]{0.85}184&\cellcolor[gray]{0.85}823&1700&0&99&204&0&44&76&0&36&82\\
\hline
25&1&
9306&9306&9978&7019&7019&9507&3549&3549&8141&1645&1645&6686\\
25&2&
7312&7466&9489&       2132&2490&5757&        310&485&2815&            24&126&1862\\
25&3&
4863&5250&8101&416&718&2668&10&65&925&0&14&523\\
25&5&
\cellcolor[gray]{0.85}1466&\cellcolor[gray]{0.85}2060&4255&10&65&449&0&2&99&0&0&38\\
\hline
50&1&
9668&9668&9993&8075&8075&9763&4408&4408&8517&1927&1927&7024\\
50&2&
8555&8600&9814&                    4013&4222&7274&          639&777&3614&          92&145&2145\\
50&5&
\cellcolor[gray]{0.85}3727&\cellcolor[gray]{0.85}4167&6835&           86&174&1034&         1&1&158&           0&0&81\\
50&7&
\cellcolor[gray]{0.85}1449&\cellcolor[gray]{0.85}1946&3888&          4&22&233&     0&0&14&             0&0&7\\
\hline
100&1&
9818&9818&9997&8886&8886&9879&5158&5158&8833&2173&2173&7271\\
100&2&
9341&9354&9951&5742&5871&8400&1013&1084&4462&167&209&2564\\
100&5&
6215&6404&8637&443&595&2122&3&11&340&0&0&96\\
100&10&
1453&1813&3490&0&1&78&0&0&1&0&0&1\\
\hline
250&1&
9917&9917&10000&9559&9559&9964&6033&6033&9208&2558&2558&7691\\
250&2&
9712&9712&9990&7840&7888&9362&1717&1764&5611&216&236&3038\\
250&5&
8293&8343&9669&2015&2192&4569&9&18&598&0&0&158\\
250&15&
1728&2014&3676&          0&1&27&        0&0&1&          0&0&0\\
250&25&
85&164&361&0&0&0&0&0&0&0&0&0\\
\hline
500&1&
9968&9968&10000&9765&9765&9986&6684&6684&9455&2667&2667&7888\\
500&2&
9864&9866&9997&8920&8933&9762&2304&2329&6352&303&314&3414\\
500&5&
9162&9178&9910&4207&4343&6662&46&50&955&0&0&186\\
500&22&
1533&1774&3148&0&0&7&0&0&0&0&0&0\\
500&50&
0&3&10&0&0&0&0&0&0&0&0&0\\
\hline
1000&1&
9984&9984&10000&9903&9903&9997&7266&7266&9615&2831&2831&8086\\
1000&2&
9926&9926&10000&9490&9491&9902&3261&3278&7194&348&350&3757\\
1000&5&
9599&9602&9984&6613&6703&8370&75&81&1486&0&0&261\\
1000&32&
1413&1588&2801&0&0&1&0&0&0&0&0&0\\
1000&100&
0&0&0&0&0&0&0&0&0&0&0&0\\
\hline
2000&1&
9994&9994&10000&9945&9945&10000&7924&7924&9773&3117&3117&8322\\
2000&2&
9963&9963&10000&9809&9812&9985&4140&4150&7842&452&456&4191\\
2000&5&
9800&9802&9992&8273&8316&9403&210&217&2140&0&0&356\\
2000&45&
1541&1728&2840&0&0&0&0&0&0&0&0&0\\
2000&200&
0&0&0&0&0&0&0&0&0&0&0&0\\
\hline\end{tabular}
  \end{center}\vspace*{-3mm}
  \caption{Numbers of instances for $P(\textit{unobserved}) = 0.25$ that are identifiable
  by use of BC, CBC, CBC$^+$ (as defined in Section~\ref{sec:experiments:algorithms:abbreviations}). 
We did not run the IDC algorithm on these data due to its high time complexity.   
%    we were unable to perform these experiments for some parameters shown in  as we report in Table~\ref{table:global:stat:unobs0.75:count}. 
Gray cells highlight where the CBC was able to identify at least 400 more graphs than the BC.
}\label{table:global:stat:unobs0.25:count}
  \end{table}
\begin{table}
  \begin{center}
  \scriptsize
  \begin{tabular}{|rr | rrr| rrr| rrr| rrr|}
  \hline
   &&\multicolumn{3}{c|}{$l=2$} &\multicolumn{3}{c|}{$l=5$} &\multicolumn{3}{c|}{$l=10$} &\multicolumn{3}{c|}{$l=20$}\\
   \bfseries $n$ & \bfseries $k$
& \bfseries BC & \bfseries CBC & \bfseries CBC$^+$
& \bfseries BC & \bfseries CBC & \bfseries CBC$^+$
& \bfseries BC & \bfseries CBC & \bfseries CBC$^+$
& \bfseries BC & \bfseries CBC & \bfseries CBC$^+$
\\\hline
10&1&
7418&7418&9799&3102&3102&8085&1500&1500&6515&1520&1520&6537\\
10&2&
\cellcolor[gray]{0.85}3289&\cellcolor[gray]{0.85}3795&7602&      303& 649&3029&    0& 223&1910&      0&251&1878\\
10&3&
\cellcolor[gray]{0.85}1000&\cellcolor[gray]{0.85}1575&4512&13&138&1038&0&57&549&0&43&539\\
10&5&
\cellcolor[gray]{0.85}165&\cellcolor[gray]{0.85}822&1684&0&87&220&0&39&85&0&47&76\\
\hline
25&1&
8912&8912&9960&5107&5107&9106&1920&1920&7224&898&898&5913\\
25&2&
6413&6555&9258&       1115&1346&5056&        154&249&2592&            9&77&1803\\
25&3&
3665&4060&7595&182&339&2326&6&24&880&0&7&521\\
25&5&
\cellcolor[gray]{0.85}820&\cellcolor[gray]{0.85}1241&3591&1&20&419&0&2&86&0&0&40\\
\hline
50&1&
9429&9429&9983&6206&6206&9487&2454&2454&7831&1004&1004&6414\\
50&2&
7928&7991&9738&                    2141&2272&6516&          260&334&3353&          43&76&2127\\
50&5&
2484&2835&6082&           21&52&921&         0&1&198&           0&0&64\\
50&7&
735&1041&3164&          0&1&205&     0&0&20&             0&0&6\\
\hline
100&1&
9725&9725&9995&7032&7032&9658&2952&2952&8281&1154&1154&6626\\
100&2&
8858&8882&9930&3198&3285&7674&403&444&4069&74&85&2422\\
100&5&
4828&5052&8211&89&129&1791&1&1&277&0&0&83\\
100&10&
619&828&2793&0&0&90&0&0&2&0&0&0\\
\hline
250&1&
9876&9876&9999&8259&8259&9908&3539&3539&8767&1314&1314&7069\\
250&2&
9542&9546&9979&5055&5097&8865&575&591&5085&83&93&2925\\
250&5&
7423&7498&9502&422&469&3608&1&1&613&0&0&155\\
250&15&
711&922&2833&          0&0&37&        0&0&0&          0&0&0\\
250&25&
12&25&243&0&0&0&0&0&0&0&0&0\\
\hline
500&1&
9937&9937&10000&8992&8992&9957&4038&4038&9062&1336&1336&7354\\
500&2&
9779&9781&9998&6569&6587&9369&791&802&5748&98&100&3210\\
500&5&
8625&8641&9852&1162&1231&5318&2&3&925&0&0&195\\
500&22&
572&685&2245&0&0&3&0&0&0&0&0&0\\
500&50&
0&1&3&0&0&0&0&0&0&0&0&0\\
\hline
1000&1&
9972&9972&10000&9420&9420&9985&4487&4487&9314&1525&1525&7565\\
1000&2&
9865&9865&9999&7872&7881&9746&1086&1094&6475&96&98&3660\\
1000&5&
9328&9335&9957&2623&2683&7140&4&5&1461&0&0&253\\
1000&32&
475&548&1910&0&0&0&0&0&0&0&0&0\\
1000&100&
0&0&0&0&0&0&0&0&0&0&0&1\\
\hline
2000&1&
9985&9985&10000&9715&9715&9992&5059&5059&9491&1693&1693&7799\\
2000&2&
9949&9949&10000&8823&8828&9937&1500&1503&7201&122&122&3905\\
2000&5&
9624&9626&9994&4614&4649&8529&19&19&2066&0&0&345\\
2000&45&
467&524&1853&0&0&0&0&0&0&0&0&0\\
2000&200&
0&0&0&0&0&0&0&0&0&0&0&0\\
\hline\end{tabular}
  \end{center}\vspace*{-3mm}
  \caption{Numbers of instances for $P(\textit{unobserved}) = 0.5$ that are identifiable 
    by use of BC, CBC, CBC$^+$ (as defined in Section~\ref{sec:experiments:algorithms:abbreviations}). 
    We did not run the IDC algorithm on these data due to its high time complexity.
     Gray cells highlight where the CBC was able to identify at least 400 more graphs than the BC.
%    Due to high time complexity of the IDC algorithm,  
%    we were not be able to perform the same experiments as we report in Table~\ref{table:global:stat:unobs0.75:count}. 
}\label{table:global:stat:unobs0.5:count}
  \end{table}